\documentclass[nohyperref]{article}

\def\GDIHmeanhns{9620.33}

\def\GDIHmedianhns{1146.39}

\def\GDIHHWRB{22}

\def\GDIHmeanHWRNS{154.27}

\def\GDIHmedianHWRNS{50.63}

\def\GDIHmeanSABER{71.26}

\def\GDIHmedianSABER{50.63}

\def\GDIHnumframes{2.00E+08}
\def\GDIHgametime{38.5}
\def\GDIImeanhns{7810.1}

\def\GDIImedianhns{832.5}

\def\GDIIHWRB{17}

\def\GDIImeanHWRNS{117.98}

\def\GDIImedianHWRNS{35.78}

\def\GDIImeanSABER{61.66}

\def\GDIImedianSABER{35.78}

\def\GDIInumframes{2.00E+08}
\def\GDIIgametime{38.5}

\def\rainbowmeanhns{873.54}

\def\rainbowmedianhns{230.99}

\def\rainbowHWRB{4}

\def\rainbowmeanHWRNS{28.39}

\def\rainbowmedianHWRNS{4.92}

\def\rainbowmeanSABER{28.39}

\def\rainbowmedianSABER{4.92}

\def\rainbownumframes{2.00E+08}
\def\rainbowgametime{38.5}

\def\impalameanhns{956.99}

\def\impalamedianhns{191.82}

\def\impalaHWRB{3}

\def\impalameanHWRNS{34.52}

\def\impalamedianHWRNS{4.31}

\def\impalameanSABER{29.45}

\def\impalamedianSABER{4.31}

\def\impalanumframes{2.00E+08}
\def\impalagametime{38.5}

\def\lasermeanhns{1740.94}

\def\lasermedianhns{454.91}

\def\laserHWRB{7}

\def\lasermeanHWRNS{45.39}

\def\lasermedianHWRNS{8.08}

\def\lasermeanSABER{36.78}

\def\lasermedianSABER{8.08}

\def\lasernumframes{2.00E+08}
\def\lasergametime{38.5}

\def\rtdtmeanhns{3373.48}

\def\rtdtmedianhns{1342.27}

\def\rtdtHWRB{15}

\def\rtdtmeanHWRNS{98.78}

\def\rtdtmedianHWRNS{33.62}

\def\rtdtmeanSABER{60.43}

\def\rtdtmedianSABER{33.62}

\def\rtdtnumframes{1.00E+10}
\def\rtdtgametime{1929}

\def\ngumeanhns{3169.07}

\def\ngumedianhns{1174.92}

\def\nguHWRB{8}

\def\ngumeanHWRNS{76.00}

\def\ngumedianHWRNS{21.19}

\def\ngumeanSABER{50.47}

\def\ngumedianSABER{21.19}

\def\ngunumframes{3.50E+10}
\def\ngugametime{6751.5}

\def\agentmeanhns{4762.17}

\def\agentmedianhns{1933.49}

\def\agentHWRB{18}

\def\agentmeanHWRNS{125.92}

\def\agentmedianHWRNS{43.62}

\def\agentmeanSABER{76.26}

\def\agentmedianSABER{43.62}

\def\agentnumframes{1.00E+11}
\def\agentgametime{19290}

\def\muzeromeanhns{4994.97}

\def\muzeromedianhns{2041.12}

\def\muzeroHWRB{19}

\def\muzeromeanHWRNS{152.10}

\def\muzeromedianHWRNS{49.80}

\def\muzeromeanSABER{71.94}

\def\muzeromedianSABER{49.80}

\def\muzeronumframes{2.00E+10}
\def\muzerogametime{3858}

\def\dreamermeanhns{642.49}

\def\dreamermedianhns{178.04}

\def\dreamerHWRB{3}

\def\dreamermeanHWRNS{38.60}

\def\dreamermedianHWRNS{4.29}

\def\dreamermeanSABER{27.73}

\def\dreamermedianSABER{4.29}

\def\dreamernumframes{2.00E+08}
\def\dreamergametime{38.58 }

\def\simplemeanhns{25.78}

\def\simplemedianhns{5.55}

\def\simpleHWRB{0}

\def\simplemeanHWRNS{4.80}

\def\simplemedianHWRNS{0.13}

\def\simplemeanSABER{4.80}

\def\simplemedianSABER{0.13}

\def\simplenumframes{1.00E+06}
\def\simplegametime{0.19}

\def\mueslimeanhns{2538.12}

\def\mueslimedianhns{1077.47}

\def\muesliHWRB{5}

\def\mueslimeanHWRNS{75.52}

\def\mueslimedianHWRNS{24.86}

\def\mueslimeanSABER{48.74}

\def\mueslimedianSABER{24.86}

\def\mueslinumframes{2.00E+08}
\def\muesligametime{38.5}

\def\goexploremeanhns{4989.31}

\def\goexploremedianhns{1451.55}

\def\goexploreHWRB{15}

\def\goexploremeanHWRNS{116.89}

\def\goexploremedianHWRNS{50.50}

\def\goexploremeanSABER{71.80}

\def\goexploremedianSABER{50.50}

\def\goexplorenumframes{1.00E+10}
\def\goexploregametime{1929}

\newcommand{\best}[1]{\textbf{#1}}

\usepackage{makecell}

\usepackage[utf8]{inputenc} 
\usepackage[T1]{fontenc}    
\usepackage{url}            
\usepackage{booktabs}       
\usepackage{amsfonts}       
\usepackage{nicefrac}       
\usepackage{microtype}      
\usepackage{appendix}
\usepackage{microtype}      
\usepackage{xcolor}         

\usepackage{svg}
\usepackage{graphicx}
\usepackage{subfigure}
\usepackage{mathrsfs}
\usepackage{algorithm}
\usepackage{algorithmic}

\usepackage{microtype}
\usepackage{graphicx}
\usepackage{subfigure}
\usepackage{booktabs} 

\usepackage{hyperref}

\usepackage[accepted]{icml2022}

\usepackage{amsmath}
\usepackage{amssymb}
\usepackage{mathtools}
\usepackage{amsthm}

\usepackage[capitalize,noabbrev]{cleveref}

\DeclareMathOperator*{\argmax}{arg\,max}

\theoremstyle{plain}

\newtheorem{definition}{Definition}[section]

\newtheorem{Theorem}{\textbf{Theorem}}
\newtheorem{Lemma}{\textbf{Lemma}}

\newtheorem*{Remark}{\textbf{Remark}}
\newtheorem*{Proof}{\textbf{Proof}}
\newtheorem{Assumption}{\textbf{Assumption}}
\newtheorem{Example}{Example}

\usepackage[textsize=tiny]{todonotes}

\icmltitlerunning{Generalized Data Distribution Iteration}

\begin{document}

\twocolumn[
\icmltitle{Generalized Data Distribution Iteration}



\icmlsetsymbol{equal}{*}

\begin{icmlauthorlist}
\icmlauthor{Jiajun Fan}{tsinghua}
\icmlauthor{Changnan Xiao}{bytedance}
\end{icmlauthorlist}

\icmlaffiliation{tsinghua}{Tsinghua Shenzhen International Graduate School, Tsinghua University, Beijing, China}
\icmlaffiliation{bytedance}{ByteDance, Beijing, China}

\icmlcorrespondingauthor{Changnan Xiao}{xiaochangnan@bytedance.com}
\icmlcorrespondingauthor{Jiajun Fan}{fanjj21@mails.tsinghua.edu.cn}

\icmlkeywords{Machine Learning, ICML}

\vskip 0.3in
]



\printAffiliationsAndNotice{}  

\begin{abstract}
To obtain higher sample efficiency and superior final performance  simultaneously has been one of the major challenges for deep reinforcement learning (DRL). Previous work could handle one of these challenges but typically failed to address them concurrently. In this paper, we try to tackle these two challenges simultaneously. To achieve this, we firstly decouple these challenges into two classic RL problems: data richness and exploration-exploitation trade-off. 
Then, we cast these two problems into the training data distribution optimization problem, namely to obtain desired training data within limited interactions, and address them concurrently via \textbf{i)} explicit modeling and control of the capacity and diversity of behavior policy and \textbf{ii)} more fine-grained and adaptive control of selective/sampling distribution of the behavior policy using a monotonic data distribution optimization.  Finally, we integrate this process into Generalized Policy Iteration (GPI)  and obtain a more general framework called \textbf{G}eneralized \textbf{D}ata Distribution \textbf{I}teration (GDI). We use the GDI framework to introduce operator-based versions of well-known RL methods from DQN to Agent57. Theoretical guarantee of the superiority of GDI compared with GPI is concluded. We also demonstrate  our state-of-the-art (SOTA) performance on Arcade Learning Environment (ALE), wherein our algorithm  has achieved \GDIHmeanhns\% mean human normalized score (HNS), \GDIHmedianhns\% median HNS and surpassed \GDIHHWRB\ human world records using only 200M training frames. Our performance is comparable to Agent57's  while we consume 500 times less data.  We argue that  there is still a long way to go before obtaining real superhuman agents in ALE.
\end{abstract}

\begin{figure}[!t]
    \centering
	\subfigure{
		\includegraphics[width=0.46\textwidth]{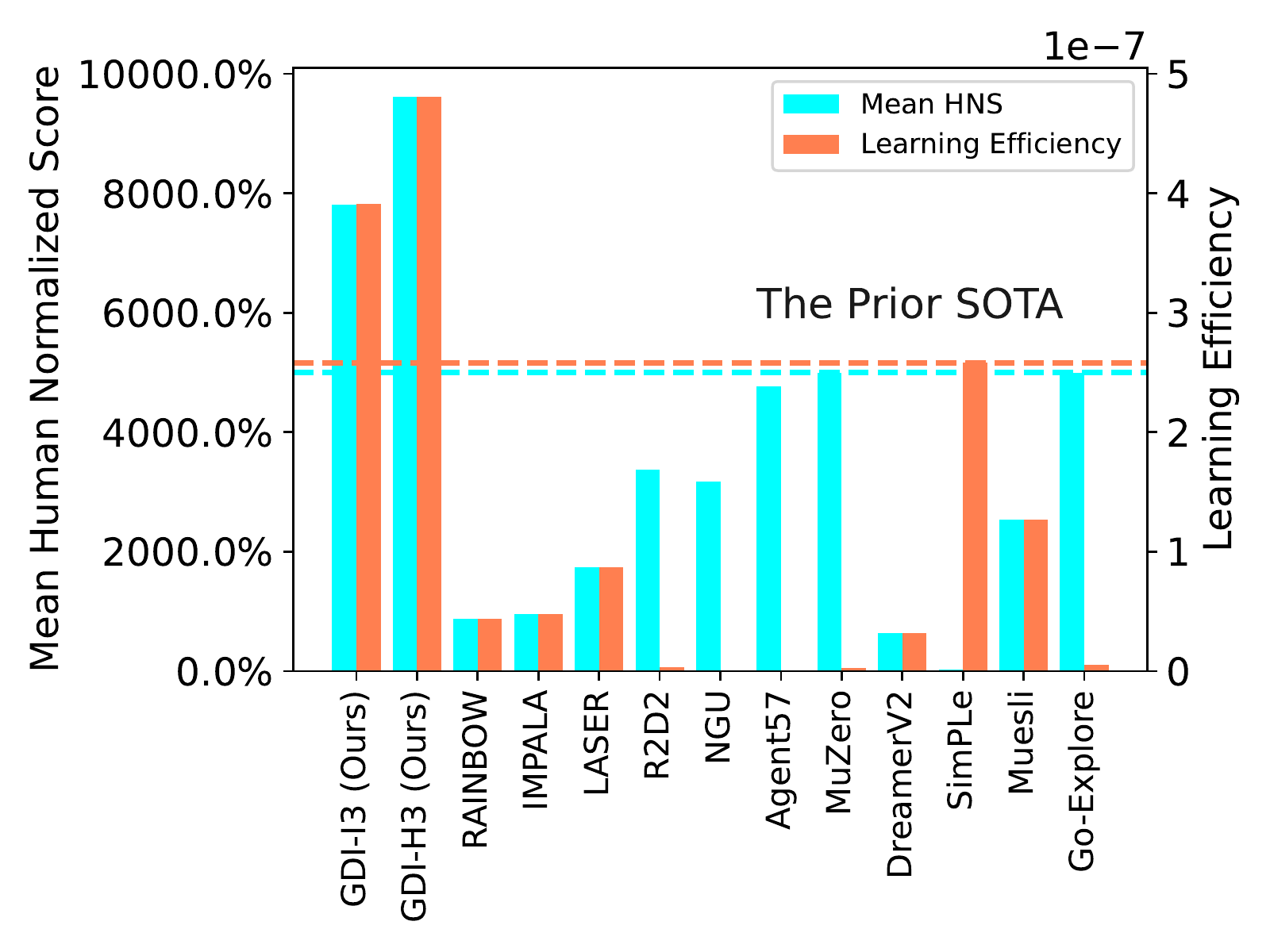}
	}
	\centering
	\caption{Performance of algorithms of Atari 57 games on mean HNS(\%) and  corresponding learning/sample efficiency calculated by $\frac{\text{Mean HNS}}{\text{Training Scale (frames)}}$. For more benchmark results, can see App. \ref{app: Summary of Benchmark Results}.} 
	\label{fig: mean med hns and learning efficiency}
\end{figure}

\begin{figure*}[!t]
    \centering
	\subfigure[Isomorphism GDI]{
		\includegraphics[width=0.46\textwidth]{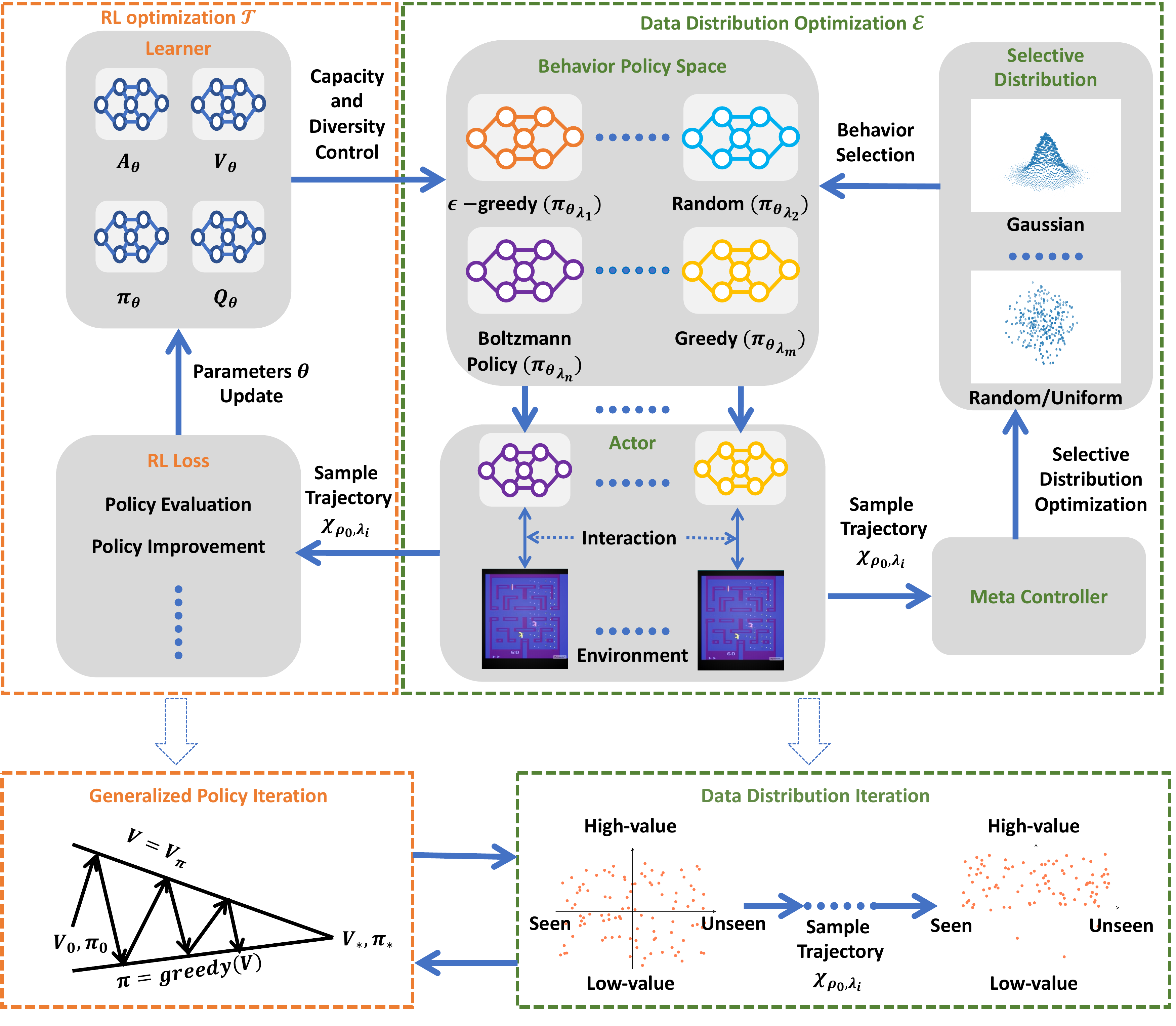}
	}
	\subfigure[Heterogeneous GDI]{
		\includegraphics[width=0.46\textwidth]{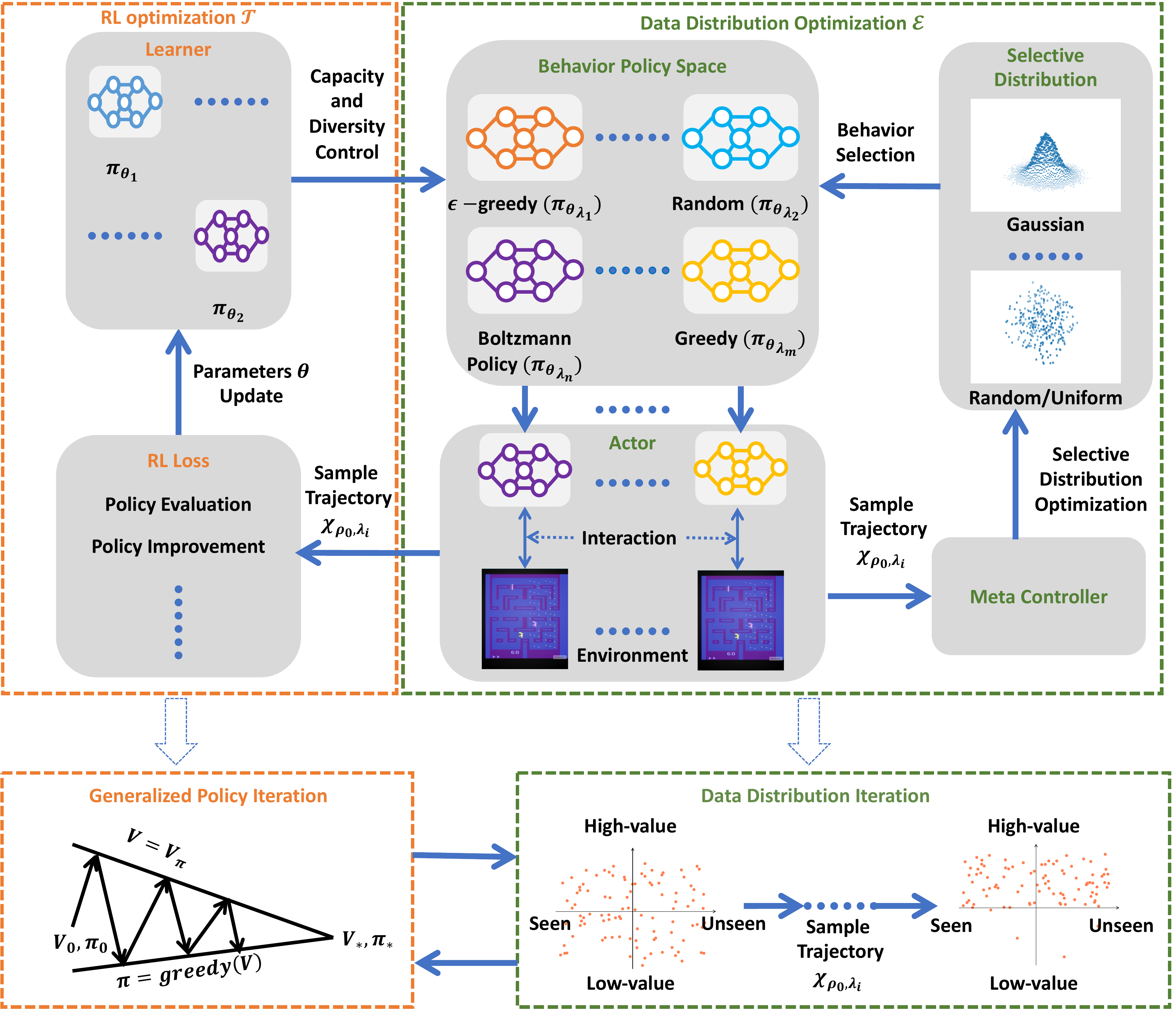}
	}
	\centering
	\caption{Algorithm Architecture Diagram. \textbf{(a)} The Isomorphism architecture of GDI, wherein the  behavior policy space (e.g., the soft entropy policy space, $\pi_{\theta_{\lambda}}=\epsilon \cdot \operatorname{Softmax}\left(\frac{A_{\theta_1}}{\tau_{1}}\right)+(1-\epsilon) \cdot \operatorname{Softmax}\left(\frac{A_{\theta_2}}{\tau_{2}}\right)$) is constructed by the base policy  with shared parameters (i.e., $\theta_1=\theta_2=\theta$) and indexed by $\lambda=(\tau_1,\tau_2,\epsilon)$. \textbf{(b)} The Heterogeneous architecture of GDI, wherein the behavior policy space is constructed by the base policy  with different parameters (i.e., $\theta_1 \neq \theta_2$) and indexed by $\lambda$. For more details, can see Sec. \ref{Sec: Methodology} and \ref{sec: experiment}.} 
	\label{fig: Algorithm Architecture Diagram}
\end{figure*}

\section{Introduction}
\label{sec: introduction}

Reinforcement learning (RL) algorithms, when combined with high-capacity deep neural networks, have shown promise in domains ranging from video games \citep{dqn} to robotic manipulation \citep{trpo,ppo}. However, it still suffers from high sample complexity and unsatisfactory final performance, especially compared to human learning \citep{tsividis2017human}. Prior work could handle one of these problems but commonly failed to tackle both of them simultaneously.

Model-free RL methods typically obtain remarkable final performance via finding a way to encourage exploration and  improve the data richness (e.g., $\frac{\text{Seen Conditions}}{\text{All Conditions}}$) that guarantees traversal of \emph{all possible} conditions. These methods \citep{goexplore,agent57} could perform remarkably well when interactions are (nearly) \emph{limitless} but normally fail when  interactions are \emph{limited}. We argue that when interactions are \emph{limited}, finding a way to guarantee traversal of \emph{all unseen} conditions is unreasonable, and perhaps we should find a way to traverse the \emph{nontrivial} conditions (e.g., unseen \citep{goexplore} and high-value \citep{discor}) first and avoid traversing the \emph{trivial/low-value} conditions repeatedly. In other words, we should explicitly control the training data distribution in RL and maximize the probability of nontrivial conditions being traversed, namely the \emph{data distribution optimization} (see Fig. \ref{fig: Algorithm Architecture Diagram}).

 In RL, training data distribution is normally controlled by the behavior policy \citep{sutton,dqn}, so that the data richness can be controlled by the capacity and diversity  of the behavior policy. Wherein the capacity describes \emph{how many different behavior policies there are in the policy space}, and the diversity describes \emph{how many different behavior policies are  selected/sampled from the policy space to generate training data} (discussed in Sec. \ref{sec: Explicit  Capacity and Diversity  Control of Behavior Policy}). When interactions are limitless, increasing the capacity and maximizing the diversity  via randomly sampling behavior policies (most prior works have achieved SOTA in this way) can significantly improve the data richness and guarantee traversal of almost all \emph{unseen} conditions, which induces better final performance \citep{agent57} and generalization  \citep{ghosh2021generalization}. However, perhaps surprisingly, this is not the case when interactions are limited, where each interaction is rare and the selection of the behavior policy becomes important. In conclusion, we should increase the probability of the traversal of \emph{unseen} conditions (i.e., exploration) via increasing the \emph{capacity} and \emph{diversity} of the behavior policy  and maximize the probability of  \emph{high-value} conditions (i.e., exploitation) being traversed via optimizing the selective distribution of the behavior policy. It's also known as  the  exploration-exploitation trade-off problem. 
 
 From this perspective, we can understand why the prior SOTA algorithms, such as Agent57  and Go-Explore, failed  to obtain high sample efficiency. They have collected massive data to guarantee the traversal of \emph{unseen} conditions but ignore the different values of data. Therefore, they wasted many  trials to collect  \emph{useless/low-value} data, which accounts for their low sample efficiency. In other words, they failed to tackle the data distribution optimization problem.

In this paper, we argue that the sample efficiency of model-free methods can be significantly improved  (even outperform the SOTA model-based schemes \cite{dreamerv2}) without degrading the final performance  via data distribution optimization.   To achieve this, we propose a data distribution optimization operator $\mathcal{E}$ to iteratively optimize the selective distribution of the behavior policy  and  thereby optimize the training data distribution. Specifically, we  construct a parameterized  policy space indexed by $\lambda$ called the soft entropy space, which enjoys a larger capacity  than Agent57. The behavior policies are sampled  from this policy space via a sampling distribution. Then, we adopt a meta-learning method to  optimize the sampling distribution of behavior policies iteratively and thereby achieve a more fine-grained exploration and exploitation trade-off. Moreover, training data collected by the optimized behavior policies will be used for RL optimization via the operator $\mathcal{T}$. This process will be illustrated in Fig. \ref{fig: Algorithm Architecture Diagram}, generalized in Sec. \ref{Sec: Methodology}, proved  superior in Sec. \ref{sec: Monotonic Data Distribution Optimization} and implemented in Sec. \ref{sec: Practical Implement Based on GDI}.

The main contributions of our work are: 
\begin{enumerate}
    \item \textbf{A General RL Framework.} \emph{Efficient learning}  within \emph{limited} interactions induces the data distribution optimization problem. To tackle this problem, we firstly explicitly  control the diversity and capacity of the behavior policy (see Sec. \ref{sec: Explicit  Capacity and Diversity  Control of Behavior Policy}) and  then optimize the sampling distribution of behavior policies iteratively via a data distribution optimization operator (see Sec. \ref{sec: Generalized Data Distribution Iteration}). After integrating  them into GPI, we obtain a general RL framework, GDI (see Fig. \ref{fig: Algorithm Architecture Diagram}).
    
    \item \textbf{An Operator View of RL Algorithms.} We use the GDI framework to introduce operator-based versions of well-known RL methods from DQN to Agent57 in Sec. \ref{sec: An Operator View of RL Methods}, which leads to a better understanding of their original counterparts.

    \item \textbf{Theoretical Proof of Superiority.}  We offer theoretical proof of the superiority of GDI in the case of both first-order and second-order optimization in Sec. \ref{sec: Monotonic Data Distribution Optimization}.

    \item \textbf{The State-Of-The-Art Performance.} From Fig. \ref{fig: mean med hns and learning efficiency}, our algorithm GDI-H$^3$ has achieved \GDIHmeanhns\% mean HNS, outperforming the SOTA model-free algorithms Agent57. Surprisingly, our learning efficiency has outperformed the SOTA model-based methods Muzero and Dreamer-V2. Furthermore, our method has surpassed 22 Human World Records in 38 playtime days.
\end{enumerate}

\section{ Related Work}
\label{sec:Problem Formulation and Related Work}


\paragraph{Data richness.} As claimed by \citep{ghosh2021generalization},  generalization to unseen test conditions from a limited number of training conditions induces implicit partial observability, effectively turning even fully observed MDPs into POMDPs, which makes generalization in RL much more difficult. Therefore, data richness (e.g., $\frac{\text{Seen Conditions}}{\text{All Conditions}}$) is vital for the generalization and performance of RL agents.  When interactions are limited, more diverse behavior policies increase the data richness and thereby reduce the proportion of unseen conditions  and  improve generalization and performance.  Therefore, we can recast this problem into the problem to control the  capacity and diversity  of the behavior policy. There are two promising ways to handle this issue. 
Firstly, some RL methods adopt intrinsic reward to encourage exploration, where unsupervised objectives, auxiliary tasks and other techniques induce the intrinsic reward \citep{icm}.
Other methods \citep{agent57} introduced a diversity-based regularizer into the RL objective and trained a family of policies with different degrees of exploratory behaviors.   Despite both obtaining SOTA performance, adopting intrinsic rewards and entropy regularization has increased the uncertainty of environmental transition. We argue that the inability to effectively tackle the  data distribution optimization accounts for their low learning efficiency. 
 


\paragraph{Exploration and exploitation trade-off. } Exploration and exploitation trade-off remains one of the significant  challenges in DRL \citep{ngu,sutton}.  In general, methods that guarantee to find an optimal policy require the number of visits to each state–action pair to approach infinity. The entropy of policy would collapse
to zero swiftly after a finite number of steps may never learn to act optimally; they may instead converge prematurely to sub-optimal policies and never gather the data they need to learn to act optimally. Therefore, to ensure that all state-action pairs are encountered infinitely, off-policy learning methods are widely used \citep{a3c,impala}, and agents must learn to adjust the entropy (exploitation degree) of the behavior policy. Adopting stochastic policies into the behavior policy has been widely used in RL algorithms \citep{dqn,rainbow}, such as the $\epsilon$-greedy \citep{epsilongreedy}. These methods  can perform remarkably well in dense reward scenarios \citep{dqn}, but fail to learn in  sparse reward environments. Recent approaches \citep{agent57} have proposed to train a family of policies and provide  intrinsic rewards and entropy regularization to agents to drive exploration. Among these methods, the intrinsic rewards  are proportional to some notion of saliency, quantifying how different the current state is from those already visited.  They have achieved SOTA performance at the cost of a relatively lower sample efficiency.  We argue that these algorithms overemphasize the role of exploration to traverse \emph{unseen} conditions but ignore the \emph{value} of data and thereby waste many trails to collect \emph{low-value} data, accounting for their low sample efficiency.
\section{Preliminaries}

 The RL problem can be formulated as a Markov Decision Process \citep[MDP]{howard1960dynamic} defined by $\left(\mathcal{S}, \mathcal{A}, p, r, \gamma, \rho_{0}\right)$. 
 Considering a discounted episodic MDP, the initial state $s_0$ is sampled from the initial distribution $\rho_0(s): \mathcal{S} \rightarrow \Delta(\mathcal{S})$, where we use $\Delta$ to represent the probability simplex.
 At each time $t$, the agent chooses an action $a_t \in \mathcal{A}$ according to the policy $\pi(a_t|s_t): \mathcal{S} \rightarrow \Delta(\mathcal{A})$ at state $s_t \in \mathcal{S}$. 
 The environment receives $a_t$, produces the reward $r_t \sim r(s,a): \mathcal{S} \times \mathcal{A} \rightarrow \mathbf{R}$ and transfers to the next state $s_{t+1}$  according to the transition distribution $p\left(s^{\prime} \mid s, a\right): \mathcal{S} \times \mathcal{A} \rightarrow \Delta(\mathcal{S})$. 
 The process continues until the agent reaches a terminal state or a maximum time step. 
 Define the discounted state visitation distribution as 
 $d_{\rho_0}^{\pi} (s) = (1 - \gamma) \textbf{E}_{s_0 \sim \rho_0} 
 \left[ \sum_{t=0}^{\infty} \gamma^t \textbf{P} (s_t = s | s_0) \right]$.
 The goal of reinforcement learning is to find the optimal policy $\pi^*$ that maximizes the expected sum of discounted rewards, denoted by $\mathcal{J}$ \citep{sutton}:
\begin{equation}
\label{eq_accmulate_reward}
\begin{aligned}
\pi^{*}
&= \underset{\pi}{\operatorname{argmax}} \textbf{E}_{s_t \sim d_{\rho_0}^{\pi}} \textbf{E}_{\pi} \left[\sum_{k=0}^{\infty} \gamma^{k} r_{t+k} | s_t \right] \\
\end{aligned}
\end{equation}
where $\gamma \in(0,1)$ is the discount factor.

\section{Methodology}
\label{Sec: Methodology}
\subsection{Notation Definition}

Let's introduce our notations first, which are also summarized in App. \ref{app: Abbreviation and Notation}.

Define $\Lambda$ to be an index set, $\Lambda \subseteq \textbf{R}^k$.
$\lambda \in \Lambda$ is an index in $\Lambda$.
$(\Lambda, \mathcal{B}|_{\Lambda}, \mathcal{P}_{\Lambda})$ is a probability space, where $\mathcal{B}|_{\Lambda}$ is a Borel $\sigma$-algebra restricted to $\Lambda$.
Under the setting of meta-RL, $\Lambda$ can be regarded as the set of all possible meta information.
Under the setting of population-based training (PBT) \citep{PBT}, $\Lambda$ can be regarded as the set of the whole population.

Define $\Theta$ to be a set of all possible values of parameters (e.g., parameters of value function network and policy network).
$\theta \in \Theta$ is some specific value of parameters.
For each index $\lambda$, there exists a specific mapping between each parameter of $\theta$ and $\lambda$, denoted as $\theta_\lambda$, to indicate the parameters in $\theta$ corresponding to $\lambda$ (e.g., $\epsilon$ in $\epsilon$-greedy behavior policies).
Under the setting of linear regression $y = w \cdot x$, $\Theta = \{w \in R^n\}$ and $\theta = w$.
If $\lambda$ represents using only the first half features to perform regression, assume $w = (w_1, w_2)$, then $\theta_\lambda = w_1$.  
Under the setting of RL, $\theta_{\lambda}$ defines a parameterized policy indexed by $\lambda$, denoted as $\pi_{\theta_{\lambda}}$.

Define $\mathcal{D} \overset{def}{=} \{d^\pi_{\rho_{0}} |\ \pi \in {\Delta (\mathcal{A})}^\mathcal{S}, \rho_{0} \in \Delta(\mathcal{S}) \}$ to be the set of all states visitation distributions.
For the parameterized policies, denote 
$\mathcal{D}_{\Lambda, \Theta, \rho_{0}} \overset{def}{=} \{d^{\pi_{\theta_{\lambda}}}_{\rho_{0}} |\ \theta \in \Theta, \lambda \in \Lambda \}$.
Note that $(\Lambda, \mathcal{B}|_{\Lambda}, \mathcal{P}_{\Lambda})$ is a probability space on $\Lambda$, 
which induces a probability space on $\mathcal{D}_{\Theta, \Lambda, \rho_{0}}$,
with the probability measure given by 
$\mathcal{P}_{\mathcal{D}} (\mathcal{D}_{\Lambda_0, \Theta, \rho_{0}}) 
= \mathcal{P}_{\Lambda} (\Lambda_0),\ \forall \Lambda_0 \in \mathcal{B}|_\Lambda$.

We use $x$ to represent one sample, which contains all necessary information for learning. 
As for DQN, $x = (s_t, a_t, r_t, s_{t+1})$.
As for R2D2, $x = (s_t, a_t, r_t, \dots, s_{t+N}, a_{t+N}, r_{t+N}, s_{t+N+1})$.
As for IMPALA, $x$ also contains the distribution of the behavior policy.
The content of $x$ depends on the algorithm, but it's assumed to be sufficient for learning.
We use $\mathcal{X}$ to represent the set of samples.
At training stage $t$, 
given the parameter $\theta = \theta^{(t)}$, 
the distribution of the index set $\mathcal{P}_{\Lambda} = \mathcal{P}^{(t)}_{\Lambda}$ (e.g., sampling distribution of behavior policy)
and the distribution of the initial state $\rho_0$, 
we denote the set of samples as
\begin{equation*}
\begin{aligned}
    \mathcal{X}_{\rho_{0}}^{(t)}
    &\overset{def}{=} \bigcup_{d_{\rho_{0}}^\pi \sim \mathcal{P}_\mathcal{D}^{(t)}} \{ x | x \sim d_{\rho_{0}}^\pi \} \\
    &= \bigcup_{\lambda \sim \mathcal{P}_{\Lambda}^{(t)}} 
    \{ x | x \sim d_{\rho_{0}}^{\pi_\theta}, 
    \theta   = {\theta^{(t)}_{\lambda}} \}
    \\
    &\triangleq \bigcup_{\lambda \sim \mathcal{P}_{\Lambda}^{(t)}} \mathcal{X}^{(t)}_{\rho_{0}, \lambda}.
\end{aligned}
\end{equation*}

\subsection{Capacity and Diversity  Control of Behavior Policy} 
\label{sec: Explicit  Capacity and Diversity  Control of Behavior Policy}

We consider the problem that behavior policies $\mu$ are sampled from a policy space $\{\pi_{\theta_{\lambda}} | \lambda \in \Lambda\}$ which is parameterized by the policy network and indexed by the index set $\Lambda$.
The capacity of $\mu$ describes \emph{how many different behavior policies are there in the policy space}, controlled by the base policy's capacity (e.g., shared parameters or not) and the size of the index set $|\Lambda|$. Noting that there are two sets of parameters, namely $\lambda$ and $\theta$. The diversity describes \emph{how many different behavior policies are actually selected from the policy space to generate training data}, controlled by the sampling/selective distribution $\mathcal{P}_\Lambda$ (see Fig. \ref{fig: Algorithm Architecture Diagram}). 
 
After the capacity of the base policy is determined, we can explicitly control the data richness via the size of the index set and the sampling distribution $\mathcal{P}_\Lambda$. On the condition that  interactions are limitless, increasing the size  of the index set can significantly improve the data richness and thus is more important for a superior final performance since the diversity can be maximized via adopting a uniform distribution (most prior works have achieved SOTA in this way). However, it's data inefficient and the condition may never hold. Considering interactions are limited, the optimization of the sampling distribution, namely to select suitable behavior policies to generate training data, is crucial for sample efficiency because each interaction is rare. It's also known as the exploration-exploitation trade-off problem.

\subsection{Data Distribution Optimization Problem}
\label{sec: Data Distribution Optimization Problem}

In conclusion, the final performance can be significantly improved via increasing the data richness controlled by  the capacity and diversity of behavior policy. The sample efficiency is significantly influenced by the exploration-exploitation trade-off, namely the sampling/selective distribution  of the behavior policy. In general, on the condition that the capacity of behavior policy is \emph{determined} and training data is totally generated by behavior policies, these problems can be cast into  the data distribution optimization problem:

\begin{definition}[Data Distribution Optimization Problem] \label{Data Distribution Optimization Problem} Finding a selective distribution $\mathcal{P}_{\Lambda}$ that samples behavior policies $\pi_{\theta_\lambda}$ from a parameterized policy space that indexed by $\Lambda$ and maximizing some target function $L_{\mathcal{E}}$, where the $L_{\mathcal{E}}$ can be any target function (e.g., RL target) that describes what kind of data do agents desire (i.e., a measure of the importance/value of the sample trajectory).
\end{definition}

 \subsection{Generalized Data Distribution Iteration}
 \label{sec: Generalized Data Distribution Iteration}
 Now we introduce our main algorithm to handle the data distribution optimization problem in RL.

\begin{figure}[ht]
  \centering
  \begin{minipage}{\linewidth}
    \begin{algorithm}[H]
      \caption{Generalized Data Distribution Iteration}  
          \begin{algorithmic}
            \STATE Initialize $\Lambda$, $\Theta$, $\mathcal{P}_{\Lambda}^{(0)}$, $\theta^{(0)}$.
            \FOR{$t=0,1,2,\dots$}
                \STATE Sample $\{\mathcal{X}^{(t)}_{\rho_0, \lambda}\}_{\lambda \sim \mathcal{P}^{(t)}_{\Lambda}}$. \COMMENT{Data Sampling}
                \STATE $\theta^{(t+1)} = \mathcal{T}( \theta^{(t)}, \{\mathcal{X}^{(t)}_{\rho_0, \lambda}\}_{\lambda \sim \mathcal{P}^{(t)}_{\Lambda}} )$. \COMMENT{Generalized Policy Iteration} 
                \STATE $\mathcal{P}_{\Lambda}^{(t+1)}  = \mathcal{E}(\mathcal{P}_{\Lambda}^{(t)}, \{\mathcal{X}^{(t)}_{\rho_0, \lambda}\}_{\lambda \sim \mathcal{P}^{(t)}_{\Lambda}} )$. \COMMENT{Data Distribution Iteration}
            \ENDFOR
          \end{algorithmic}
        \label{alg:GDI}
    \end{algorithm}
  \end{minipage}
\end{figure}

$\mathcal{T}$ defined as $\theta^{(t+1)} 
= \mathcal{T}( \theta^{(t)}, \{\mathcal{X}^{(t)}_{\rho_0, \lambda}\}_{\lambda \sim \mathcal{P}^{(t)}_{\Lambda}} )$
is a typical optimization operator of RL algorithms, 
which utilizes the collected samples to update the parameters for maximizing some function $L_{\mathcal{T}}$.
For instance, $L_{\mathcal{T}}$ may contain the policy gradient and the state value evaluation for the policy-based methods, 
may contain generalized policy iteration for the value-based methods, 
and may also contain some auxiliary tasks or intrinsic rewards for specially designed methods.

$\mathcal{E}$ defined as $\mathcal{P}_{\Lambda}^{(t+1)}  
= \mathcal{E}(\mathcal{P}_{\Lambda}^{(t)}, \{\mathcal{X}^{(t)}_{\rho_0, \lambda}\}_{\lambda \sim \mathcal{P}^{(t)}_{\Lambda}} )$ 
is a data distribution optimization operator.
It uses the samples $\{\mathcal{X}^{(t)}_{\rho_0, \lambda}\}_{\lambda \sim \mathcal{P}^{(t)}_{\Lambda}}$ to update $\mathcal{P}_{\Lambda}$ and maximize some function $L_{\mathcal{E}}$, namely,
\begin{equation*}
    \mathcal{P}_{\Lambda}^{(t+1)} = \argmax_{\mathcal{P}_{\Lambda}} L_{\mathcal{E}} (\{\mathcal{X}^{(t)}_{\rho_0, \lambda}\}_{\lambda \sim \mathcal{P}_{\Lambda}}).
\end{equation*}
Since $\mathcal{P}_{\Lambda}$ is parameterized,  we abuse the notation and use $\mathcal{P}_{\Lambda}$ to represent the parameter of $\mathcal{P}_{\Lambda}$.
If $\mathcal{E}$ is a first-order optimization operator, then we can write $\mathcal{E}$ explicitly as
\begin{equation*}
    \mathcal{P}_{\Lambda}^{(t+1)} = \mathcal{P}_{\Lambda}^{(t)} + \eta \nabla_{\mathcal{P}_{\Lambda}^{(t)}} L_{\mathcal{E}} (\{\mathcal{X}^{(t)}_{\rho_0, \lambda}\}_{\lambda \sim \mathcal{P}^{(t)}_{\Lambda}}).
\end{equation*}
If $\mathcal{E}$ is a second-order optimization operator, like natural gradient, we can write $\mathcal{E}$ formally as
    \begin{gather*}
        \mathcal{P}_{\Lambda}^{(t+1)} = \mathcal{P}_{\Lambda}^{(t)} + \eta
        \textbf{F}(\mathcal{P}_{\Lambda}^{(t)})^\dagger
        \nabla_{\mathcal{P}_{\Lambda}^{(t)}} L_{\mathcal{E}} (\{\mathcal{X}^{(t)}_{\rho_0, \lambda}\}_{\lambda \sim \mathcal{P}^{(t)}_{\Lambda}}), \\
        \textbf{F}(\mathcal{P}_{\Lambda}^{(t)}) = \left[\nabla_{\mathcal{P}_{\Lambda}^{(t)}} \log \mathcal{P}_{\Lambda}^{(t)} \right]
        \cdot
        \left[\nabla_{\mathcal{P}_{\Lambda}^{(t)}} \log \mathcal{P}_{\Lambda}^{(t)} \right]^\top, \\
    \end{gather*}
where $\dagger$ denotes the Moore-Penrose pseudoinverse of the matrix.

\subsection{An Operator View of RL Methods}
\label{sec: An Operator View of RL Methods}
We can further divide all algorithms into two categories, GDI-I$^n$ and GDI-H$^n$.
$n$ represents the degree of freedom of $\Lambda$, which is the dimension of selective distribution.
I represents Isomorphism. 
We say one algorithm belongs to GDI-I$^n$, if $\theta = \theta_{\lambda}, \, \forall \lambda \in \Lambda$.
H represents Heterogeneous.
We say one algorithm belongs to GDI-H$^n$, if $\theta_{\lambda_1} \neq \theta_{\lambda_2}, \, \exists \lambda_1, \lambda_2 \in \Lambda$.
By definition, GDI-H$^n$ is a much larger set than GDI-I$^n$, but many algorithms belong to GDI-I$^n$ rather than GDI-H$^n$.
We say one algorithm is "w/o $\mathcal{E}$" if it doesn't contain the operator $\mathcal{E}$, which means its $\mathcal{E}$ is an identical mapping and the data distribution is not additionally optimized. 
Now, we could understand some well-known RL methods from the view of GDI.


For DQN, RAINBOW, PPO and IMPALA, they are in GDI-I$^0$ w/o $\mathcal{E}$. Let $|\Lambda| = 1$, WLOG, assume $\Lambda = \{\lambda_0\}$.
Then, the probability measure $\mathcal{P}_{\Lambda}$ collapses to $\mathcal{P}_{\Lambda} (\lambda_0) = 1$. 
$\Theta = \{\theta_{\lambda_0}\}$.
$\mathcal{E}$ is an identical mapping of $\mathcal{P}_{\Lambda}^{(t)}$.
$\mathcal{T}$ is the first-order operator that optimizes the loss functions.

For Ape-X and R2D2, they are in GDI-I$^1$ w/o $\mathcal{E}$. 
Let $\Lambda = \{\epsilon_l |\ l = 1, \dots, 256\}$.
$\mathcal{P}_{\Lambda}$ is uniform, $\mathcal{P}_{\Lambda} (\epsilon_l) = |\Lambda|^{-1}$.
Since all actors and the learner share parameters, we have $\theta_{\epsilon_1} = \theta_{\epsilon_2}$ for $\forall \epsilon_1, \epsilon_2 \in \Lambda$, hence $\Theta = \bigcup_{\epsilon \in \Lambda} \{\theta_{\epsilon}\} = \{\theta_{\epsilon_l}\}, \ \forall\ l = 1,\dots, 256$.
$\mathcal{E}$ is an identical mapping, because $\mathcal{P}_{\Lambda}^{(t)}$ is always a uniform distribution.
$\mathcal{T}$ is the first-order operator that optimizes the loss functions.

For LASER, it's in GDI-H$^1$ w/o $\mathcal{E}$. 
Let $\Lambda = \{i |\ i = 1, \dots, K\}$ to be the number of learners.
$\mathcal{P}_{\Lambda}$ is uniform, $\mathcal{P}_{\Lambda} (i) = |\Lambda|^{-1}$.
Since different learners don't share parameters, $\theta_{i_1} \cap \theta_{i_2} = \emptyset$ for $\forall i_1, i_2 \in \Lambda$, hence $\Theta = \bigcup_{i \in \Lambda} \{\theta_i\}$.
$\mathcal{E}$ is an identical mapping.
$\mathcal{T}$ can be formulated as a union of $\theta^{(t+1)}_i 
= \mathcal{T}_{i}( \theta^{(t)}_i, \{\mathcal{X}^{(t)}_{\rho_0, \lambda}\}_{\lambda \sim \mathcal{P}^{(t)}_{\Lambda}} )$, 
which represents optimizing $\theta_i$ of the $i$th learner with shared samples from other learners.

For PBT, it's in GDI-H$^{n+1}$, where $n$ is the number of  searched hyperparameters.
Let $\Lambda = \{h\} \times \{i| i=1,\dots, K\}$, where $h$ represents the hyperparameters being searched and $K$ is the population size.
$\Theta = \bigcup_{i=1,\dots, K} \{\theta_{i, h}\}$, where $\theta_{i, h_1} = \theta_{i, h_2}$ for $\forall (h_1, i), (h_2, i) \in \Lambda$.
$\mathcal{E}$ is the meta-controller that adjusts $h$ for each $i$, which can be formally written as 
$\mathcal{P}_{\Lambda}^{(t+1)}(\cdot, i)
= \mathcal{E}_{i}(\mathcal{P}_{\Lambda}^{(t)}(\cdot, i), \{\mathcal{X}^{(t)}_{\rho_0, (h, i)}\}_{h \sim \mathcal{P}^{(t)}_{\Lambda}(\cdot, i)} )$,
which optimizes $\mathcal{P}_{\Lambda}$ according to the performance of all agents in the population.
$\mathcal{T}$ can also be formulated as a union of $\mathcal{T}_i$, but is 
$\theta^{(t+1)}_i 
= \mathcal{T}_{i}( \theta^{(t)}_i, \{\mathcal{X}^{(t)}_{\rho_0, (h, i)}\}_{h \sim \mathcal{P}^{(t)}_{\Lambda}(\cdot, i)})$,
which represents optimizing the $i$th agent with samples from the $i$th agent.

For NGU and Agent57, it's in GDI-I$^2$. 
Let $\Lambda = \{\beta_i | i=1,\dots,m\} \times \{\gamma_j | j=1,\dots,n\}$, where $\beta$ is the weight of the intrinsic value function and $\gamma$ is the discount factor.
Since all actors and the learner share variables, $\Theta = \bigcup_{(\beta, \gamma) \in \Lambda} \{\theta_{(\beta, \gamma)}\} = \{\theta_{(\beta, \gamma)}\}$ for $\forall (\beta, \gamma) \in \Lambda$.
$\mathcal{E}$ is an optimization operator of a multi-arm bandit controller with UCB, which aims to maximize the expected cumulative rewards by adjusting $\mathcal{P}_{\Lambda}$.
Different from above, $\mathcal{T}$ is identical to our general definition $\theta^{(t+1)} 
= \mathcal{T}( \theta^{(t)}, \{\mathcal{X}^{(t)}_{\rho_0, \lambda}\}_{\lambda \sim \mathcal{P}^{(t)}_{\Lambda}} )$,
which utilizes samples from all $\lambda$s to update the shared $\theta$.

For Go-Explore, it's in GDI-H$^1$. 
Let $\Lambda = \{\tau\}$, where $\tau$ represents the stopping time of switching between robustification and exploration.
$\Theta = \{\theta_r\} \cup \{\theta_e\}$, where $\theta_r$ is the robustification model and $\theta_e$ is the exploration model.
$\mathcal{E}$ is a search-based controller, which defines the next $\mathcal{P}_{\Lambda}$ for better exploration.
$\mathcal{T}$ can be decomposed into $(\mathcal{T}_r, \mathcal{T}_e)$.

\subsection{Monotonic Data Distribution Optimization}
\label{sec: Monotonic Data Distribution Optimization}

We see that many algorithms can be formulated as a special case of GDI.
For algorithms without a meta-controller, whose data distribution optimization operator $\mathcal{E}$ is an identical mapping, the guarantee that the learned policy could converge to the optimal policy has been widely studied, for instance, GPI in \citep{sutton} and policy gradient in \citep{pgtheory}.
However, for algorithms with a meta-controller, whose data distribution optimization operator $\mathcal{E}$ is non-identical, though most algorithms in this class show superior performance, it still lacks a general study on why the data distribution optimization operator $\mathcal{E}$ helps.
In this section, with a few assumptions, we show that given the same optimization operator $\mathcal{T}$, a GDI with a non-identical data distribution optimization operator $\mathcal{E}$ is always superior to that without $\mathcal{E}$.

For brevity, we denote the expectation of $L_\mathcal{E}, L_\mathcal{T}$ for each $\lambda \in \Lambda$ as $\mathcal{L}_{\mathcal{E}} (\lambda, \theta_\lambda) $ and $\mathcal{L}_{\mathcal{T}} (\lambda, \theta_\lambda)$, calculated as
\begin{equation*}
\begin{aligned}
        \mathcal{L}_{\mathcal{E}} (\lambda, \theta_\lambda) 
    && = \textbf{E}_{x \sim \pi_{\theta_\lambda}} [L_{\mathcal{E}} (\{\mathcal{X}_{\rho_0, \lambda}\})]\\ 
    \mathcal{L}_{\mathcal{T}} (\lambda, \theta_\lambda) 
    && = \textbf{E}_{x \sim \pi_{\theta_\lambda}} [L_{\mathcal{T}} (\{\mathcal{X}_{\rho_0, \lambda}\})]
\end{aligned}
\end{equation*}
and denote the expectation of $\mathcal{L}_\mathcal{E}(\lambda, \theta_\lambda), \mathcal{L}_\mathcal{T}(\lambda, \theta_\lambda)$ for any $\mathcal{P}_{\Lambda}$ as 
    $\mathcal{L}_{\mathcal{E}} (\mathcal{P}_{\Lambda}, \theta) 
    = \textbf{E}_{\lambda \sim \mathcal{P}_{\Lambda}} [\mathcal{L}_{\mathcal{E}} (\lambda, \theta_\lambda) ], \ 
    \mathcal{L}_{\mathcal{T}} (\mathcal{P}_{\Lambda}, \theta) 
    = \textbf{E}_{\lambda \sim \mathcal{P}_{\Lambda}}  [\mathcal{L}_{\mathcal{T}} (\lambda, \theta_\lambda)].$

\begin{Assumption}[Uniform Continuous Assumption]
    For $\forall \epsilon > 0,\  \forall s \in \mathcal{S},\ \exists\, \delta > 0,\ s.t. |V^{\pi_1} (s) - V^{\pi_2} (s)| < \epsilon,\ \forall\, d_{\pi} (\pi_1, \pi_2) < \delta$,
    where $d_\pi$ is a metric on ${\Delta(\mathcal{A})}^\mathcal{S}$.
    If $\pi$ is parameterized by $\theta$, then for $\forall \epsilon > 0,\  \forall s \in \mathcal{S},\ \exists\, \delta > 0,\ s.t. |V^{\pi_{\theta_1}} (s) - V^{\pi_{\theta_2}} (s)| < \epsilon,\ \forall\, || \theta_1 - \theta_2 || < \delta$.
\label{asp:1}
\end{Assumption}
\begin{Remark}
    \citep{polytope} shows $V^\pi$ is infinitely differentiable everywhere on $\Delta (\mathcal{A})^{\mathcal{S}}$ if $|\mathcal{S}| < \infty, |\mathcal{A}| < \infty$.
    \citep{pgtheory} shows $V^\pi$ is $\beta$-smooth, namely bounded second-order derivative, for direct parameterization.
    If $\Delta (\mathcal{A})^{\mathcal{S}}$ is compact, continuity implies uniform continuity.
\end{Remark}

\begin{Assumption}[Formulation of $\mathcal{E}$ Assumption]
    Assume
    $\mathcal{P}_{\Lambda}^{(t+1)} 
    = \mathcal{E}(\mathcal{P}_{\Lambda}^{(t)}, \{\mathcal{X}^{(t)}_{\rho_0, \lambda}\}_{\lambda \sim \mathcal{P}^{(t)}_{\Lambda}}) $
    can be written as 
    $\mathcal{P}_{\Lambda}^{(t+1)} (\lambda)= \mathcal{P}_{\Lambda}^{(t)}(\lambda) \frac{\exp (\eta \mathcal{L}_{\mathcal{E}} (\lambda, \theta_{\lambda}^{(t)})  )}{Z^{(t+1)}}$, 
    $Z^{(t+1)} = \textbf{E}_{\lambda \sim \mathcal{P}_{\Lambda}^{(t)}}[\exp (\eta \mathcal{L}_{\mathcal{E}} (\lambda, \theta_{\lambda}^{(t)})  )]$.
\label{asp:2}
\end{Assumption}
\begin{Remark}
    The assumption is actually general.
    Regarding $\Lambda$ as an action space and 
    $r_\lambda 
    = \mathcal{L}_{\mathcal{E}} (\lambda, \theta_{\lambda}^{(t)})$, when solving $\argmax_{\mathcal{P}_{\Lambda}} \textbf{E}_{\lambda \sim \mathcal{P}_{\Lambda}} [\mathcal{L}_{\mathcal{E}} (\lambda, \theta_{\lambda}^{(t)})] 
    = \argmax_{\mathcal{P}_{\Lambda}} \textbf{E}_{\lambda \sim \mathcal{P}_{\Lambda}} [r_\lambda]$, the data distribution optimization operator $\mathcal{E}$ is equivalent to solving a multi-arm bandit (MAB) problem.
    For the first-order optimization, \citep{eq_pg_q} shows that the solution of a KL-regularized version, $\argmax_{\mathcal{P}_{\Lambda}} \textbf{E}_{\lambda \sim \mathcal{P}_{\Lambda}} [r_\lambda] - \eta KL(\mathcal{P}_{\Lambda} || \mathcal{P}_{\Lambda}^{(t)})$, is exactly the assumption.
    For the second-order optimization, let $\mathcal{P}_{\Lambda} = softmax (\{r_\lambda\})$, \citep{pgtheory} shows that the natural policy gradient of a softmax parameterization also induces exactly the assumption.
\end{Remark}

\begin{Assumption}[First-Order Optimization Co-Monotonic Assumption]
    For $\forall\, \lambda_1, \lambda_2 \in \Lambda$, we have
    $[ \mathcal{L}_{\mathcal{E}} (\lambda_1, \theta_{\lambda_1})  -  \mathcal{L}_{\mathcal{E}} (\lambda_2, \theta_{\lambda_2}) ] \cdot
    [ \mathcal{L}_{\mathcal{T}} (\lambda_1, \theta_{\lambda_1})  -  \mathcal{L}_{\mathcal{T}} (\lambda_2, \theta_{\lambda_2}) ] \geq 0$.
\label{asp:3}
\end{Assumption}

\begin{Assumption}[Second-Order Optimization Co-Monotonic Assumption]
    For $\forall\, \lambda_1, \lambda_2 \in \Lambda$,
    $\exists\, \eta_0 > 0$, s.t. $\forall\, 0 < \eta < \eta_0$, we have
    $[ \mathcal{L}_{\mathcal{E}} (\lambda_1, \theta_{\lambda_1})  -  \mathcal{L}_{\mathcal{E}} (\lambda_2, \theta_{\lambda_2}) ] \cdot
    [ G^{\eta} \mathcal{L}_{\mathcal{T}} (\lambda_1, \theta_{\lambda_1}) 
    - G^{\eta} \mathcal{L}_{\mathcal{T}} (\lambda_2, \theta_{\lambda_2}) ] \geq 0$,
    where $\theta_{\lambda}^{\eta} = \theta_{\lambda} + \eta \nabla_{\theta_{\lambda}} \mathcal{L}_{\mathcal{T}} (\lambda, \theta_{\lambda})$ and
    $G^{\eta} \mathcal{L}_{\mathcal{T}} (\lambda, \theta_{\lambda})
    = \frac{1}{\eta} \left[\mathcal{L}_{\mathcal{T}} (\lambda, \theta_{\lambda}^{\eta}) - \mathcal{L}_{\mathcal{T}} (\lambda, \theta_{\lambda}) \right]$.
\label{asp:4}
\end{Assumption}

Under assumptions \eqref{asp:1} \eqref{asp:2} \eqref{asp:3}, if $\mathcal{T}$ is a first-order operator, namely a gradient accent operator, to maximize $\mathcal{L}_{\mathcal{T}}$, GDI can be guaranteed to be superior to that w/o $\mathcal{E}$.
Under assumptions \eqref{asp:1} \eqref{asp:2} \eqref{asp:4}, if $\mathcal{T}$ is a second-order operator, namely a natural gradient operator, to maximize $\mathcal{L}_{\mathcal{T}}$, GDI can also be guaranteed to be superior to that w/o $\mathcal{E}$.

\begin{Theorem}[First-Order Optimization with Superior Target]
    Under assumptions \eqref{asp:1} \eqref{asp:2} \eqref{asp:3}, we have
    $\mathcal{L}_{\mathcal{T}} (\mathcal{P}_{\Lambda}^{(t+1)}, \theta^{(t+1)}) 
     = \textbf{E}_{\lambda \sim \mathcal{P}_{\Lambda}^{(t+1)}}  [\mathcal{L}_{\mathcal{T}} (\lambda, \theta^{(t+1)}_{\lambda})]
     \geq \textbf{E}_{\lambda \sim \mathcal{P}_{\Lambda}^{(t)}}  [\mathcal{L}_{\mathcal{T}} (\lambda, \theta^{(t+1)}_{\lambda})]
     = \mathcal{L}_{\mathcal{T}} (\mathcal{P}_{\Lambda}^{(t)}, \theta^{(t+1)})$.
\label{thm:1st_gdi}
\end{Theorem}

\begin{Proof}
    By \textbf{Theorem} \ref{thm:cts_Rp} (see App. \ref{App: proof}), the upper triangular transport inequality, 
    let $f(\lambda) = \mathcal{L}_{\mathcal{T}} (\lambda, \theta_{\lambda})$ and 
    $g(\lambda) = \mathcal{L}_{\mathcal{E}} (\lambda, \theta_{\lambda})$,
    the proof is done.
\end{Proof}

\begin{Remark}[Superiority of Target]
     In Algorithm \ref{alg:GDI}, if $\mathcal{E}$ updates $\mathcal{P}_{\Lambda}^{(t)}$ at time $t$, then the operator $\mathcal{T}$ at time $t+1$ can be written as $\theta^{(t+2)} = \theta^{(t+1)} + \eta \nabla_{\theta^{(t+1)}} \mathcal{L}_{\mathcal{T}} (\mathcal{P}_{\Lambda}^{(t+1)}, \theta^{(t+1)})$.
     If $\mathcal{P}_{\Lambda}^{(t)}$ hasn't been updated at time $t$, then the operator $\mathcal{T}$ at time $t+1$ can be written as $\theta^{(t+2)} = \theta^{(t+1)} + \eta \nabla_{\theta^{(t+1)}} \mathcal{L}_{\mathcal{T}} (\mathcal{P}_{\Lambda}^{(t)}, \theta^{(t+1)})$.
     \textbf{Theorem} \ref{thm:1st_gdi} shows that the target of $\mathcal{T}$ at time $t+1$ becomes higher if $\mathcal{P}_{\Lambda}^{(t)}$ is updated by $\mathcal{E}$ at time $t$.
\end{Remark}

\begin{table*}[!htbp]
\small
\setlength{\tabcolsep}{1.0pt}
    \centering
    \caption{Experiment results of Atari. Playtime is the equivalent human playtime, HWRB is the human world record breakthrough, HNS is the human normalized score, HWRNS is the human world records normalized score, $\text{SABER}=\max\{\min\{\text{HWRNS},2\},0\}$.}
    \label{tab:atari_results}
    \begin{tabular}{c c c c c c c c c}
    \toprule
                      & GDI-H$^3$                   & GDI-I$^3$               & Muesli & RAINBOW & LASER & R2D2 & NGU & Agent57\\
    \midrule
     Training Scale (Num. Frames)      &\textbf{2E+8}        &\textbf{2E+8}       &\textbf{2E+8}  & \textbf{2E+8} & \textbf{2E+8}  & 1E+10   & 3.5E+10  &1E+11 \\
    Playtime (Day)  & \textbf{\GDIHgametime}       & \textbf{\GDIIgametime}      & \textbf{\muesligametime} & \textbf{\rainbowgametime} & \textbf{\lasergametime}  & \rtdtgametime   & \ngugametime    & \agentgametime \\
    HWRB              &\textbf{\GDIHHWRB}          &\GDIIHWRB         & \muesliHWRB             & \rainbowHWRB             & \laserHWRB              & \rtdtHWRB      & \nguHWRB & \agentHWRB \\
    Mean HNS(\%)      &\textbf{\GDIHmeanhns}     &\GDIImeanhns    & \mueslimeanhns        & \rainbowmeanhns        & \lasermeanhns        & \rtdtmeanhns &\ngumeanhns   &\agentmeanhns \\
    Median HNS(\%)    &\GDIHmedianhns               &\GDIImedianhns               & \mueslimedianhns        & \rainbowmedianhns        & \lasermedianhns         & \rtdtmedianhns & \ngumedianhns   &\textbf{\agentmedianhns}\\
    Mean HWRNS(\%)    &\textbf{\GDIHmeanHWRNS}      &\GDIImeanHWRNS              & \mueslimeanHWRNS         & \rainbowmeanHWRNS         & \lasermeanHWRNS           & \rtdtmeanHWRNS   & \ngumeanHWRNS     &\agentmeanHWRNS\\
    Median HWRNS(\%)  &\textbf{\GDIHmedianHWRNS}                &\GDIImedianHWRNS               & \mueslimedianHWRNS          & \rainbowmedianHWRNS          & \lasermedianHWRNS           &\rtdtmedianHWRNS    & \ngumedianHWRNS    &\agentmedianHWRNS\\
    Mean SABER(\%)    &\GDIHmeanSABER                &\GDIImeanSABER               & \mueslimeanSABER          & \rainbowmeanSABER         & \lasermeanSABER          &\rtdtmeanSABER    &\ngumeanSABER &\textbf{\agentmeanSABER}\\
    Median SABER(\%)  &\textbf{\GDIHmedianSABER}                & \GDIImedianSABER              & \mueslimedianSABER          & \rainbowmedianSABER         & \lasermedianSABER           &\rtdtmedianSABER    & \ngumedianSABER     &\agentmedianSABER\\
    \bottomrule
    \end{tabular}
\end{table*}
\normalsize

\begin{Example}[Practical Implementation]
    Let $\mathcal{L}_{\mathcal{E}} (\lambda, \theta_{\lambda}) = \mathcal{J}_{\pi_{\theta_{\lambda}}}$ and $\mathcal{L}_{\mathcal{T}} (\lambda, \theta_{\lambda}) = \mathcal{J}_{\pi_{\theta_{\lambda}}}$.
    $\mathcal{E}$ can update $\mathcal{P}_{\Lambda}$ by the Monte-Carlo estimation of $\mathcal{J}_{\pi_{\theta_{\lambda}}}$.
    $\mathcal{T}$ is to maximize $\mathcal{J}_{\pi_{\theta_{\lambda}}}$, which can be any RL algorithms.
\end{Example}

\begin{Theorem}[Second-Order Optimization with Superior Improvement]
    Under assumptions \eqref{asp:1} \eqref{asp:2} \eqref{asp:4}, we have
    $\textbf{E}_{\lambda \sim \mathcal{P}_{\Lambda}^{(t+1)}}  [G^{\eta} \mathcal{L}_{\mathcal{T}} (\lambda, \theta_{\lambda}^{(t+1)})] 
    \geq \textbf{E}_{\lambda \sim \mathcal{P}_{\Lambda}^{(t)}}  [G^{\eta} \mathcal{L}_{\mathcal{T}} (\lambda, \theta_{\lambda}^{(t+1)})] $, more specifically,
   \begin{equation*}
       \begin{aligned}
               &\textbf{E}_{\lambda \sim \mathcal{P}_{\Lambda}^{(t+1)}}
    [\mathcal{L}_{\mathcal{T}} (\lambda, \theta_{\lambda}^{(t+1),\eta}) - \mathcal{L}_{\mathcal{T}} (\lambda, \theta_{\lambda}^{(t+1)}) ] \\
    &\geq \textbf{E}_{\lambda \sim \mathcal{P}_{\Lambda}^{(t)}}
    [\mathcal{L}_{\mathcal{T}} (\lambda, \theta_{\lambda}^{(t+1),\eta}) - \mathcal{L}_{\mathcal{T}} (\lambda, \theta_{\lambda}^{(t+1)}) ]
       \end{aligned}
   \end{equation*}
\label{thm:2nd_gdi}
\end{Theorem}

\begin{Proof}
    By \textbf{Theorem} \ref{thm:cts_Rp} (see App. \ref{App: proof}), the upper triangular transport inequality,
    let $f(\lambda) = G^{\eta} \mathcal{L}_{\mathcal{T}} (\lambda, \theta_{\lambda})$ and 
    $g(\lambda) = \mathcal{L}_{\mathcal{E}} (\lambda, \theta_{\lambda})$,
    the proof is done.
\end{Proof}

\begin{Remark}[Superiority of Improvement]
    \textbf{Theorem} \ref{thm:2nd_gdi} shows that, if $\mathcal{P}_{\Lambda}$ is updated by $\mathcal{E}$, the expected improvement of $\mathcal{T}$ is higher.
\end{Remark}

\begin{Example}[Practical Implementation]
\label{example: GDI meta}
    Let $\mathcal{L}_{\mathcal{E}} (\lambda, \theta_{\lambda}) = \textbf{E}_{s \sim d_{\rho_0}^{\pi}} \textbf{E}_{a \sim \pi(\cdot | s) \exp(\epsilon A^{\pi}(s, \cdot))/Z} [A^{\pi}(s, a)]$, where $\pi = \pi_{\theta_{\lambda}}$.
    Let $\mathcal{L}_{\mathcal{T}} (\lambda, \theta_{\lambda}) = \mathcal{J}_{\pi_{\theta_{\lambda}}}$.
    If we optimize $\mathcal{L}_{\mathcal{T}} (\lambda, \theta_{\lambda})$ by natural gradient, 
    \citep{pgtheory} shows that, for direct parameterization, the natural policy gradient gives $\pi^{(t+1)} \propto \pi^{(t)} \exp (\epsilon A^{\pi^{(t)}})$, by \textbf{Lemma} \ref{lemma:perfdiff} (see App. \ref{App: proof}), the performance difference lemma,
    $V^{\pi} (s_0) - V^{\pi'} (s_0) = \frac{1}{1 - \gamma} \textbf{E}_{s \sim d_{s_0}^\pi} \textbf{E}_{a \sim \pi (\cdot | s)} [ A^{\pi'} (s, a) ]$, 
    hence if we ignore the gap between the states visitation distributions of $\pi^{(t)}$ and $\pi^{(t+1)}$, 
    $\mathcal{L}_{\mathcal{E}} (\lambda, \theta_{\lambda}^{(t)}) \approx \frac{1}{1 - \gamma} \textbf{E}_{s \sim d_{\rho_0}^{\pi}} [V^{\pi^{(t+1)}}(s) - V^{\pi^{(t)}} (s)]$, 
    where $\pi^{(t)} = \pi_{\theta_{\lambda}^{(t)}}$.
    Hence, $\mathcal{E}$ is actually putting more measure on $\lambda$ that can achieve more improvement.
\end{Example}

\begin{figure*}[!t]
\subfigure[Performance of Control Groups]{
\includegraphics[width=0.4\textwidth]{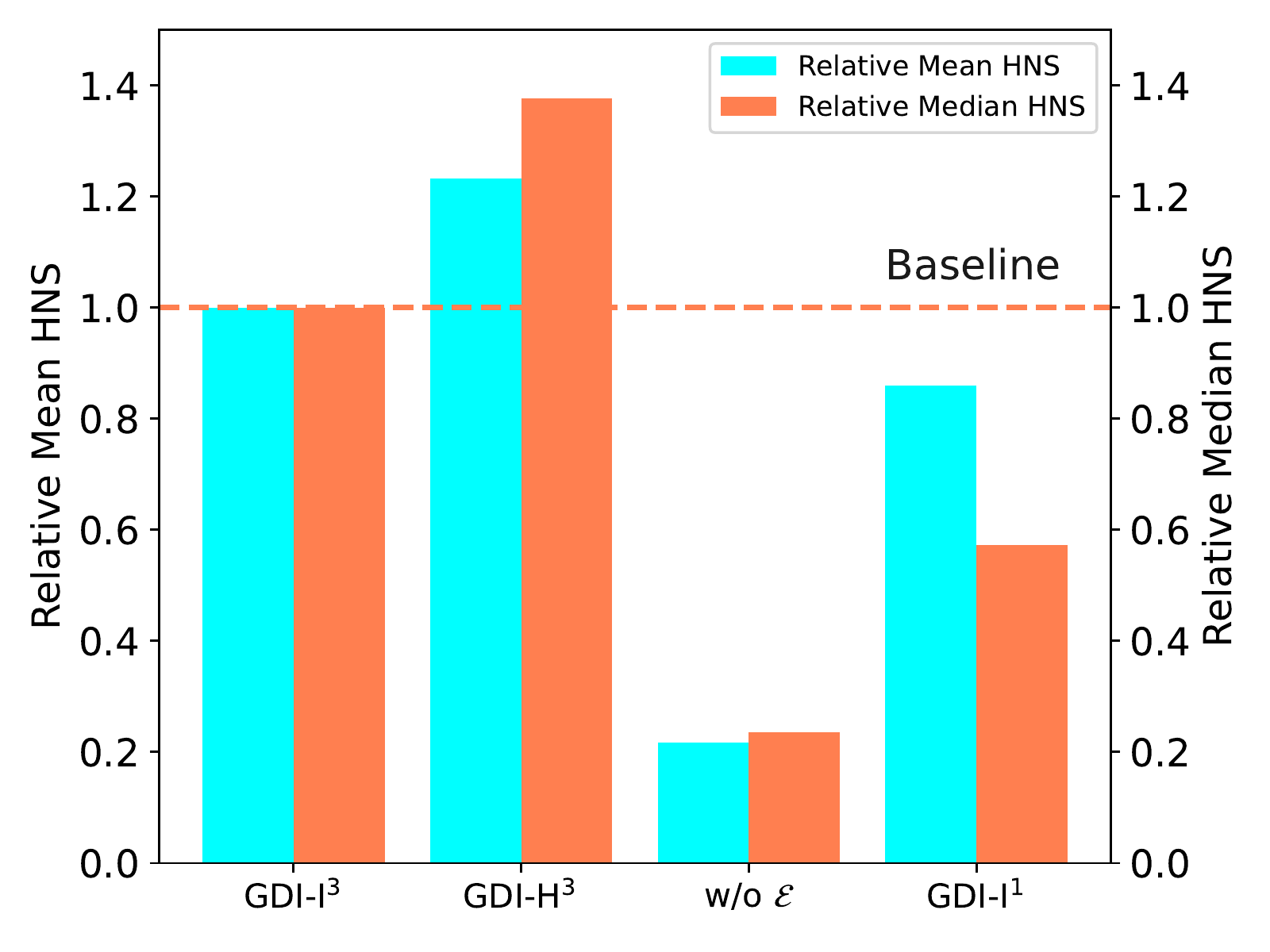}
}
\subfigure[t-SNE of GDI-I$^3$]{
		\includegraphics[width=0.25\textwidth]{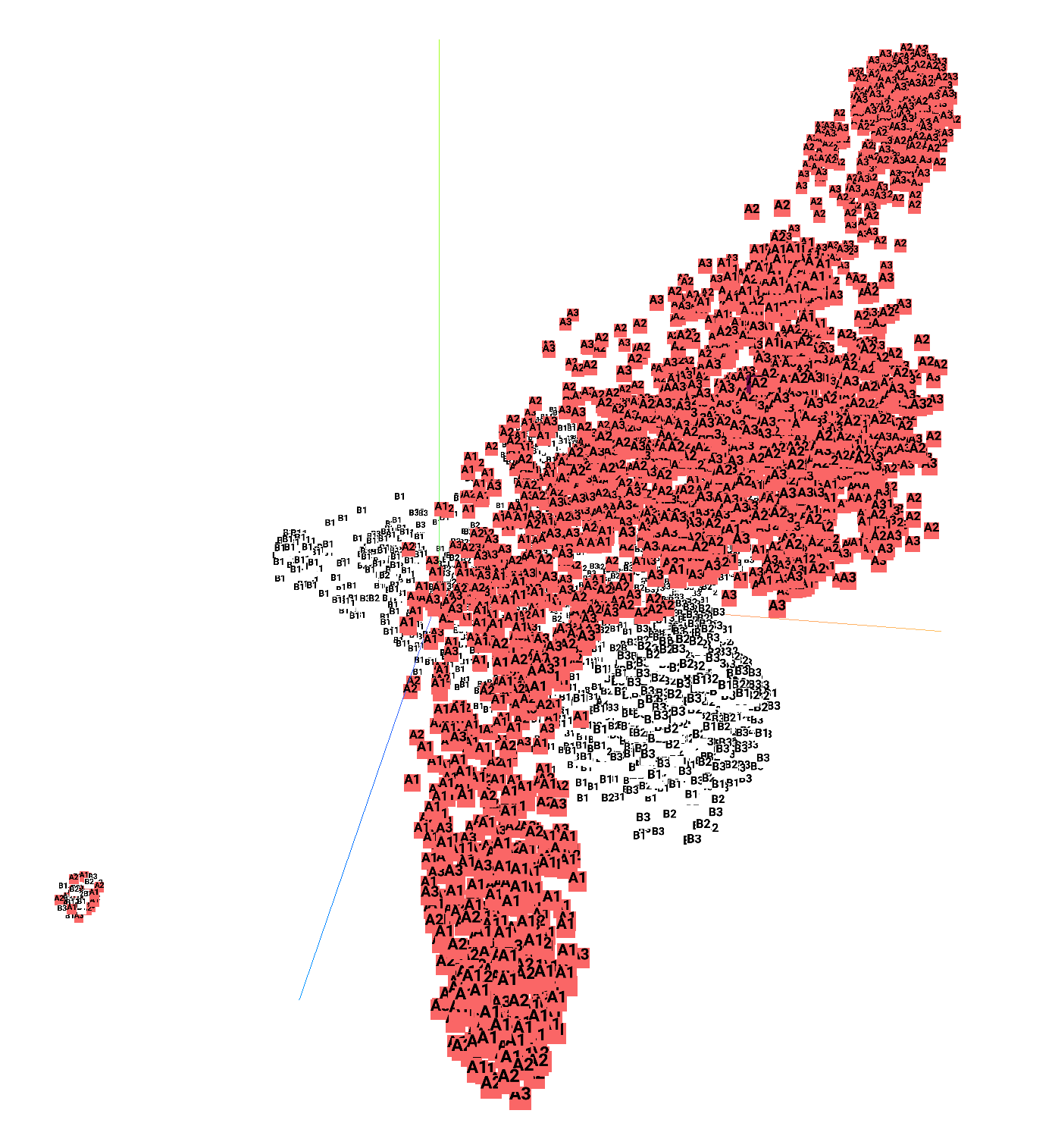}
	}
	\subfigure[t-SNE of GDI-I$^1$]{
		\includegraphics[width=0.25\textwidth]{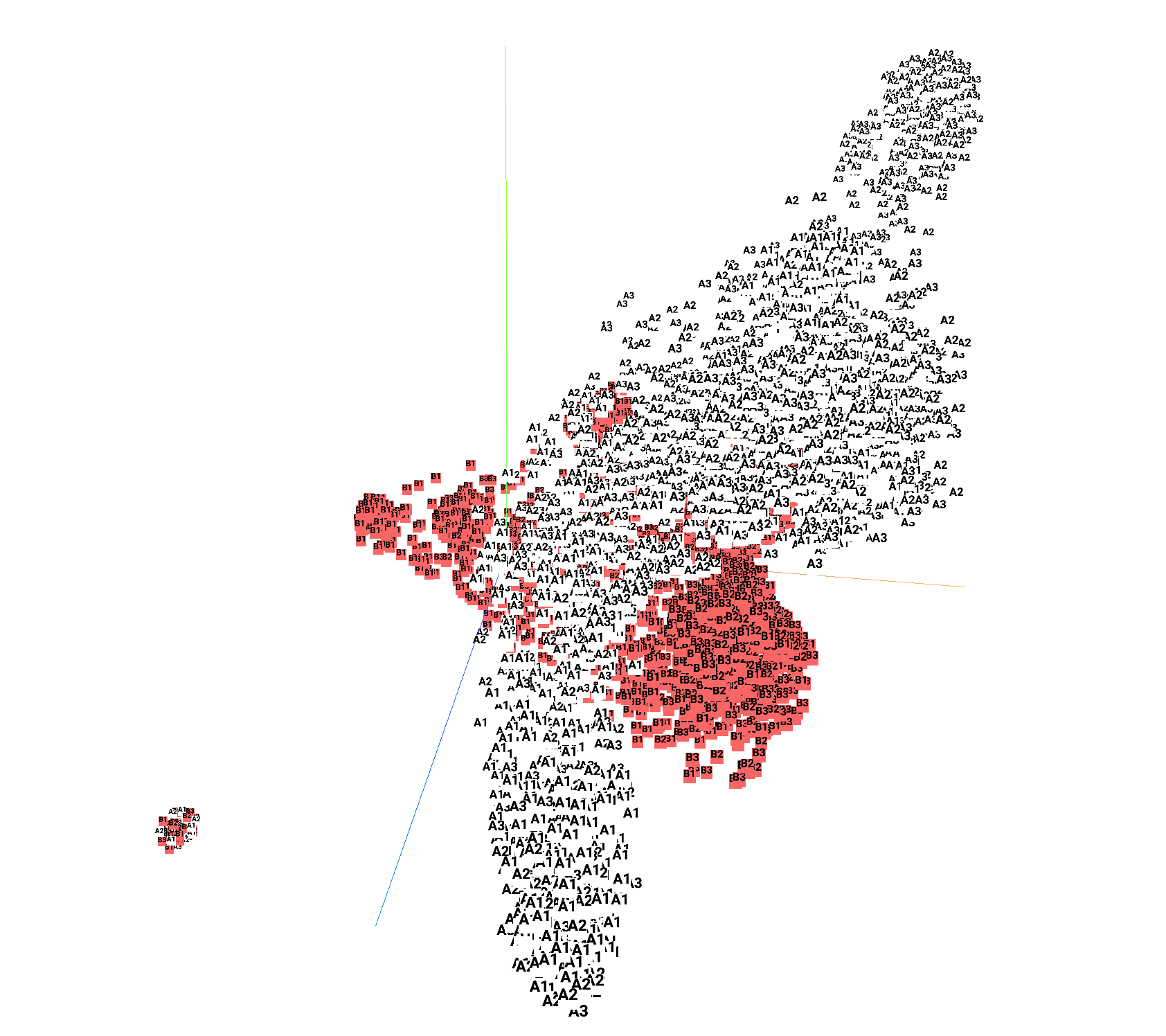}
	}
\caption{Figures of ablation study. \textbf{(a)} shows  how   the ablation groups (see App. \ref{Sec: appendix Ablation Study}) perform compared with the baseline (i.e., GDI-I$^3$). Noting that the performance has been normalized by GDI-I$^3$ (e.g., $\frac{\text{Mean HNS of  GDI-I}^1}{\text{Mean HNS of GDI-I}^3}$), and w/o $\mathcal{E}$ means without the meta-controller. \textbf{(b)} and \textbf{(c)} illustrate the  data richness (e.g., $\frac{\text{Seen Conditions}}{\text{All Conditions}}$) of GDI-I$^1$ and GDI-I$^3$ via t-SNE  of visited states (see App. \ref{app: tsne}).}
\label{fig:ablation study}
\end{figure*}

\section{Experiment}
\label{sec: experiment}
In this section, we designed our experiment to answer the following questions:
\begin{itemize}
    \item How to implement RL algorithms based on GDI step by step (see Sec. \ref{sec: Practical Implement Based on GDI})?  Whether the proposed methods can outperform all prior SOTA RL algorithms in both sample efficiency and final performance  (see Tab. \ref{tab:atari_results})?
    \item How to construct a behavior policy space (see Sec. \ref{sec: Policy Space Construction})? What's the impact of the size of the index set $\Lambda$, namely, whether the data richness can be improved via increasing the capacity and diversity (see Fig. \ref{fig:ablation study})?
    \item How to design a data distribution optimization operator (e.g., a meta-controller) to tackle the exploration and exploitation trade-off (see Sec. \ref{sec: Practical Implement Based on GDI})? How much performance would be degraded without data distribution optimization, namely no meta-controller (see Fig. \ref{fig:ablation study})?
\end{itemize}

\subsection{Practical Implementation Based on GDI}
\label{sec: Practical Implement Based on GDI}

\paragraph{Policy Space Construction}
\label{sec: Policy Space Construction}

To illustrate the effectiveness of GDI, we give two representative practical implementations of GDI, namely GDI-I$^3$ and GDI-H$^3$, the capacity of whose behavior policy space is larger than Agent57.  Let $\Lambda = \{\lambda | \lambda = (\tau_1, \tau_2, \epsilon)\}$.
The behavior policy belongs to a \emph{soft} entropy policy space including policies ranging from very exploratory to purely exploitative and thereby the optimization of the sampling distribution  of behavior policy $\mathcal{P}_{\Lambda}$ can be reframed into the trade-off between exploration and exploitation.
We define the behavior policy $\pi_{\theta_{\lambda}}$ as
\begin{equation}
\label{equ: soft epsilon policy space}
    \pi_{\theta_{\lambda}}=\epsilon \cdot \operatorname{Softmax}\left(\frac{A_{\theta_1}}{\tau_{1}}\right)+(1-\epsilon) \cdot \operatorname{Softmax}\left(\frac{A_{\theta_2}}{\tau_{2}}\right)
\end{equation}
wherein $\pi_{\theta_{\lambda}}$ constructs a parameterized policy space, and the index set $\Lambda$ is constructed by $\lambda = (\tau_1, \tau_2, \epsilon)$.
For GDI-I$^3$, $A_{\theta_1}$ and $A_{\theta_2}$ are identical advantage functions \citep{dueling_q}. Namely, they are estimated by an isomorphic family of trainable variables $\theta$.
The learning policy is also $\pi_{\theta_{\lambda}}$.
For GDI-H$^3$,  $A_{\theta_1}$ and $A_{\theta_2}$ are different, and they are estimated by two different families of trainable variables (i.e., $\theta_1 \neq \theta_2$).
Since \textbf{GDI needn't assume $A_{\theta_1}$ and $A_{\theta_2}$ are learned from the same MDP}, we adopt two kinds of reward shaping to learn $A_{\theta_1}$ and $A_{\theta_2}$ respectively, which  can see App. \ref{Sec: appendix hyperparameters}.
More implementation details see App. \ref{App: Algorithm Pseudocode}.

\paragraph{Data Distribution Optimization Operator} The operator $\mathcal{E}$, which optimizes $\mathcal{P}_{\Lambda}$,  is achieved by Multi-Arm Bandits \citep[MAB]{sutton}, 
where assumption \eqref{asp:2} holds naturally.
For more details, can see App. \ref{Sec: appendix MAB}.

\paragraph{Reinforcement Learning  Optimization Operator}  The operator $\mathcal{T}$ is achieved by policy gradient, V-Trace and ReTrace \citep{impala, retrace} (see App. \ref{app: background on RL}), which meets Theorem \ref{thm:1st_gdi} by first-order optimization. 

\subsection{Summary of Results}
\label{sec: Summary of Results}

\paragraph{Experimental Details} Recommended by \citep{agent57,atarihuman}, we construct an evaluation system to highlight the superiority of GDI from multiple levels (see App. \ref{app:Evaluation Metrics for ALE}). Furthermore, to avoid any issues that aggregated metrics may have, App. \ref{appendix: experiment results} provides full learning curves for all games and detailed comparison tables of raw and normalized scores. More details see App. \ref{sec:app Experiment Details}.

\paragraph{Effectiveness of GDI} The aggregated results across games are reported in Tab. \ref{tab:atari_results}. Our agents obtain the highest mean HNS with the minimal training frames, leading to the best learning efficiency. Furthermore, our agents have surpassed 22 human world records within  38 playtime days, which is \textbf{500 times} more efficient than Agent57. Extensive experiments have demonstrated the fact that either GDI-I$^3$ or GDI-H$^3$ could obtain  superhuman performance with remarkable learning efficiency.

\paragraph{Discussion of the Results} Agent57 could obtain the highest median HNS but relatively lower learning efficiency via \textbf{i)} a relatively larger behavior policy space and  a meta-controller \textbf{ii)} intrinsic rewards and nearly unlimited data. However, Agent57 fails to distinguish the value of data and thereby collects many useless/low-value samples.  Other algorithms are struggling to match our performance.

\subsection{Ablation Study}

\paragraph{Ablation Study Design} In the ablation study, we further investigate the effects of several properties of GDI. In the first experiment, we demonstrate the effectiveness of the capacity and diversity control via exploring how different sizes  of the index set of the policy space influence the performance and data richness. In the second experiment, we highlight the effectiveness of data distribution optimization operator $\mathcal{E}$ via ablating $\mathcal{E}$. More details can see App. \ref{Sec: appendix Ablation Study}.

\paragraph{Effectiveness of Capacity and Diversity Control}

In this experiment, we firstly implement a GDI-I$^1$ algorithm with Boltzmann policy space (i.e., $\pi_{\theta_{\lambda}}=\text{Softmax}(\frac{A}{\tau})$) to explore the impact of the  capacity and diversity control. Then, we explore whether the data richness is indeed improved via a case study of t-SNE of GDI-I$^3$ and GDI-I$^1$. Results are illustrated in Fig. \ref{fig:ablation study}, from which we could find the visited states of GDI-I$^3$ are indeed  richer than GDI-I$^1$, which concludes its better performance. In the same way, the behavior policy space of GDI-I$^3$ is a sub-space (i.e., $\theta_1=\theta_2$) of that of GDI-H$^3$, leading to further performance improvement.

\paragraph{Effectiveness of  Data Distribution Optimization} 
From  Fig. \ref{fig:ablation study}, we could also find that not using a meta-controller (e.g., the index $\lambda$ of behavior policy takes a fixed value) will  dramatically degrade performance, which confirms the effectiveness of the data distribution optimization and echoes the previous theoretical proof.

\section{Conclusion}

Simultaneously obtaining superior sample efficiency and  better final performance is an important and challenging problem in RL. In this paper, we present the first attempt to address this problem from  training data distribution control, namely to obtain any desired (e.g., nontrivial) data within \emph{limited} interactions. To tackle this problem, we firstly cast it into a data distribution optimization problem. Then, we handle this problem via \textbf{i)} explicitly modeling and controlling the diversity  of the behavior policies and \textbf{ii)} adaptively tackling the  exploration-exploitation trade-off using meta-learning. After integrating this process into GPI, we surprisingly find a more general framework GDI and then we give an operation-version of recent SOTA algorithms. Under the guidance of GDI, we propose feasible implementations and achieve the superhuman final performance with remarkable learning efficiency within only 38 playtime days.


\section*{Acknowledgements}
We are grateful for the careful reading and insightful reviews of meta-reviewers and reviewers.

\nocite{langley00}

\bibliography{example_paper}
\bibliographystyle{icml2022}

\newpage
\appendix
\onecolumn

\section{Summary of Notation and Abbreviation}
\label{app: Abbreviation and Notation}
In this section, we briefly summarize some common notations and abbreviations in this paper for the convenience of readers, which are illustrated in Tab. \ref{tab: notation} and Tab. \ref{tab: abbreviation}.

\begin{table}[!hb]
	\centering
	\caption{Summary of Notation}
	\label{tab: notation}
	\begin{tabular}{|c |c| }
	    \hline
		\textbf{Notation} &\textbf{Description}\\
		\hline
	    $s$ & state \\
	    \hline
	    $a$ & action \\
	    \hline
	    $\mathcal{S} $ & set of all states \\
	    \hline
	    $\mathcal{A} $ & set of all actions \\
	    \hline
	    $\Delta$  & probability simplex \\
	    \hline
	    $\mu $ & behavior policy \\
	    \hline
	    $\pi $ & target policy \\
	    \hline
	    $G_t $ & \makecell[c]{cumulative discounted reward \\ or return at $t$} \\
	    \hline
	    $d_{\rho_0}^{\pi}$  & \makecell[c]{the states visitation distribution of $\pi$ \\ with the initial state distribution 
	    $\rho_0$} \\
	    \hline
	    $J_{\pi}$     & \makecell[c]{the expectation of the returns \\ with the states visitation distribution of $\pi$} \\
	    \hline
	    $V^{\pi}$ & the state value function of $\pi$\\
	    \hline
	    $Q^{\pi}$ &  the state-action value function of $\pi$\\
	    \hline
	    $\gamma$ & discount-rate parameter \\
	    \hline
	    $\delta_{t}$ & temporal-difference error at $t$\\
	    \hline
	    $\Lambda$ & set of indexes  \\
	    \hline
	    $\lambda$ & one index in $\Lambda$ \\
	    \hline
	    $\mathcal{P}_{\Lambda}$ & one probability measure on $\Lambda$ \\
	    \hline
	    $\Theta$  & set of all possible parameter values \\
	    \hline
	    $\theta$  & one parameter value in $\Theta$ \\
	    \hline
	    $\theta_\lambda$ & \makecell[c]{a subset of $\theta$, indicates \\ the parameter in $\theta$ being used by the index $\lambda$} \\
	    \hline
	    $\mathcal{X}$ & set of samples \\
	    \hline
	    $x$ & one sample in $\mathcal{X}$ \\
	    \hline
	    $\mathcal{D}$ & set of all possible states visitation distributions \\
	    \hline
	    $\mathcal{E}$ & the data distribution optimization operator \\
	    \hline
	    $\mathcal{T}$ & the RL algorithm optimization operator \\
	    \hline
	    $L_{\mathcal{E}}$ & \makecell[c]{the loss function of $\mathcal{E}$ to be maximized, \\ calculated by the samples set $\mathcal{X}$ }\\
	    \hline
	    $\mathcal{L}_\mathcal{E}$ & \makecell[c]{expectation of $L_{\mathcal{E}}$, \\ with respect to each sample $x \in \mathcal{X}$} \\
	    \hline
	    $L_{\mathcal{T}}$ & \makecell[c]{the loss function of $\mathcal{T}$ to be maximized,\\  calculated by the samples set $\mathcal{X}$} \\
	    \hline
	    $\mathcal{L}_\mathcal{T}$ & \makecell[c]{expectation of $L_{\mathcal{T}}$, \\ with respect to each sample $x \in \mathcal{X}$} \\
		\hline
	\end{tabular} 
\end{table}

\begin{table}[!hb]
	\centering
	\caption{Summary ofAbbreviation}
	\label{tab: abbreviation}
	\begin{tabular}{|c| c|}
		\hline
		\textbf{Abbreviation} &\textbf{Description}\\
		\hline
		Sec.  & Section \citep{agent57} \\
		\hline
		Figs. & Figures \citep{dreamerv2} \\
		\hline
		Fig. & Figure \citep{agent57} \\
		\hline
		Eq.     & Equation \citep{agent57}\\
		\hline
		Tab.    & Table \citep{agent57} \\
		\hline
		App.    & Appendix \citep{agent57} \\
		\hline
		SOTA & State-of-The-Art \citep{agent57}    \\
		\hline
		RL  & Reinforcement Learning \citep{sutton} \\
		\hline
		DRL & Deep Reinforcement Learning \citep{sutton} \\
        \hline
        GPI & Generalized Policy Iteration \citep{sutton} \\
        \hline
        PG  & Policy Gradient \citep{sutton} \\
        \hline
        AC  & Actor Critic \citep{sutton} \\
        \hline
        ALE & Atari Learning Environment \citep{ale} \\
        \hline
        HNS   & Human Normalized Score \citep{ale} \\
        \hline
        HWRB & Human World Records Breakthrough \\
        \hline
        HWRNS & Human World Records Normalized Score \\
        \hline
        SABER & Standardized Atari BEnchmark for RL \citep{atarihuman}\\
        \hline
        CHWRNS & Capped Human World Records Normalized Score \\      
        \hline
        WLOG   & Without Loss of Generality \\
        \hline
        w/o    & Without \\
		\hline
	\end{tabular} 
\end{table}

\clearpage
\section{Background on RL}
\label{app: background on RL}

 The RL problem can be formulated by a Markov decision process \citep[MDP]{howard1960dynamic} defined by the tuple  $\left(\mathcal{S}, \mathcal{A}, p, r, \gamma, \rho_{0}\right)$. 
 Considering a discounted episodic MDP, the initial state $s_0$ will be sampled from the distribution denoted by $\rho_0(s): \mathcal{S} \rightarrow \Delta(\mathcal{S})$. 
 At each time t, the agent choose an action $a_t \in \mathcal{A}$ according to the policy $\pi(a_t|s_t): \mathcal{S} \rightarrow \Delta(\mathcal{A})$ at state $s_t \in \mathcal{S}$. 
 The environment receives the action, produces a reward $r_t \sim r(s,a): \mathcal{S} \times \mathcal{A} \rightarrow \mathbf{R}$ and transfers to the next state $s_{t+1}$  submitted to the transition distribution $p\left(s^{\prime} \mid s, a\right): \mathcal{S} \times \mathcal{A} \rightarrow \Delta(\mathcal{S})$. 
 The process continues until the agent reaches a terminal state or a maximum time step. 
 Define return $G_t = \sum_{k=0}^\infty \gamma^k r_{t+k}$, state value function $V^{\pi}(s_t) = \textbf{E}\left[ \sum_{k=0}^\infty \gamma^k r_{t+k} | s_t \right]$, state-action value function $Q^{\pi}(s_t, a_t) = \textbf{E}\left[ \sum_{k=0}^\infty \gamma^k r_{t+k} | s_t, a_t \right]$, and advantage function $A^{\pi}(s_t,a_t) = Q^{\pi}(s_t, a_t) - V^{\pi}(s_t)$, wherein $\gamma \in(0,1)$ is the discount factor.
The connections between $V^\pi$ and $Q^\pi$ is given by the Bellman equation,
\begin{equation*}
    \mathcal{T}Q^{\pi} (s_t, a_t) = \textbf{E}_{'\pi} [r_t + \gamma V^{\pi}(s_{t+1})],
\end{equation*}
where
\begin{equation*}
    V^{\pi} (s_t)  = \textbf{E}_{\pi} [Q^{\pi} (s_t, a_t)].
\end{equation*}
The goal of reinforcement learning is to find the optimal policy $\pi^*$ that maximizes the expected sum of discounted rewards, denoted by $\mathcal{J}$ \citep{sutton}:
\begin{equation*}
\pi^{*}=\underset{\pi}{\operatorname{argmax}} \mathcal{J}_{\pi}(\tau) = \underset{\pi}{\operatorname{argmax}} \textbf{E}_{\pi}\left[G_{t}\right]= \underset{\pi}{\operatorname{argmax}} \textbf{E}_{\pi}[\sum_{k=0}^{\infty} \gamma^{k} r_{t+k}]
\end{equation*}


Model-free reinforcement learning (MFRL) has made many impressive breakthroughs in a wide range of Markov decision processes  \citep[MDP]{alpha_star,ftw,agent57}.
MFRL mainly consists of two categories, valued-based methods \citep{dqn,rainbow} and policy-based methods \citep{trpo,ppo,impala}.

Value-based methods learn state-action values and select actions according to these values. 
One merit of value-based methods is to accurately control the exploration rate of the behavior policies by some trivial mechanism, such like $\epsilon$-greedy.
The drawback is also apparent. 
The policy improvement of valued-based methods totally depends on the policy evaluation. 
Unless the selected action is changed by a more accurate policy evaluation, the policy won't be improved. 
So the policy improvement of each policy iteration is limited, which leads to a low learning efficiency.
Previous works equip valued-based methods with many appropriated designed structures, achieving a more promising learning efficiency \citep{dueling_q,priority_q,r2d2}.

In practice, value-based methods maximize $\mathcal{J}$ by policy iteration \citep{sutton}. 
The policy evaluation is fulfilled by minimizing $\textbf{E}_{\pi} [(G - Q^\pi) ^ 2]$, which gives the gradient ascent direction 
$\textbf{E}_{\pi} [(G - Q^\pi) \nabla Q^\pi]$. 
The policy improvement is usually achieved by $\epsilon$-greedy.

Q-learning is a typical value-based methods, which updates the state-action value function $Q(s,a)$ with Bellman Optimality Equation \citep{qlearning}: 
\begin{equation*}
    \begin{array}{c}
    \delta_{t}=r_{t+1}+\gamma \arg \max _{a} Q\left(s_{t+1}, a\right)-Q\left(s_{t}, a_{t}\right) \\
    Q\left(s_{t}, a_{t}\right) \leftarrow Q\left(s_{t}, a_{t}\right)+\alpha \delta_{t}
    \end{array}
\end{equation*}
wherein $\delta_t$ is the temporal difference error \citep{TDerror}, and $\alpha$ is the learning rate.

A refined structure design of $Q^\pi$ is achieved by \citep{dueling_q}. It estimates $Q^\pi$ by a summation of two separated networks, $Q^\pi = A^\pi + V^\pi$, which has been widely studied in \citep{dueling_q,casa_bridge}.

Policy gradient \citep[PG]{williams1992simple} methods is an outstanding representative of policy-based RL algorithms, which directly parameterizes the policy and  updates through optimizing the following objective: 
\begin{equation*}
    \mathcal{J} (\theta)=\mathbb{E}_{\pi}\left[\sum_{t=0}^{\infty} \log \pi_{\theta}\left(a_{t} \mid s_{t}\right) R(\tau)\right]
\end{equation*}
wherein $R(\tau)$ is the cumulative return on trajectory $\tau$. In PG method, policy improves via ascending  along the gradient of the above equation, denoted as policy gradient:
\begin{equation*}
\nabla_{\theta} \mathcal{J} \left(\pi_{\theta}\right) =\underset{\tau \sim \pi_{\theta}}{\mathrm{E}}\left[\sum_{t=0}^{\infty} \nabla_{\theta} \log \pi_{\theta}\left(a_{t} \mid s_{t}\right) R(\tau)\right]
\end{equation*}

One merit of policy-based methods is that they incorporate a policy improvement phase every training step, suggesting a higher learning efficiency than value-based methods.
Nevertheless, policy-based methods easily fall into a suboptimal solution, where the entropy drops to $0$ \citep{sac}.
The actor-critic methods introduce a value function as the baseline to reduce the variance of the policy gradient \citep{a3c}, but maintain the other characteristics unchanged.

Actor-Critic \citep[AC]{sutton} reinforcement learning updates the policy gradient with an value-based critic, which can reduce variance of estimates and thereby ensure  more stable and rapid optimization.
\begin{equation*}
    \nabla_{\theta} \mathcal{J}(\theta)=\mathbb{E}_{\pi}\left[\sum_{t=0}^{\infty} \psi_{t} \nabla_{\theta} \log \pi_{\theta}\left(a_{t} \mid s_{t}\right)\right]
\end{equation*}
wherein $\psi_{t}$ is the critic to guide the improvement directions of policy improvement, which can be the state-action value function $Q^{\pi}\left(s_{t}, a_{t}\right)$, the advantage function $A^{\pi}\left(s_{t}, a_{t}\right)=Q^{\pi}\left(s_{t}, a_{t}\right)-V^{\pi}(s_t)$.

\subsection{Retrace}

When large scale training is involved, the off-policy problem is inevitable.
Denote $\mu$ to be the behavior policy, $\pi$ to be the target policy, and $c_t = \min\{\frac{\pi_t}{\mu_t}, \Bar{c}\}$ to be the clipped importance sampling. 
For brevity, denote $c_{[t: t+k]} = \prod_{i=0}^{k} c_{t+i}$. 
ReTrace \citep{retrace} estimates $Q(s_t, a_t)$ by clipped per-step importance sampling
\begin{equation*}
\label{Equ: retrace}
    Q^{\Tilde{\pi}} (s_t, a_t) 
= \textbf{E}_{\mu} [ Q(s_t, a_t) + \sum_{k \geq 0} \gamma^k 
c_{[t+1:t+k]} \delta^{Q}_{t+k} Q ],
\end{equation*}
where $\delta^{Q}_t Q \overset{def}{=} r_t + \gamma Q(s_{t+1}, a_{t+1}) - Q(s_t, a_t)$. 
The above operator is a contraction mapping, 
and $Q$ converges to $Q^{\Tilde{\pi}_{ReTrace}}$ that corresponds to some $\Tilde{\pi}_{ReTrace}$.

\subsection{Vtrace}
Policy-based methods maximize $\mathcal{J}$ by policy gradient. 
It's shown \citep{sutton} that $\nabla \mathcal{J} = \textbf{E}_\pi [G \nabla \log \pi]$. 
When involved with a baseline, it becomes an actor-critic algorithm such as $\nabla \mathcal{J} = \textbf{E}_\pi [(G - V^\pi) \nabla \log \pi]$, where $V^\pi$ is optimized by minimizing $\textbf{E}_\pi [(G - V^\pi)^2]$, i.e. gradient ascent direction $\textbf{E}_\pi [(G - V^\pi)\nabla V^\pi]$.

IMPALA \citep{impala} introduces V-Trace off-policy actor-critic algorithm to correct for the discrepancy between target policy and behavior policy. Denote $\rho_t = \min\{\frac{\pi_t}{\mu_t}, \Bar{\rho} \}$. V-Trace estimates $V(s_t)$ by
\begin{equation*}
\label{Equ: vtrace}
    V^{\Tilde{\pi}} (s_t) 
        = \textbf{E}_{\mu} [ 
        V(s_t) + \sum_{k \geq 0} \gamma^k 
     c_{[t:t+k-1]} \rho_{t+k}  \delta^{V}_{t+k} V ],
\end{equation*}
where $\delta^{V}_t V \overset{def}{=} r_t + \gamma V(s_{t+1}) - V(s_t)$. 
If $\Bar{c} \leq \Bar{\rho}$, the above operator is a contraction mapping, and $V$ converges to $V^{\Tilde{\pi}}$ that corresponds to 
$$
        \Tilde{\pi}(a|s) = \frac
        {\min \left\{\Bar{\rho} \mu (a|s), \pi(a|s)\right\}}
        {\sum_{b \in \mathcal{A}}\min \left\{\Bar{\rho} \mu (b|s), \pi(b|s)\right\}}.
$$
The policy gradient is given by
$$
\textbf{E}_\mu \left[\rho_t (r_t + \gamma V^{\Tilde{\pi}}(s_{t+1}) - V(s_t)) \nabla \log \pi \right].
$$

\clearpage

\section{Theoretical Proof}
\label{App: proof}

For a monotonic sequence of numbers which satisfies $a = x_0 < x_1 < \dots < x_n < b$,
we call it a split of interval $[a, b]$.

\begin{Lemma}[Discretized Upper Triangular Transport Inequality for Increasing Functions in $\mathbf{R}^1$]
Assume $\mu$ is a continuous probability measure supported on $[0, 1]$. 
Let $0 = x_0 < x_1 < \dots < x_n < 1$ to be any split of $[0, 1]$.
Define $\Tilde{\mu}(x_i) = \mu([x_i, x_{i+1}))$. 
Define 
$$\Tilde{\beta}(x_i) = \Tilde{\mu}(x_i) \exp(x_i) / Z,\  Z = \sum_{i} \Tilde{\mu}(x_i) \exp (x_i).$$ 
Then there exists a probability measure $\gamma: \{x_i\}_{i=0,\dots, n} \times \{x_i\}_{i=0,\dots, n} \rightarrow [0, 1]$, s.t. 
\begin{equation}
    \left\{
    \begin{aligned}
        &\sum_j \gamma (x_i, y_j) = \Tilde{\mu} (x_i),& &\ i = 0, \dots, n; \\
        &\sum_i \gamma (x_i, y_j) = \Tilde{\beta} (y_j),& &\ j = 0, \dots, n; \\
        &\gamma (x_i, y_j) = 0,& &\ i > j. \\
    \end{aligned}
    \right.
\label{eq:UTTC_R1}
\end{equation}
Then for any monotonic increasing function $f:\{x_i\}_{i=0,\dots, n} \rightarrow $ \textbf{R},
we have
$$
\textbf{E}_{\Tilde{\mu}} [f] \leq \textbf{E}_{\Tilde{\beta}} [f].
$$
\label{lemma:dct_inc_R1}
\end{Lemma}

\begin{proof}[Proof of Lemma \ref{lemma:dct_inc_R1}]

For any couple of measures $(\mu, \beta)$, we say the couple satisfies Upper Triangular Transport Condition (UTTC), if there exists $\gamma$ s.t. \eqref{eq:UTTC_R1} holds.

Given $0 = x_0 < x_1 < \dots < x_n < 1$, we prove the existence of $\gamma$ by induction.

Define 
\begin{equation*}
    \Tilde{\mu}_m(x_i) = 
\left\{
\begin{aligned}
    &\mu ([x_i, x_{i+1})),& &i < m, \\
    &\mu ([x_i, 1)),& &i = m, \\
    & 0,& &i > m.
\end{aligned}
\right.
\end{equation*}

Define
$$
\Tilde{\beta}_m (x_i) = \Tilde{\mu}_m (x_i) \exp(x_i) / Z_m,
\  Z_m = \sum_{i} \Tilde{\mu}_m (x_i) \exp (x_i).$$ 

Noting if we prove that $(\Tilde{\mu}_m, \Tilde{\beta}_m)$ satisfies UTTC for $m=n$, it's equivalent to prove the existence of $\gamma$ in \eqref{eq:UTTC_R1}.

To clarify the proof, we use $x_i$ to represent the point for $\Tilde{\mu}$-axis in coupling and $y_j$ to represent the point for $\Tilde{\beta}$-axis, but they are actually identical, i.e. $x_i = y_j$ when $i = j$. 

When $m = 0$, it's obvious that $(\Tilde{\mu}_0, \Tilde{\beta}_0)$ satisfies UTTC, as
$$
\gamma_0 (x_i, y_j) = \left\{
\begin{aligned}
    &1, \ i = 0, j = 0, \\
    &0, \ else.
\end{aligned}
\right.
$$

Assume UTTC holds for $m$, i.e. there exists $\gamma_m$ s.t. $(\Tilde{\mu}_m, \Tilde{\beta}_m)$ satisfies UTTC, we want to prove it also holds for $m+1$.

By definition of $\Tilde{\mu}_m$, we have
\begin{equation*}
    \left\{
    \begin{aligned}
        &\Tilde{\mu}_m (x_i) = \Tilde{\mu}_{m+1} (x_i),& &i < m, \\
        &\Tilde{\mu}_m (x_i) = \Tilde{\mu}_{m+1} (x_i) + \Tilde{\mu}_{m+1} (x_{i+1}),& &i=m, \\
        &\Tilde{\mu}_m (x_{m+1}) = \Tilde{\mu}_m (x_i) = \Tilde{\mu}_{m+1} (x_i) = 0,& &i>m+1. \\
    \end{aligned}
    \right.
\end{equation*}

By definition of $\Tilde{\beta}_m$, we have
\begin{equation*}
    \left\{
    \begin{aligned}
        &\Tilde{\beta}_m (x_i) = \Tilde{\beta}_{m+1} (x_i) \cdot \frac{Z_{m+1}}{Z_m},& &i < m, \\
        &\Tilde{\beta}_m (x_i) = \left( \Tilde{\beta}_{m+1} (x_i)  + \Tilde{\beta}_{m+1} (x_{i+1}) \exp(x_i - x_{i+1}) \right) \cdot \frac{Z_{m+1}}{Z_m},& &i=m, \\
        &\Tilde{\beta}_m (x_{m+1}) = \Tilde{\beta}_m (x_i) = \Tilde{\beta}_{m+1} (x_i) = 0,& &i>m+1. \\
    \end{aligned}
    \right.
\end{equation*}

Multiplying $\gamma_m$ by $\frac{Z_m}{Z_{m+1}}$, we get the following UTTC
\begin{equation*}
    \left\{
    \begin{aligned}
        &\sum_j \frac{Z_m}{Z_{m+1}} \gamma_m (x_i, y_j) = \frac{Z_m}{Z_{m+1}} \Tilde{\mu}_{m+1} (x_i),& &\ i < m; \\
        &\sum_j \frac{Z_m}{Z_{m+1}} \gamma_m (x_i, y_j) = \frac{Z_m}{Z_{m+1}} (\Tilde{\mu}_{m+1} (x_i) + \Tilde{\mu}_{m+1} (x_{i+1})),& &\ i = m; \\
        &\sum_j \frac{Z_m}{Z_{m+1}} \gamma_m (x_i, y_j) = 0,& &\ i = m+1; \\
        &\sum_j \frac{Z_m}{Z_{m+1}} \gamma_m (x_i, y_j) = \Tilde{\mu}_{m+1} (x_i) = 0,& &\ i > m+1; \\
        &\sum_i \frac{Z_m}{Z_{m+1}} \gamma_m (x_i, y_j) = \Tilde{\beta}_{m+1} (y_j),& &\ j < m; \\
        &\sum_i \frac{Z_m}{Z_{m+1}} \gamma_m (x_i, y_j) = \Tilde{\beta}_{m+1} (y_i)  + \Tilde{\beta}_{m+1} (y_{j+1}) \exp(y_j - y_{j+1}),& &\ j = m; \\
        &\sum_i \frac{Z_m}{Z_{m+1}} \gamma_m (x_i, y_j) = 0,& &\ j = m+1; \\
        &\sum_i \frac{Z_m}{Z_{m+1}} \gamma_m (x_i, y_j) = \Tilde{\beta}_{m+1} (y_j) = 0,& &\ j > m+1; \\
        &\frac{Z_m}{Z_{m+1}} \gamma_m (x_i, y_j) = 0,& &\ i > j. \\
    \end{aligned}
    \right.
\end{equation*}

By definition of $Z_m$,
\begin{equation}
    Z_{m+1} - Z_m = \Tilde{\mu}_{m+1} (x_{m+1}) (\exp(x_{m+1}) - \exp(x_m)) > 0,
\label{eq:Zm_inc}
\end{equation}
so we have $\frac{Z_m}{Z_{m+1}} \Tilde{\mu}_{m+1} (x_i) < \Tilde{\mu}_{m+1} (x_i)$.

Noticing that $\Tilde{\beta}_{m+1} (y_{i+1}) \exp(y_i - y_{i+1}) <  \Tilde{\beta}_{m+1} (y_{i+1})$ and $\frac{Z_m}{Z_{m+1}} \Tilde{\mu}_{m+1} (x_i) < \Tilde{\mu}_{m+1} (x_i)$, we decompose the measure of $\frac{Z_m}{Z_{m+1}} \gamma_m$ at $(x_i, y_m)$ to $(x_i, y_m), (x_i, y_{m+1})$ for $i = 0, \dots, m-1$, and complement a positive measure at $(x_i, y_{m+1})$ to make up the difference between $\frac{Z_m}{Z_{m+1}} \Tilde{\mu}_{m+1} (x_i)$ and $\Tilde{\mu}_{m+1} (x_i)$.
For $i=m$, we decompose the measure at $(x_m, y_m)$ to $(x_m, y_m), (x_m, y_{m+1}), (x_{m+1}, y_{m+1})$ and also complement a proper positive measure.

Now we define $\gamma_{m+1}$ by
\begin{equation*}
    \left\{
    \begin{aligned}
        &\gamma_{m+1} (x_i, y_j) = \frac{Z_m}{Z_{m+1}} \gamma_m (x_i, y_j),\qquad\qquad\qquad\qquad\quad i < m\ and\  j < m, \\
        &\gamma_{m+1} (x_i, y_j) = \left( \frac{Z_m}{Z_{m+1}} \gamma_m (x_i, y_j) + \frac{Z_{m+1} - Z_m}{Z_{m+1}} \Tilde{\mu}_{m+1} (x_i) \right) \\
        &\ \quad\qquad\qquad\quad \cdot \frac{\Tilde{\beta}_{m+1} (y_{j})}{\Tilde{\beta}_{m+1} (y_j)  + \Tilde{\beta}_{m+1} (y_{j+1})}, \qquad\qquad\quad i<m\ and\ j=m, \\
        &\gamma_{m+1} (x_i, y_j) = \left( \frac{Z_m}{Z_{m+1}} \gamma_m (x_i, y_j) + \frac{Z_{m+1} - Z_m}{Z_{m+1}} \Tilde{\mu}_{m+1} (x_i) \right) \\
        &\ \quad\qquad\qquad\quad \cdot \frac{\Tilde{\beta}_{m+1} (y_{j+1})}{\Tilde{\beta}_{m+1} (y_j)  + \Tilde{\beta}_{m+1} (y_{j+1})},\quad\qquad\  i<m\ and\ j=m+1, \\
        &\gamma_{m+1} (x_i, y_j) = 0,\quad\qquad\qquad\qquad\qquad i>j\ or\ i>m+1\ or\ j>m+1, \\
        &\gamma_{m+1} (x_m, y_m) = u, \\
        &\gamma_{m+1} (x_{m}, y_{m+1}) = v, \\
        &\gamma_{m+1} (x_{m+1}, y_{m+1}) = w, \\
    \end{aligned}
    \right.
\end{equation*}
where we assume $u, v, w$ to be the solution of the following equations
\begin{equation}
    \left\{
    \begin{aligned}
        &u + v + w = \Tilde{\mu}_{m+1} (x_m) + \Tilde{\mu}_{m+1} (x_{m+1}), \\
        &\frac{w}{u+v} = \frac{\Tilde{\mu}_{m+1} (x_{m+1})}{\Tilde{\mu}_{m+1} (x_m)}, \\
        &\frac{v+w}{u} = \frac{\Tilde{\beta}_{m+1} (x_{m+1})}{\Tilde{\beta}_{m+1} (x_m)}, \\
        &u, v, w \geq 0. \\
    \end{aligned}
    \right.
\label{eq:uvw_R1}
\end{equation}

It's obvious that 
\begin{equation*}
    \left\{
    \begin{aligned}
        &\sum_j \gamma_{m+1} (x_i, y_j) = \Tilde{\mu}_{m+1} (x_i) = 0,& &\ i>m+1, \\
        &\sum_i \gamma_{m+1} (x_i, y_j) = \Tilde{\beta}_{m+1} (y_j) = 0,& &\ j>m+1,\\
        &\gamma (x_i, y_j) = 0,& &\ i > j. \\
    \end{aligned}
    \right.
\end{equation*}

For $j < m$, since $\sum_i \frac{Z_m}{Z_{m+1}} \gamma_m (x_i, y_j) = \Tilde{\beta}_{m+1} (y_j)$, we have 
$$\sum_i \gamma_{m+1} (x_i, y_j) = \Tilde{\beta}_{m+1} (y_j), \ j < m.$$

For $i < m$, since $\sum_j \frac{Z_m}{Z_{m+1}} \gamma_m (x_i, y_j) = \frac{Z_m}{Z_{m+1}} \Tilde{\mu}_{m+1} (x_i) < \Tilde{\mu}_{m+1} (x_i)$, we add $\frac{Z_{m+1} - Z_m}{Z_{m+1}} \Tilde{\mu}_{m+1} (x_i) \frac{\Tilde{\beta}_{m+1} (y_{m})}{\Tilde{\beta}_{m+1} (y_m)  + \Tilde{\beta}_{m+1} (y_{m+1})}$, $\frac{Z_{m+1} - Z_m}{Z_{m+1}} \Tilde{\mu}_{m+1} (x_i) \frac{\Tilde{\beta}_{m+1} (y_{m+1})}{\Tilde{\beta}_{m+1} (y_m)  + \Tilde{\beta}_{m+1} (y_{m+1})}$ to $\gamma_{m+1} (x_i, y_{m})$, $\gamma_{m+1} (x_i, y_{m+1})$, respectively.
So we have 
$$\sum_j \gamma_{m+1} (x_i, y_j) = \Tilde{\mu}_{m+1} (x_i), \ i < m.$$

For $i = m, m+1$, since assumption \eqref{eq:uvw_R1} holds, we have $u + v + w = \Tilde{\mu}_{m+1} (x_m) + \Tilde{\mu}_{m+1} (x_{m+1}), \frac{w}{u+v} = \frac{\Tilde{\mu}_{m+1} (x_{m+1})}{\Tilde{\mu}_{m+1} (x_m)}$, it's obvious that $u + v = \Tilde{\mu}_{m+1} (x_m), w = \Tilde{\mu}_{m+1} (x_{m+1})$, which is 
$$\sum_j \gamma_{m+1} (x_i, y_j) = \Tilde{\mu}_{m+1} (x_i), \ i = m, m+1.$$

For $j = m, m+1$, we firstly have
\begin{equation*}
    \begin{aligned}
        \sum_{j=m, m+1} \sum_i \gamma_{m+1} (x_i, y_j)
        &= \sum_{j} \sum_i \gamma_{m+1} (x_i, y_j) - \sum_{j \neq m, m+1} \sum_i \gamma_{m+1} (x_i, y_j) \\
        &=  \sum_{i} \sum_j \gamma_{m+1} (x_i, y_j) - \sum_{j \neq m, m+1} \Tilde{\beta}_{m+1}(y_j) \\
        &= \sum_{i} \Tilde{\mu}_{m+1}(x_i) - \sum_{j \neq m, m+1} \Tilde{\beta}_{m+1}(y_j) \\
        &= 1 - (1 - \Tilde{\beta}_{m+1}(y_m) - \Tilde{\beta}_{m+1}(y_{m+1})) \\
        &= \Tilde{\beta}_{m+1}(y_m) + \Tilde{\beta}_{m+1}(y_{m+1}). \\
    \end{aligned}
\end{equation*}
By definition of $\gamma_{m+1}$, we know $\frac{\gamma_{m+1} (x_i, y_m)}{\gamma_{m} (x_i, y_m)} = \frac{\Tilde{\beta}_{m+1} (x_{m+1})}{\Tilde{\beta}_{m+1} (x_m)}$ for $i < m$.
By assumption \eqref{eq:uvw_R1}, we know $\frac{v+w}{u} = \frac{\Tilde{\beta}_{m+1} (x_{m+1})}{\Tilde{\beta}_{m+1} (x_m)}$.
Combining three equations above together, we have
$$\sum_i \gamma_{m+1} (x_i, y_j) = \Tilde{\beta}_{m+1} (y_j), \ j = m, m+1.$$

Now we only need to prove assumption \eqref{eq:uvw_R1} holds.
With linear algebra, we solve \eqref{eq:uvw_R1} and have
\begin{equation*}
    \left\{
    \begin{aligned}
        &u = w \frac{1 + \frac{\Tilde{\mu}_{m+1} (x_{m+1})}{\Tilde{\mu}_{m+1} (x_m)}}{\frac{\Tilde{\mu}_{m+1} (x_{m+1})}{\Tilde{\mu}_{m+1} (x_m)} \left(1 + \frac{\Tilde{\beta}_{m+1} (x_{m+1})}{\Tilde{\beta}_{m+1} (x_m)}\right)}, \\
        &v = w \frac{\frac{\Tilde{\beta}_{m+1} (x_{m+1})}{\Tilde{\beta}_{m+1} (x_m)} - \frac{\Tilde{\mu}_{m+1} (x_{m+1})}{\Tilde{\mu}_{m+1} (x_m)}}{\frac{\Tilde{\mu}_{m+1} (x_{m+1})}{\Tilde{\mu}_{m+1} (x_m)} \left(1 + \frac{\Tilde{\beta}_{m+1} (x_{m+1})}{\Tilde{\beta}_{m+1} (x_m)}\right)}, \\
        &w = \frac{\left(\Tilde{\mu}_{m+1} (x_m) + \Tilde{\mu}_{m+1} (x_{m+1})\right)\frac{\Tilde{\mu}_{m+1} (x_{m+1})}{\Tilde{\mu}_{m+1} (x_m)} \left(1 + \frac{\Tilde{\beta}_{m+1} (x_{m+1})}{\Tilde{\beta}_{m+1} (x_m)}\right)}{\left(1 + \frac{\Tilde{\mu}_{m+1} (x_{m+1})}{\Tilde{\mu}_{m+1} (x_m)}\right) \left(1 + \frac{\Tilde{\beta}_{m+1} (x_{m+1})}{\Tilde{\beta}_{m+1} (x_m)}\right)}. \\
    \end{aligned}
    \right.
\end{equation*}

It's obvious that $u, w \geq 0$. 
$v \geq 0$ also holds, because
\begin{equation}
    \begin{aligned}
    \frac{\Tilde{\beta}_{m+1} (x_{m+1})}{\Tilde{\beta}_{m+1} (x_m)} - \frac{\Tilde{\mu}_{m+1} (x_{m+1})}{\Tilde{\mu}_{m+1} (x_m)} 
    &= \frac{\Tilde{\mu}_{m+1} (x_{m+1})\exp(x_{m+1})}{\Tilde{\mu}_{m+1} (x_m)\exp(x_{m})} - \frac{\Tilde{\mu}_{m+1} (x_{m+1})}{\Tilde{\mu}_{m+1} (x_m)} \\
    &= \frac{\Tilde{\mu}_{m+1} (x_{m+1})}{\Tilde{\mu}_{m+1} (x_m)} \left(\exp(x_{m+1} - x_m) - 1 \right) \geq 0.
\end{aligned} 
\label{eq:v_exist}
\end{equation}

So we can find a proper solution of assumption \eqref{eq:uvw_R1}.

So $\gamma_{m+1}$ defined above satisfies UTTC for $(\Tilde{\mu}_{m+1}, \Tilde{\beta}_{m+1})$.

By induction, for any $0 = x_0 < x_1 < \dots < x_n < 1$, there exists $\gamma$ s.t. UTTC \eqref{eq:UTTC_R1} holds for $(\Tilde{\mu}, \Tilde{\beta})$.

Then for any monotonic increasing function, since $\gamma (x_i, y_j) = 0$ when $i > j$, we know $\gamma (x_i, y_j) f(x_i) \leq \gamma (x_i, y_j) f(y_j)$.
Hence we have
\begin{equation*}
    \begin{aligned}
        \textbf{E}_{\Tilde{\mu}} [f] 
        &= \sum_i \Tilde{\mu} (x_i) f(x_i) 
        = \sum_i \sum_j \gamma (x_i, y_j) f(x_i) \\
        &\leq \sum_i \sum_j \gamma (x_i, y_j) f(y_j) \\
        &= \sum_j \sum_i \gamma (x_i, y_j) f(y_j) \\
        &= \sum_j \Tilde{\beta} (y_j) f(y_j) 
        = \textbf{E}_{\Tilde{\beta}} [f]. \\
    \end{aligned}
\end{equation*}

\end{proof}

\begin{Lemma}[Discretized Upper Triangular Transport Inequality for Co-Monotonic Functions in $\mathbf{R}^1$]
Assume $\mu$ is a continuous probability measure supported on $[0, 1]$.
Let $0 = x_0 < x_1 < \dots < x_n < 1$ to be any split of $[0, 1]$.
Let $f, g:\{x_i\}_{i=0,\dots, n} \rightarrow $ \textbf{R} to be two co-monotonic functions that satisfy
$$(f(x_i) - f(x_j)) \cdot (g(x_i) - g(x_j)) \geq 0, \ \forall \ i, j.$$
Define $\Tilde{\mu}(x_i) = \mu([x_i, x_{i+1}))$. 
Define 
$$\Tilde{\beta}(x_i) = \Tilde{\mu}(x_i) \exp(g(x_i)) / Z, \ Z = \sum_{i} \Tilde{\mu}(x_i) \exp (g(x_i)).$$
Then we have
$$
\textbf{E}_{\Tilde{\mu}} [f] \leq \textbf{E}_{\Tilde{\beta}} [f].
$$
\label{lemma:dct_R1}
\end{Lemma}

\begin{proof}[Proof of Lemma \ref{lemma:dct_R1}]

If the Upper Triangular Transport Condition (UTTC) holds for $(\Tilde{\mu}, \Tilde{\beta})$, i.e. there exists a probability measure $\gamma: \{x_i\}_{i=0,\dots, n} \times \{x_i\}_{i=0,\dots, n} \rightarrow [0, 1]$, s.t. 
\begin{equation*}
    \left\{
    \begin{aligned}
        &\sum_j \gamma (x_i, y_j) = \Tilde{\mu} (x_i),& &\ i = 0, \dots, n; \\
        &\sum_i \gamma (x_i, y_j) = \Tilde{\beta} (y_j),& &\ j = 0, \dots, n; \\
        &\gamma (x_i, y_j) = 0,& &\ g(x_i) > g(y_j), \\
    \end{aligned}
    \right.
\end{equation*}
then we finish the proof by
\begin{equation*}
    \begin{aligned}
        \textbf{E}_{\Tilde{\mu}} [f] 
        &= \sum_i \Tilde{\mu} (x_i) f(x_i) 
        = \sum_i \sum_j \gamma (x_i, y_j) f(x_i) \\
        &\leq \sum_i \sum_j \gamma (x_i, y_j) f(y_j) \\
        &= \sum_j \sum_i \gamma (x_i, y_j) f(y_j) \\
        &= \sum_j \Tilde{\beta} (y_j) f(y_j) 
        = \textbf{E}_{\Tilde{\beta}} [f], \\
    \end{aligned}
\end{equation*}
where $\gamma (x_i, y_j) f(x_i) \leq \gamma (x_i, y_j) f(y_j)$ is because of $\gamma (x_i, y_j) = 0, \ g(x_i) > g(y_j)$ and $(f(x_i) - f(x_j)) \cdot (g(x_i) - g(x_j)) \geq 0$.

Now we only need to prove UTTC holds for $(\Tilde{\mu}, \Tilde{\beta})$.

Given $0 = x_0 < x_1 < \dots < x_n < 1$, we prove the existence of $\gamma$ by induction.
With $g$ to be the transition function in the definition of $\Tilde{\beta}$, we mimic the proof of \textbf{Lemma} \ref{lemma:dct_inc_R1} and sort $(x_0, \dots, x_n)$ in the increasing order of $g$,
which is 
$$g(x_{k_0}) \leq g(x_{k_1}) \leq \dots \leq g(x_{k_n}).$$

Define 
\begin{equation*}
    \Tilde{\mu}_m(x_{k_i}) = 
    \left\{
    \begin{aligned}
        &\mu ([x_{k_i}, \min\{1, x_{k_l} |\ x_{k_l} > x_{k_i}, l \leq m \})),& &i \leq m,\  
        x_{k_i} \neq \min\{x_{k_l} |\ l \leq m \}, \\
        &\mu ([0, \min\{1, x_{k_l} |\ x_{k_l} > x_{k_i}, l \leq m \})),& &i \leq m,\  
        x_{k_i} = \min\{x_{k_l} |\ l \leq m \}, \\
        & 0,& &i > m.
    \end{aligned}
    \right.
\end{equation*}

Define 
$$
\Tilde{\beta}_m (x_{k_i}) = \Tilde{\mu}_m (x_{k_i}) \exp(g(x_{k_i})) / Z_m,
\  Z_m = \sum_{i} \Tilde{\mu}_m (x_{k_i}) \exp (g(x_{k_i})).
$$ 

To clarify the proof, we use $x_{k_i}$ to represent the point for $\Tilde{\mu}$-axis in coupling and $y_{k_j}$ to represent the point for $\Tilde{\beta}$-axis, but they are actually identical, i.e. $x_{k_i} = y_{k_j}$ when $i = j$. 

When $m = 0$, it's obvious that $(\Tilde{\mu}_0, \Tilde{\beta}_0)$ satisfies UTTC, as 
\begin{equation*}
    \gamma_0 (x_{k_i}, y_{k_j}) = 
    \left\{
    \begin{aligned}
        &1,& &\ i=0, j=0,\\
        &0,& &\ else. \\
    \end{aligned}
    \right.
\end{equation*}

Assume UTTC holds for $m$, i.e. there exists $\gamma_m$ s.t. $(\Tilde{\mu}_m, \Tilde{\beta}_m)$ satisfies UTTC, we want to prove it also holds for $m+1$.

When $x_{k_{m+1}} > \min\{x_{k_l} |\ l \leq m \}$, let $x_{k^*} = \max \{x_{k_l} | \ x_{k_l} < x_{k_{m+1}}, l \leq m \}$ to be the closest left neighbor of $x_{k_{m+1}}$ in $\{x_{k_l}|\ l \leq m \}$.
Then we have $\Tilde{\mu}_{m} (x_{k^*}) = \Tilde{\mu}_{m+1} (x_{k^*}) + \Tilde{\mu}_{m+1} (x_{k^{m+1}})$.

When $x_{k_{m+1}} < \min\{x_{k_l} |\ l \leq m \}$, let $x_{k^*} = \min\{x_{k_l} |\ l \leq m \}$ to be the leftmost point in $\{x_{k_l}|\  l \leq m \}$. 
Then we have $\Tilde{\mu}_{m} (x_{k^*}) = \Tilde{\mu}_{m+1} (x_{k^*}) + \Tilde{\mu}_{m+1} (x_{k^{m+1}})$.

In either case, we always have $\Tilde{\mu}_{m} (x_{k^*}) = \Tilde{\mu}_{m+1} (x_{k^*}) + \Tilde{\mu}_{m+1} (x_{k_{m+1}})$.
By definition of $\Tilde{\mu}_m$ and $\Tilde{\beta}_m$, we have
\begin{equation*}
    \left\{
    \begin{aligned}
        &\Tilde{\mu}_m (x_{k_i}) = \Tilde{\mu}_{m+1} (x_{k_i}),& &i \leq m,\ k_i \neq k^*, \\
        &\Tilde{\mu}_m (x_{k_i}) = \Tilde{\mu}_{m+1} (x_{k_i}) + \Tilde{\mu}_{m+1} (x_{k_{m+1}}),& &i \leq m,\ k_i = k^*, \\
        &\Tilde{\mu}_m (x_{k_{m+1}}) = \Tilde{\mu}_m (x_{k_i}) = \Tilde{\mu}_{m+1} (x_{k_i}) = 0,& &i>m+1, \\
    \end{aligned}
    \right.
\end{equation*}
\begin{equation*}
    \left\{
    \begin{aligned}
        &\Tilde{\beta}_m (x_{k_i}) = \Tilde{\beta}_{m+1} (x_{k_i}) \cdot \frac{Z_{m+1}}{Z_m},& &i \leq m,\ k_i \neq k^*, \\
        &\Tilde{\beta}_m (x_{k_i}) = \left( \Tilde{\beta}_{m+1} (x_{k_i})  + \Tilde{\beta}_{m+1} (x_{k_{m+1}}) \exp\left(g(x_{k_i}) - g(x_{k_{m+1}})\right) \right) \cdot \frac{Z_{m+1}}{Z_m},& &i \leq m,\ k_i = k^*, \\
        &\Tilde{\beta}_m (x_{m+1}) = \Tilde{\beta}_m (x_i) = \Tilde{\beta}_{m+1} (x_i) = 0,& &i>m+1. \\
    \end{aligned}
    \right.
\end{equation*}

If $g(x_{k^*}) = g(x_{k_{m+1}})$, it's easy to check that 
$\frac{\Tilde{\mu}_{m+1} (x_{k_{m+1}})}{\Tilde{\mu}_{m+1} (x_{k^*})} = \frac{\Tilde{\beta}_{m+1} (x_{k_{m+1}})}{\Tilde{\beta}_{m+1} (x_{k^*})}$, 
we can simply define the following $\gamma_{m+1}$ which achieves UTTC for $(\Tilde{\mu}_{m+1}, \Tilde{\beta}_{m+1})$:
\begin{equation*}
    \left\{
    \begin{aligned}
        &\gamma_{m+1} (x_{k^*}, y_{k_j}) = \gamma_{m} (x_{k^*}, y_{k_j}) \frac{\Tilde{\mu}_{m+1} (x_{k^*})}{\Tilde{\mu}_{m+1} (x_{k^*}) + \Tilde{\mu}_{m+1} (x_{k_{m+1}})},& &\ j \leq m,\ k_j \neq k^*, \\
        &\gamma_{m+1} (x_{k_{m+1}}, y_{k_j}) = \gamma_{m} (x_{k_{m+1}}, y_{k_j}) \frac{\Tilde{\mu}_{m+1} (x_{k_{m+1}})}{\Tilde{\mu}_{m+1} (x_{k^*}) + \Tilde{\mu}_{m+1} (x_{k_{m+1}})},& &\ j \leq m,\ k_j \neq k^*, \\
        &\gamma_{m+1} (x_{k_i}, y_{k^*}) = \gamma_{m} (x_{k_i}, y_{k^*}) \frac{\Tilde{\beta}_{m+1} (y_{k^*})}{\Tilde{\beta}_{m+1} (y_{k^*}) + \Tilde{\beta}_{m+1} (y_{k_{m+1}})},& &\ i \leq m,\ k_i \neq k^*, \\
        &\gamma_{m+1} (x_{k_{i}}, y_{k_{m+1}}) = \gamma_{m} (x_{k_{i}}, y_{k_{m+1}}) \frac{\Tilde{\beta}_{m+1} (y_{k_{m+1}})}{\Tilde{\beta}_{m+1} (y_{k^*}) + \Tilde{\mu}_{m+1} (x_{k_{m+1}})},& &\ i \leq m,\ k_i \neq k^*, \\
        &\gamma_{m+1} (x_{k^*}, y_{k^*}) = \gamma_{m} (x_{k^*}, y_{k^*}) \frac{\Tilde{\mu}_{m+1} (x_{k^*})}{\Tilde{\mu}_{m+1} (x_{k^*}) + \Tilde{\mu}_{m+1} (x_{k_{m+1}})},& & \\
        &\gamma_{m+1} (x_{k_{m+1}}, y_{k_{m+1}}) = \gamma_{m} (x_{k_{m+1}}, y_{k_{m+1}}) \frac{\Tilde{\mu}_{m+1} (x_{k_{m+1}})}{\Tilde{\mu}_{m+1} (x_{k^*}) + \Tilde{\mu}_{m+1} (x_{k_{m+1}})}, \\
        &\gamma_{m+1} (x_{k_i}, y_{k_j}) = 0,& &\ others.
    \end{aligned}
    \right.
\end{equation*}

If $g(x_{k^*}) < g(x_{k_{m+1}})$, recalling the proof of \textbf{Lemma} \ref{lemma:dct_inc_R1}, it's crucial to prove inequalities \eqref{eq:Zm_inc} and \eqref{eq:v_exist}.
Inequality \eqref{eq:Zm_inc} guarantees that $\frac{Z_m}{Z_{m+1}} < 1$, so we can shrinkage $\gamma_m$ entrywise by $\frac{Z_m}{Z_{m+1}}$ and add some proper measure at proper points.
Inequality \eqref{eq:v_exist} guarantees that $(x_m, y_m)$ can be decomposed to $(x_m, y_m)$, $(x_m, y_{m+1})$, $(x_{m+1}, y_{m+1})$.
Following the idea, we check that 
\begin{equation*}
    Z_{m+1} - Z_m = \Tilde{\mu}_{m+1} (x_{k_{m+1}}) \left( \exp (g(x_{k_{m+1}}) - g(x_{k^*})) \right) > 0,
\end{equation*}
\begin{equation*}
    \begin{aligned}
        \frac{\Tilde{\beta}_{m+1} (x_{k_{m+1}})}{\Tilde{\beta}_{m+1} (x_{k^*})} 
        - \frac{\Tilde{\mu}_{m+1} (x_{k_{m+1}})}{\Tilde{\mu}_{m+1} (x_{k^*})} 
        &= \frac{\Tilde{\mu}_{m+1} (x_{k_{m+1}}) \exp(g(x_{k_{m+1}}))}{\Tilde{\mu}_{m+1} (x_{k^*}) \exp (g(x_{k^*}))} 
        - \frac{\Tilde{\mu}_{m+1} (x_{k_{m+1}})}{\Tilde{\mu}_{m+1} (x_{k^*})}  \\
        &= \frac{\Tilde{\mu}_{m+1} (x_{k_{m+1}})}{\Tilde{\mu}_{m+1} (x_{k^*})}
        \left(\exp(g(x_{k_{m+1}}) - g(x_{k^*})) - 1 \right) > 0.
    \end{aligned}
\end{equation*}

Replacing $x_m, x_{m+1}$ in the proof of \textbf{Lemma} \ref{lemma:dct_inc_R1} by $x_{k^*}, x_{k_{m+1}}$, we can construct $\gamma_{m+1}$ all the same way as in the proof of \textbf{Lemma} \ref{lemma:dct_inc_R1}.

By induction, we prove UTTC for $(\Tilde{\mu}, \Tilde{\beta})$.
The proof is done.
\end{proof}

\begin{Theorem}[Upper Triangular Transport Inequality for Co-Monotonic Functions in $\mathbf{R}^1$]
Assume $\mu$ is a continuous probability measure supported on $[0, 1]$.
Let $f, g: [0, 1] \rightarrow $ \textbf{R} to be two co-monotonic functions that satisfy
$$(f(x) - f(y)) \cdot (g(x) - g(y)) \geq 0, \ \forall \ x, y \in [0, 1].$$
$f$ is continuous.
Define 
$$\beta(x) = \mu(x) \exp(g(x)) / Z, \ Z = \int_{[0, 1]} \mu(x) \exp (g(x)).$$
Then we have 
$$\textbf{E}_{\mu} [f] \leq \textbf{E}_{\beta} [f].$$
\label{thm:cts_R1}
\end{Theorem}

\begin{proof}[Proof of Theorem \ref{thm:cts_R1}]

For $\forall \epsilon > 0$, since $f$ is continuous, $f$ is uniformly continuous, so there exists $\delta > 0$ s.t. $|f(x) - f(y)| < \epsilon, \forall x, y \in [0, 1]$.
We can split $[0, 1]$ by $0 < x_0 < x_1 < \dots < x_n < 1$ s.t. $x_{i+1} - x_i < \delta$.
Define $\Tilde{\mu}$ and $\Tilde{\beta}$ as in \textbf{Lemma} \ref{lemma:dct_R1}.
Since $x_{i+1} - x_i < \delta$, by uniform continuity and the definition of the expectation, we have
$$
| \textbf{E}_{\mu} [f] - \textbf{E}_{\Tilde{\mu}} [f] | < \epsilon,\ 
| \textbf{E}_{\beta} [f] - \textbf{E}_{\Tilde{\beta}} [f] | < \epsilon,
$$
By \textbf{Lemma} \ref{lemma:dct_R1}, we have 
$$\textbf{E}_{\Tilde{\mu}} [f] \leq \textbf{E}_{\Tilde{\beta}} [f].$$
So we have
$$
\textbf{E}_{\mu} [f] 
< \textbf{E}_{\Tilde{\mu}} [f] + \epsilon
\leq \textbf{E}_{\Tilde{\beta}} [f] + \epsilon
< \textbf{E}_{\beta} [f] + 2\epsilon.
$$
Since $\epsilon$ is arbitrary, we prove $\textbf{E}_{\mu} [f] \leq \textbf{E}_{\beta} [f].$

\end{proof}

\begin{Lemma}[Discretized Upper Triangular Transport Inequality for Co-Monotonic Functions in $\mathbf{R}^p$]
Assume $\mu$ is a continuous probability measure supported on $[0, 1]^p$. 
Let $0 = x_0^d < x_1^d < \dots < x_n^d < 1$ to be any split of $[0, 1]$, $d = 1, \dots, p$.
Denote $\textbf{x}_{\textbf{i}} \overset{def}{=} (x_{i_1}^1, \dots, x_{i_p}^p)$.
Define $\Tilde{\mu}(\textbf{x}_{\textbf{i}}) = \mu(\prod_{d=1,\dots,p} [x_{i_d}^d, x_{i_d+1}^d))$. 
Let $f, g: \{\textbf{x}_{\textbf{i}}\}_{\textbf{i} \in \{0, \dots, n\}^p} \rightarrow $ \textbf{R} to be two co-monotonic functions that satisfy
$$(f(\textbf{x}_{\textbf{i}})
- f(\textbf{x}_{\textbf{j}})) 
\cdot (g(\textbf{x}_{\textbf{i}}) 
- g(\textbf{x}_{\textbf{j}})) \geq 0, \ \forall \ \textbf{i}, \textbf{j}.$$
Define 
$$\Tilde{\beta}(\textbf{x}_{\textbf{i}}) 
= \Tilde{\mu}(\textbf{x}_{\textbf{i}}) \exp(g(\textbf{x}_{\textbf{i}})) / Z,\  Z = \sum_{\textbf{i}} \Tilde{\mu}(\textbf{x}_{\textbf{i}}) \exp (g(\textbf{x}_{\textbf{i}})).$$ 
Then there exists a probability measure 
$\gamma: \{\textbf{x}_{\textbf{i}}\}_{\textbf{i} \in \{0,\dots,n\}^p} \times \{\textbf{x}_{\textbf{j}}\}_{\textbf{j} \in \{0,\dots,n\}^p} \rightarrow [0, 1]$, s.t. 
$$
\begin{aligned}
    \sum_{\textbf{j}} \gamma (\textbf{x}_{\textbf{i}}, \textbf{y}_{\textbf{j}}) &= \Tilde{\mu} (\textbf{x}_{\textbf{i}}), \ \forall \ \textbf{i}; \\
    \sum_{\textbf{i}} \gamma (\textbf{x}_{\textbf{i}}, \textbf{y}_{\textbf{j}}) &= \Tilde{\beta} (\textbf{y}_{\textbf{j}}), \ \forall \ \textbf{j}; \\
    \gamma (\textbf{x}_{\textbf{i}}, \textbf{y}_{\textbf{j}}) &= 0, \ g(\textbf{x}_{\textbf{i}}) > g(\textbf{y}_{\textbf{j}}). \\
\end{aligned}
$$
Then we have
$$
\textbf{E}_{\Tilde{\mu}} [f] \leq \textbf{E}_{\Tilde{\beta}} [f].
$$
\label{lemma:dct_Rp}
\end{Lemma}

\begin{proof}[Proof of Lemma \ref{lemma:dct_Rp}]

The proof is almost identical to the proof of \textbf{Lemma} \ref{lemma:dct_R1}, except for the definition of $(\Tilde{\mu}_m, \Tilde{\beta}_m)$ in $\textbf{R}^p$.


Given $\{\textbf{x}_{\textbf{i}}\}_{\textbf{i} \in \{0,\dots,n\}^p}$, we sort $\textbf{x}_{\textbf{i}}$ in the increasing order of $g$, which is 
$$
g(\textbf{x}_{\textbf{k}_0}) \leq g(\textbf{x}_{\textbf{k}_1}) \leq \dots \leq g(\textbf{x}_{\textbf{k}_{(n+1)^p - 1}}),
$$
where $\{\textbf{k}_i\}_{i \in \{0, \dots, (n+1)^p - 1\}}$ is a permutation of $\{\textbf{\textbf{i}}\}_{\textbf{i} \in \{0, \dots, n\}^p}$.

For $\textbf{i}, \textbf{j} \in \{0, \dots, n\}^p$, we define the partial order $\textbf{i} < \textbf{j}$ on $\{0, \dots, n\}^p$, if 
$$
\exists 0 \leq d_0 \leq n,\  s.t.\  \textbf{i}_d \leq \textbf{j}_d,\ \forall d < d_0 \ and\  \textbf{i}_{d_0} < \textbf{j}_{d_0}.
$$
It's obvious that 
\begin{equation*}
    \left\{
    \begin{aligned}
        &\forall \textbf{i} \in \{0, \dots, n\}^p,\  \textbf{i} \nless \textbf{i}, \\
        &\forall \textbf{i}, \textbf{j} \in \{0, \dots, n\}^p,\ \textbf{i} < \textbf{j} \Rightarrow \textbf{j} \nless \textbf{i}, \\
        &\forall \textbf{i}, \textbf{j}, \textbf{k} \in \{0, \dots, n\}^p,\ \textbf{i} < \textbf{j},\  \textbf{j} < \textbf{k} \Rightarrow \textbf{i} < \textbf{k}. \\
    \end{aligned}
    \right.
\end{equation*}
We define $\textbf{i} = \textbf{j}$ if $\textbf{i}_d = \textbf{j}_d,\, \forall\, 0 \leq d \leq n$.
So we define the partial order relation, and we can further define the $\min$ function and the $\max$ function on $\{0, \dots, n\}^p$.

Now using this partial order relation, we define 
\begin{equation*}
    \Tilde{\mu}_m (\textbf{x}_{\textbf{k}_i}) = 
    \left\{
    \begin{aligned}
        &\sum_{\textbf{k} \geq \textbf{k}_i, \textbf{k} < \min \{ \textbf{k}_l | \textbf{k}_l > \textbf{k}_i, l \leq m \} } \Tilde{\mu} (\textbf{x}_{\textbf{k}}) 
        ,& &\ i \leq m,\ \textbf{k}_i \neq \min\{ \textbf{k}_l |\ l \leq m \}, \\
        &\sum_{\textbf{k} < \min\{ \textbf{k}_l | \textbf{k}_l > \textbf{k}_i, l \leq m \} } \Tilde{\mu} (\textbf{x}_{\textbf{k}})     
        ,& &\ i \leq m,\ \textbf{k}_i = \min\{\textbf{k}_l |\ l \leq m \}, \\
        &0         ,& &\  i > m. \\
    \end{aligned}
    \right.
\end{equation*}

With this definition of $\Tilde{\mu}_m$, other parts are identical to the proof of \textbf{Lemma} \ref{lemma:dct_R1}.
The proof is done.

\end{proof}

\begin{Theorem}[Upper Triangular Transport Inequality for Co-Monotonic Functions in $\mathbf{R}^p$]
Assume $\mu$ is a continuous probability measure supported on $[0, 1]^p$. 
Denote $\textbf{x} \overset{def}{=} (x^1, \dots, x^p)$.
Let $f, g: [0, 1]^p \rightarrow $ \textbf{R} to be two co-monotonic functions that satisfy
$$(f(\textbf{x})
- f(\textbf{y})) 
\cdot (g(\textbf{x}) 
- g(\textbf{y})) \geq 0, \ \forall \ \textbf{x}, \textbf{y} \in [0, 1]^p.$$
$f$ is continuous.
Define 
$$\beta(\textbf{x}) 
= \mu(\textbf{x}) \exp(g(\textbf{x})) / Z,\  
Z = \int_{[0, 1]^p} \mu(\textbf{x}) \exp (g(\textbf{x})).$$ 
Let $f, g: [0, 1]^p \rightarrow $ \textbf{R} to be two co-monotonic functions that satisfy
$$(f(\textbf{x})
- f(\textbf{y})) 
\cdot (g(\textbf{x}) 
- g(\textbf{y})) \geq 0, \ \forall \ \textbf{x}, \textbf{y} \in [0, 1]^p.$$
Then we have
$$
\textbf{E}_{\mu} [f] \leq \textbf{E}_{\beta} [f].
$$
\label{thm:cts_Rp}
\end{Theorem}

\begin{proof}[Proof of Theorem \ref{thm:cts_Rp}]

For $\forall \epsilon > 0$, since $f$ is continuous, $f$ is uniformly continuous, so there exists $\delta > 0$ s.t. $|f(\textbf{x}) - f(\textbf{y})| < \epsilon, \forall \textbf{x}, \textbf{y} \in [0, 1]^p$.
We can split $[0, 1]$ by $0 < x_0 < x_1 < \dots < x_n < 1$ s.t. $x_{i+1} - x_i < \delta / \sqrt{p}$.
Define $x_i^d = x_i, \ \forall 0 \leq d \leq p$.
Define $\Tilde{\mu}$ and $\Tilde{\beta}$ as in \textbf{Lemma} \ref{lemma:dct_Rp}.
Since $x_{i+1} - x_i < \delta / \sqrt{p}$, 
$|(x_{i+1}^0, \dots, x_{i+1}^p) - (x_{i}^0, \dots, x_{i}^p)| < \delta$, by uniform continuity and the definition of the expectation, we have
$$
| \textbf{E}_{\mu} [f] - \textbf{E}_{\Tilde{\mu}} [f] | < \epsilon,\ 
| \textbf{E}_{\beta} [f] - \textbf{E}_{\Tilde{\beta}} [f] | < \epsilon,
$$
By \textbf{Lemma} \ref{lemma:dct_Rp}, we have 
$$\textbf{E}_{\Tilde{\mu}} [f] \leq \textbf{E}_{\Tilde{\beta}} [f].$$
So we have
$$
\textbf{E}_{\mu} [f] 
< \textbf{E}_{\Tilde{\mu}} [f] + \epsilon
\leq \textbf{E}_{\Tilde{\beta}} [f] + \epsilon
< \textbf{E}_{\beta} [f] + 2\epsilon.
$$
Since $\epsilon$ is arbitrary, we prove $\textbf{E}_{\mu} [f] \leq \textbf{E}_{\beta} [f].$

\end{proof}

\begin{Lemma}[Performance Difference Lemma]
For any policies $\pi, \pi'$ and any state $s_0$, we have
\begin{equation*}
    V^{\pi} (s_0) - V^{\pi'} (s_0) = \frac{1}{1 - \gamma} \textbf{E}_{s \sim d_{s_0}^\pi} \textbf{E}_{a \sim \pi (\cdot | s)} \left[ A^{\pi'} (s, a) \right].
\end{equation*}
\label{lemma:perfdiff}
\end{Lemma}

\begin{proof}
    See \citep{kakade2002approximately}.
\end{proof}

\clearpage

\section{Algorithm Pseudocode}
\label{App: Algorithm Pseudocode}

\subsection{GDI-I$^3$}
In this section, we provide the implementation pseudocode of GDI-I$^3$, which is shown in \textbf{Algorithm} \ref{alg:i3}.

\begin{equation}
\label{Equ: i3 casa equ}
    \left\{
    \begin{aligned}
        &A=A_{\theta}\left(s_{t}\right),& 
        &V=V_{\theta}\left(s_{t}\right), \\
        &\bar{A}=A-E_{\pi}[A],& 
        &Q=\bar{A}+V. \\
    \end{aligned}
    \right.
\end{equation}

\begin{equation}
\label{Equ: i3 soft entropy}
    \lambda = (\tau_1, \tau_2, \epsilon), \ 
    \pi_{\theta_{\lambda}}=\epsilon \cdot \underbrace{\operatorname{Softmax}\left(\frac{A}{\tau_{1}}\right)}_{Exploration}+(1-\epsilon) \cdot \underbrace{\operatorname{Softmax}\left(\frac{A}{\tau_{2}}\right)}_{Exploitation}
\end{equation}

\begin{figure}[ht]
  \centering
  \begin{minipage}{.7\linewidth}
    \begin{algorithm}[H]
      \caption{GDI-I$^3$ Algorithm.}  
          \begin{algorithmic}
            \STATE Initialize Parameter Server (PS) and Data Collector (DC).
            \STATE
            \STATE // LEARNER
            \STATE Initialize $d_{push}$.
            \STATE Initialize $\theta$ as Eq. \eqref{Equ: i3 casa equ} and \eqref{Equ: i3 soft entropy}.
            \STATE Initialize $count = 0$.
            \WHILE{$True$}
                \STATE Load data from DC.
                \STATE Estimate $qs$ and $vs$ by proper off-policy algorithms.
                \STATE \ \ \ \ (For instance, ReTrace \eqref{Equ: retrace} for $qs$ and V-Trace \eqref{Equ: vtrace} for $vs$.)
                \STATE Update $\theta$ via policy gradient and policy evaluation.
                \IF{$count$ mod $d_{push}$ = 0}
                    \STATE Push $\theta$ to PS.
                \ENDIF
                \STATE $count \leftarrow count + 1$.
            \ENDWHILE
            \STATE
            \STATE // ACTOR
            \STATE Initialize $d_{pull}$, $M$.
            \STATE Initialize $\theta$ as Eq. \eqref{Equ: i3 casa equ} and \eqref{Equ: i3 soft entropy}.
            \STATE Initialize $\{\mathcal{B}_m\}_{m=1,...,M}$ and sample $\lambda$ as in \textbf{Algorithm} \ref{alg:bva}.
            \STATE Initialize $count = 0$, $G = 0$.
            \WHILE{$True$}
                \STATE Calculate $\pi_{\theta_{\lambda}}(\cdot | s)$.
                \STATE Sample $a \sim \pi_{\theta_{\lambda}}(\cdot | s)$.
                \STATE $s, r, done \sim p(\cdot | s, a)$.
                \STATE $G \leftarrow G + r$.
                \IF{$done$}
                    \STATE Update $\{\mathcal{B}_m\}_{m=1,...,M}$ with $(\lambda, G)$ as in \textbf{Algorithm} \ref{alg:bva}.
                    \STATE Send data to DC and reset the environment.
                    \STATE $G \leftarrow 0$.
                    \STATE Sample $\lambda$ as in \textbf{Algorithm} \ref{alg:bva}
                \ENDIF
                \IF{$count \mod d_{pull}$ = 0}
                    \STATE Pull $\theta$ from PS and update $\theta$.
                \ENDIF
                \STATE $count \leftarrow count + 1$.
            \ENDWHILE
          \end{algorithmic}
        \label{alg:i3}
    \end{algorithm}
  \end{minipage}
\end{figure}

\clearpage
\subsection{GDI-H$^3$}
In this section, we provide the implementation pseudocode of GDI-H$^3$, which is shown in \textbf{Algorithm} \ref{alg:h3}.

\begin{equation}
\label{Equ: h3 casa equ}
    \begin{aligned}
    \left\{
    \begin{aligned}
        &A_{\theta_1}=A_{\theta_1}\left(s_{t}\right),& 
        &V_{\theta_1}=V_{\theta_1}\left(s_{t}\right), \\
        &\bar{A}_{\theta_1}=A_{\theta_1}-E_{\pi}[A_{\theta_1}],& 
        &Q_{\theta_1}=\bar{A}_{\theta_1}+V_{\theta_1}. \\
    \end{aligned}
    \right.\\
    \left\{
    \begin{aligned}
        &A_{\theta_2}=A_{\theta_2}\left(s_{t}\right),& 
        &V_{\theta_2}=V_{\theta_2}\left(s_{t}\right), \\
        &\bar{A}_{\theta_2}=A_{\theta_2}-E_{\pi}[A_{\theta_2}],& 
        &Q_{\theta_2}=\bar{A}_{\theta_2}+V_{\theta_2}. \\
    \end{aligned}
    \right.
    \end{aligned}
\end{equation}

\begin{equation}
\label{Equ: h3 soft entropy}
    \lambda = (\tau_1, \tau_2, \epsilon), \ 
    \pi_{\theta_{\lambda}}=\epsilon \cdot \operatorname{Softmax}\left(\frac{A_{\theta_1}}{\tau_{1}}\right)+(1-\epsilon) \cdot \operatorname{Softmax}\left(\frac{A_{\theta_2}}{\tau_{2}}\right)
\end{equation}

\begin{figure}[ht]
  \centering
  \begin{minipage}{.7\linewidth}
    \begin{algorithm}[H]
      \caption{GDI-H$^3$ Algorithm.}  
          \begin{algorithmic}
            \STATE Initialize Parameter Server (PS) and Data Collector (DC).
            \STATE
            \STATE // LEARNER
            \STATE Initialize $d_{push}$.
            \STATE Initialize $\theta$  as Eq. \eqref{Equ: h3 casa equ} and \eqref{Equ: h3 soft entropy}.
            \STATE Initialize $count = 0$.
            \WHILE{$True$}
                \STATE Load data from DC.
                \STATE Estimate $qs_1, qs_2$ and $vs_1, vs_2$ by proper off-policy algorithms.
                \STATE \ \ \ \ (For instance, ReTrace \eqref{Equ: retrace} for $qs1, qs_2$ and V-Trace \eqref{Equ: vtrace} for $vs_1, vs_2$.)
                \STATE Update $\theta_1, \theta_2$ via policy gradient and policy evaluation, respectively.
                \IF{$count$ mod $d_{push}$ = 0}
                    \STATE Push $\theta_1, \theta_2$ to PS.
                \ENDIF
                \STATE $count \leftarrow count + 1$.
            \ENDWHILE
            \STATE
            \STATE // ACTOR
            \STATE Initialize $d_{pull}$, $M$.
            \STATE Initialize $\theta_1, \theta_2$ as Eq. \eqref{Equ: h3 casa equ} and \eqref{Equ: h3 soft entropy}.
            \STATE Initialize $\{\mathcal{B}_m\}_{m=1,...,M}$ and sample $\lambda$ as in \textbf{Algorithm} \ref{alg:bva}.
            \STATE Initialize $count = 0$, $G = 0$.
            \WHILE{$True$}
                \STATE Calculate $\pi_{\theta_{\lambda}}(\cdot | s)$.
                \STATE Sample $a \sim \pi_{\theta_{\lambda}}(\cdot | s)$.
                \STATE $s, r, done \sim p(\cdot | s, a)$.
                \STATE $G \leftarrow G + r$.
                \IF{$done$}
                    \STATE Update $\{\mathcal{B}_m\}_{m=1,...,M}$ with $(\lambda, G)$ as in \textbf{Algorithm} \ref{alg:bva}.
                    \STATE Send data to DC and reset the environment.
                    \STATE $G \leftarrow 0$.
                    \STATE Sample $\lambda$ as in \textbf{Algorithm} \ref{alg:bva}
                \ENDIF
                \IF{$count \mod d_{pull}$ = 0}
                    \STATE Pull $\theta$ from PS and update $\theta$.
                \ENDIF
                \STATE $count \leftarrow count + 1$.
            \ENDWHILE
          \end{algorithmic}
        \label{alg:h3}
    \end{algorithm}
  \end{minipage}
\end{figure}

\clearpage

\section{Adaptive Controller Formalism}
\label{Sec: appendix MAB}

In practice, we use a Bandits Controller (BC) to control the behavior sampling distribution adaptively, which has been widely used in prior works \citep{agent57,casa_entropy}. More details on Bandits  can see \citep{sutton}. The whole algorithm is shown in \textbf{Algorithm} \ref{alg:bva}. As the behavior policy can be parameterized and thereby sampling behaviors from the policy space is equivalent to sampling indexes $x$ from the index set. 

Let's firstly define a bandit as $B = Bandit(mode, l, r, lr, d, acc, ta, to, \textbf{w}, \textbf{N})$.
\begin{itemize}
    \item $mode$ is the mode of sampling, with two choices, $argmax$ and $random$, wherein $argmax$ greedily chooses the behaviors with top estimated value from the policy space, and $random$ samples behaviors according to a distribution calculated by $Softmax(V)$.
    \item $l$ is the left boundary of the index set, and each $x$ is clipped to $x = \max \{x, l\}$.
    \item $r$ is the right boundary of the index set, and each $x$ is clipped to $x = \min \{x, r\}$.
    \item $acc$ is the accuracy of space to be optimized, where each $x$ is located in the $\lfloor (\min\{\max\{x, l\}, r\} - l) / acc \rfloor$th block.
    \item tile coding is a representation method of continuous space \citep{sutton}, and each kind of tile coding can be uniquely determined by $l$, $r$, $to$ and $ta$, wherein $to$ represents the tile offset and $ta$ represents the accuracy of the tile coding.
    \item $to$ is the offset of each tile coding, which represents the relative offset of the basic coordinate system (normally we select the space to be optimized as basic coordinate system).
    \item $ta$ is the accuracy of each tile coding, where each $x$ is located in the $\lfloor (\min\{\max\{x-to, l\}, r\} - l) / ta \rfloor$th tile.
    \item $M_{btt}$ represents block-to-tile, which is a mapping from the block of the original space to the tile coding space.
    \item $M_{ttb}$ represents tile-to-block, which is a mapping from the tile coding space to the block of the original space.
    \item $\textbf{w}$ is a vector in $\mathbf{R}^{\lfloor (r-l) / ta \rfloor}$, which represents the weight of each tile.
    \item $\textbf{N}$ is a vector in $\mathbf{R}^{\lfloor (r-l) / ta \rfloor}$, which counts the number of sampling of each tile.
    \item $lr$ is the learning rate.
    \item $d$ is an integer, which represents how many candidates is provided by each bandit when sampling.
\end{itemize}

During the evaluation process, we evaluate the value of the $i$th tile by
\begin{equation}
\label{eq:bandit_eval}
V_i = \frac{\sum_{k}^{M_{btt}(block_i)} \textbf{w}_k}{len(M_{btt}(block_i))}
\end{equation}

During the training process, for each sample $(x, g)$, where $g$ is the target value. Since $x$ locates in the $j$th tile of $k$th tile\_coding, we update $B$ by
\begin{equation}
\label{eq:bandit_update}
\left\{
\begin{aligned}
&j = \lfloor (\min\{\max\{x-to_{k}, l\}, r\} - l) / ta_{k} \rfloor, \\
&\textbf{w}_j 
\leftarrow \textbf{w}_j + lr * \left(g - V_i\right)\\
& \textbf{N}_j \leftarrow \textbf{N}_j + 1
\end{aligned}
\right.
\end{equation}

During the sampling process, we firstly evaluate $\mathcal{B}$ by \eqref{eq:bandit_eval} and get $(V_1, ..., V_{\lfloor (r-l) / acc \rfloor})$.
We calculate the score of $i$th tile by
\begin{equation}
\label{eq:bandit_score}
score_i = \frac{V_i - \mu(\{V_j\}_{j=1,...,\lfloor(r-l)/acc\rfloor})}{\sigma(\{V_j\}_{j=1,...,\lfloor(r-l)/acc\rfloor})} + c \cdot \sqrt{\frac{\log (1 + \sum_j \textbf{N}_j)}{1 + \textbf{N}_i}}.
\end{equation}
For different $mode$s, we sample the candidates by the following mechanism,
\begin{itemize}
    \item if $mode$ = $argmax$, find blocks with top-$d$ $score$s, then sample $d$ candidates from these blocks, one uniformly from a block;
    \item if $mode$ = $random$, sample $d$ blocks with $score$s as the logits without replacement, then sample $d$ candidates from these blocks, one uniformly from a block;
\end{itemize}

In practice, we define a set of bandits $\mathcal{B}_m = \{B_m\}_{m=1,...,M}$.
At each step, we sample $d$ candidates $\{c_{m, i}\}_{i=1,...,d}$ from each $B_m$, so we have a set of $m \times d$ candidates $\{c_{m, i}\}_{m=1,...,M; i=1,...,d}$.
Then we sample uniformly from these $m \times d$ candidates to get $x$. 
At last, we transform the selected $x$ to $\alpha=\{\tau_1,\tau_2,\epsilon\}$ by $\tau_{1,2} = \frac{1}{\exp (x_{1,2}) - 1}$ and $\epsilon = x_{3}$
When we receive $(\alpha, g)$, we transform $\alpha$ to $x$ by $x_{1,2} = \log (1 + 1 / \tau_{1,2})$, and  $x_{3} = \epsilon$.
Then we update each $B_m$ by \eqref{eq:bandit_update}.

\begin{figure}[ht]
  \centering
  \begin{minipage}{.9\linewidth}
    \begin{algorithm}[H]
      \caption{Bandits Controller}  
          \begin{algorithmic}
            \FOR{$m=1,...,M$}
                \STATE Sample $mode \sim \{argmax, random\}$ and other initialization parameters
                \STATE Initialize $B_m = Bandit(mode, l, r, lr, d, acc, to, ta, \textbf{w}, \textbf{N})$
                \STATE Ensemble $B_m$ to constitute $\mathcal{B}_m$ 
            \ENDFOR
            \WHILE{$True$}  
                \FOR{$m=1,...,M$}
                    \STATE Evaluate $\mathcal{B}_m$ by \eqref{eq:bandit_eval}.
                    \STATE Sample candidates $c_{m, 1}, ..., c_{m, d}$  from $\mathcal{B}_m$ via \eqref{eq:bandit_score} following its $mode$.
                \ENDFOR
                \STATE Sample $x$ from $\{c_{m, i}\}_{m=1,...,M; i=1,...,d}$.
                \STATE Execute $x$ and receive the return $G$.
                \FOR{$m=1,...,M$}
                    \STATE Update $\mathcal{B}_m$ with $(x, G)$ by \eqref{eq:bandit_update}.
                \ENDFOR
            \ENDWHILE
          \end{algorithmic}  
        \label{alg:bva}
    \end{algorithm}
  \end{minipage}
\end{figure}

\clearpage

\section{Experiment Details}
\label{sec:app Experiment Details}

The overall training architecture is on the top of the Learner-Actor  framework \citep{impala}, which supports large-scale training. Additionally, the recurrent encoder with LSTM \citep{lstm} is used to handle the partially observable MDP problem \citep{ale}. 
The burn-in technique is adopted to deal with the representational drift \citep{r2d2}, and we train each sample twice.
A complete description of the hyperparameters  can see App. \ref{Sec: appendix hyperparameters}. 
We employ additional environments to evaluate the scores during training, and the undiscounted episode returns averaged over 32 environments with different seeds have been recorded. 
Details on relevant evaluation criteria  can see App. \ref{app:Evaluation Metrics for ALE}.
 
We evaluated all agents on 57 Atari 2600 games from the arcade learning environment  \citep[ALE]{ale} by recording the average score of the population of agents during training. We have demonstrated our evaluation metrics for ALE in App. \ref{app:Evaluation Metrics for ALE}, and we will describe more details in the following. Besides, all the experiment is accomplished using a single CPU with 92 cores and a single Tesla-V100-SXM2-32GB GPU.

Noting that episodes will be truncated at 100K frames (or 30 minutes of simulated play) as other baseline algorithms \citep{rainbow,agent57,laser,ngu,r2d2} and thereby we calculate the mean playtime over 57 games which is called Playtime. In addition to comparing the mean and median human normalized scores (HNS), we also report the performance based on human world records among these algorithms and the related learning efficiency to further highlight the significance of our algorithm. Inspired by \citep{atarihuman}, human world records normalized score (HWRNS) and SABER are  better descriptors for evaluating algorithms on human top level on Atari games, which simultaneously give rise to more challenges and lead the related research into a new journey to train the superhuman agent instead of  just paying attention to  the human average level.

\clearpage
\section{Hyperparameters}
\label{Sec: appendix hyperparameters}

In this section, we firstly detail the hyperparameters we use to pre-process the environment frames received from the Arcade Learning Environment. The hyperparameters that we used in all experiments are almost the same as Agent57  \citep{agent57}, NGU \citep{ngu}, MuZero \citep{muzero} and R2D2 \citep{r2d2}.
In Tab. \ref{tab:ale_process}, we detail these pre-processing hyperparameters. Then we will detail the hyperparameters we used for Atari experiments, which is demonstrated in Tab. \ref{tab:fixed_model_hyperparameters_atari}. 

\begin{table}[H]
\begin{center}
\caption{Atari pre-processing hyperparameters.}
\label{tab:ale_process}
\begin{tabular}{|c|c|}
\hline
\textbf{Hyperparameter} & \textbf{Value}  \\
\hline
Random modes and difficulties & No \\
\hline
Sticky action probability  & 0.0 \\
\hline
Life information & Not allowed \\
\hline
Image Size & (84, 84) \\
\hline
Num. Action Repeats & 4 \\
\hline
Num. Frame Stacks & 4 \\
\hline
Action Space & Full \\
\hline
Max episode length   & 100000 \\
\hline
Random noops range  & 30\\
\hline
Grayscaled/RGB      & Grayscaled\\
\hline
\end{tabular}

\end{center}
\end{table}
\clearpage

\begin{table}[H]
\begin{center}
\caption{Hyperparameters for Atari experiments.}
\label{tab:fixed_model_hyperparameters_atari}
\begin{tabular}{|c|c|}
\hline
\textbf{Parameter} & \textbf{Value}  \\
\hline
Num. Frames & 200M (2E+8) \\
\hline
Replay & 2 \\
\hline
Num. Environments & 160 \\
\hline
GDI-I$^3$ Reward Shape & $\log (abs (r) + 1.0) \cdot (2 \cdot 1_{\{r \geq 0\}} - 1_{\{r < 0\}})$ \\
\hline
GDI-H$^3$ Reward Shape 1 & $\log (abs (r) + 1.0) \cdot (2 \cdot 1_{\{r \geq 0\}} - 1_{\{r < 0\}})$ \\
\hline
GDI-H$^3$ Reward Shape 2 & $sign(r) \cdot ((abs (r) + 1.0)^{0.25} - 1.0) + 0.001 \cdot r$ \\
\hline
Reward Clip & No \\
\hline
Intrinsic Reward & No \\
\hline
Entropy Regularization & No \\
\hline
Burn-in & 40 \\
\hline
Seq-length & 80 \\
\hline
Burn-in Stored Recurrent State & Yes \\
\hline
Bootstrap & Yes \\
\hline
Batch size & 64 \\
\hline
Discount ($\gamma$) & 0.997 \\
\hline
$V$-loss Scaling ($\xi$) & 1.0 \\
\hline
$Q$-loss Scaling ($\alpha$) & 10.0 \\
\hline
$\pi$-loss Scaling ($\beta$) & 10.0 \\
\hline
Importance Sampling Clip $\Bar{c}$ & 1.05 \\
\hline
Importance Sampling Clip $\Bar{\rho}$ & 1.05 \\
\hline
Backbone & IMPALA,deep \\
\hline
LSTM Units & 256 \\
\hline
Optimizer & Adam Weight Decay \\
\hline
Weight Decay Rate & 0.01 \\
\hline
Weight Decay Schedule & Anneal linearly to 0 \\
\hline
Learning Rate & 5e-4 \\
\hline
Warmup Steps & 4000 \\
\hline
Learning Rate Schedule & Anneal linearly to 0 \\
\hline
AdamW $\beta_1$ & 0.9 \\
\hline
AdamW $\beta_2$ & 0.98 \\
\hline
AdamW $\epsilon$ & 1e-6 \\
\hline
AdamW Clip Norm & 50.0 \\
\hline
Auxiliary Forward Dynamic Task & Yes \\
\hline
Auxiliary Inverse Dynamic Task & Yes \\
\hline
Learner Push Model Every $N$ Steps & 25 \\
\hline
Actor Pull Model Every $N$ Steps & 64 \\
\hline
Num. Bandits & 7 \\
\hline
Bandit Learning Rate & Uniform([0.05, 0.1, 0.2]) \\
\hline
Bandit Tiling Width & Uniform([2, 3, 4]) \\
\hline
Num. Bandit Candidates & 3 \\
\hline
Offset of Tile coding & Uniform([0, 60]) \\
\hline
Accuracy of Tile coding & Uniform([2, 3, 4]) \\
\hline
Accuracy of Search Range for [$1/\tau_1$,$1/\tau_2$,$\epsilon$]& [1.0, 1.0, 0.1]\\
\hline
Fixed Selection for [$1/\tau_1$,$1/\tau_2$,$\epsilon$]&[1.0,0.0,1.0]\\
\hline
Bandit Search Range for $1/\tau_1$ & [0.0, 50.0] \\
\hline
Bandit Search Range for $1/\tau_2$ & [0.0, 50.0] \\
\hline
Bandit Search Range for $\epsilon$ & [0.0, 1.0] \\
\hline
\end{tabular}
\end{center}
\end{table}
\clearpage

\section{Evaluation Metrics for ALE}
\label{app:Evaluation Metrics for ALE}
In this section, we will mainly introduce the evaluation metrics in ALE, including those that have been commonly used by previous works like the raw score and the normalized score over all the Atari games based on human average score baseline, and some novel evaluation criteria for the superhuman Atari benchmark such as the normalized score based on human world records, learning efficiency, and human world record breakthrough. For the summary of benchmark results on these evaluation metrics can see App. \ref{app: Summary of Benchmark Results}. For more details on these evaluation metrics, we refer to see \citep{review2021_atari}.

\subsection{Raw Score}

Raw score refers to using tables (e.g., Table of Scores) or figures (e.g., Training Curve) to show the total scores of RL algorithms on all Atari games, which can be calculated by the sum of the undiscounted reward of the \textbf{g}th game of Atari using algorithm \textbf{i} as follows:

\begin{equation}
\label{equ: raw score}
G_{g,i} =  \textbf{E}_{s_t \sim d_{\rho_0}^{\pi}} \textbf{E}_{\pi} \left[\sum_{k=0}^{\infty} r_{t+k} | s_t \right], g \in [1,57]
\end{equation}

As \citet{ale} firstly put it, raw score over the whole 57 Atari games can reflect the performance and generality of RL agents to a certain extent. However, this evaluation metric has many limitations:

\begin{enumerate}
    \item It is difficult to compare the performance of the two algorithms directly.
    \item Its value is easily affected by the score scale. For example, the score scale of Pong is [-21,21], but that of Chopper Command is [0,999900], so the Chopper Command will dominate the mean score of those games.
\end{enumerate}

In recent RL advances, this metric is used to avoid any issues that aggregated metrics may have \citep{agent57}. Furthermore, this paper used these metrics to prove whether the RL agents have surpassed the human world records, which will be introduced in detail later.

\subsection{Normalized Scores}
To handle the drawbacks of the raw score, some methods \citep{ale,dqn} proposed the normalized score. The normalized score of the \textbf{g}th game of Atari using algorithm \textbf{i} can be calculated as follows:
\begin{equation}
\label{equ: Normalized Score}
    Z_{g,i} = \frac{G_{g,i}-G_{g,base}}{G_{g,reference}-G_{g,base}}
\end{equation}
As \citet{ale} put it, we can compare games with different scoring scales by  normalizing scores, which makes the numerical values become comparable. In practice, we can make $G_{g,base} = r_{g,min}$ and $G_{g,reference} = r_{g,max}$, where $[r_{g,min},r_{g,max}]$ is the score scale of the \textbf{g}th game. Then Equ. \eqref{equ: Normalized Score} becomes $Z_{g,i} = \frac{G_{g,i}-r_{g,min}}{r_{i,max}-r_{g,min}}$, which is a \textbf{Min-Max Scaling} and thereby $Z_{g,i} \in [0,1]$ become comparable across the 57 games. It seems this metric can be served to compare the performance between two different algorithms. However, the Min-Max normalized score fail to intuitively reflect the gap between the algorithm and the average level of humans. Thus, we need a human baseline normalized score.

\subsubsection{Human Average Score Baseline}
As we mentioned above, recent reinforcement learning advances \citep{agent57,ngu,r2d2,goexplore,muzero,muesli,rainbow} are seeking agents that can achieve superhuman performance. Thus, we need a metric to intuitively reflect the level of the algorithms compared to human performance. Since being proposed by \citep{ale}, the Human Normalized Score (HNS) is widely used in the RL research\citep{ale2}. HNS can be calculated as follows:

\begin{equation}
     \text{HNS}_{g,i} = \frac{G_{g,i}-G_{g,\text{random}}}{G_{g,\text{human average}}-G_{g,\text{random}}}
\end{equation}
wherein g denotes the gth game of Atari, i represents the algorithm i, $G_{g,\text{human average}}$ represents the human average score baseline \citep{atarihuman}, and $G_{g,\text{random}}$ represents the performance of a random policy. Adopting HNS as an evaluation metric has the following advantages:

\begin{enumerate}
    \item \textbf{Intuitive comparison with human performance.} $\text{HNS}_{g,i} \geq 100\%$ means algorithm $i$ have surpassed the human average performance in game g. Therefore, we can directly use HNS to reflect which games the RL agents have surpassed the average human performance. 
    \item \textbf{Performance across algorithms become comparable.} Like Max-Min Scaling, the human normalized score can also make two different algorithms comparable. The value of $\text{HNS}_{g,i}$ represents the degree to which algorithm i surpasses the average level of humans in game g.
\end{enumerate}

\textbf{Mean HNS} represents the mean performance of the algorithms across the 57 Atari games  based on the human average score. However, it is susceptible to interference from individual high-scoring games like the hard-exploration problems in Atari \citep{goexplore}. While taking it as the only evaluation metric, Go-Explore\citep{goexplore} has achieved SOTA compared to Agent57\citep{agent57}, NGU\citep{ngu}, R2D2\citep{r2d2}. However, Go-Explore fails to handle many other games in Atari like Demon Attack, Breakout, Boxing, Phoenix. Additionally, Go-Explore fails to balance the trade-off between exploration and exploitation, which makes it suffer from the low sample efficiency problem, which will be discussed later.

\textbf{Median HNS} represents the median performance of the algorithms across the 57 Atari games based on the human average score. Some methods \citep{muzero,muesli} have adopted it as a more reasonable metric for comparing performance between different algorithms. The median HNS has overcome the interference from individual high-scoring games. However, As far as we can see, there are at least two problems while only referring to it as the evaluation metrics. First of all, the median HNS only represents the mediocre performance of an algorithm. How about the top performance? One algorithm \citep{muesli} can easily achieve high median HNS, but at the same time obtain a poor mean HNS by adjusting the hyperparameters of algorithms for games near the median score. It shows that these metrics can show the generality of the algorithms but fail to reflect the algorithm's potential. Moreover, adopting these metrics will urge us to pursue rather mediocre methods.

In practice, we often use \textbf{mean HNS} or \textbf{median HNS} to show the final performance or generality of an algorithm. Dispute upon whether the mean value or the median value is more representative to show the generality and performance of the algorithms lasts for several years \citep{dqn,rainbow,dreamerv2,muesli,ale,ale2}. To avoid any issues that aggregated metrics may have, \textbf{we advocate calculating both of them in the final results} because they serve different purposes, and we could not evaluate any algorithm via a single one of them.

\subsubsection{Capped Normalized Score}
Capped Normalized Score is also widely used in many reinforcement learning advances \citep{atarihuman,agent57}. Among them, Agent57 \citep{agent57} adopts the capped human normalized score (CHNS) as a better descriptor for evaluating general performance, which can be calculated as $\mathrm{CHNS}=\max \{\min \{\mathrm{HNS}, 1\}, 0\}$. Agent57 claimed CHNS emphasizes the games that are below the average human performance benchmark and used CHNS to judge whether an algorithm has surpassed the human performance via $\mathrm{CHNS} \geq 100\%$.  The mean/median CHNS represents the mean/median completeness of surpassing human performance. However, there are several problems while adopting 
these metrics:
\begin{enumerate}
    \item CHNS fails to reflect the real performance in specific games. For example, $\mathrm{CHNS} \geq 100\%$ represents the algorithms surpassed the human performance but failed to reveal how good the algorithm is in this game. From the view of CHNS, Agent57 \citep{agent57} has achieved SOTA performance across 57 Atari games, but while referring to the mean HNS or median HNS, Agent57 lost to MuZero.
    \item It is still controversial that using $\mathrm{CHNS} \geq 100\%$ to represent the superhuman performance because it underestimates the human performance \citep{atarihuman}.
    \item CHNS ignores the low sample efficiency problem as other metrics using normalized scores.
\end{enumerate}

In practice, CHNS can serve as an indicator to reflect whether RL agents can surpass the average human performance. The mean/median CHNS can be used to reflect the generality of the algorithms.

\subsubsection{Human World Records Baseline}
As \citep{atarihuman} put it, the Human Average Score Baseline potentially underestimates human performance relative to what is possible. To better reflect the performance of the algorithm compared to the human world record, we introduced a complete human world record baseline extended from \citep{dreamerv2,atarihuman} to normalize the raw score, which is called the Human World Records Normalized Score (HWRNS), which can be calculated as follows:

\begin{equation}
     \text{HWRNS}_{g,i} = \frac{G_{g,i}-G_{g,\text{random}}}{G_{g,\text{human world records}}-G_{g,\text{random}}}
\end{equation}
wherein g denotes the gth game of Atari, i represents the RL algorithm, $G_{i,human}$ represents the human world records, and $G_{g,random}$ represents means the performance of a random policy. Adopting HWRNS as an evaluation metric of algorithm performance has the following advantages:

\begin{enumerate}
    \item \textbf{Intuitive comparison with human world records.} As $\text{HNS}_{g,i} \geq 100\%$ means algorithm $i$ have surpassed the human world records performance in game g. We can directly use HWRNS to reflect which games the RL agents have surpassed the human world records, which can be used to calculate the human world records breakthrough in Atari benchmarks.
    \item \textbf{Performance across algorithms become comparable.} Like the Max-Min Scaling, the HWRNS can also make two different algorithms comparable. The value of $\text{HWRNS}_{g,i}$ represents the degree to which algorithm i has surpassed the human world records in game g.
\end{enumerate}

\textbf{Mean HWRNS} represents the mean performance of the algorithms across the 57 Atari games. Compared to mean HNS, mean HWRNS put forward higher requirements on the algorithm. Poor performance algorithms like SimPLe \citep{modelbasedatari} will can be directly distinguished from other algorithms. It requires the algorithms to pursue a better performance across all the games rather than concentrate on one or two of them because breaking through any human world record is a huge milestone, which puts forward significant challenges to the performance and generality of the algorithm. For example, current model-free SOTA algorithms on HNS is Agent57 \citep{agent57}, which only acquires 125.92\% mean HWRNS, while GDI-H$^3$ obtained 154.27\% mean HWRNS and thereby became the new state-of-the-art.

\textbf{Median HWRNS} represents the median performance of the algorithms across the 57 Atari games. Compared to Median HNS, median HWRNS also puts forward higher requirements for the algorithm. For example, current SOTA RL algorithms like Muzero \citep{muzero} obtain much higher median HNS over GDI-H$^3$ but relatively lower median HWRNS.

\textbf{Capped HWRNS} Capped HWRNS (also called SABER) is firstly proposed and used by \citep{atarihuman}, which is calculated by $\mathrm{SABER}=\max \{\min \{\mathrm{HWRNS}, 2\}, 0\}$. SABER also has the same problems as CHNS, and we will not repeat them here. For more details on SABER, can see \citep{atarihuman}.

\subsection{Learning Efficiency}
As we mentioned above, traditional SOTA algorithms typically ignore the low learning efficiency problem, which makes the data used for training continuously increasing (e.g., from 10B \citep{r2d2} to 100B \citep{agent57}). Increasing the training volume hinders the application of reinforcement learning algorithms into the real world. In this paper, we advocate not to improve the final performance via improving the learning efficiency instead of increasing the training volume. We advocate achieving SOTA within 200M training frames for Atari. To evaluate the learning efficiency of an algorithm, we introduce three promising metrics.

\subsubsection{Training Scale}
As one of the commonly used metrics to reveal the learning efficiency for machine learning algorithms, training scale can also serve the purpose in RL problems. In ALE, the training scale means the scale of video frames used for training. Training frames for world modeling or planning via real-world models also need to be counted in model-based settings.

\subsubsection{Playtime}
Playtime is a unique metric of Atari, which means the equivalent real-time gameplay \citep{ale2}. We can use the following formula to calculate this metric:
\begin{equation}
    \text{Playtime (day)} = \frac{\text{Num.Frames}}{\text{108000*2*24}}
\end{equation}
For example, 200M training frames equal to 38.5 days real-time gameplay, and 100B training frames equal to 19250 days (52.7 years) real-time gameplay \citep{agent57}. As far as we know, no Atari human world record was achieved by playing a game continuously for more than 52.7 years because it is less than 52.7 years since the birth of the Atari games.

\subsubsection{Learning Efficiency}
As we mentioned several times while discussing the drawbacks of the normalized score, learning efficiency has been ignored in massive SOTA algorithms. Many SOTA algorithms achieved SOTA through training with vast amounts of data, which may equal 52.7 years continuously playing for a human. In this paper, we argue it is unreasonable to rely on the increase of data to improve the algorithm's performance. Thus, we proposed the following metric to evaluate the learning efficiency of an algorithm: 
\begin{equation}
    \text{Learning Efficiency}=\frac{\text{Related Evaluation Metric}}{\text{Num.Frames}}
\end{equation}
For example, the learning efficiency of an algorithm over means HNS is $\frac{\text{mean HNS}}{\text{Num.Frames}}$, which means the algorithms obtaining higher mean HNS via lower training frames are better than those acquiring more training data methods.
\subsection{Human World Record Breakthrough}
As we mentioned above, we need higher requirements to prove RL agents achieve real superhuman performance. Therefore, like the CHNS \citep{agent57}, the Human World Record Breakthrough (HWRB) can serve as the metric to reveal whether the algorithm has achieved the real superhuman performance, which can be calculated by $\text{HWRB} =\sum_{i=1}^{57}(\text{HWRNS} \geq 1) $.

\clearpage
\section{Atari Benchmark}
\label{app: atari benchmark}

In this section, we introduce some SOTA algorithms in the Atari Benchmarks. For more details on evaluation metrics for ALE, can see App. \ref{app:Evaluation Metrics for ALE}. For summary of the benchmark results on those evaluation metrics can see App. \ref{app: Summary of Benchmark Results}.

\subsection{Model-Free Reinforcement Learning}
\subsubsection{Rainbow}
Rainbow \citep{rainbow} is a classic value-based RL algorithm among the  DQN algorithm family, which has fruitfully combined six extensions of the DQN algorithm family. It is recognized to achieve state-of-the-art performance on the ALE benchmark. Thus, we select it as one of the representative algorithms of the SOTA DQN algorithms.
\subsubsection{IMPALA}
IMPALA, namely the Importance Weighted Actor Learner Architecture \citep{impala}, is a classic distributed off-policy actor-critic framework, which decouples acting from learning and learning from experience trajectories using V-trace. IMPALA actors communicate trajectories of experience (sequences of states, actions, and rewards) to a centralized learner, which boosts distributed large-scale training. Thus, we select it as one of the representative algorithms of the traditional distributed RL algorithm.
\subsubsection{LASER}
LASER \citep{laser} is a classic Actor-Critic algorithm, which investigated the combination of Actor-Critic algorithms with a uniform large-scale experience replay. It trained populations of actors with shared experiences and claimed to achieve SOTA in Atari. Thus, we select it as one of the SOTA RL algorithms within 200M training frames.

\subsubsection{R2D2}
\citep{r2d2}
Like IMPALA, R2D2 \citep{r2d2} is also a classic distributed RL algorithms. It trained  RNN-based RL agents from distributed prioritized experience replay, which achieved SOTA in Atari. Thus, we select it as one of the representative value-based distributed RL algorithms.

\subsubsection{NGU}
One of the classical problems in ALE for RL agents is the hard exploration problems \citep{goexplore,ale,agent57} like \textit{Private Eye, Montezuma’s Revenge, Pitfall!}. NGU \citep{ngu}, or Never Give Up, try to ease this problem by augmenting the reward signal with an internally generated intrinsic reward that is sensitive to novelty at two levels: short-term novelty within an episode and long-term novelty across episodes. It then learns a family
of policies for exploring and exploiting (sharing the same parameters) to obtain the highest score under the exploitative policy. NGU has achieved SOTA in Atari and thus  we selected it as one of the representative population-based model-free RL algorithms.

\subsubsection{Agent57}
Agent57 \citep{agent57} is the SOTA model-free RL algorithms on CHNS or Median HNS of Atari Benchmark. Built on the NGU agents, Agent57 proposed a novel state-action value function parameterization method and adopted an adaptive exploration over a family of policies, which overcome the drawback of NGU \citep{agent57}. We select it as one of the SOTA model-free RL algorithms.

\subsubsection{GDI}
GDI, or Generalized Data Distribution Iteration, claimed to have achieved  SOTA on mean/median HWRNS, mean HNS, HWRB, median SABER of Atari Benchmark. GDI is one of the novel Reinforcement Learning paradigms, which combined a data distribution optimization operator into the traditional generalized policy iteration (GPI) \citep{sutton} and thus achieved human-level learning efficiency.  Thus, we select them as one of the SOTA model-free RL algorithms.

\subsection{Model-Based Reinforcement Learning}
\subsubsection{SimPLe}
As one of the classic model-based RL algorithms on Atari, SimPLe, or Simulated Policy Learning \citep{modelbasedatari}, adopted a video prediction model to enable RL agents to solve Atari problems with higher sample efficiency. It claimed to outperform the SOTA model-free algorithms in most games, so we selected it as representative model-based RL algorithms.
\subsubsection{Dreamer-V2}
Dreamer-V2 \citep{dreamerv2} built world models to facilitate generalization across the experience and allow learning behaviors from imagined outcomes in the compact latent space of the world model to increase sample efficiency. Dreamer-V2 is claimed to achieve SOTA in Atari and thus  we select it as one of the SOTA model-based RL algorithms within the 200M training scale.
\subsubsection{Muzero}
Muzero \citep{muzero} combined a tree-based search with a learned model and has achieved superhuman performance on Atari. We thus selected it as one of the SOTA model-based RL algorithms.
\subsection{Other SOTA algorithms}
\subsubsection{Go-Explore}
As mentioned in NGU, a grand challenge in reinforcement learning is intelligent exploration, which is called the hard-exploration problem \citep{ale2}. Go-Explore \citep{goexplore} adopted three principles to solve this problem. Firstly, agents remember previously visited states. Secondly, agents first return to a promising state and then explore it. Finally, solve simulated environment through any available means, and then robustify via imitation learning. Go-Explore has achieved SOTA in Atari, so we select it as one of the SOTA algorithms of the hard exploration problem.

\subsubsection{Muesli}
Muesli \citep{muesli} proposed a novel policy update that combines regularized policy optimization with model learning as an auxiliary loss. It acts directly with a policy network and has a computation speed comparable to model-free baselines. As it claimed to achieve SOTA in Atari within 200M training frames, we select it as one of the SOTA  RL algorithms within 200M training frames.

\subsection{Summary of Benchmark Results}
This part summarizes the results among all the algorithms we mentioned above on the human world record benchmark for Atari. In Figs, we illustrated the benchmark results on HNS, HWRNS, SABER, and the corresponding training scale. \ref{fig: scale mean HNS time}, \ref{fig: scale mean HWRNS time} and \ref{fig: scale mean SABER time}, HWRB and corresponding game time  and learning efficiency in Fig. \ref{fig: HWRB time}. From those results, we see GDI  has achieved SOTA in learning efficiency, HWRB, HWRNS, mean HNS, and median SABER within 200M training frames. Agent57 has achieved SOTA in mean SABER, and Muzero \citep{muzero} has achieved SOTA in median HNS.

\clearpage
\section{Summary of Benchmark Results}
\label{app: Summary of Benchmark Results}

In this section, we illustrate the benchmark results of all the SOTA algorithms mentioned in this paper. For more details on these algorithms, can see App. \ref{app: atari benchmark}.

\subsection{RL Benchmarks on HNS}
\label{app: RL Benchmarks on HNS}
We report several milestones of Atari benchmarks on HNS, including DQN \citep{dqn}, RAINBOW \citep{rainbow}, IMPALA \citep{impala}, LASER \citep{laser}, R2D2 \citep{r2d2}, NGU \citep{ngu}, Agent57 \citep{agent57}, Go-Explore \citep{goexplore}, MuZero \citep{muzero}, DreamerV2 \citep{dreamerv2}, SimPLe \citep{modelbasedatari} and Muesli \citep{muesli}. We summarize mean HNS and median HNS of these algorithms  with their playtime (human playtime), learning efficiency , and training scale in Fig \ref{fig: year mean HNS time}, \ref{fig: efficiency mean HNS time} and \ref{fig: scale mean HNS time}.

\begin{figure*}[!t]
    \centering
	\subfigure{
		\includegraphics[width=0.45\textwidth]{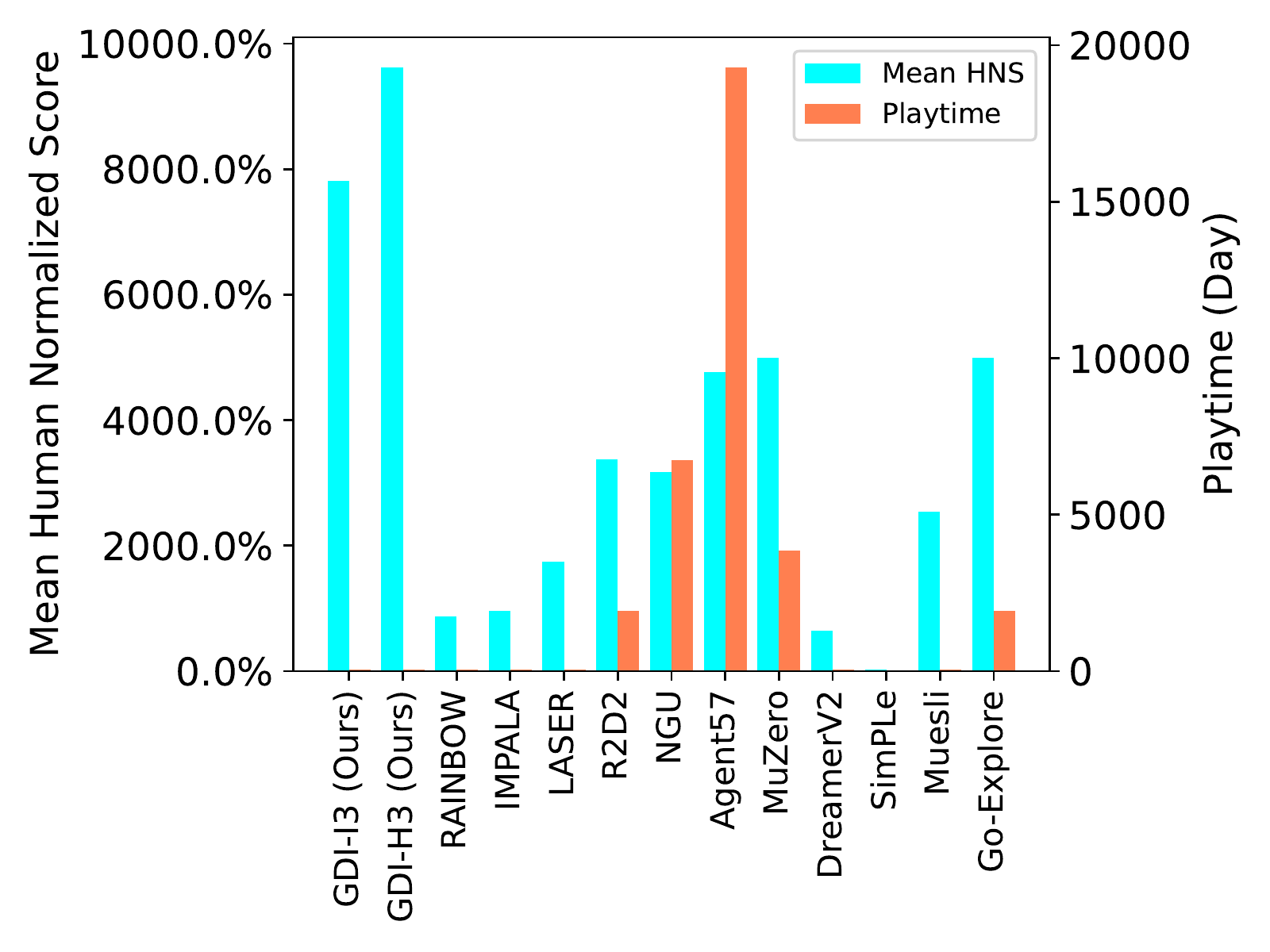}
	}
	\subfigure{
		\includegraphics[width=0.45\textwidth]{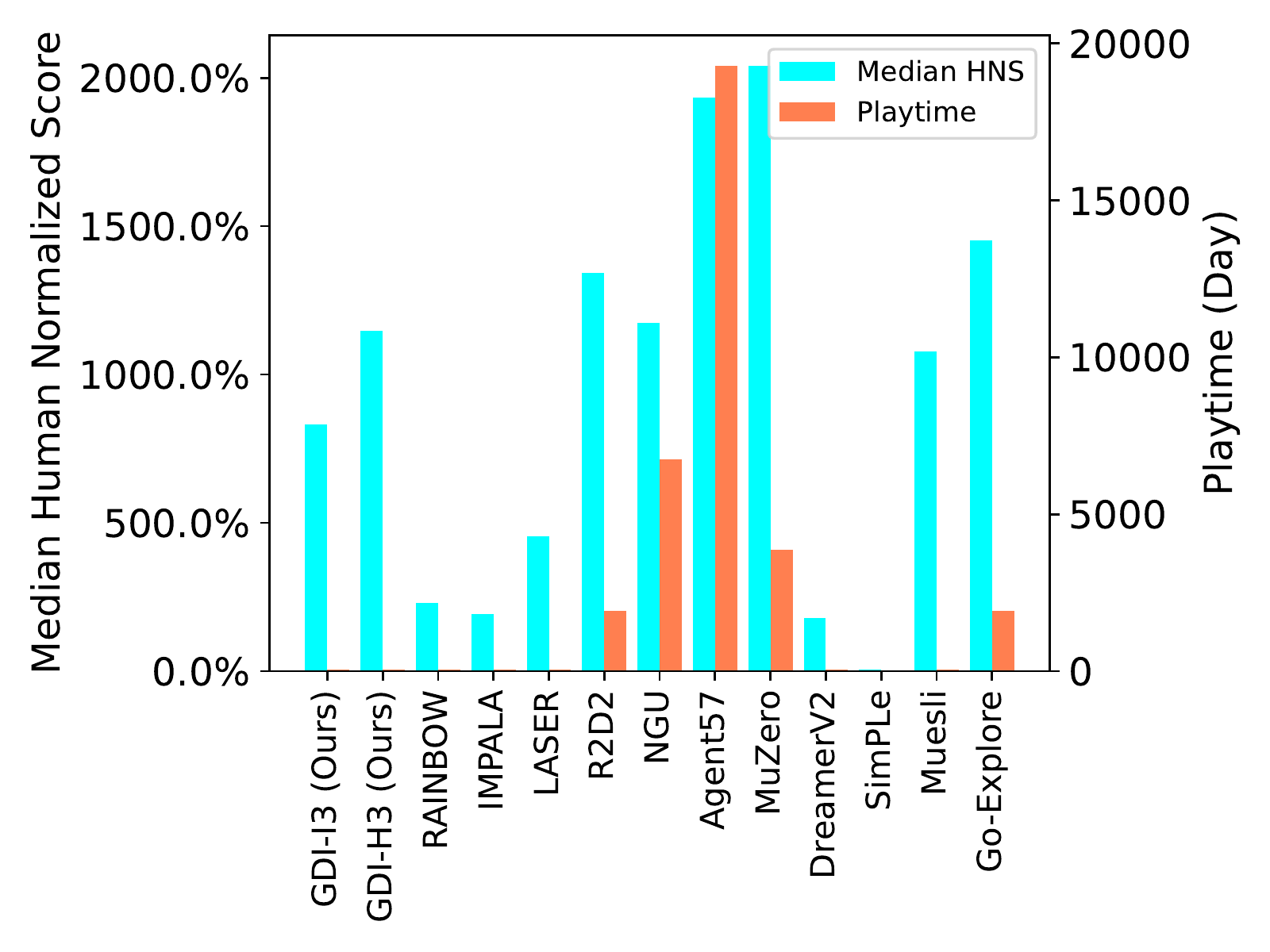}
	}
	\centering
	\caption{SOTA algorithms of Atari 57 games on mean and median HNS (\%) and playtime.}
	\label{fig: year mean HNS time}
\end{figure*}

\begin{figure*}[!t]
    \centering
	\subfigure{
		\includegraphics[width=0.45\textwidth]{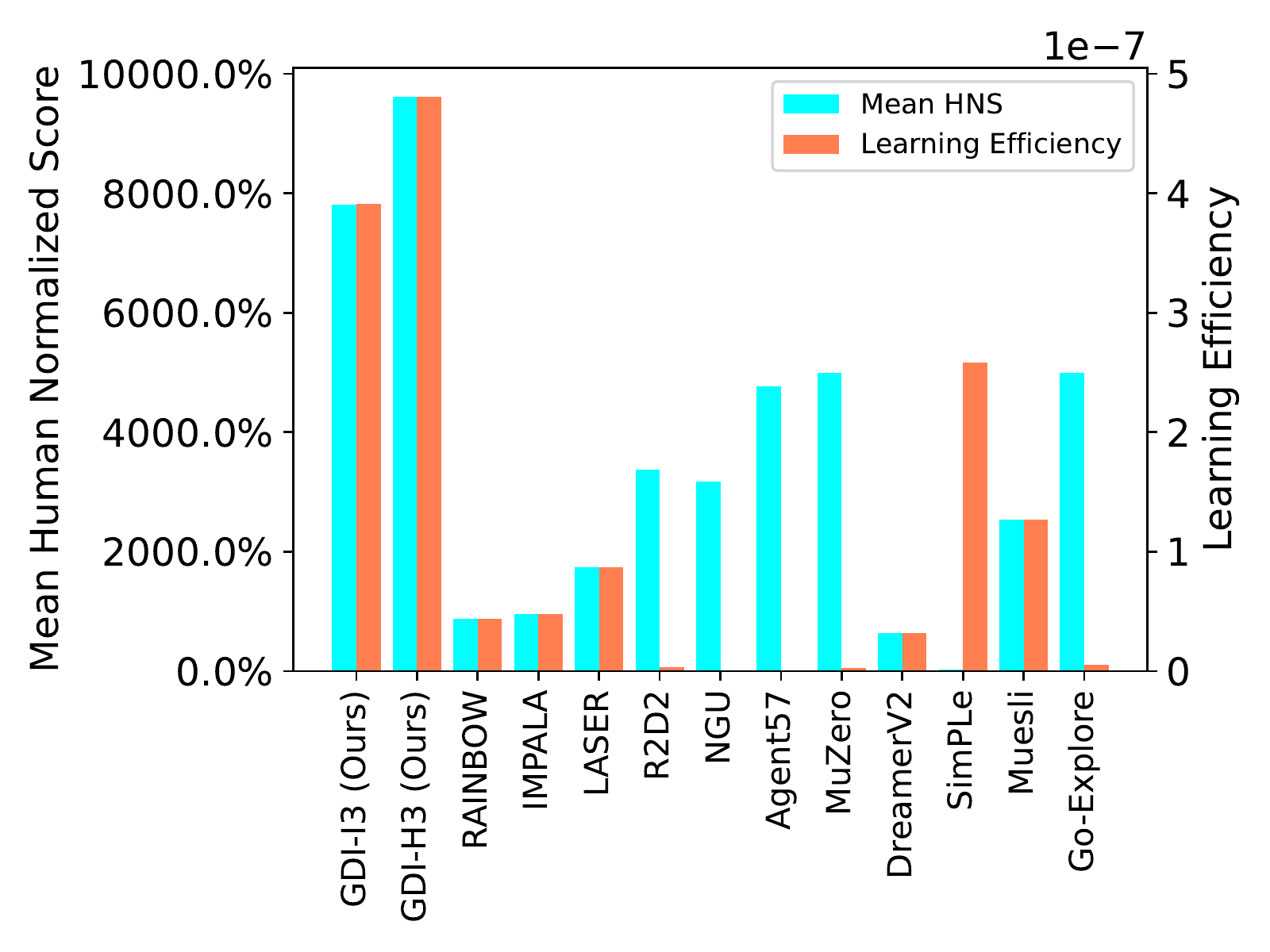}
	}
	\subfigure{
		\includegraphics[width=0.45\textwidth]{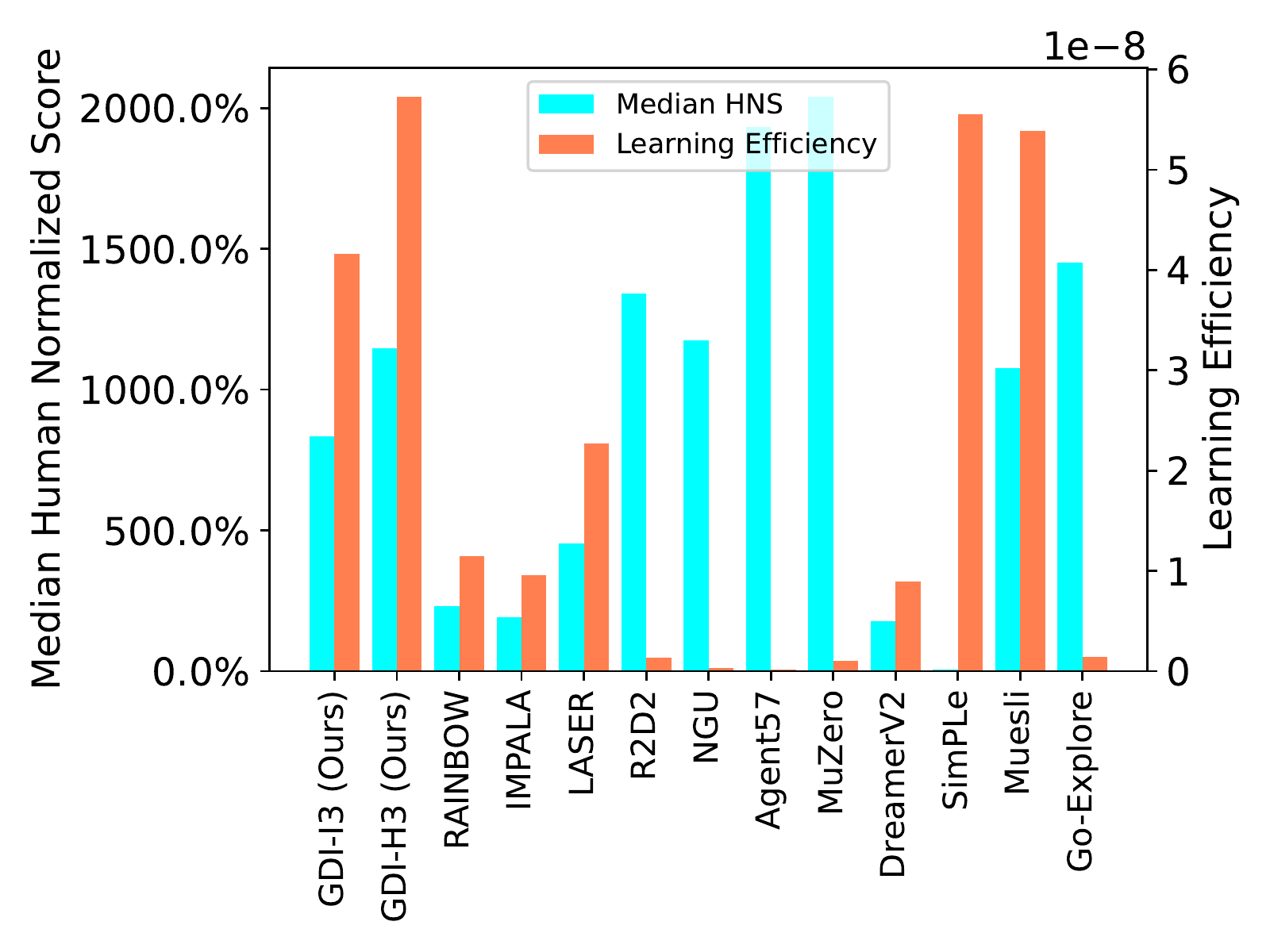}
	}
	\centering
	\caption{SOTA algorithms of Atari 57 games on mean and median HNS (\%) and corresponding learning efficiency calculated by $\frac{\text{MEAN HNS/MEDIAN HNS}}{\text{TRAINING FRAMES}}$.}
	\label{fig: efficiency mean HNS time}
\end{figure*}

\begin{figure*}[!t]
    \centering
	\subfigure{
		\includegraphics[width=0.45\textwidth]{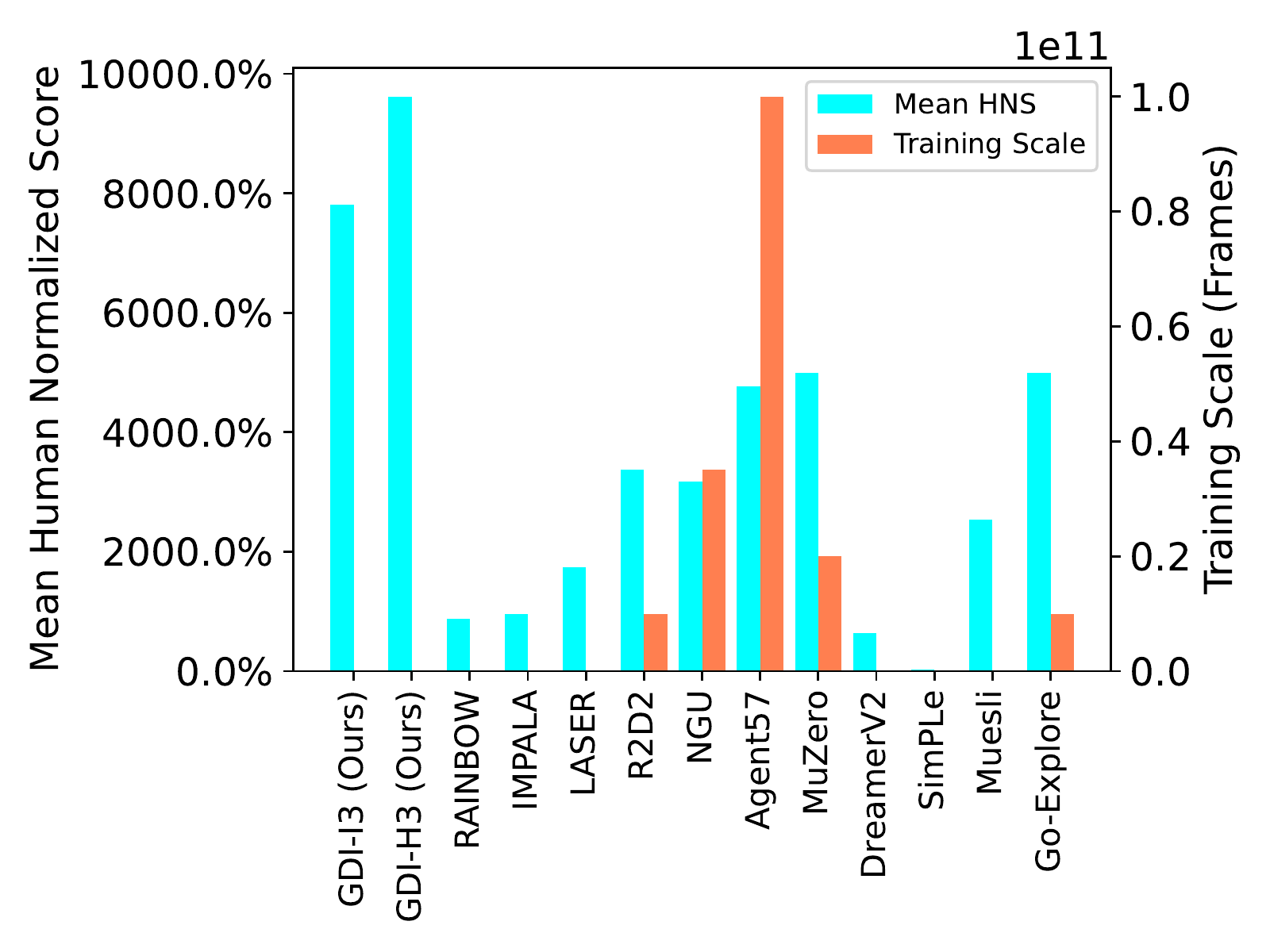}
	}
	\subfigure{
		\includegraphics[width=0.45\textwidth]{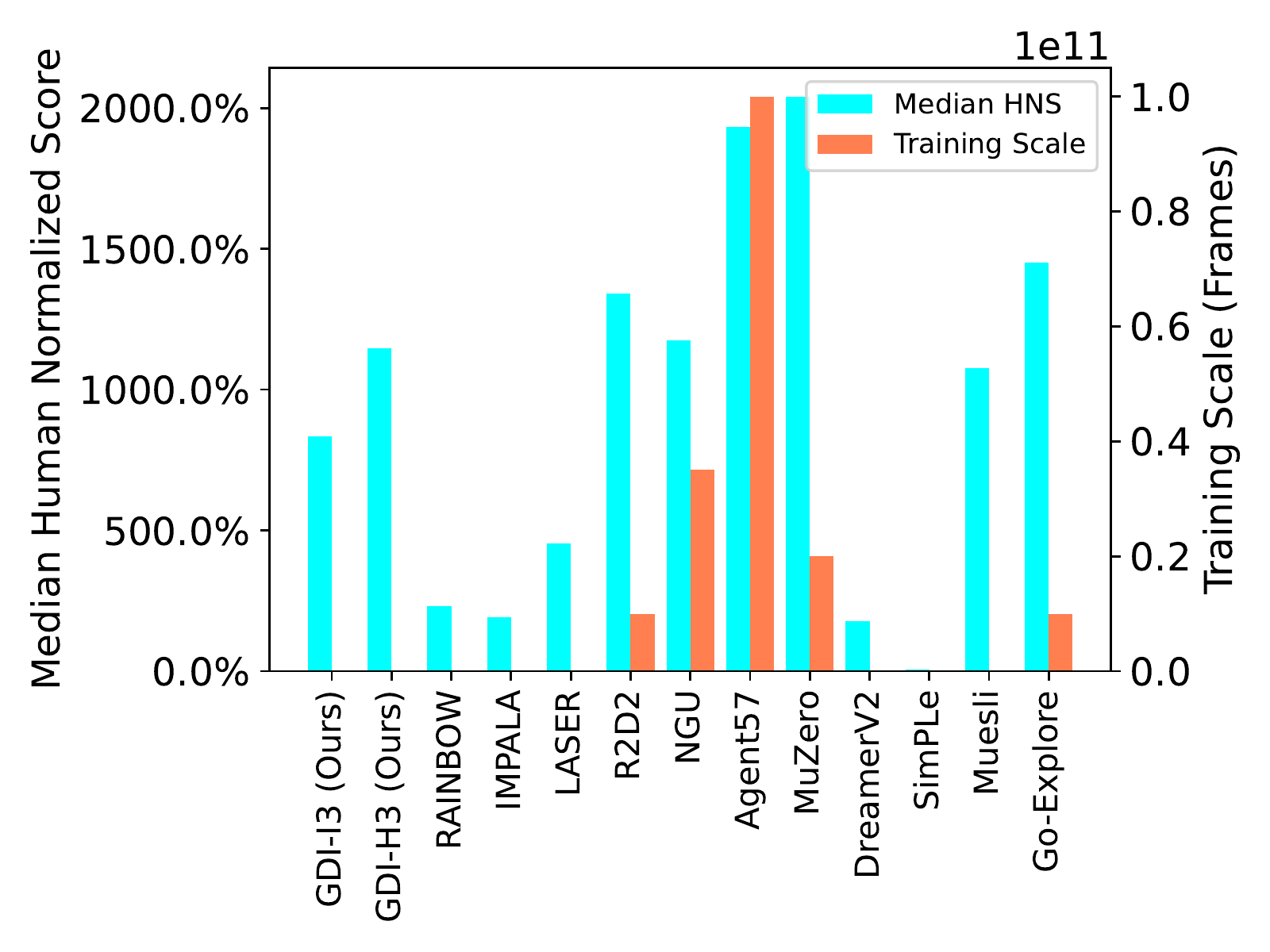}
	}
	\centering
	\caption{SOTA algorithms of Atari 57 games on mean and median HNS (\%) and corresponding training scale.}
	\label{fig: scale mean HNS time}
\end{figure*}

\subsection{RL Benchmarks on HWRNS}
\label{app: RL Benchmarks on HWRNS}

We report several milestones of Atari benchmarks on Human World Records Normalized Score (HWRNS), including DQN \citep{dqn}, RAINBOW \citep{rainbow}, IMPALA \citep{impala}, LASER \citep{laser}, R2D2 \citep{r2d2}, NGU \citep{ngu}, Agent57 \citep{agent57}, Go-Explore \citep{goexplore}, MuZero \citep{muzero}, DreamerV2 \citep{dreamerv2}, SimPLe \citep{modelbasedatari} and Muesli \citep{muesli}. We summarize mean HWRNS and median HWRNS of these algorithms  with their playtime (day), learning efficiency , and training scale in Fig \ref{fig: year mean HWRNS time}, \ref{fig: efficiency mean HWRNS time} and \ref{fig: scale mean HWRNS time}.

\begin{figure*}[!t]
    \centering
	\subfigure{
		\includegraphics[width=0.45\textwidth]{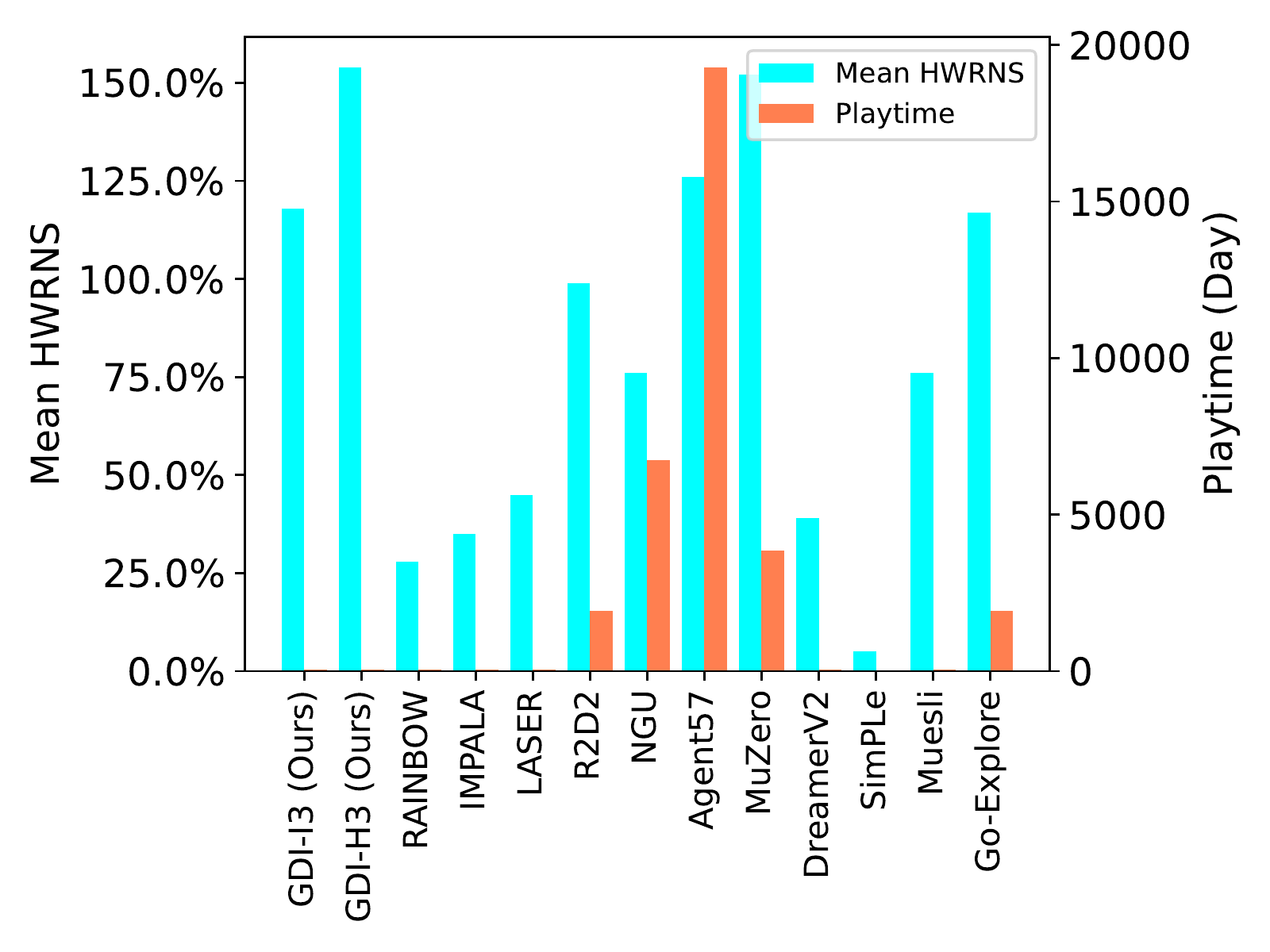}
	}
	\subfigure{
		\includegraphics[width=0.45\textwidth]{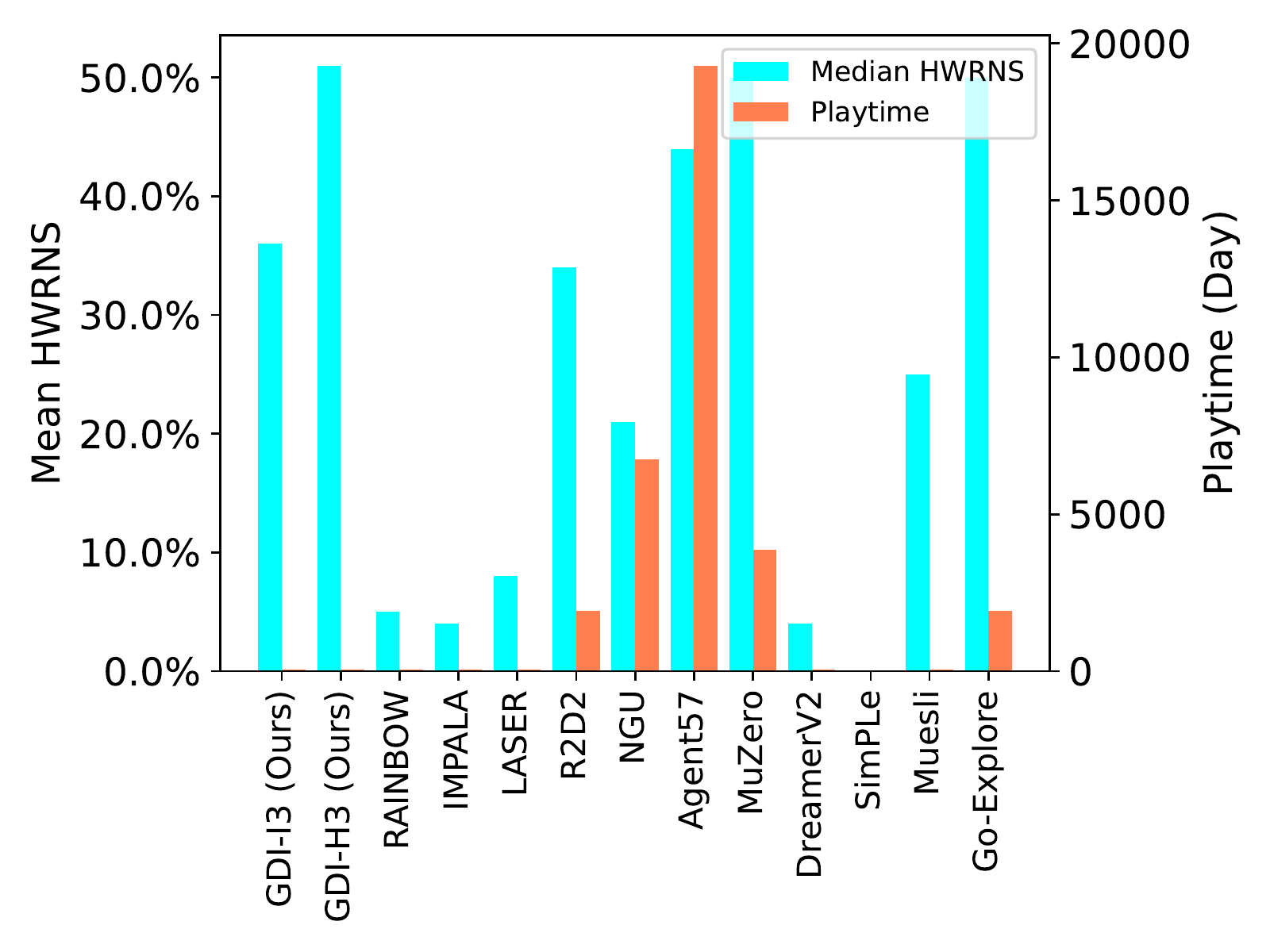}
	}
	\centering
	\caption{SOTA algorithms of Atari 57 games on mean and median HWRNS (\%) and corresponding playtime.}
	\label{fig: year mean HWRNS time}
\end{figure*}

\begin{figure*}[!t]
    \centering
	\subfigure{
		\includegraphics[width=0.45\textwidth]{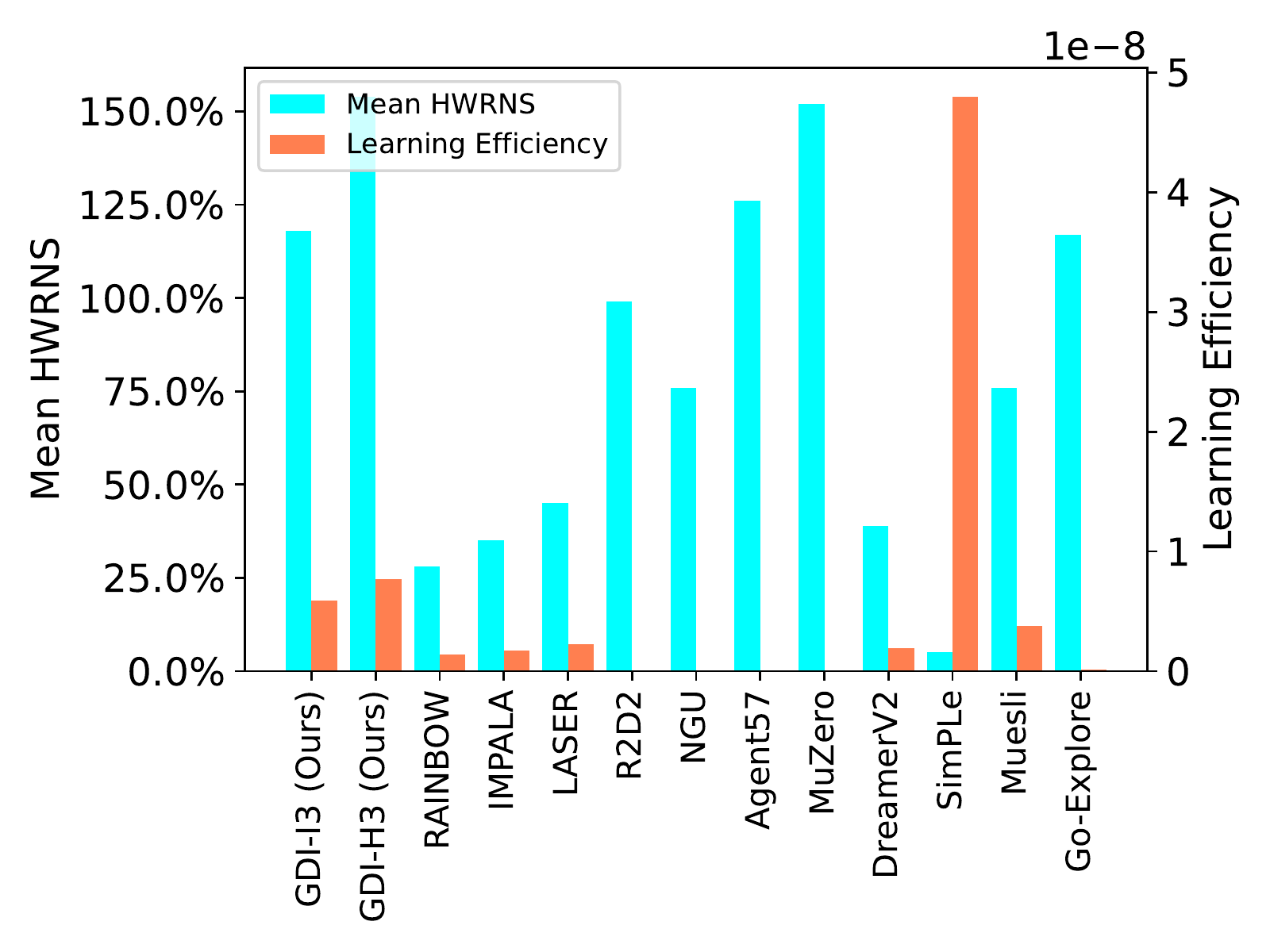}
	}
	\subfigure{
		\includegraphics[width=0.45\textwidth]{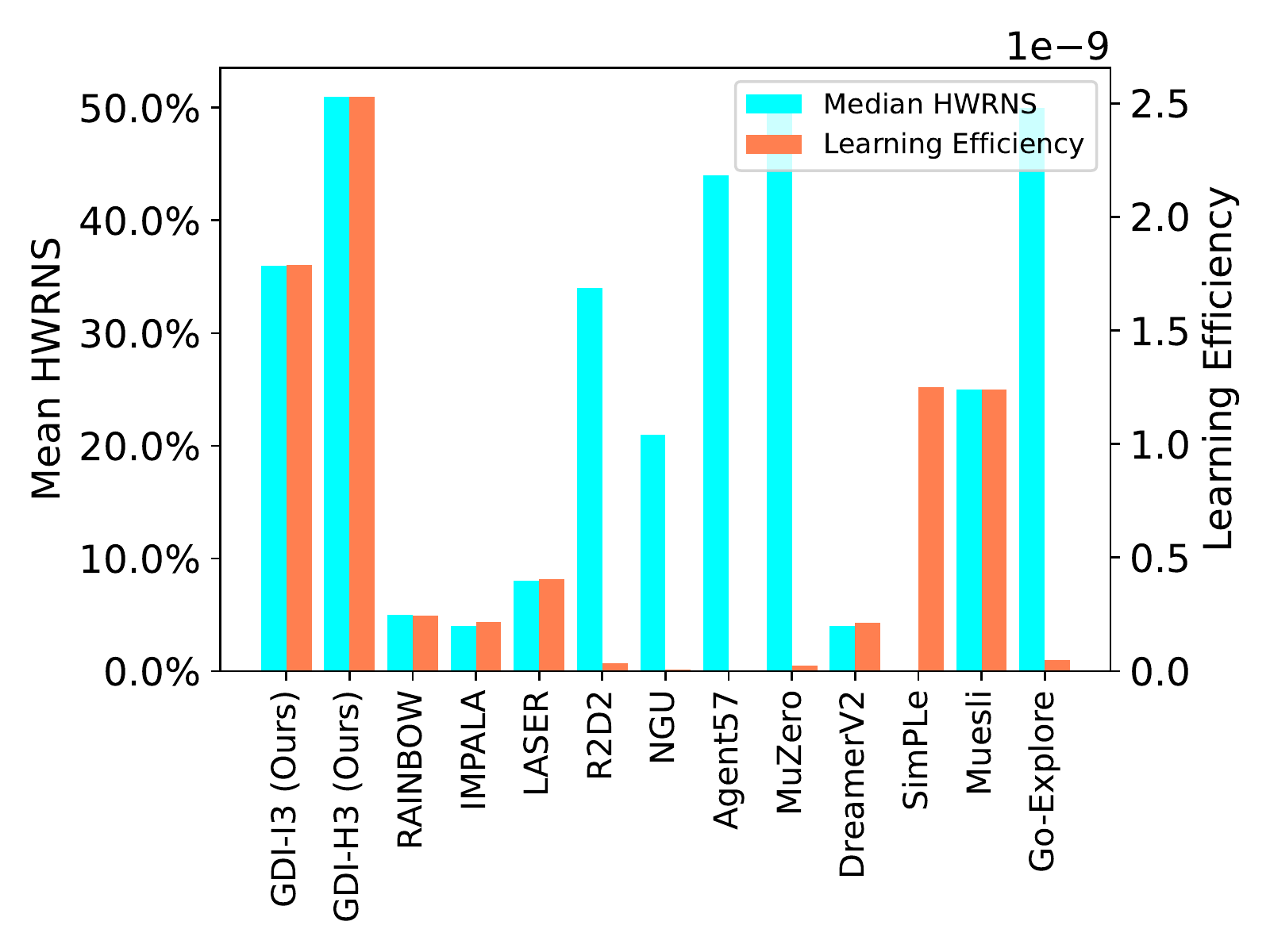}
	}
	\centering
	\caption{SOTA algorithms of Atari 57 games on mean and median HWRNS (\%) and corresponding learning efficiency calculated by $\frac{\text{MEAN HWRNS/MEDIAN HWRNS}}{\text{TRAINING FRAMES}}$.}
	\label{fig: efficiency mean HWRNS time}
\end{figure*}

\begin{figure*}[!t]
    \centering
	\subfigure{
		\includegraphics[width=0.45\textwidth]{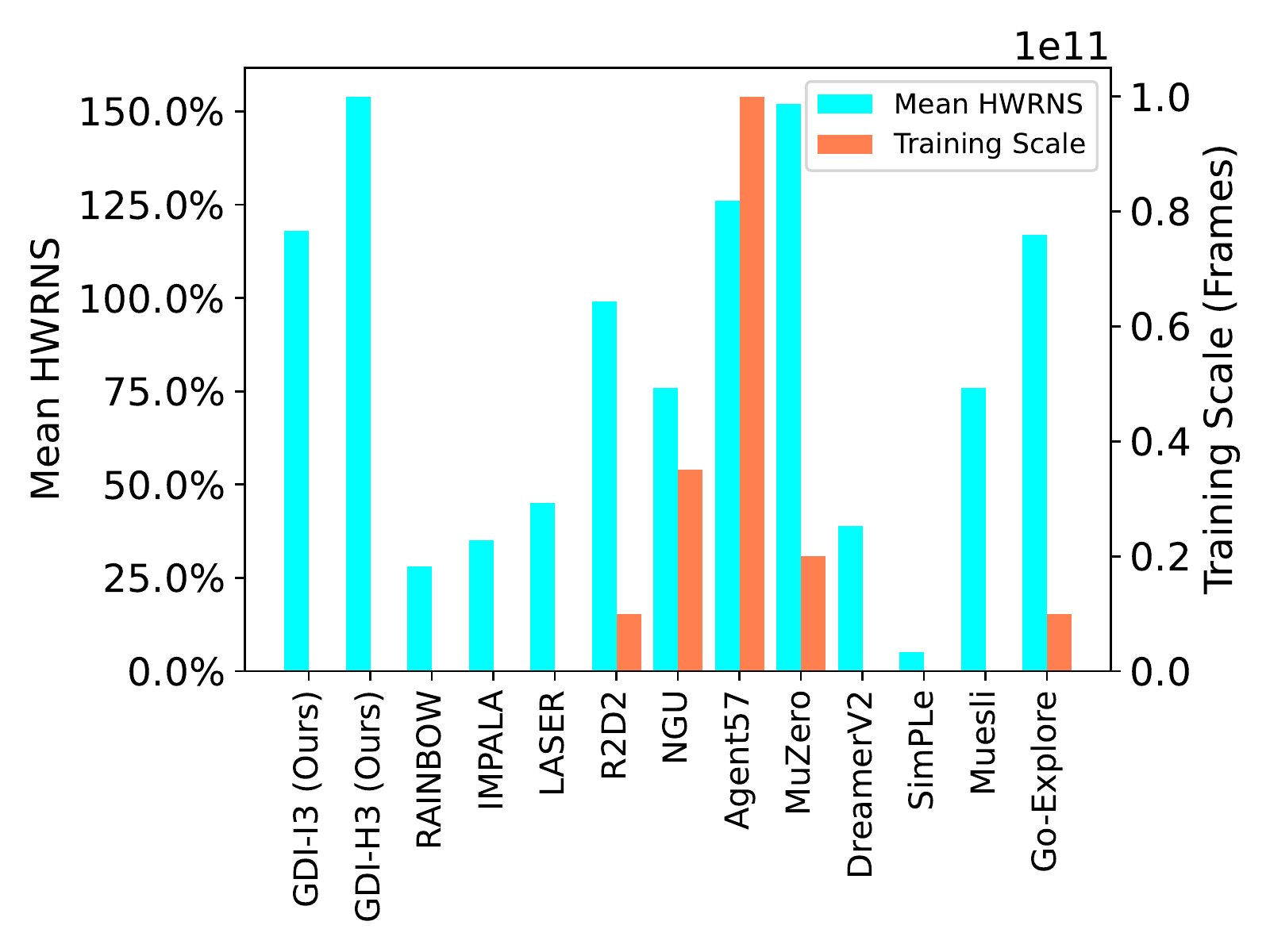}
	}
	\subfigure{
		\includegraphics[width=0.45\textwidth]{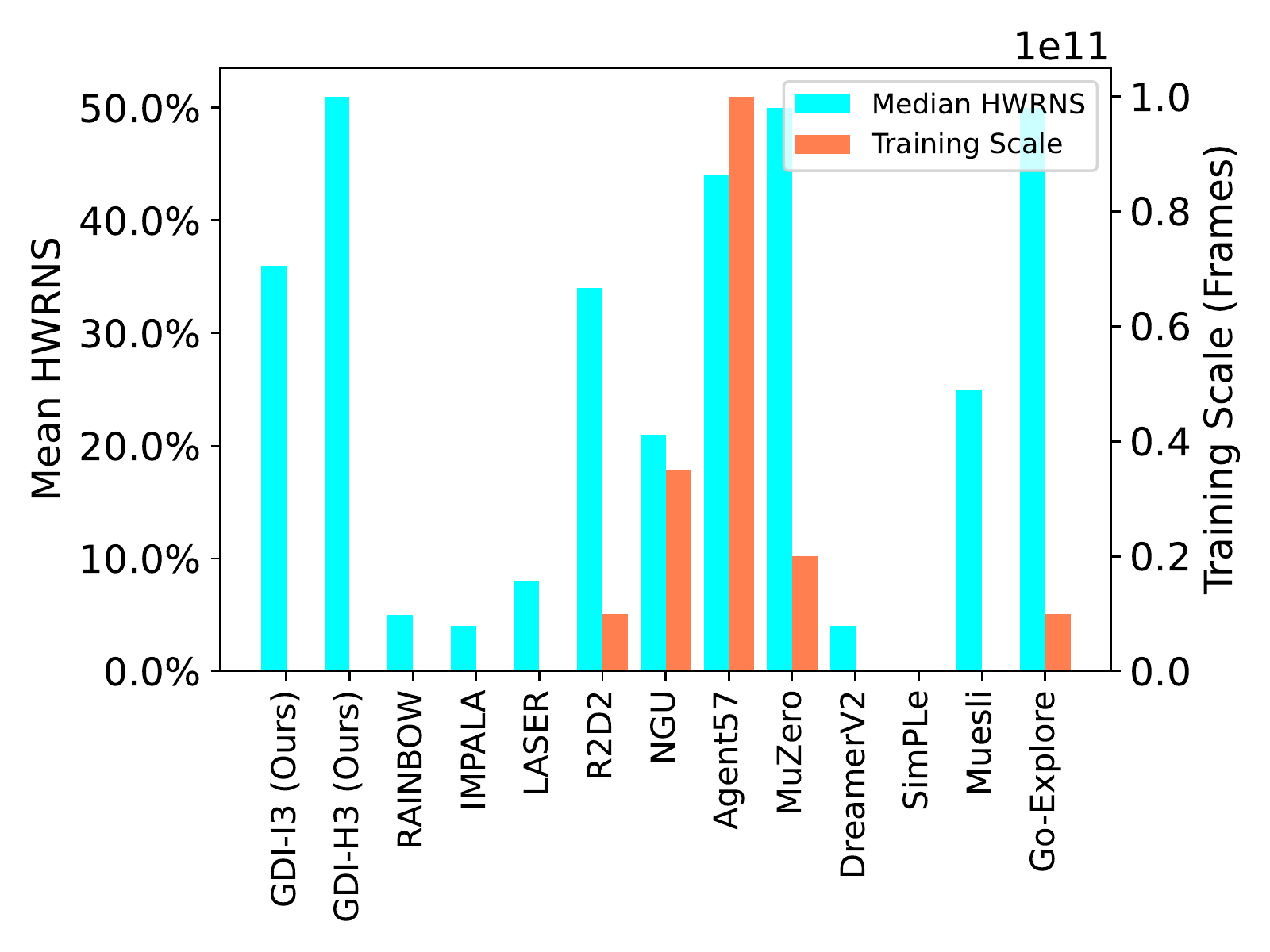}
	}
	\centering
	\caption{SOTA algorithms of Atari 57 games on mean and median HWRNS (\%) and corresponding training scale.}
	\label{fig: scale mean HWRNS time}
\end{figure*}

\subsection{RL Benchmarks on SABER}
\label{app: RL Benchmarks on SABER}

We report several milestones of Atari benchmarks on Standardized Atari BEnchmark for RL (SABER), including DQN \citep{dqn}, RAINBOW \citep{rainbow}, IMPALA \citep{impala}, LASER \citep{laser}, R2D2 \citep{r2d2}, NGU \citep{ngu}, Agent57 \citep{agent57}, Go-Explore \citep{goexplore}, MuZero \citep{muzero}, DreamerV2 \citep{dreamerv2}, SimPLe \citep{modelbasedatari} and Muesli \citep{muesli}. We summarize mean SABER and median SABER of these algorithms  with their playtime, learning efficiency, and training scale in Figs \ref{fig: year mean SABER time}, \ref{fig: efficiency mean SABER time} and \ref{fig: scale mean SABER time}.

\begin{figure*}[!t]
    \centering
	\subfigure{
		\includegraphics[width=0.45\textwidth]{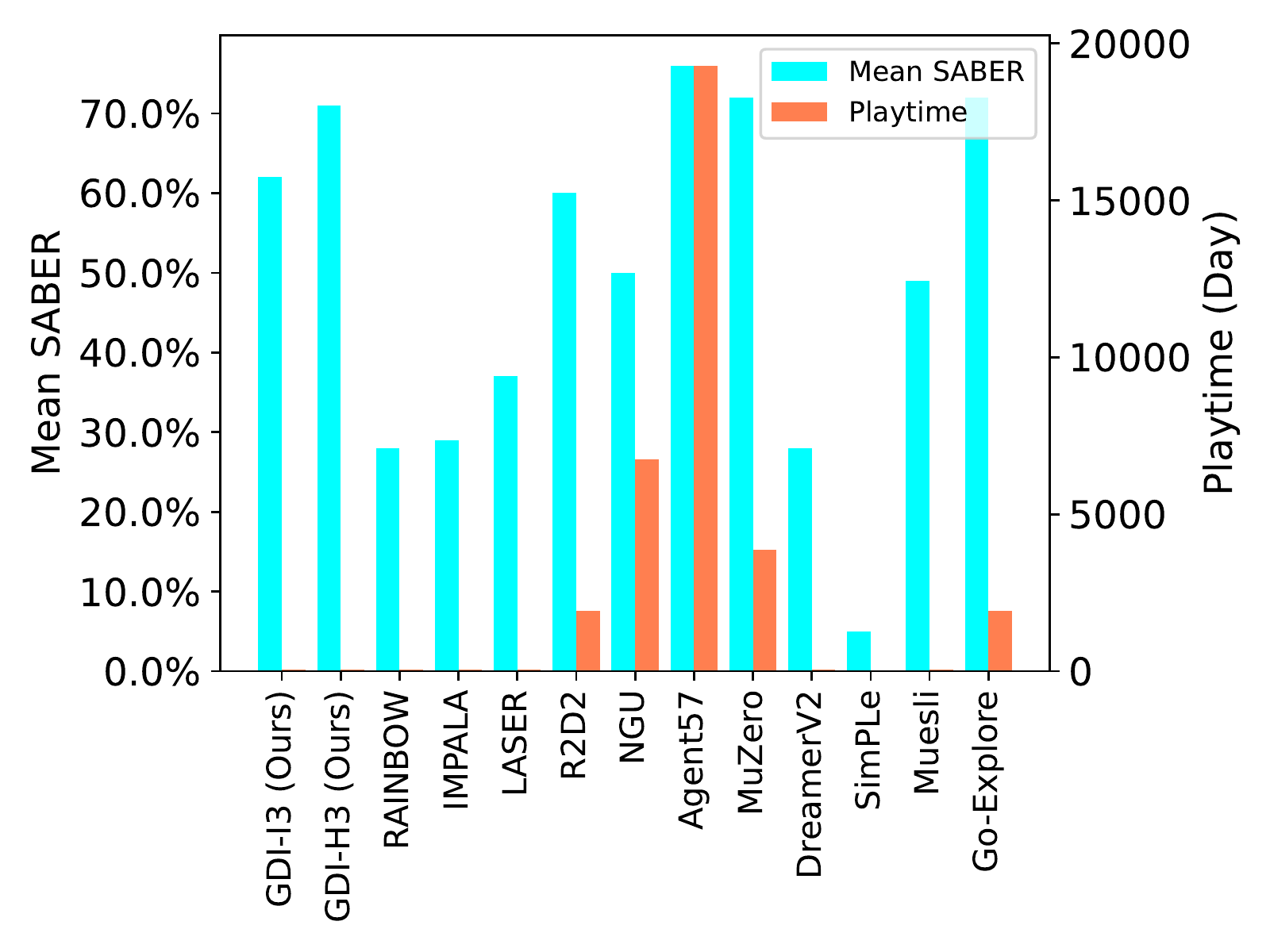}
	}
	\subfigure{
		\includegraphics[width=0.45\textwidth]{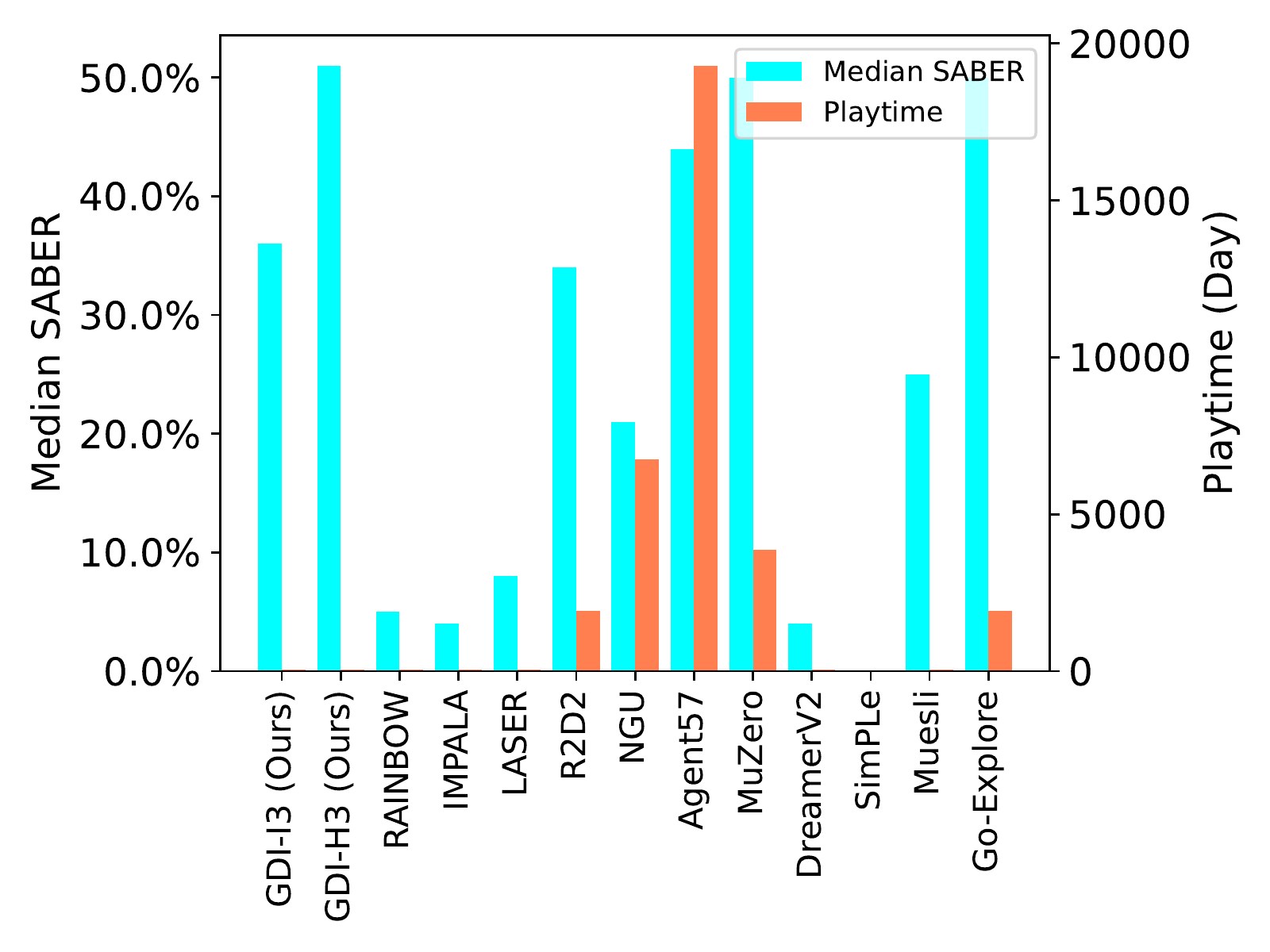}
	}
	\centering
	\caption{SOTA algorithms of Atari 57 games on mean and median SABER (\%) and corresponding playtime.}
	\label{fig: year mean SABER time}
\end{figure*}

\begin{figure*}[!t]
    \centering
	\subfigure{
		\includegraphics[width=0.45\textwidth]{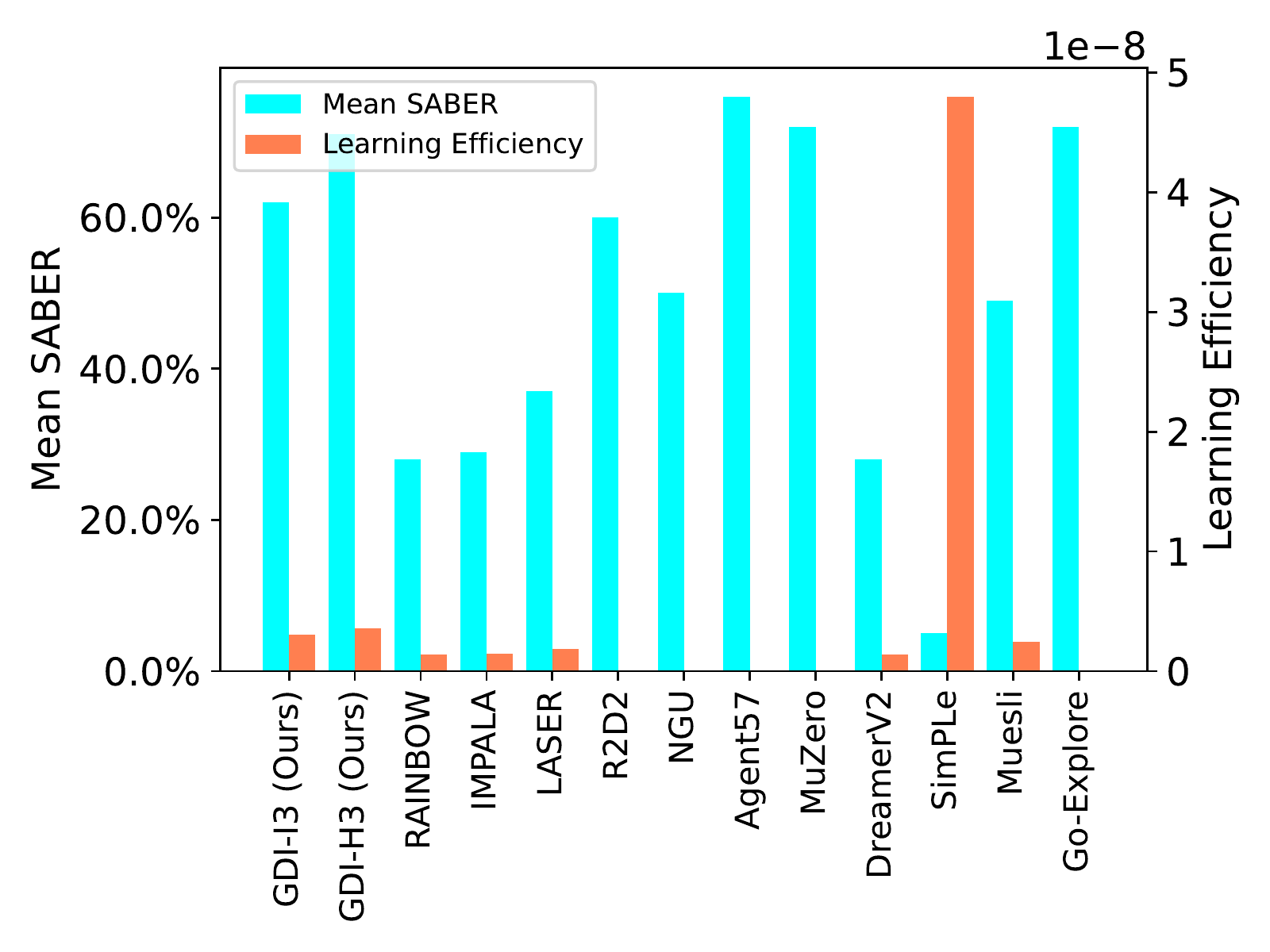}
	}
	\subfigure{
		\includegraphics[width=0.45\textwidth]{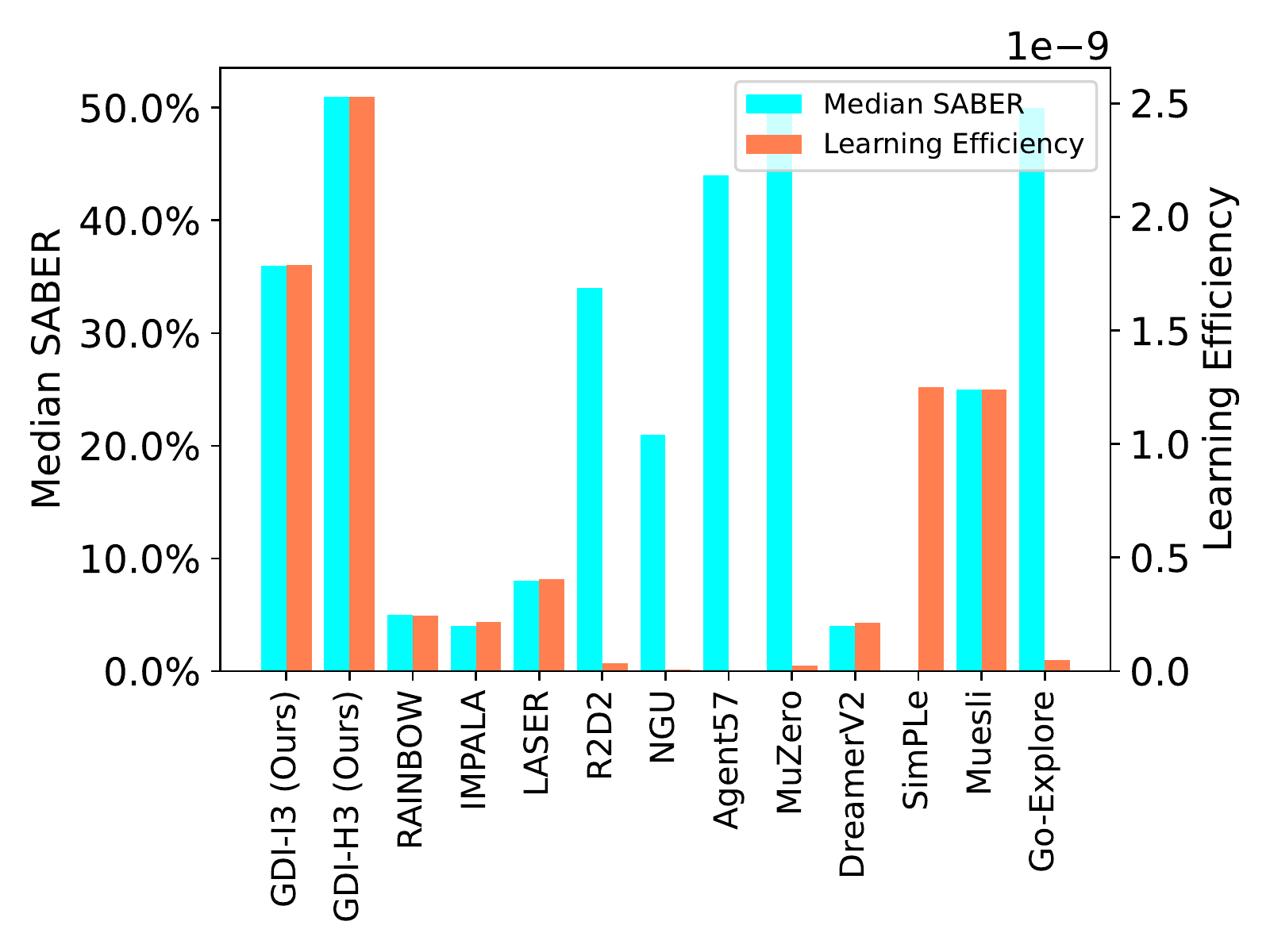}
	}
	\centering
	\caption{SOTA algorithms of Atari 57 games on mean and median SABER (\%) and corresponding learning efficiency calculated by $\frac{\text{MEAN SABER/MEDIAN SABER}}{\text{TRAINING FRAMES}}$.}
	\label{fig: efficiency mean SABER time}
\end{figure*}

\begin{figure*}[!t]
    \centering
	\subfigure{
		\includegraphics[width=0.45\textwidth]{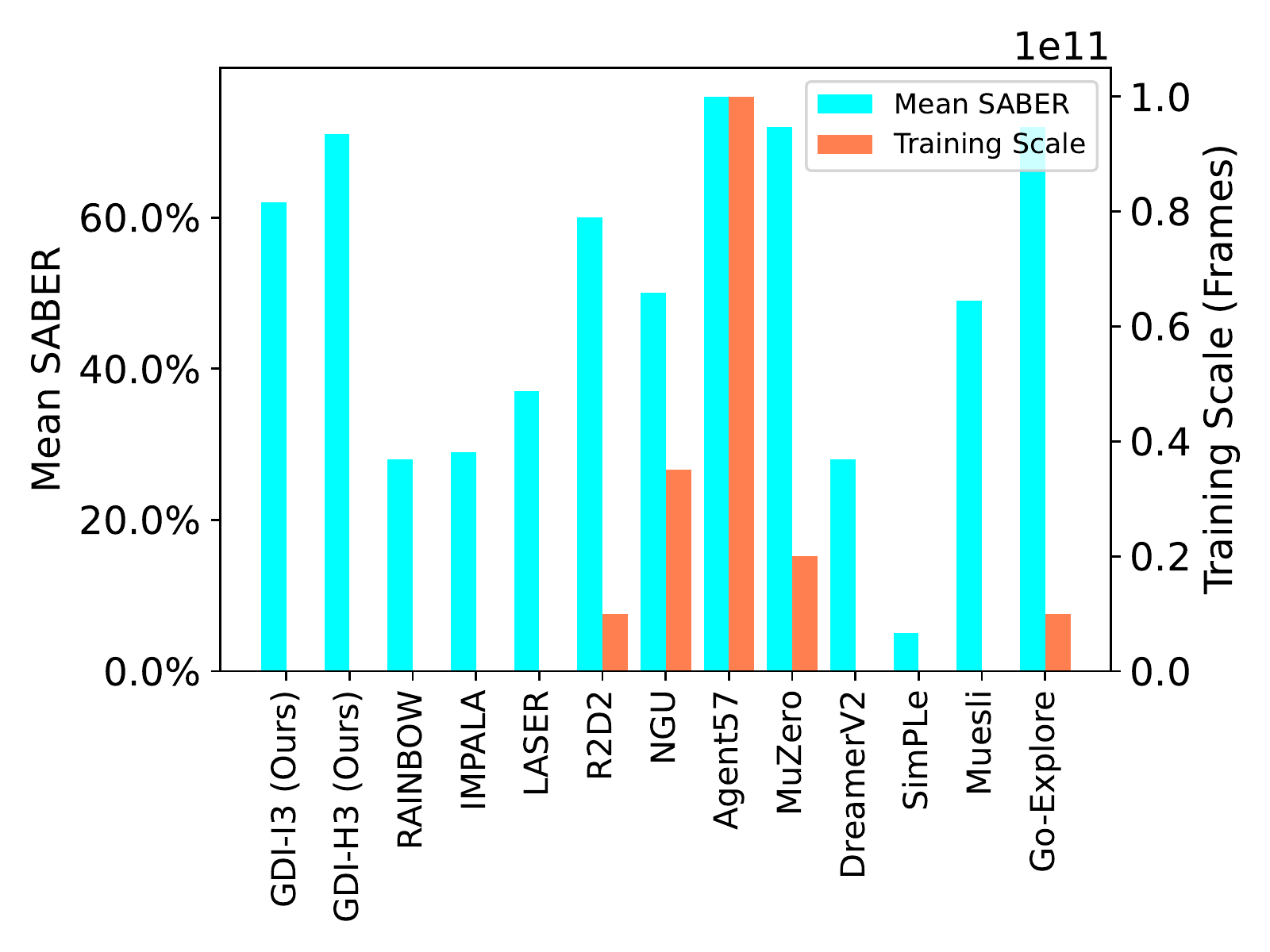}
	}
	\subfigure{
		\includegraphics[width=0.45\textwidth]{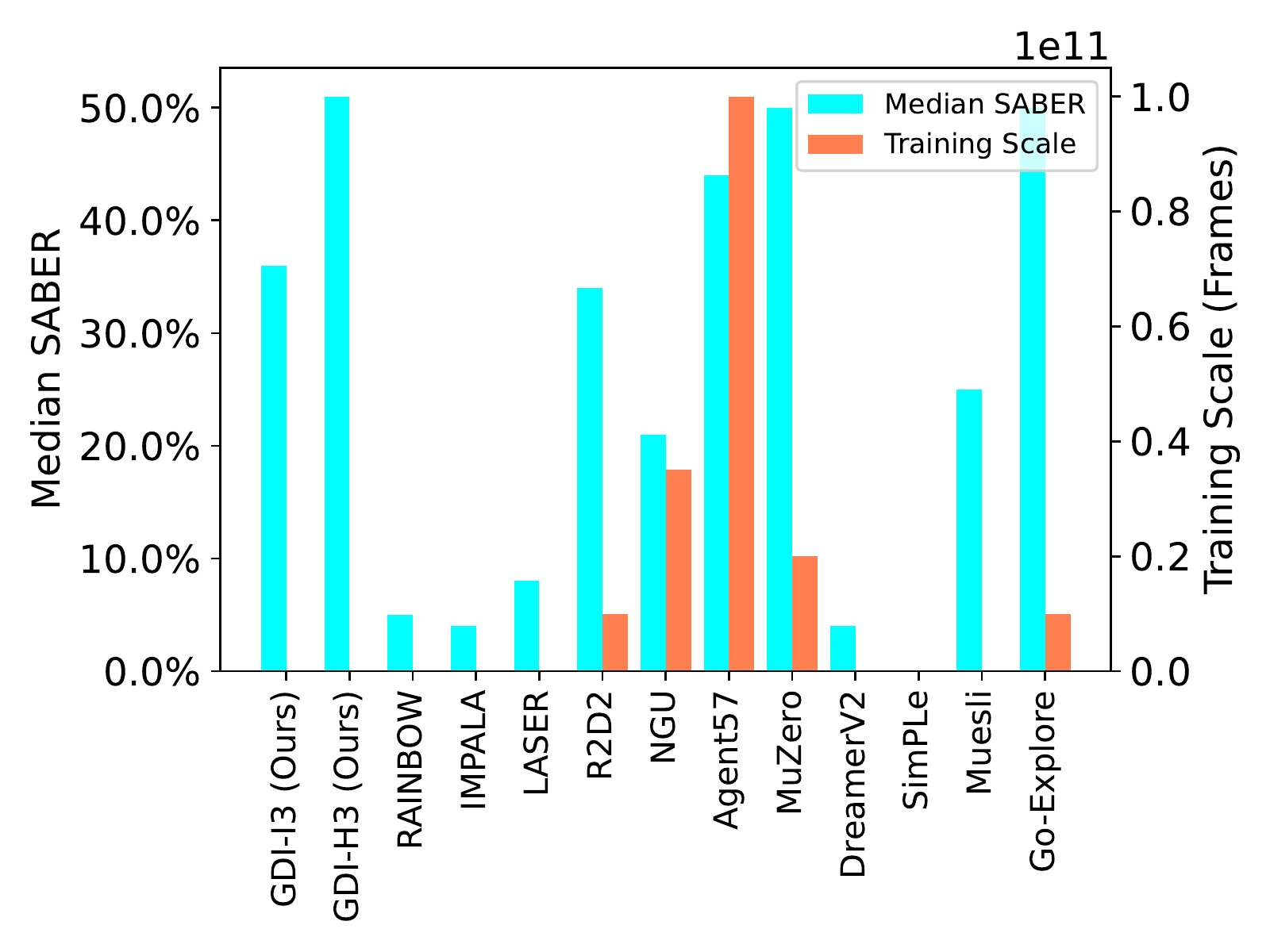}
	}
	\centering
	\caption{SOTA algorithms of Atari 57 games on mean and median SABER (\%) and corresponding training scale.}
	\label{fig: scale mean SABER time}
\end{figure*}

\subsection{RL Benchmarks on HWRB}
\label{app: RL Benchmarks on HWRB}

We report several milestones of Atari benchmarks on HWRB, including DQN \citep{dqn}, RAINBOW \citep{rainbow}, IMPALA \citep{impala}, LASER \citep{laser}, R2D2 \citep{r2d2}, NGU \citep{ngu}, Agent57 \citep{agent57}, Go-Explore \citep{goexplore}, MuZero \citep{muzero}, DreamerV2 \citep{dreamerv2}, SimPLe \citep{modelbasedatari} and Muesli \citep{muesli}. We summarize HWRB of these algorithms  with their playtime, learning efficiency , and training scale in Figs \ref{fig: HWRB time}.

\begin{figure*}[!t]
    \centering
	\subfigure{
		\includegraphics[width=0.45\textwidth]{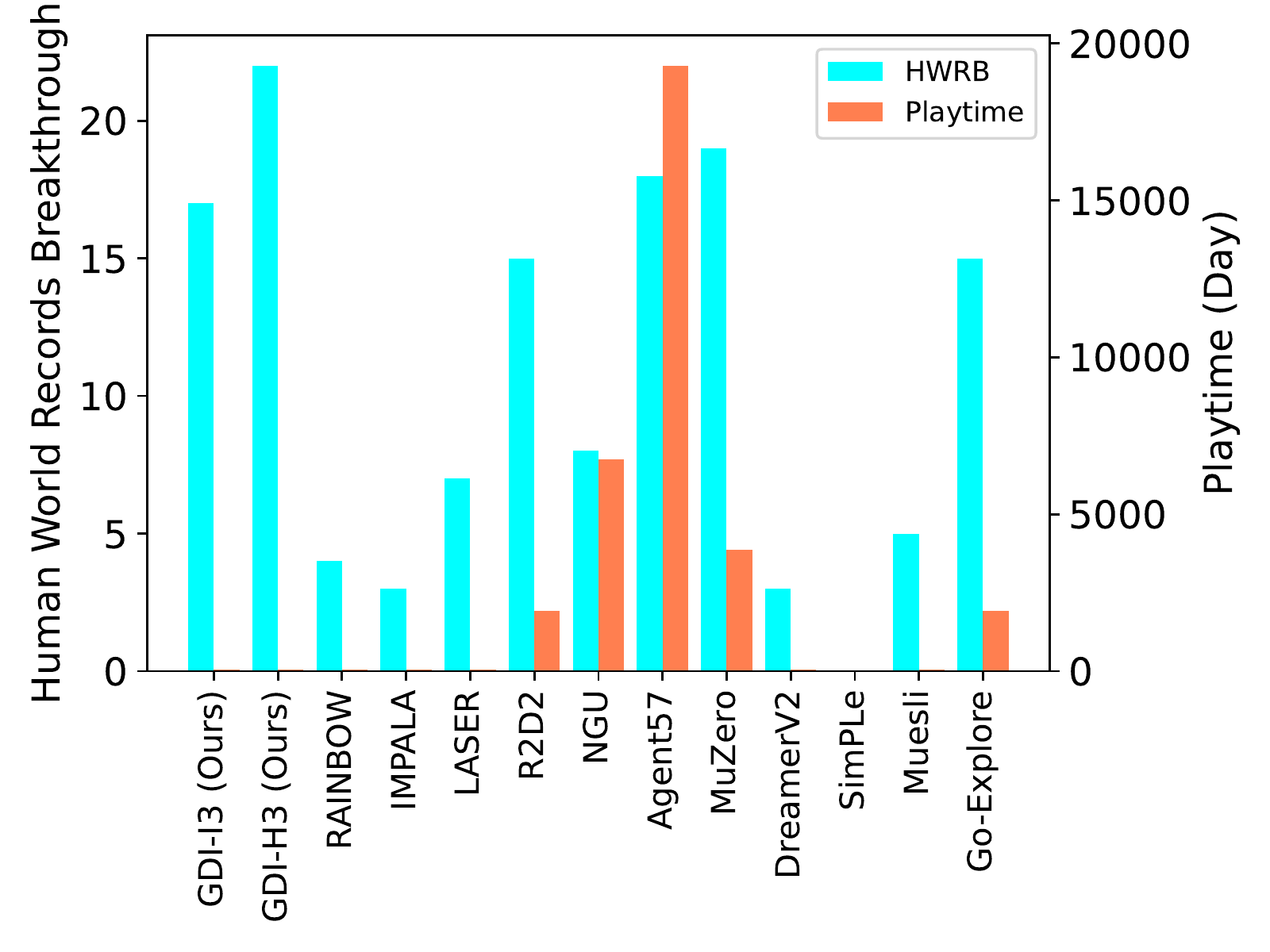}
	}
	\subfigure{
		\includegraphics[width=0.45\textwidth]{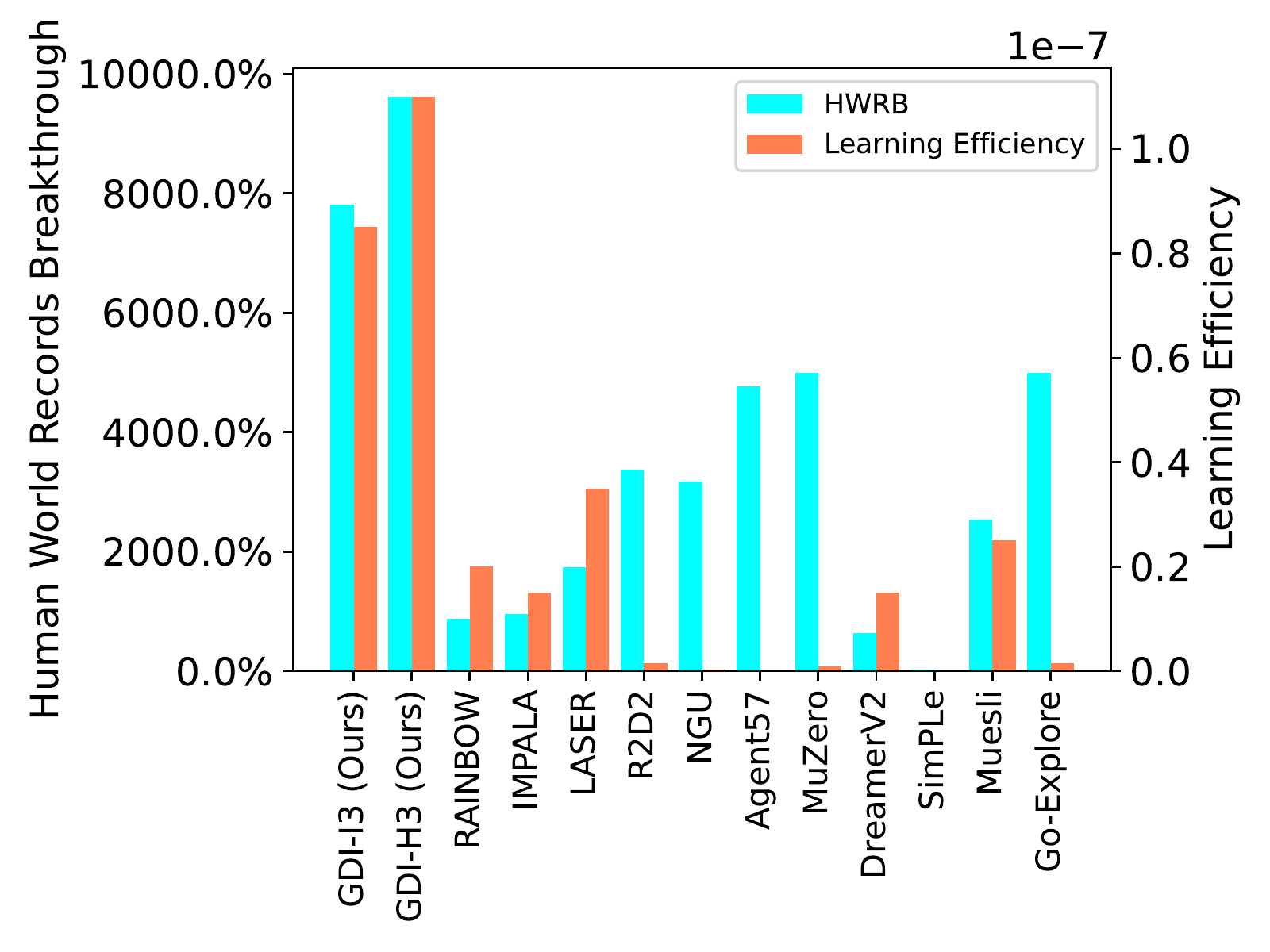}
	}
	\subfigure{
		\includegraphics[width=0.45\textwidth]{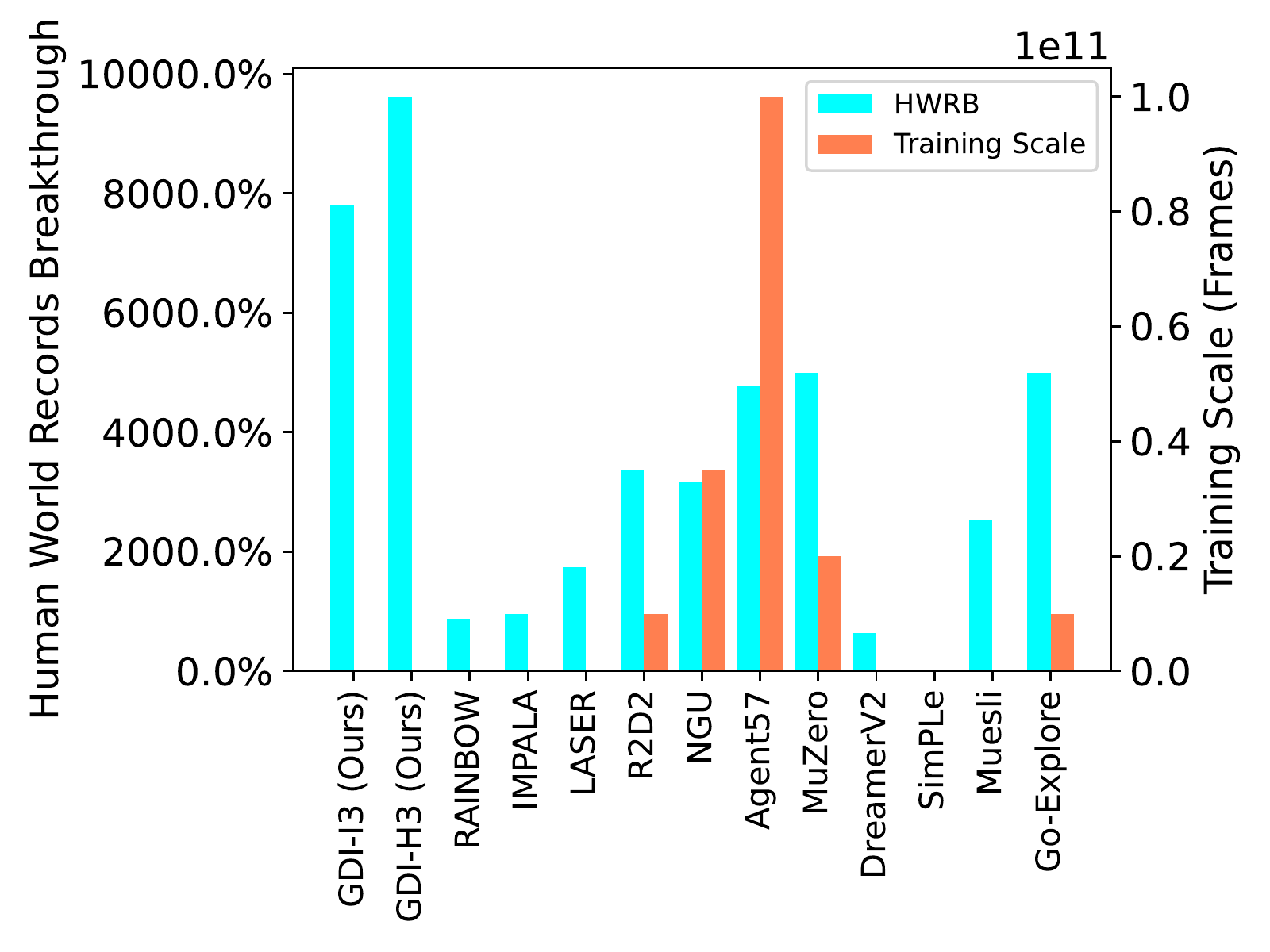}
	}
	\centering
	\caption{SOTA algorithms of Atari 57 games on HWRB.
	HWRB of SimPLe is 0, so it's not shown in the up-right figure.}
	\label{fig: HWRB time}
\end{figure*}

\clearpage

\section{Experimental Results}
\label{appendix: experiment results}

In this section, we report the performance of GDI-H$^3$, GDI-I$^3$, and many well-known SOTA algorithms, including both the model-based and model-free methods (see App. \ref{app: atari benchmark}). First of all, we summarize the performance of all the algorithms over all the evaluation criteria of our evaluation system in App. \ref{app: Full Performance Comparison} which is mentioned in App. \ref{sec:app Experiment Details}. In the following three parts, we visualize the performance of GDI-H$^3$, GDI-I$^3$ over HNS in App. \ref{app: Figure of HNS}, HWRNS in App. \ref{app: Figure of HWRNS}, SABER in App. \ref{app: Figure of SABER} via histogram. Furthermore, we detail all the original scores of all the algorithms and provide raw data that calculates those  evaluation criteria, wherein we first provide all the human world records in 57 Atari games and calculate the HNS in App. \ref{app: Atari Games Table of Scores Based on Human Average Records}, HWRNS in App. \ref{app: Atari Games Table of Scores Based on Human World Records} and SABER in App. \ref{app: Atari Games Table of Scores Based on SABER} of all 57 Atari games. We further provide all the evaluation curves of GDI-H$^3$ and GDI-I$^3$ over 57 Atari games in the App. \ref{app: Atari Games Learning Curves}.

\subsection{Full Performance Comparison}
\label{app: Full Performance Comparison}
In this part, we summarize the performance of all mentioned algorithms over all the evaluation criteria in Tab. \ref{Tab: full performance comparison}. In the following sections, we will detail the performance of each algorithm on all Atari games one by one. 

\begin{table}[H]
\footnotesize
\centering
\caption{Full performance comparison on Atari. The units of training scale is sampled frames. The units of playtime is huamn playtime (day). HNS(\%), HWRNS(\%), and SABER(\%) adopts the  percentage format (i.e., \%). Bold scores indicate the SOTA performance}
\label{Tab: full performance comparison}
\setlength{\tabcolsep}{1.0pt}
\begin{tabular}{ |c | c | c | c| c c | c c |c c|}
\hline
Algorithms & Training Scale & Playtime & HWRB & Mean HNS & Median HNS & Mean HWRNS & Median HWRNS & Mean SABER  & Median SABER\\
\hline
GDI-I$^3$        &\GDIInumframes    & \GDIIgametime    & \textbf{\GDIIHWRB} & \textbf{\GDIImeanhns} & \textbf{\GDIImedianhns}      &\textbf{\GDIImeanHWRNS} & \textbf{\GDIImedianHWRNS}& \textbf{\GDIImeanSABER} & \textbf{\GDIImedianSABER}\\
Rainbow     &\rainbownumframes    & \rainbowgametime     & \rainbowHWRB  & \rainbowmeanhns  & \rainbowmedianhns     &\rainbowmeanHWRNS  & \rainbowmedianHWRNS & \rainbowmeanSABER & \rainbowmedianSABER\\
IMPALA      &\impalanumframes    & \impalagametime     & \impalaHWRB  & \impalameanhns  & \impalamedianhns     &\impalameanHWRNS  & \impalamedianHWRNS & \impalameanSABER & \impalamedianSABER\\
LASER       &\lasernumframes    & \lasergametime     & \laserHWRB  & \lasermeanhns & \lasermedianhns     &\lasermeanHWRNS  & \lasermedianHWRNS & \lasermeanSABER & \lasermedianSABER\\
\hline
\hline
GDI-I$^3$        &\GDIInumframes    &\GDIIgametime   & \GDIIHWRB & \textbf{\GDIImeanhns} & \GDIImedianhns     &\GDIImeanHWRNS &\GDIImedianHWRNS & \GDIImeanSABER & \GDIImedianSABER\\
R2D2        &\rtdtnumframes   &\rtdtgametime     & \rtdtHWRB & \rtdtmeanhns & \rtdtmedianhns &\rtdtmeanHWRNS  & \rtdtmedianHWRNS& \rtdtmeanSABER & \rtdtmedianSABER \\
NGU         &\ngunumframes & \ngugametime    & \nguHWRB  & \ngumeanhns & \ngumedianhns    &\ngumeanHWRNS  & \ngumedianHWRNS& \ngumeanSABER & \ngumedianSABER\\
Agent57     &\agentnumframes   & \agentgametime      & \textbf{\agentHWRB} & \agentmeanhns & \textbf{\agentmedianhns}    &\textbf{\agentmeanHWRNS} & \textbf{\agentmedianHWRNS}& \textbf{\agentmeanSABER} & \textbf{\agentmedianSABER}\\
\hline
\hline
GDI-I$^3$       &\GDIInumframes    &\GDIIgametime    & \GDIIHWRB & \textbf{\GDIImeanhns} & \GDIImedianhns     &\GDIImeanHWRNS &\GDIImedianHWRNS & \GDIImeanSABER & \GDIImedianSABER\\
SimPLe      &\simplenumframes   & \simplegametime  & \simpleHWRB  & \simplemeanhns   & \simplemedianhns       &\simplemeanHWRNS   & \simplemedianHWRNS & \simplemeanSABER  & \simplemedianSABER\\
DreamerV2   &\dreamernumframes    & \dreamergametime     & \dreamerHWRB  & \dreamermeanhns  & \dreamermedianhns     &\dreamermeanHWRNS  & \dreamermedianHWRNS & \dreamermeanSABER & \dreamermedianSABER \\
MuZero      & \muzeronumframes   & \muzerogametime     & \textbf{\muzeroHWRB} & \muzeromeanhns & \textbf{\muzeromedianhns}    &\textbf{\muzeromeanHWRNS} & \textbf{\muzeromedianHWRNS}& \textbf{\muzeromeanSABER} & \textbf{\muzeromedianSABER} \\
\hline
\hline
GDI-I$^3$       &\GDIInumframes    &\GDIIgametime    & \textbf{\GDIIHWRB} & \textbf{\GDIImeanhns} & \GDIImedianhns     &\GDIImeanHWRNS &\GDIImedianHWRNS & \GDIImeanSABER & \GDIImedianSABER\\
Muesli      &\mueslinumframes    & \muesligametime     & \muesliHWRB           &  \mueslimeanhns         & \mueslimedianhns    & \mueslimeanHWRNS          & \mueslimedianHWRNS  & \mueslimeanSABER & \mueslimedianSABER \\
Go-Explore  & \goexplorenumframes   & \goexploregametime     & \goexploreHWRB & \goexploremeanhns & \textbf{\goexploremedianhns}    &\goexploremeanHWRNS & \textbf{\goexploremedianHWRNS}& \textbf{\goexploremeanSABER} & \textbf{\goexploremedianSABER}\\
\hline
\hline
GDI-H$^3$   &\GDIHnumframes    & \GDIHgametime    &\textbf{ \GDIHHWRB} & \textbf{\GDIHmeanhns} & \textbf{\GDIHmedianhns}     &\textbf{\GDIHmeanHWRNS} &\textbf{\GDIHmedianHWRNS} & \textbf{\GDIHmeanSABER} & \textbf{\GDIHmedianSABER}\\
Rainbow     &\rainbownumframes    & \rainbowgametime     & \rainbowHWRB  & \rainbowmeanhns  & \rainbowmedianhns     &\rainbowmeanHWRNS  & \rainbowmedianHWRNS & \rainbowmeanSABER & \rainbowmedianSABER\\
IMPALA      &\impalanumframes    & \impalagametime     & \impalaHWRB  & \impalameanhns  & \impalamedianhns     &\impalameanHWRNS  & \impalamedianHWRNS & \impalameanSABER & \impalamedianSABER\\
LASER       &\lasernumframes    & \lasergametime     & \laserHWRB  & \lasermeanhns & \lasermedianhns     &\lasermeanHWRNS  & \lasermedianHWRNS & \lasermeanSABER & \lasermedianSABER\\
\hline
\hline
GDI-H$^3$  &\GDIHnumframes    & \GDIHgametime     & \textbf{\GDIHHWRB} & \textbf{\GDIHmeanhns} & \GDIHmedianhns     &\textbf{\GDIHmeanHWRNS} &\textbf{\GDIHmedianHWRNS} & \GDIHmeanSABER &\textbf{ \GDIHmedianSABER}\\
R2D2        &\rtdtnumframes   &\rtdtgametime     & \rtdtHWRB & \rtdtmeanhns & \rtdtmedianhns &\rtdtmeanHWRNS  & \rtdtmedianHWRNS& \rtdtmeanSABER & \rtdtmedianSABER \\
NGU         &\ngunumframes & \ngugametime    & \nguHWRB  & \ngumeanhns & \ngumedianhns    &\ngumeanHWRNS  & \ngumedianHWRNS& \ngumeanSABER & \ngumedianSABER\\
Agent57     &\agentnumframes   & \agentgametime      & \agentHWRB & \agentmeanhns & \textbf{\agentmedianhns}    &\agentmeanHWRNS & \agentmedianHWRNS& \textbf{\agentmeanSABER} & \agentmedianSABER\\
\hline
\hline
GDI-H$^3$   &\GDIHnumframes    & \GDIHgametime      & \textbf{\GDIHHWRB} & \textbf{\GDIHmeanhns} & \GDIHmedianhns     &\textbf{\GDIHmeanHWRNS} &\textbf{\GDIHmedianHWRNS} & \GDIHmeanSABER & \textbf{\GDIHmedianSABER}\\
SimPLe      &\simplenumframes   & \simplegametime  & \simpleHWRB  & \simplemeanhns   & \simplemedianhns       &\simplemeanHWRNS   & \simplemedianHWRNS & \simplemeanSABER  & \simplemedianSABER\\
DreamerV2   &\dreamernumframes    & \dreamergametime     & \dreamerHWRB  & \dreamermeanhns  & \dreamermedianhns     &\dreamermeanHWRNS  & \dreamermedianHWRNS & \dreamermeanSABER & \dreamermedianSABER \\
MuZero      & \muzeronumframes   & \muzerogametime     & \muzeroHWRB & \muzeromeanhns & \textbf{\muzeromedianhns}    &\muzeromeanHWRNS & \muzeromedianHWRNS& \textbf{\muzeromeanSABER} & \muzeromedianSABER \\
\hline
\hline
GDI-H$^3$    &\GDIHnumframes    & \GDIHgametime       & \textbf{\GDIHHWRB} & \textbf{\GDIHmeanhns} & \GDIHmedianhns     &\textbf{\GDIHmeanHWRNS} &\textbf{\GDIHmedianHWRNS} & \GDIHmeanSABER & \textbf{\GDIHmedianSABER}\\
Muesli      &\mueslinumframes    & \muesligametime     & \muesliHWRB           &  \mueslimeanhns         & \mueslimedianhns    & \mueslimeanHWRNS          & \mueslimedianHWRNS  & \mueslimeanSABER & \mueslimedianSABER \\
Go-Explore  & \goexplorenumframes   & \goexploregametime     & \goexploreHWRB & \goexploremeanhns & \textbf{\goexploremedianhns}    &\goexploremeanHWRNS & \goexploremedianHWRNS& \textbf{\goexploremeanSABER} & \goexploremedianSABER\\
\hline
\end{tabular}
\end{table}

\normalsize
\clearpage

\subsection{Figure of HNS}
\label{app: Figure of HNS}
In this part, we  visualize the HNS using GDI-H$^3$ and  GDI-I$^3$ in all 57 games. The HNS histogram of GDI-I$^3$ is illustrated in Fig. \ref{fig: HNS of GDII}. The HNS histogram of GDI-H$^3$ is illustrated in Fig. \ref{fig: HNS of GDIH}. 

\begin{figure*}[!ht]
	\subfigure{
		\includegraphics[width=0.8\textwidth]{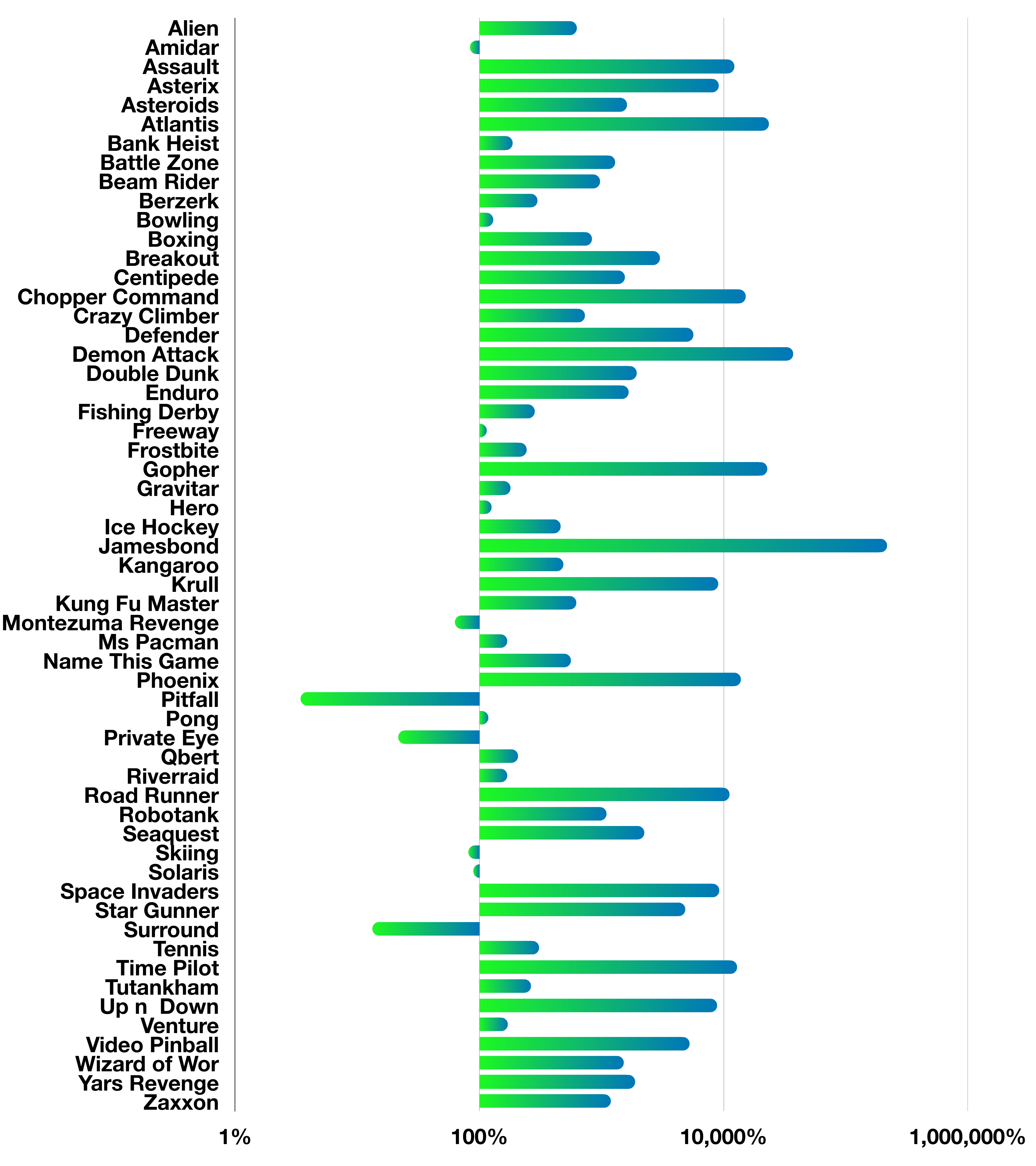}
	}
	\caption{HNS (\%) of Atari 57 games using GDI-I$^3$.}
	\label{fig: HNS of GDII}
\end{figure*}

\begin{figure*}[!ht]
	\subfigure{
		\includegraphics[width=0.8\textwidth]{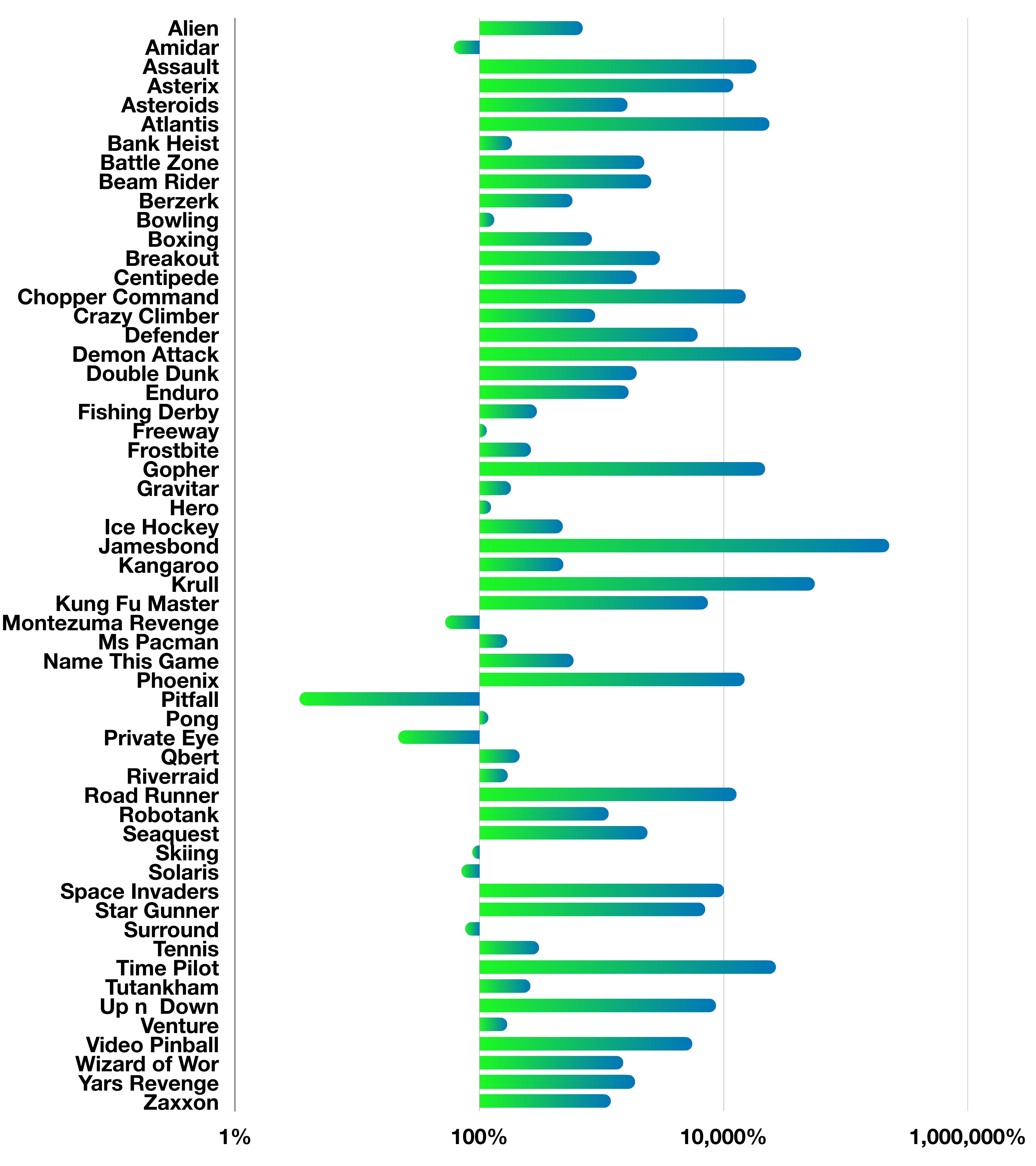}
	}
	\caption{HNS (\%) of Atari 57 games using GDI-H$^3$.}
	\label{fig: HNS of GDIH}
\end{figure*}

\clearpage

\subsection{Figure of HWRNS}
\label{app: Figure of HWRNS}
In this part, we  visualize the HWRNS \citep{dreamerv2,atarihuman} using GDI-H$^3$ and  GDI-I$^3$ in all 57 games. The HWRNS histogram of GDI-I$^3$ is illustrated in Fig. \ref{fig: HWRNS of GDII}. The HWRNS histogram of GDI-H$^3$ is illustrated in Fig. \ref{fig: HWRNS of GDIH}.

\begin{figure*}[!ht]
	\subfigure{
		\includegraphics[width=0.8\textwidth]{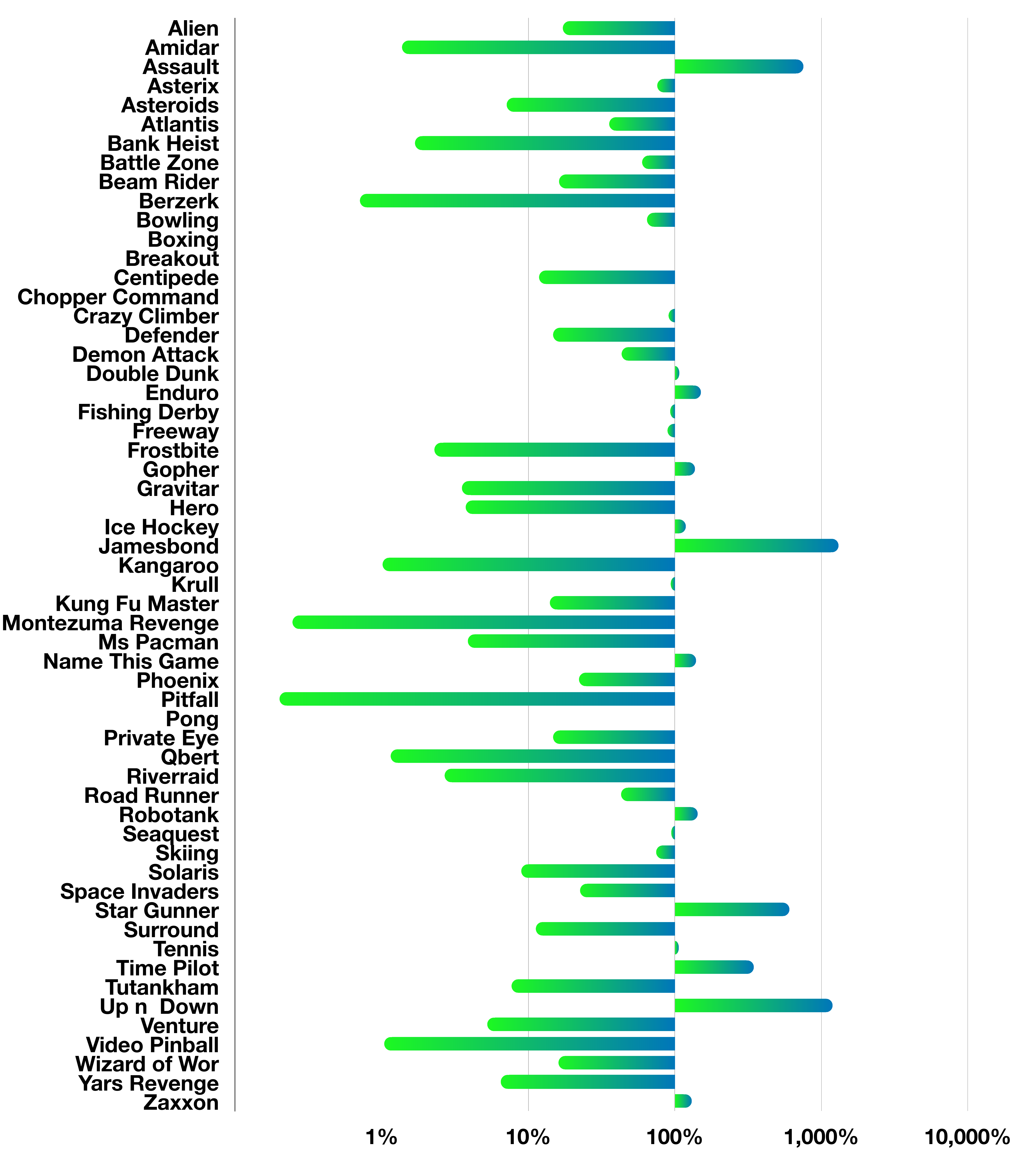}
	}
	\caption{HWRNS (\%) of Atari 57 games using GDI-I$^3$.}
	\label{fig: HWRNS of GDII}
\end{figure*}

\begin{figure*}[!ht]
	\subfigure{
		\includegraphics[width=0.8\textwidth]{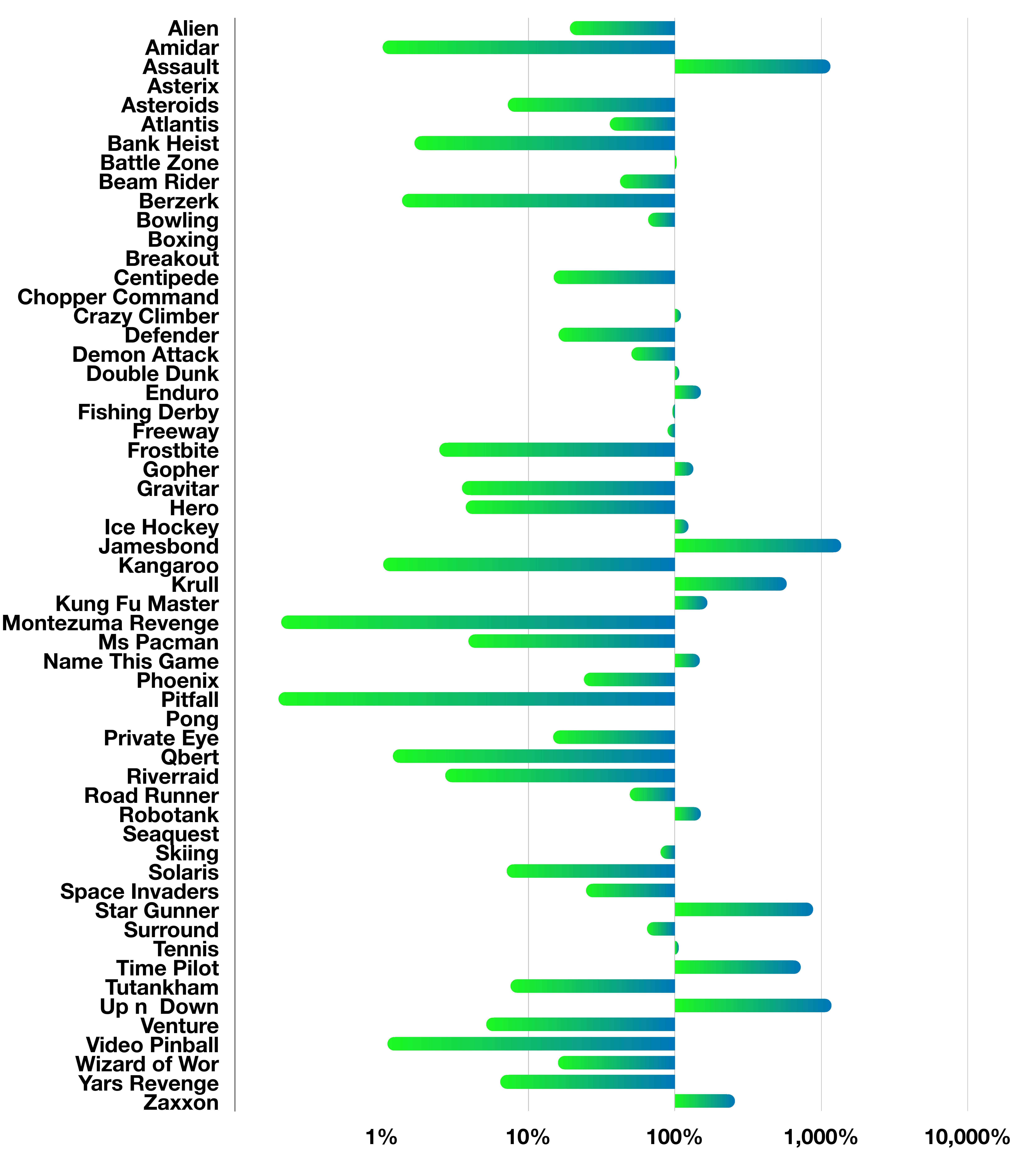}
	}
	\caption{HWRNS (\%) of Atari 57 games using GDI-H$^3$.}
	\label{fig: HWRNS of GDIH}
\end{figure*}

\clearpage

\subsection{Figure of SABER}
\label{app: Figure of SABER}

In this part, we  illustrate the SABER \citep{dreamerv2,atarihuman} using GDI-H$^3$ and  GDI-I$^3$ in all 57 games. The SABER histogram of GDI-I$^3$ is illustrated in Fig. \ref{fig: SABER of GDII}. The SABER histogram of GDI-H$^3$ is illustrated in Fig. \ref{fig: SABER of GDIH}.

\begin{figure*}[!ht]
	\subfigure{
		\includegraphics[width=0.8\textwidth]{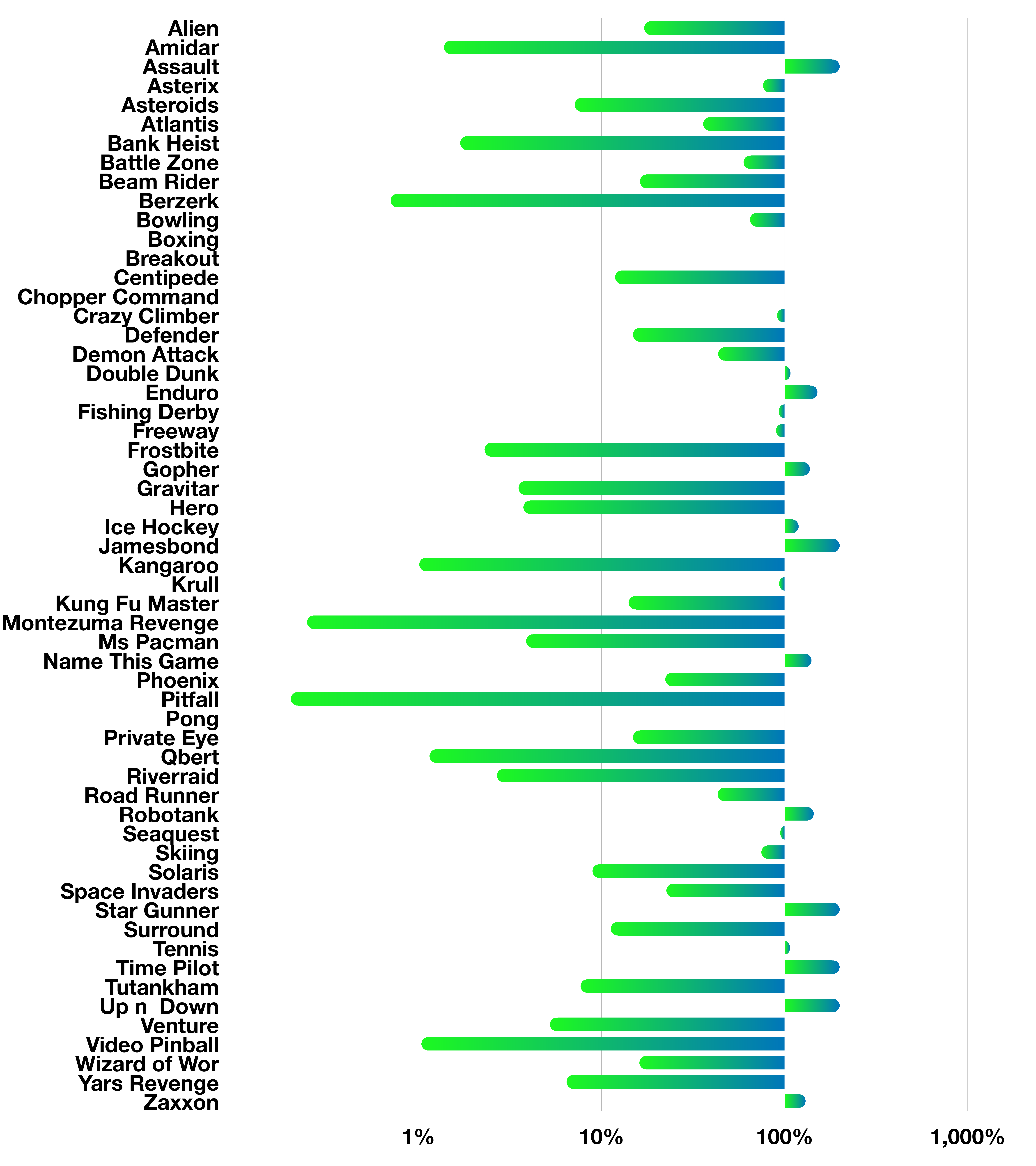}
	}
	\caption{SABER (\%) of Atari 57 games using GDI-I$^3$.}
	\label{fig: SABER of GDII}
\end{figure*}

\begin{figure*}[!ht]
	\subfigure{
		\includegraphics[width=0.8\textwidth]{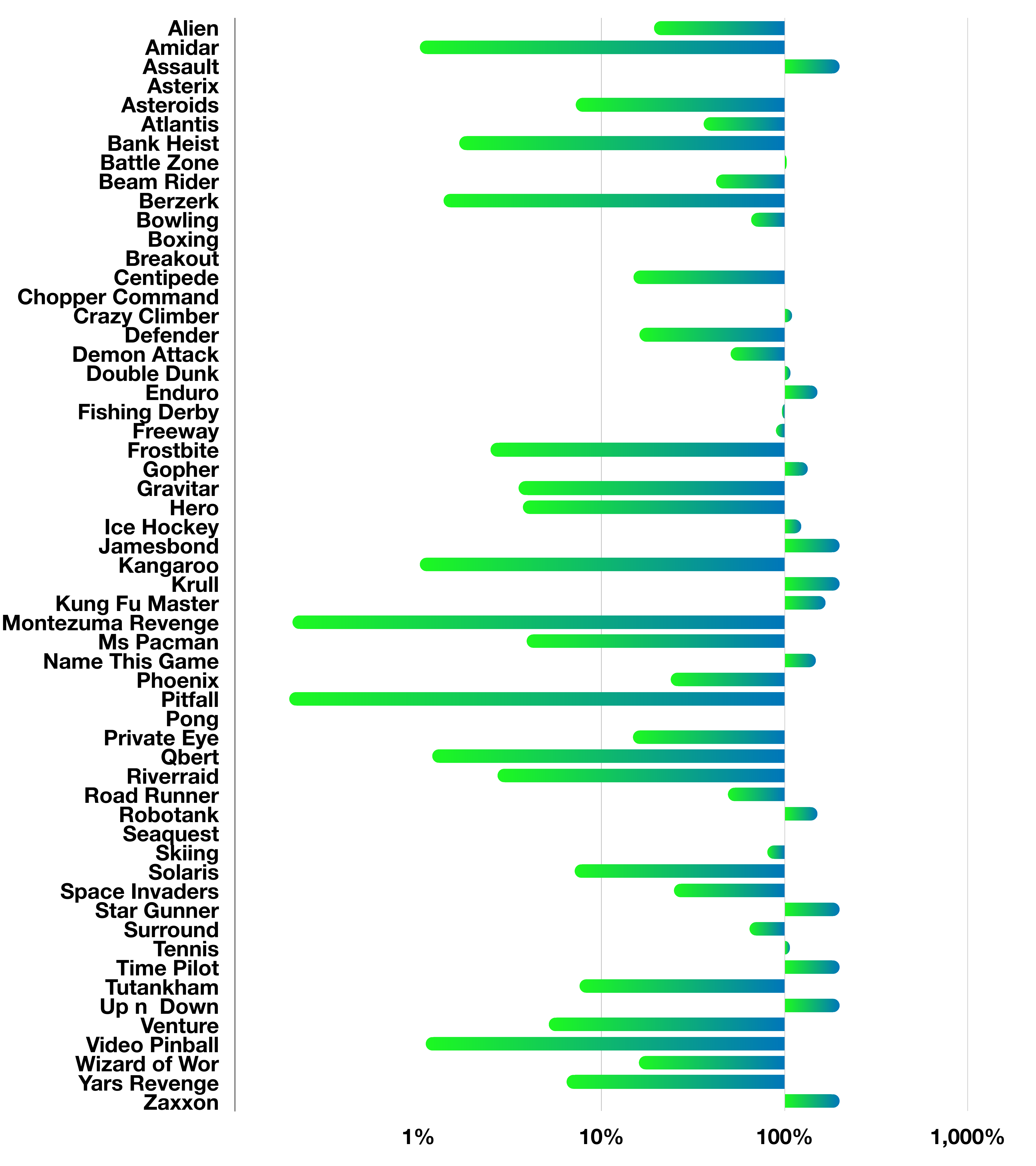}
	}
	\caption{SABER (\%) of Atari 57 games using GDI-H$^3$. }
	\label{fig: SABER of GDIH}
\end{figure*}
\clearpage

\subsection{Atari Games Table of Scores Based on Human Average Records}
\label{app: Atari Games Table of Scores Based on Human Average Records}
In this part, we detail the raw score of several representative SOTA algorithms , including the SOTA 200M model-free algorithms, SOTA 10B+ model-free algorithms, SOTA model-based algorithms and other SOTA algorithms.\footnote{200M and 10B+ represent the training scale.} Additionally, we calculate the Human Normalized Score (HNS) of each game with each algorithm. First of all, we demonstrate the sources of the scores that we used.
Random scores and average human's scores are from \citep{agent57}.
Rainbow's scores are from \citep{rainbow}.
IMPALA's scores are from \citep{impala}.
LASER's scores are from \citep{laser}, with no sweep at 200M.
As there are many versions of R2D2 and NGU, we use original papers'.
R2D2's scores are from \citep{r2d2}.
NGU's scores are from \citep{ngu}.
Agent57's scores are from \citep{agent57}.
MuZero's scores are from \citep{muzero}.
DreamerV2's scores are from \citep{dreamerv2}.
SimPLe's scores are from \citep{modelbasedatari}.
Go-Explore's scores are from \citep{goexplore}.
Muesli's scores are from \citep{muesli}.
In the following,  we detail the raw scores and HNS of each algorithm on 57 Atari games.

\begin{table}[!hb]
\footnotesize
\begin{center}
\caption{Score table of SOTA 200M model-free algorithms on HNS(\%) (GDI-I$^3$).}
\setlength{\tabcolsep}{1.0pt}
\begin{tabular}{ |c| c| c| c c| c c| c c| c c| }
\hline
Games & RND & HUMAN & RAINBOW & HNS & IMPALA & HNS & LASER & HNS & GDI-I$^3$ & HNS \\
\hline
Scale  &     &       & 200M   &       &  200M    &        & 200M   &         &  200M   &  \\
\hline
 Alien  & 227.8 & 7127.8            & 9491.7 & 134.26 & 15962.1  & 228.03 & 35565.9 & 512.15                        &43384             &625.45             \\
 Amidar & 5.8   & 1719.5            & \textbf{5131.2} & \textbf{299.08} & 1554.79  & 90.39  & 1829.2  & 106.4       &1442              &83.81                \\
 Assault & 222.4 & 742              & 14198.5 & 2689.78 & 19148.47 & 3642.43  & 21560.4 & 4106.62                   &63876             &12250.50          \\
 Asterix & 210   & 8503.3           & 428200 & 5160.67 & 300732   & 3623.67  & 240090  & 2892.46                    &759910            &9160.41            \\
 Asteroids & 719 & 47388.7          & 2712.8 & 4.27   & 108590.05 & 231.14  & 213025  &  454.91                     &751970            &1609.72          \\
 Atlantis & 12850 & 29028.1         & 826660 & 5030.32 & 849967.5 & 5174.39 & 841200 & 5120.19                      &3803000           &23427.66          \\
 Bank Heist & 14.2 & 753.1          & 1358   & 181.86  & 1223.15  & 163.61  & 569.4  & 75.14                        &\best{1401}       &\best{187.68}       \\
 Battle Zone & 236 & 37187.5        & 62010 & 167.18  & 20885    & 55.88  & 64953.3 & 175.14                        &478830            &1295.20            \\
 Beam Rider & 363.9 & 16926.5       & 16850.2 & 99.54 & 32463.47 & 193.81 & 90881.6 & 546.52                        &162100            &976.51             \\
 Berzerk & 123.7 & 2630.4           & 2545.6   & 96.62  & 1852.7   & 68.98  & \textbf{25579.5} & \textbf{1015.51}   &7607              &298.53              \\
 Bowling & 23.1 & 160.7             & 30   & 5.01        & 59.92    & 26.76  & 48.3    & 18.31                      &201.9             &129.94              \\
 Boxing  & 0.1  & 12.1              & 99.6 & 829.17      & 99.96    & 832.17 & \textbf{100} & \textbf{832.5}        &\best{100}        &\best{832.50}        \\
 Breakout & 1.7 & 30.5              & 417.5 & 1443.75    & 787.34   & 2727.92 & 747.9 & 2590.97                     &\best{864}        &\best{2994.10}      \\
 Centipede & 2090.9 & 12017         & 8167.3 & 61.22   & 11049.75 & 90.26   & \textbf{292792} & \textbf{2928.65}    &155830            &1548.84             \\
 Chopper Command & 811 & 7387.8     & 16654 & 240.89 & 28255  & 417.29  & 761699 & 11569.27                         &\best{999999}     &\best{15192.62}     \\
 Crazy Climber & 10780.5 & 36829.4  & 168788.5 & 630.80 & 136950 & 503.69 & 167820  & 626.93                        &201000            &759.39           \\
 Defender & 2874.5 & 18688.9        & 55105 & 330.27 & 185203 & 1152.93 & 336953  & 2112.50                         &893110            &5629.27           \\
 Demon Attack & 152.1 & 1971        & 111185 & 6104.40 & 132826.98 & 7294.24 & 133530 & 7332.89                     &675530                 &37131.12       \\
 Double Dunk & -18.6 & -16.4        & -0.3   & 831.82  & -0.33     & 830.45  & 14     & 1481.82                     &\best{24}         &\best{1936.36}      \\
 Enduro      & 0   & 860.5          & 2125.9 & 247.05  & 0       & 0.00     & 0    & 0.00                           &\best{14330}      &\best{1665.31 }    \\
 Fishing Derby & -91.7 & -38.8      & 31.3 & 232.51  & 44.85   & 258.13    & 45.2   & 258.79                        &59         &285.71                 \\
 Freeway       & 0     & 29.6       & \textbf{34} & \textbf{114.86}  & 0     & 0.00       & 0    & 0.00             &\best{34}          &\best{114.86   }  \\
 Frostbite     & 65.2  & 4334.7     & 9590.5 & 223.10 & 317.75 & 5.92     & 5083.5 & 117.54                         &10485              &244.05           \\
 Gopher  & 257.6 & 2412.5           & 70354.6 & 3252.91    & 66782.3 & 3087.14 & 114820.7 & 5316.40                 &\best{488830}     &\best{22672.63}    \\
 Gravitar & 173 & 3351.4            & 1419.3  & 39.21   & 359.5      & 5.87    & 1106.2   & 29.36                   &5905       &180.34                   \\
 Hero   & 1027 & 30826.4            & \textbf{55887.4} & \textbf{184.10}   & 33730.55  & 109.75  & 31628.7 & 102.69 &38330             &125.18              \\
 Ice Hockey & -11.2 & 0.9           & 1.1    & 101.65   & 3.48      & 121.32   & 17.4    & 236.36                   &44.94      &463.97   \\
 Jamesbond  & 29    & 302.8         & 19809 & 72.24   & 601.5     & 209.09   & 37999.8 & 13868.08                   &594500     &217118.70   \\
 Kangaroo   & 52    & 3035          & \textbf{14637.5} & \textbf{488.05} & 1632    & 52.97    & 14308   & 477.91    &14500             &484.34              \\
 Krull     & 1598   & 2665.5        & 8741.5  & 669.18 & 8147.4  & 613.53   & 9387.5  &  729.70                     &97575      &8990.82       \\
 Kung Fu Master & 258.5 & 22736.3   & 52181 & 230.99 & 43375.5 & 191.82 & 607443 & 2701.26        &140440            &623.64                   \\
 Montezuma Revenge&0&\textbf{4753.3}& 384   & 8.08   & 0       & 0.00   & 0.3    & 0.01                             &3000              &63.11           \\
 Ms Pacman  & 307.3 & 6951.6        & 5380.4  & 76.35   & 7342.32 & 105.88 & 6565.5   & 94.19                       &11536      &169.00       \\
 Name This Game & 2292.3 & 8049     & 13136 & 188.37   & 21537.2 & 334.30 & 26219.5 & 415.64                        &34434      &558.34        \\
 Phoenix & 761.5 & 7242.6  & 108529 & 1662.80   & 210996.45  & 3243.82 & 519304 & 8000.84                           &894460     &13789.30    \\
 Pitfall & -229.4 & \textbf{6463.7} & 0      & 3.43      & -1.66      & 3.40    & -0.6   & 3.42                     &0                 &3.43            \\
 Pong    & -20.7  & 14.6   & 20.9   & 117.85    & 20.98      & 118.07  & \textbf{21}     &  \textbf{118.13}         &\best{21   }      &\best{118.13}     \\
 Private Eye & 24.9&\textbf{69571.3}& 4234 & 6.05     & 98.5       & 0.11    & 96.3   & 0.10                        &15100             &21.68                 \\
 Qbert  & 163.9 & 13455.0 & 33817.5 & 253.20   & \textbf{351200.12}  & \textbf{2641.14} & 21449.6 & 160.15          &27800             &207.93         \\
 Riverraid & 1338.5 & 17118.0       & 22920.8 & 136.77 & 29608.05  & 179.15  & \textbf{40362.7} & \textbf{247.31}   &28075             &169.44             \\
 Road Runner & 11.5 & 7845          & 62041   & 791.85 & 57121     & 729.04  & 45289   & 578.00                     &878600            &11215.78         \\
 Robotank   & 2.2   & 11.9          & 61.4   & 610.31    & 12.96     & 110.93  & 62.1    & 617.53                   &108.2             &1092.78            \\
 Seaquest  & 68.4 & 42054.7         & 15898.9 & 37.70    & 1753.2    & 4.01    & 2890.3  & 6.72                     &943910	           &2247.98  \\
 Skiing & -17098  & \textbf{-4336.9}& -12957.8 & 32.44  & -10180.38 & 54.21   & -29968.4 & -100.86                  &-6774             &80.90          \\
 Solaris & 1236.3 & \textbf{12326.7}& 3560.3  & 20.96  & 2365      & 10.18   & 2273.5   & 9.35                      &11074             &88.70             \\
 Space Invaders & 148 & 1668.7      & 18789 & 1225.82 & 43595.78 & 2857.09 & 51037.4 & 3346.45                      &140460            &9226.80            \\
 Star Gunner & 664 & 10250          & 127029    & 1318.22 & 200625   & 2085.97 & 321528  & 3347.21                  &465750            &4851.72             \\
 Surround    & -10 & 6.5            & \textbf{9.7}       & \textbf{119.39}  & 7.56     & 106.42  & 8.4     & 111.52 &-7.8              &13.33           \\
 Tennis  & -23.8   & -8.3           & 0        & 153.55    & 0.55     & 157.10  & 12.2    & 232.26                  &\best{24       }  &\best{308.39   }   \\
 Time Pilot & 3568 & 5229.2         & 12926 & 563.36     & 48481.5  & 2703.84 & 105316  & 6125.34                   &216770     &12834.99           \\
 Tutankham  & 11.4 & 167.6          & 241   & 146.99     & 292.11   & 179.71  & 278.9   & 171.25                    &\best{423.9 }     &\best{264.08   }   \\
 Up N Down  & 533.4 & 11693.2       & 125755 & 1122.08 & 332546.75 & 2975.08 & 345727 & 3093.19                     &\best{986440}     &\best{8834.45 }   \\
 Venture    & 0     & 1187.5        & 5.5    & 0.46    & 0         & 0.00    & 0      & 0.00                        &\best{2035     }  &\best{171.37   }   \\
 Video Pinball & 0 & 17667.9        & 533936.5 & 3022.07 & 572898.27 & 3242.59 & 511835 & 2896.98                   &925830     &5240.18                   \\
 Wizard of Wor & 563.5 & 4756.5     & 17862.5 & 412.57 & 9157.5    & 204.96  & 29059.3 & 679.60                     &\best{64239 }     &\best{1519.90 }    \\
 Yars Revenge & 3092.9 & 54576.9    & 102557 & 193.19 & 84231.14  & 157.60 & 166292.3  & 316.99                     &\textbf{972000}     &\textbf{1881.96}       \\
 Zaxxon       & 32.5   & 9173.3     & 22209.5 & 242.62 & 32935.5   & 359.96 & 41118    & 449.47                     &109140     &1193.63       \\
\hline
MEAN HNS(\%)        &     0.00 & 100.00   &         & \rainbowmeanhns &         & \impalameanhns  &        & \lasermeanhns &      & \GDIImeanhns        \\
\hline
MEDIAN HNS(\%)      & 0.00   & 100.00   &         & \rainbowmedianhns &         &  \impalamedianhns  &        & \lasermedianhns  &      & \GDIImedianhns        \\
\hline
\end{tabular}
\end{center}
\end{table}
\clearpage

\begin{table}[!hb]
\footnotesize
\begin{center}
\caption{Score table of SOTA 200M model-free algorithms on HNS(\%) (GDI-H$^3$).}
\setlength{\tabcolsep}{1.0pt}
\begin{tabular}{ |c| c| c| c c| c c| c c| c c|}
\hline
Games & RND & HUMAN & RAINBOW & HNS & IMPALA & HNS & LASER & HNS  & GDI-H$^3$ & HNS\\
\hline
Scale  &     &       & 200M   &       &  200M    &        & 200M   &         &  200M   &\\
\hline
 Alien  & 227.8 & 7127.8            & 9491.7 & 134.26 & 15962.1  & 228.03 & 35565.9 & 512.15                                 &\textbf{48735}             &\textbf{703.00}      \\
 Amidar & 5.8   & 1719.5            & \textbf{5131.2} & \textbf{299.08} & 1554.79  & 90.39  & 1829.2  & 106.4                 &1065                       &61.81       \\
 Assault & 222.4 & 742              & 14198.5 & 2689.78 & 19148.47 & 3642.43  & 21560.4 & 4106.62                         &\textbf{97155}             &\textbf{18655.23}    \\
 Asterix & 210   & 8503.3           & 428200 & 5160.67 & 300732   & 3623.67  & 240090  & 2892.46                          &\textbf{999999}            &\textbf{12055.38}    \\
 Asteroids & 719 & 47388.7          & 2712.8 & 4.27   & 108590.05 & 231.14  & 213025  &  454.91                          &\textbf{760005 }           &\textbf{1626.94}     \\
 Atlantis & 12850 & 29028.1         & 826660 & 5030.32 & 849967.5 & 5174.39 & 841200 & 5120.19                            &\textbf{3837300}           &\textbf{23639.67}     \\
 Bank Heist & 14.2 & 753.1          & 1358   & 181.86  & 1223.15  & 163.61  & 569.4  & 75.14                         &1380                       &184.84       \\
 Battle Zone & 236 & 37187.5        & 62010 & 167.18  & 20885    & 55.88  & 64953.3 & 175.14                                &\textbf{824360}            &\textbf{2230.29}      \\
 Beam Rider & 363.9 & 16926.5       & 16850.2 & 99.54 & 32463.47 & 193.81 & 90881.6 & 546.52                                &\textbf{422890}            &\textbf{2551.09}      \\
 Berzerk & 123.7 & 2630.4           & 2545.6   & 96.62  & 1852.7   & 68.98  & \textbf{25579.5} & \textbf{1015.51}           &14649             &579.46       \\
 Bowling & 23.1 & 160.7             & 30   & 5.01        & 59.92    & 26.76  & 48.3    & 18.31                             &\textbf{205.2}             &\textbf{132.34}       \\
 Boxing  & 0.1  & 12.1              & 99.6 & 829.17      & 99.96    & 832.17 & \textbf{100} & \textbf{832.5}         &\textbf{100}               &\textbf{832.50}       \\
 Breakout & 1.7 & 30.5              & 417.5 & 1443.75    & 787.34   & 2727.92 & 747.9 & 2590.97                      &\textbf{864}               &\textbf{2994.10}      \\
 Centipede & 2090.9 & 12017         & 8167.3 & 61.22   & 11049.75 & 90.26   & \textbf{292792} & \textbf{2928.65}         &195630                     &1949.80      \\
 Chopper Command & 811 & 7387.8     & 16654 & 240.89 & 28255  & 417.29  & 761699 & 11569.27                         &\textbf{999999}            &\textbf{15192.62}     \\
 Crazy Climber & 10780.5 & 36829.4  & 168788.5 & 630.80 & 136950 & 503.69 & 167820  & 626.93                               &\textbf{241170}            &\textbf{919.76}  \\
 Defender & 2874.5 & 18688.9        & 55105 & 330.27 & 185203 & 1152.93 & 336953  & 2112.50                             &\textbf{970540}            &\textbf{6118.89}     \\
 Demon Attack & 152.1 & 1971        & 111185 & 6104.40 & 132826.98 & 7294.24 & 133530 & 7332.89                       &\textbf{787985}                     &\textbf{43313.70}     \\
 Double Dunk & -18.6 & -16.4        & -0.3   & 831.82  & -0.33     & 830.45  & 14     & 1481.82                      &\textbf{24}                &\textbf{1936.36}     \\
 Enduro      & 0   & 860.5          & 2125.9 & 247.05  & 0       & 0.00     & 0    & 0.00                           &14300                      &1661.82      \\
 Fishing Derby & -91.7 & -38.8      & 31.3 & 232.51  & 44.85   & 258.13    & 45.2   & 258.79                                  &\textbf{65}                &\textbf{296.22}   \\
 Freeway       & 0     & 29.6       & \textbf{34} & \textbf{114.86}  & 0     & 0.00       & 0    & 0.00             &\textbf{34}               &\textbf{114.86}   \\
 Frostbite     & 65.2  & 4334.7     & 9590.5 & 223.10 & 317.75 & 5.92     & 5083.5 & 117.54                                  &\textbf{11330}            &\textbf{263.84}   \\
 Gopher  & 257.6 & 2412.5           & 70354.6 & 3252.91    & 66782.3 & 3087.14 & 114820.7 & 5316.40                  &473560           &21964.01    \\
 Gravitar & 173 & 3351.4            & 1419.3  & 39.21   & 359.5      & 5.87    & 1106.2   & 29.36                                 &\textbf{5915}             &\textbf{180.66}   \\
 Hero   & 1027 & 30826.4            & \textbf{55887.4} & \textbf{184.10}   & 33730.55  & 109.75  & 31628.7 & 102.69          &38225            &124.83    \\
 Ice Hockey & -11.2 & 0.9           & 1.1    & 101.65   & 3.48      & 121.32   & 17.4    & 236.36                            &\textbf{481.90}   \\
 Jamesbond  & 29    & 302.8         & 19809 & 72.24   & 601.5     & 209.09   & 37999.8 & 13868.08                   &\textbf{620780}          &\textbf{226716.95}    \\
 Kangaroo   & 52    & 3035          & \textbf{14637.5} & \textbf{488.05} & 1632    & 52.97    & 14308   & 477.91              &14636           &488.00    \\
 Krull     & 1598   & 2665.5        & 8741.5  & 669.18 & 8147.4  & 613.53   & 9387.5  &  729.70                        &\textbf{594540}          &\textbf{55544.92}     \\
 Kung Fu Master & 258.5 & 22736.3   & 52181 & 230.99 & 43375.5 & 191.82 & 607443 & 2701.26                  &\textbf{1666665}          &\textbf{7413.57}         \\
 Montezuma Revenge&0&\textbf{4753.3}& 384   & 8.08   & 0       & 0.00   & 0.3    & 0.01                                       &2500            &52.60   \\
 Ms Pacman  & 307.3 & 6951.6        & 5380.4  & 76.35   & 7342.32 & 105.88 & 6565.5   & 94.19                         &\textbf{11573}           &\textbf{169.55}    \\
 Name This Game & 2292.3 & 8049     & 13136 & 188.37   & 21537.2 & 334.30 & 26219.5 & 415.64                          &\textbf{36296}           &\textbf{590.68}    \\
 Phoenix & 761.5 & 7242.6  & 108529 & 1662.80   & 210996.45  & 3243.82 & 519304 & 8000.84                            &\textbf{959580}          &\textbf{14794.07}  \\
 Pitfall & -229.4 & \textbf{6463.7} & 0      & 3.43      & -1.66      & 3.40    & -0.6   & 3.42                                 &-4.345            &3.36 \\
 Pong    & -20.7  & 14.6   & 20.9   & 117.85    & 20.98      & 118.07  & \textbf{21}     &  \textbf{118.13}             &\textbf{21}              &\textbf{118.13}    \\
 Private Eye & 24.9&\textbf{69571.3}& 4234 & 6.05     & 98.5       & 0.11    & 96.3   & 0.10                                  &15100           &21.68       \\
 Qbert  & 163.9 & 13455.0 & 33817.5 & 253.20   & \textbf{351200.12}  & \textbf{2641.14} & 21449.6 & 160.15                 &28657           &214.38    \\
 Riverraid & 1338.5 & 17118.0       & 22920.8 & 136.77 & 29608.05  & 179.15  & \textbf{40362.7} & \textbf{247.31}            &28349           &171.17    \\
 Road Runner & 11.5 & 7845          & 62041   & 791.85 & 57121     & 729.04  & 45289   & 578.00                            &\textbf{999999} &\textbf{12765.53} \\
 Robotank   & 2.2   & 11.9          & 61.4   & 610.31    & 12.96     & 110.93  & 62.1    & 617.53                          &\textbf{113.4}           &\textbf{1146.39}  \\
 Seaquest  & 68.4 & 42054.7         & 15898.9 & 37.70    & 1753.2    & 4.01    & 2890.3  & 6.72                     &\textbf{1000000}          &\textbf{2381.57}\\
 Skiing & -17098  & \textbf{-4336.9}& -12957.8 & 32.44  & -10180.38 & 54.21   & -29968.4 & -100.86                             &-6025	          &86.77   \\
 Solaris & 1236.3 & \textbf{12326.7}& 3560.3  & 20.96  & 2365      & 10.18   & 2273.5   & 9.35                                &9105            &70.95   \\
 Space Invaders & 148 & 1668.7      & 18789 & 1225.82 & 43595.78 & 2857.09 & 51037.4 & 3346.45                               &\textbf{154380} &\textbf{10142.17}   \\
 Star Gunner & 664 & 10250          & 127029    & 1318.22 & 200625   & 2085.97 & 321528  & 3347.21                           &\textbf{677590} &\textbf{7061.61}    \\
 Surround    & -10 & 6.5            & \textbf{9.7}       & \textbf{119.39}  & 7.56     & 106.42  & 8.4     & 111.52            &2.606           &76.40    \\
 Tennis  & -23.8   & -8.3           & 0        & 153.55    & 0.55     & 157.10  & 12.2    & 232.26                  &\textbf{24}              &\textbf{308.39}  \\
 Time Pilot & 3568 & 5229.2         & 12926 & 563.36     & 48481.5  & 2703.84 & 105316  & 6125.34                            &\textbf{450810}          &\textbf{26924.45}   \\
 Tutankham  & 11.4 & 167.6          & 241   & 146.99     & 292.11   & 179.71  & 278.9   & 171.25                    &418.2           &260.44  \\
 Up N Down  & 533.4 & 11693.2       & 125755 & 1122.08 & 332546.75 & 2975.08 & 345727 & 3093.19                       &966590          &8656.58    \\
 Venture    & 0     & 1187.5        & 5.5    & 0.46    & 0         & 0.00    & 0      & 0.00                         &2000            &168.42   \\
 Video Pinball & 0 & 17667.9        & 533936.5 & 3022.07 & 572898.27 & 3242.59 & 511835 & 2896.98                                  &\textbf{978190} &\textbf{5536.54}     \\
 Wizard of Wor & 563.5 & 4756.5     & 17862.5 & 412.57 & 9157.5    & 204.96  & 29059.3 & 679.60                       &63735           &1506.59     \\
 Yars Revenge & 3092.9 & 54576.9    & 102557 & 193.19 & 84231.14  & 157.60 & 166292.3  & 316.99                        &968090          &1874.36     \\
 Zaxxon       & 32.5   & 9173.3     & 22209.5 & 242.62 & 32935.5   & 359.96 & 41118    & 449.47                       &\textbf{216020} &\textbf{2362.89}     \\
\hline
MEAN HNS(\%)        &     0.00 & 100.00   &         & \rainbowmeanhns &         & \impalameanhns  &        & \lasermeanhns      &      & \textbf{\GDIHmeanhns} \\
\hline
MEDIAN HNS(\%)      & 0.00   & 100.00   &         & \rainbowmedianhns &         &  \impalamedianhns  &        & \lasermedianhns         &      & \textbf{\GDIHmedianhns} \\
\hline
\end{tabular}
\end{center}
\end{table}
\clearpage

\begin{table}[!hb]
\footnotesize
\begin{center}
\caption{Score table of 10B+ SOTA  model-free algorithms on HNS(\%).}
\label{Tab:Score table of SOTA  model-free algorithms on HNS.}
\setlength{\tabcolsep}{1.0pt}
\begin{tabular}{| c| c c| c c| c c| c c| c c|}
\hline
 Games & R2D2 & HNS & NGU & HNS & AGENT57 & HNS & GDI-I$^3$ & HNS & GDI-H$^3$ & HNS \\
\hline
Scale  & 10B   &        & 35B &         & 100B     &        & 200M &     &  200M   &\\
\hline
 Alien  & 109038.4 & 1576.97 & 248100 & 3592.35 & \textbf{297638.17} & \textbf{4310.30}             &43384             &625.45                &48735             &703.00       \\
 Amidar & 27751.24 & 1619.04 & 17800  & 1038.35 & \textbf{29660.08}  & \textbf{1730.42}             &1442              &83.81                 &1065              &61.81        \\
 Assault &  90526.44 & 17379.53 & 34800 & 6654.66 & 67212.67 & 12892.66            &63876           &12250.50          &\textbf{97155}        &\textbf{18655.23}     \\
 Asterix &  999080   & 12044.30 & 950700 & 11460.94 & 991384.42 & 11951.51         &759910          &9160.41           &\textbf{999999}       &\textbf{12055.38}     \\
 Asteroids & 265861.2 & 568.12 & 230500 & 492.36   & 150854.61 & 321.70                             &751970            &1609.72               &\textbf{760005}            &\textbf{1626.94}      \\
 Atlantis & 1576068   & 9662.56 & 1653600 & 10141.80 & 1528841.76 & 9370.64                         &3803000           &23427.66              &\textbf{3837300}           &\textbf{23639.67}     \\
 Bank Heist & \textbf{46285.6} & \textbf{6262.20} & 17400   & 2352.93  & 23071.5& 3120.49           &1401              &187.68                &1380              &184.84       \\
 Battle Zone & 513360 & 1388.64 & 691700  & 1871.27  & \textbf{934134.88}& \textbf{2527.36}         &478830            &1295.20               &824360            &2230.29      \\
 Beam Rider & 128236.08 & 772.05 & 63600  & 381.80   & 300509.8 & 1812.19         &162100            &976.51                &\textbf{422390}            &\textbf{2548.07}      \\
 Berzerk & 34134.8      & 1356.81 & 36200 & 1439.19  & \textbf{61507.83} & \textbf{2448.80}         &7607              &298.53                &14649             &579.46       \\
 Bowling & 196.36       & 125.92  & 211.9 & 137.21   & \textbf{251.18}   & \textbf{165.76}          &201.9             &129.94                &205.2             &132.34       \\
 Boxing  & 99.16        & 825.50  & 99.7  & 830.00   & \textbf{100}      & \textbf{832.50}          &\best{100}        &\best{832.50}         &\textbf{100}               &\textbf{832.50}       \\
 Breakout & 795.36      & 2755.76 & 559.2 & 1935.76  & 790.4 & 2738.54                  &\best{864}        &\best{2994.10}                    &\textbf{864}             &\textbf{2994.10}      \\
 Centipede & 532921.84  & 5347.83 & \textbf{577800} & \textbf{5799.95} & 412847.86& 4138.15         &155830            &1548.84               &195630            &1949.80\\
 Chopper Command&960648&14594.29&999900&15191.11&999900&15191.11                                    &\best{999999}&\best{15192.62}            &\textbf{999999}            &\textbf{15192.62}\\
 Crazy Climber & 312768   & 1205.59  & 313400 & 1208.11&\textbf{565909.85}&\textbf{2216.18}         &201000            &759.39                &241170            &919.76\\
 Defender & 562106        & 3536.22  & 664100 & 4181.16  & 677642.78 & 4266.80                      &893110     &5629.27                      &\textbf{970540}            &\textbf{6118.89}\\
 Demon Attack & 143664.6  & 7890.07  & 143500 & 7881.02  & 143161.44 & 7862.41                      &675530     &37131.12       &\textbf{787985}                     &\textbf{43313.70}\\
 Double Dunk & 23.12      & 1896.36  & -14.1  & 204.55   & 23.93& 1933.18         &\textbf{24}      &\textbf{1936.36}                         &\textbf{24}                &\textbf{1936.36}\\
 Enduro      & 2376.68    & 276.20   & 2000   & 232.42   & 2367.71   & 275.16                       &\best{14330}      &\best{1665.31}        &14300             &1661.82\\
 Fishing Derby & 81.96    & 328.28   & 32     & 233.84   & \textbf{86.97}& \textbf{337.75}          &59                &285.71                &65               &296.22\\
 Freeway       & \textbf{34}       & \textbf{114.86}   & 28.5   & 96.28    & 32.59& 110.10          &\best{34}         &\best{114.86}         &\textbf{34}               &\textbf{114.86}\\
 Frostbite    & 11238.4  & 261.70   & 206400 & 4832.76&\textbf{541280.88}&\textbf{12676.32}         &10485             &244.05                &11330	           &263.84\\
 Gopher  & 122196        & 5658.66  & 113400 & 5250.47  & 117777.08 & 5453.59                       &\best{488830}     &\best{22672.63}       &473560           &21964.01\\
 Gravitar & 6750         & 206.93   & 14200  & 441/32  &\textbf{19213.96}&\textbf{599.07}           &5905              &180.34                &5915             &180.66\\
 Hero   & 37030.4        & 120.82   & 69400  & 229.44&\textbf{114736.26}&\textbf{381.58}            &38330             &125.18                &38225	            &124.83\\
 Ice Hockey & \textbf{71.56}      & \textbf{683.97}   &-4.1   & 58.68    & 63.64& 618.51            &44.94             &463.97                &47.11           &481.90\\
 Jamesbond  & 23266      & 8486.85  & 26600  & 9704.53  & 135784.96 & 49582.16                      &594500            &217118.70             &\textbf{620780	}          &\textbf{226716.95}\\
 Kangaroo   & 14112      & 471.34   & \textbf{35100}  & \textbf{1174.92}&24034.16& 803.96           &14500             &484.34                &14636           &488.90\\
 Krull     & 145284.8    & 13460.12 & 127400&11784.73& 251997.31&23456.61         &97575             &8990.82                                 &\textbf{594540}          &\textbf{55544.92}\\
 Kung Fu Master & 200176 & 889.40   & 212100 & 942.45   & 206845.82 & 919.07      &140440            &623.64                &\textbf{1666665}	         &\textbf{7413.57}\\
 Montezuma Revenge & 2504 & 52.68   & \textbf{10400}  & \textbf{218.80} &9352.01& 196.75            &3000              &63.11                 &2500            &52.60\\
 Ms Pacman  & 29928.2     & 445.81  & 40800  & 609.44& \textbf{63994.44}&\textbf{958.52}            &11536             &169.00                &11573           &169.55\\
 Name This Game & 45214.8 & 745.61  & 23900  & 375.35&\textbf{54386.77}&\textbf{904.94}             &34434             &558.34                &36296           &590.68\\
 Phoenix & 811621.6       & 125.11  & 959100 &14786.66 &908264.15&14002.29         &894460            &13789.30              &\textbf{959580}	          &\textbf{14794.07}\\
 Pitfall & 0              & 3.43    & 7800   & 119.97&\textbf{18756.01}&\textbf{283.66}             &0                 &3.43                  &-4.3            &3.36\\
 Pong    & \textbf{21}             & \textbf{118.13}  & 19.6   & 114.16   & 20.67& 117.20           &\best{21}         &\best{118.13}         &\textbf{21}     &\textbf{118.13}\\
 Private Eye & 300        & 0.40    & \textbf{100000} & \textbf{143.75}& 79716.46&114.59            &15100             &21.68                 &15100           &21.68\\
 Qbert  & 161000          & 1210.10 & 451900 & 3398.79&\textbf{580328.14}&\textbf{4365.06}          &27800             &207.93                &28657           &214.38\\
 Riverraid & 34076.4      & 207.47  & 36700  & 224.10 & \textbf{63318.67}&\textbf{392.79}           &28075             &169.44                &28349           &171.17\\
 Road Runner & 498660     & 6365.59 & 128600 & 1641.52  & 243025.8&3102.24                          &878600            &11215.78       &\textbf{999999} &\textbf{12765.53}\\
 Robotank   & \textbf{132.4}       & \textbf{1342.27} & 9.1    & 71.13 &127.32 &1289.90             &108.2             &1092.78               &113.4           &1146.39\\
 Seaquest  & 999991.84    & 2381.55 & \textbf{1000000} & \textbf{2381.57}&999997.63&2381.56         &943910	           &2247.98        &\textbf{1000000}          &\textbf{2381.57}\\
 Skiing & -29970.32       & -100.87 & -22977.9 & -46.08 & \textbf{-4202.6}  &\textbf{101.05}        &-6774             &80.90                 &-6025	       &86.77\\
 Solaris & 4198.4         & 26.71   & 4700     & 31.23  & \textbf{44199.93}& \textbf{387.39}        &11074             &88.70                 &9105            &70.95\\
 Space Invaders & 55889   & 3665.48 & 43400    & 2844.22 & 48680.86 & 3191.48                       &140460            &9226.80               &\textbf{154380} &\textbf{10142.17}\\
 Star Gunner & 521728     & 5435.68 & 414600   &4318.13&\textbf{839573.53}&\textbf{8751.40}         &465750            &4851.72               &677590	         &7061.61\\
 Surround    & \textbf{9.96}       & \textbf{120.97}  & -9.6     & 2.42    & 9.5&118.18             &-7.8              &13.33                 &2.606           &76.40\\
 Tennis  & \textbf{24}             & \textbf{308.39}  & 10.2     & 219.35 & 23.84& 307.35           &\best{24}         &\best{308.39}         &\textbf{24}     &\textbf{308.39}             \\
 Time Pilot & 348932    & 20791.28 & 344700 & 20536.51&405425.31&24192.24         &216770            &12834.99              &\textbf{450810}	          &\textbf{26924.45}\\
 Tutankham  & 393.64    & 244.71   & 191.1   & 115.04   & \textbf{2354.91}&\textbf{1500.33}         &423.9             &264.08                &418.2           &260.44\\
 Up N Down  & 542918.8  & 4860.17  & 620100  & 5551.77  & 623805.73 & 5584.98                       &\best{986440}     &\best{8834.45}        &966590          &8656.58\\
 Venture    & 1992      & 167.75   & 1700    & 143.16   &\textbf{2623.71}  &\textbf{220.94}         &2035              &171.37                &2000	            &168.42\\
 Video Pinball & 483569.72 & 2737.00 & 965300 & 5463.58 &\textbf{992340.74}&\textbf{5616.63}        &925830            &5240.18               &978190          &5536.54\\
 Wizard of Wor & 133264 & 3164.81  & 106200  & 2519.35  &\textbf{157306.41}&\textbf{3738.20}        &64293             &1519.90               &63735           &1506.59\\
 Yars Revenge & 918854.32 & 1778.73 & 986000 & 1909.15  &\textbf{998532.37}&\textbf{1933.49}        &972000            &1881.96               &968090          &1874.36\\
 Zaxxon & 181372        & 1983.85  & 111100  & 1215.07  &\textbf{249808.9} &\textbf{2732.54}        &109140            &1193.63               &216020	         &2362.89\\
\hline
MEAN HNS(\%)            &               & \rtdtmeanhns  &         &  \ngumeanhns &           & \agentmeanhns  &     & \GDIImeanhns &      & \textbf{\GDIHmeanhns} \\
\hline
MEDIAN HNS(\%) &            &\rtdtmedianhns   &         & \ngumedianhns  &           & \textbf{\agentmedianhns}  &     & \GDIImedianhns &      & \GDIHmedianhns \\
\hline
\end{tabular}
\end{center}
\end{table}
\clearpage

\begin{table}[!hb]
\footnotesize
\begin{center}
\caption{Score table of SOTA model-based algorithms on HNS(\%). SimPLe \citep{modelbasedatari} and DreamerV2\citep{dreamerv2} haven't evaluated all 57 Atari Games in their paper. For fairness, we set the score on those games as N/A, which will not be considered when calculating the median and mean HNS.}
\label{Tab:Score table of SOTA model-based algorithms on HNS.}
\setlength{\tabcolsep}{1.0pt}
\begin{tabular}{|c |c c| c c| c c| c c| c c|}
\hline
 Games              & MuZero         & HNS      & DreamerV2 & HNS    & SimPLe             & HNS          & GDI-I$^3$     & HNS & GDI-H$^3$ & HNS \\
\hline
Scale               & 20B            &              & 200M      &            & 1M               &                  & 200M     &      &  200M   &\\
\hline  
 Alien              & \textbf{741812.63}      & \textbf{10747.61}     &3483       & 47.18      &616.9     & 5.64    & 43384       & 625.45                    &48735	             &703.00      \\
 Amidar             & \textbf{28634.39 }      & \textbf{1670.57    }  &2028       & 118.00     &74.3      & 4.00    & 1442        & 83.81                     &1065              &61.81 \\
 Assault            & \textbf{143972.03}      & \textbf{27665.44}     &7679       & 1435.07    &527.2     & 58.66   & 63876       & 12250.50                  &97155	             &18655.23\\
 Asterix            & 998425                  & 12036.40     &25669      & 306.98     &1128.3    & 11.07   & 759910      & 9160.41                   &\textbf{999999}   &\textbf{12055.38} \\
 Asteroids          & 678558.64             & 1452.42      &3064       & 5.02       &793.6     & 0.16               &751970& 1609.72           &\textbf{760005}            &\textbf{1626.94}  \\
 Atlantis           & 1674767.2             & 10272.64     &989207     & 6035.05    &20992.5   & 50.33              &3803000&23427.66          &\textbf{3837300}           &\textbf{23639.67}   \\
 Bank Heist         & 1278.98               & 171.17       &1043       & 139.23     &34.2      & 2.71               &\best{1401}  & \best{187.68  }           &1380              &184.84 \\
 Battle Zone        & \textbf{848623         }& \textbf{2295.95}      &31225      & 83.86      &4031.2    & 10.27   & 478830      & 1295.20                   &824360            &2230.29\\
 Beam Rider         & \textbf{454993.53}      & \textbf{2744.92}      &12413      & 72.75      &621.6     & 1.56    & 162100      & 976.51                    &422390            &2548.07\\
 Berzerk            & \textbf{85932.6        }& \textbf{3423.18}      &751        & 25.02      &N/A       & N/A     & 7607        & 298.53                    &14649             &579.46\\
 Bowling            & \textbf{260.13         }& \textbf{172.26 }      &48         & 18.10      &30        & 5.01    & 202         & 129.94                    &205.2             &132.34\\
 Boxing             & \textbf{100}                   & \textbf{832.50}       &87         & 724.17     &7.8       & 64.17   & \best{100}  & \best{832.50  }    &\textbf{100}      &\textbf{832.50}  \\
 Breakout           & \textbf{864}                   & \textbf{2994.10}      &350        & 1209.38    &16.4      & 51.04   & \best{864}  & \best{2994.10    } &\textbf{864}      &\textbf{2994.10}\\
 Centipede          & \textbf{1159049.27}     & \textbf{11655.72}     &6601       & 45.44      &N/A       & N/A     & 155830      & 1548.84                   &195630            &1949.80\\
 Chopper Command    & 991039.7              & 15056.39     &2833       & 30.74      & 979.4    & 2.56    & \best{999999}& \best{15192.62}                     &\textbf{999999}   &\textbf{15192.62}\\
 Crazy Climber      & \textbf{458315.4}       & \textbf{1786.64    }  &141424     & 521.55     & 62583.6  & 206.81  & 201000      & 759.39                    &241170	            &919.76\\
 Defender           & 839642.95             & 5291.18      & N/A       & N/A        & N/A      & N/A     & 893110      & 5629.27               &\textbf{970540}   &\textbf{6118.89}\\
 Demon Attack       & 143964.26             & 7906.55      & 2775     &144.20      & 208.1    & 3.08    & 675530      & 37131.12                &\textbf{787985}                     &\textbf{43313.70}\\
 Double Dunk        & 23.94          & 1933.64      & 22        &1845.45     & N/A      & N/A     & \textbf{24}          & \textbf{1936.36}                   &\textbf{24 }      &\textbf{1936.36}\\
 Enduro             & 2382.44               & 276.87       & 2112     &245.44      & N/A      & N/A     & \best{14330}       & \best{1665.31}                 &14300             &1661.82\\
 Fishing Derby      & \textbf{91.16}          & \textbf{345.67     }  & 60        &286.77      &-90.7     & 1.89    & 59          & 285.71                    &65               &296.22\\
 Freeway            & 33.03                 & 111.59       & \textbf{34}        &\textbf{114.86}      &16.7      & 56.42   & \best{34}      & \best{114.86 }  &\textbf{34}      &\textbf{114.86}\\
 Frostbite          & \textbf{631378.53}      & \textbf{14786.59}     & 15622    &364.37      &236.9     & 4.02    & 10485       & 244.05                     &11330	            &263.84\\
 Gopher             & 130345.58             & 6036.85      & 53853    &2487.14     &596.8     & 15.74   & \best{488830}      & \best{22672.6}                 &473560           &21964.01\\
 Gravitar           & \textbf{6682.7     }    & \textbf{204.81     }  & 3554     &106.37      &173.4     & 0.01    & 5905        & 180.34                     &5915             &180.66\\
 Hero               & \textbf{49244.11}       & \textbf{161.81     }  & 30287    &98.19       &2656.6    & 5.47    & 38330       & 125.18                     &38225	            &124.83\\
 Ice Hockey         & \textbf{67.04      }    & \textbf{646.61     }  & 29        &332.23      &-11.6     & -3.31   & 44.94       & 463.97                    &47.11           &481.90\\
 Jamesbond          & 41063.25              & 14986.94     &  9269     &3374.73     &100.5     & 26.11   & 594500      & 217118.70              &\textbf{620780	}  &\textbf{226716.95}\\
 Kangaroo           & \textbf{16763.6        }& \textbf{560.23     }  & 11819     &394.47      &51.2      & -0.03   & 14500       & 484.34                    &14636           &488.90\\
 Krull              & 269358.27      & 25082.93     & 9687     &757.75      &2204.8    & 56.84   & 97575       & 8990.82                    &\textbf{594540}          &\textbf{55544.92}\\
 Kung Fu Master     & 204824         & 910.08       & 66410    &294.30      &14862.5   & 64.97   & 140440      & 623.64                     &\textbf{1666665}	          &\textbf{7413.57}\\
 Montezuma Revenge  & 0                     & 0.00         & 1932     &40.65       &N/A       & N/A     & \best{3000}        & \best{63.11  }                 &2500            &52.60\\
 Ms Pacman          & \textbf{243401.1 }      & \textbf{3658.68    }  & 5651     &80.43       &1480      & 17.65   & 11536       & 169.00                     &11573           &169.55\\
 Name This Game     & \textbf{157177.85}      & \textbf{2690.53    }  & 14472    &211.57      &2420.7    & 2.23    & 34434       & 558.34                     &36296           &590.68\\
 Phoenix            & 955137.84      & 14725.53     & 13342     &194.11      &N/A       & N/A     & 894460      & 13789.30                  &\textbf{959580 }         &	\textbf{14794.07}\\
 Pitfall            &\textbf{ 0}                     & \textbf{3.43}         & -1        &3.41        &N/A       & N/A     & \best{0}       & \best{3.43  }   &-4.3            &3.36\\
 Pong               & \textbf{21}                    & \textbf{118.13}       & 19        &112.46      & 12.8     & 94.90   & \best{21}      & \best{118.13}   &\textbf{21}     &\textbf{118.13}   \\
 Private Eye        & \textbf{15299.98 }      & \textbf{21.96  }      & 158       &0.19        & 35       & 0.01    & 15100       & 21.68                     &15100           &21.68\\
 Qbert              & \textbf{72276          }& \textbf{542.56 }      & 162023    &1217.80     & 1288.8   & 8.46    & 27800       & 207.93                    &28657           &214.38\\
 Riverraid          & \textbf{323417.18}      & \textbf{2041.12}      & 16249    &94.49       & 1957.8   & 3.92    & 28075       & 169.44                     &28349           &171.17\\
 Road Runner        & 613411.8              & 7830.48      & 88772    &1133.09     & 5640.6   & 71.86   & 878600      & 11215.78                              &\textbf{999999	} &\textbf{12765.53}\\
 Robotank           & \textbf{131.13}         & \textbf{1329.18}      & 65        &647.42      & N/A      & N/A     & 108         & 1092.78                   &113.4           &1146.39\\
 Seaquest           & 999976.52             & 2381.51      & 45898    &109.15      & 683.3    & 1.46    &943910	           &2247.98             &\textbf{1000000}          &\textbf{2381.57}\\
 Skiing             & -29968.36      & -100.86     & -8187    &69.83       & N/A      & N/A     & -6774       & 80.90                       &\textbf{-6025}  &\textbf{86.77}\\
 Solaris            & 56.62                 & -10.64       & 883       &-3.19       & N/A      & N/A     & \best{11074}       & \best{88.70   }               &9105            &70.95\\
 Space Invaders     & 74335.3               & 4878.50      & 2611      &161.96      & N/A      & N/A     & 140460     & 9226.80                 &\textbf{154380}          &\textbf{10142.17}\\
 Star Gunner        & 549271.7       & 5723.01      & 29219    &297.88      & N/A      & N/A     & 465750      & 4851.72                    &\textbf{677590}          &\textbf{7061.61}\\
 Surround           & \textbf{9.99       }    & \textbf{121.15     }  & N/A       &N/A         & N/A      & N/A     & -7.8          & 13.33                   &2.606           &76.40 \\
 Tennis             & 0       & 153.55  & 23        &301.94      & N/A      & N/A     & \textbf{24}          & \textbf{308.39}                                &\textbf{24}              &\textbf{308.39}  \\
 Time Pilot         & \textbf{476763.9}       & \textbf{28486.90}     & 32404    &1735.96     & N/A      & N/A     & 216770      & 12834.99                   &450810	          &26924.45 \\
 Tutankham          & \textbf{491.48     }    & \textbf{307.35     }  & 238       &145.07      & N/A      & N/A     & 424         & 264.08                    &418.2           &260.44 \\
 Up N Down          & 715545.61             & 6407.03      & 648363   &5805.03     & 3350.3   & 25.24   & \best{986440}      & \best{8834.45}                 &966590          &8656.58\\
 Venture            & 0.4                   & 0.03         & 0         &0.00        & N/A      & N/A     & \best{2035}        & \best{171.37 }                &2000	            &168.42\\
 Video Pinball      & \textbf{981791.88}      & \textbf{5556.92}      & 22218    &125.75      & N/A      & N/A     & 925830      & 5240.18                    &978190          &5536.54\\
 Wizard of Wor      & \textbf{197126         }& \textbf{4687.87}      & 14531    &333.11      & N/A      & N/A     & 64439       & 1523.38                    &63735           &1506.59\\
 Yars Revenge       & 553311.46             & 1068.72      & 20089    &33.01       & 5664.3   & 4.99    & \best{972000}      & \best{1881.96}                 &968090          &1874.36\\
 Zaxxon             & \textbf{725853.9}       & \textbf{7940.46}      & 18295    &199.79      & N/A      & N/A     & 109140      & 1193.63                    &216020	          &2362.89\\
\hline    
MEAN HNS(\%)        &                & \muzeromeanhns     &           &  \dreamermeanhns   &           &\simplemeanhns    &             & \GDIImeanhns&      & \textbf{\GDIHmeanhns} \\
\hline
MEDIAN HNS(\%)      &                & \textbf{\muzeromedianhns}      &           & \dreamermedianhns    &           & \simplemedianhns    &             & \GDIImedianhns &      & \GDIHmedianhns \\
\hline
\end{tabular}
\end{center}
\end{table}
\clearpage

\begin{table}[!hb]
\footnotesize
\begin{center}
\caption{Score table of other SOTA algorithms on HNS(\%). Go-Explore \citep{goexplore} and Muesli \citep{muesli}.}
\label{Tab:Score table of other SOTA algorithms on HNS.}
\setlength{\tabcolsep}{1.0pt}
\begin{tabular}{|c |c c| c c| c c |c c|}            
\hline
 Games        & Muesli & HNS      & Go-Explore              & HNS                     & GDI-I$^3$ & HNS               & GDI-H$^3$ & HNS \\
\hline
Scale         &  200M   &            & 10B                     &                             & 200M              &                &  200M   &\\
\hline
 Alien        &139409          &2017.12                   &\textbf{959312}       &\textbf{13899.77}              & 43384             &625.45            &48735	             &703.00              \\
 Amidar       &\textbf{21653}  &\textbf{1263.18}          &19083                 &1113.22                        & 1442              &83.81             &1065              &61.81           \\
 Assault      &36963           &7070.94                   &30773                 &5879.64                        & 63876      &12250.50   &\textbf{97155	}    &\textbf{18655.23}       \\
 Asterix      &316210          &3810.30                   &999500       &12049.37              & 759910            &9160.41           &\textbf{999999}    &\textbf{12055.38}\\
 Asteroids    &484609          &1036.84                   &112952                &240.48                         & 751970     &1609.72    &\textbf{760005}            &\textbf{1626.94}       \\
 Atlantis     &1363427         &8348.18                   &286460                &1691.24                        & 3803000    &23427.66   &\textbf{3837300}           &\textbf{23639.67}       \\
 Bank Heist   &1213            &162.24                    &\textbf{3668}         &\textbf{494.49}                & 1401              &187.68            &1380              &184.84\\
 Battle Zone  &414107          &1120.04                   &\textbf{998800}       &\textbf{2702.36}               & 478830            &1295.20           &824360            &2230.29\\
 Beam Rider   &288870          &1741.91                   &371723       &2242.15               & 162100            &976.51            &\textbf{422390}   &\textbf{2548.07}\\
 Berzerk      &44478           &1769.43                   &\textbf{131417}       &\textbf{5237.69}               & 7607              &298.53            &14649             &579.46\\
 Bowling      &191             &122.02                    &\textbf{247}           &\textbf{162.72}                & 202               &129.94           &205.2             &132.34\\
 Boxing       &99              &824.17                    &91                     &757.50                         & \best{100}        &\best{832.50}    &\textbf{100}      &\textbf{832.50}        \\
 Breakout     &791             &2740.63                   &774                    &2681.60                        & \best{864}        &\best{2994.10}   &\textbf{864}      &\textbf{2994.10}        \\
 Centipede    &\textbf{869751} &\textbf{8741.20}          &613815                &6162.78               & 155830            &1548.84                    &195630            &1949.80\\
 Chopper Command &101289       &1527.76            &996220                &15135.16                       & \best{999999}     &\best{15192.62}          &\textbf{999999}   &\textbf{15192.62}\\
 Crazy Climber   &175322       &656.88             &235600       &897.52                & 201000            &759.39                   &\textbf{241170}	            &\textbf{919.76}\\
 Defender        &629482       &3962.26            &N/A                    &N/A                            & 893110     &5629.27                        &\textbf{970540}   &\textbf{6118.89}\\
 Demon Attack    &129544       &7113.74            &239895                 &13180.65                       & 675530     &37131.12         &\textbf{787985}                     &\textbf{43313.70}\\
 Double Dunk     &-3           &709.09             &\textbf{24}                     &\textbf{1936.36}                        & \best{24}         &\best{1936.36}          &\textbf{24}       &\textbf{1936.36}\\
 Enduro          &2362         &274.49             &1031                   &119.81                         & \best{14330}      &\best{1665.31}          &14300             &1661.82\\
 Fishing Derby   &51           &269.75             &\textbf{67}            &\textbf{300.00}                & 59                &285.71                  &65               &296.22\\
 Freeway         &33           &111.49             &\textbf{34}            &\textbf{114.86}                & \best{34}         &\best{114.86}           &\textbf{34}        &\textbf{114.86}\\
 Frostbite       &301694       &7064.73            &\textbf{999990}       &\textbf{23420.19}              & 10485             &244.05                   &11330	            &263.84\\
 Gopher          &104441       &4834.72            &134244                &6217.75                        & \best{488830}     &\best{22672.63}          &473560           &21964.01\\
 Gravitar        &11660        &361.41             &\textbf{13385}        &\textbf{415.68}                & 5905              &180.34                   &5915             &180.66\\
 Hero            &37161        &121.26            &37783                  &123.34                         & \textbf{38330}      &\textbf{125.18}            &38225	   &124.83\\
 Ice Hockey      &25           &299.17             &33                     &365.29                         & 44.94         &463.97        &\textbf{47.11}           &\textbf{481.90}    \\
 Jamesbond       &19319        &7045.29            &200810                &73331.26                       & 594500     &217118.70         &\textbf{620780	}          &\textbf{226716.95}\\
 Kangaroo        &14096        &470.80             &\textbf{24300}        &\textbf{812.87}                & 14500             &484.34                   &14636           &488.90\\
 Krull           &34221        &3056.02            &63149                 &5765.90                        & 97575      &8990.82           &\textbf{594540}          &\textbf{55544.92}\\
 Kung Fu Master  &134689       &598.06             &24320                 &107.05                         & 140440     &623.64            &\textbf{1666665	}          &\textbf{7413.57}\\
 Montezuma Revenge  &2359      &49.63               &\textbf{24758}        &\textbf{520.86}                & 3000              &63.11                   &2500            &52.60\\
 Ms Pacman          &65278     &977.84              &\textbf{456123}       &\textbf{6860.25}               & 11536             &169.00                  &11573           &169.55\\
 Name This Game     &105043    &1784.89              &\textbf{212824}       &\textbf{3657.16}               & 34434             &558.34                 &36296           &590.68\\
 Phoenix        &805305        &12413.69                    &19200                 &284.50                  & 894460     &13789.30 &\textbf{959580	}          &\textbf{14794.07}   \\
 Pitfall        &0             &3.43                    &\textbf{7875}          &\textbf{121.09}                & 0                 &3.43               &-4.3            &3.36\\
 Pong           &20            &115.30                  &\textbf{21}            &\textbf{118.13}                & \best{21}         &\best{118.13}      &\textbf{21}              &\textbf{118.13}      \\
 Private Eye    &10323         &14.81                   &\textbf{69976}        &\textbf{100.58}                & 15100             &21.68               &15100           &21.68\\
 Qbert          &157353        &1182.66                 &\textbf{999975}       &\textbf{7522.41}               & 27800             &207.93              &28657           &214.38\\
 Riverraid      &\textbf{47323}&\textbf{291.42}         &35588                 &217.05                & 28075             &169.44                       &28349           &171.17\\
 Road Runner    &327025        &4174.55                 &999900        &12764.26              & 878600            &11215.78           &\textbf{999999}	          &\textbf{12765.53}\\
 Robotank       &59            &585.57                  &\textbf{143}           &\textbf{1451.55}               & 108               &1092.78            &113.4           &1146.39\\
 Seaquest       &815970        &1943.26                 &539456                &1284.68               &943910	           &2247.98               &\textbf{1000000}          &\textbf{2381.57}\\
 Skiing         &-18407        &-10.26                  &\textbf{-4185}        &\textbf{101.19}                & -6774             &80.90               &-6025	         &86.77\\
 Solaris        &3031          &16.18                   &\textbf{20306}        &\textbf{171.95}                & 11074             &88.70               &9105            &70.95\\
 Space Invaders &59602         &3909.65                &93147                 &6115.54                        & 140460     &9226.80       &\textbf{154380}          &\textbf{10142.17}\\
 Star Gunner    &214383        &2229.49                &609580       &6352.14               & 465750     &4851.72                     &\textbf{677590}	          &\textbf{7061.61}\\
 Surround       &\textbf{9}    &\textbf{115.15}                 &N/A                    &N/A                  & -8        &13.33                        &2.606           &76.40\\
 Tennis         &12            &230.97                 &\best{24}              &\best{308.39}                  & \best{24}         &\best{308.39}       &\textbf{24}    &\textbf{308.39}     \\
 Time Pilot     &\textbf{359105} &\textbf{21403.71}                   &183620                &10839.32       & 216770     &12834.99                     &450810	          &26924.45\\
 Tutankham      &252           &154.03                 &\textbf{528}           &\textbf{330.73}                & 424               &264.08              &418.2           &260.44\\
 Up N Down      &649190        &5812.44                &553718                &4956.94                        & \best{986440}     &\best{8834.45}       &966590          &8656.58    \\
 Venture        &2104          &177.18                 &\textbf{3074}         &\textbf{258.86}                & 2035              &171.37               &2000	            &168.42\\
 Video Pinball  &685436        &3879.56                &\textbf{999999}       &\textbf{5659.98}               & 925830            &5240.18              &978190          &5536.54\\
 Wizard of Wor  &93291         &2211.48                &\textbf{199900}       &\textbf{4754.03}               & 64293             &1519.90              &63735           &1506.59\\
 Yars Revenge   &557818        &1077.47                &\textbf{999998}       &\textbf{1936.34}               & 972000            &1881.96              &968090          &1874.36\\
 Zaxxon         &65325         &714.30                 &18340                 &200.28                         & 109140     &1193.63       &\textbf{216020	}          &\textbf{2362.89}    \\
\hline    
MEAN HNS(\%)        & & \mueslimeanhns           &                       & \goexploremeanhns                       &            &  \GDIImeanhns &      & \textbf{\GDIHmeanhns} \\
\hline
MEDIAN HNS(\%)      & & \mueslimedianhns           &                       & \textbf{\goexploremedianhns}              &            & \GDIImedianhns   &      & \GDIHmedianhns \\
\hline
\end{tabular}
\end{center}
\end{table}

\clearpage

\subsection{Atari Games Table of Scores Based on Human World Records}
\label{app: Atari Games Table of Scores Based on Human World Records}
In this part, we detail the raw score of several representative SOTA algorithms , including the SOTA 200M model-free algorithms, SOTA 10B+ model-free algorithms, SOTA model-based algorithms and other SOTA algorithms.\footnote{200M and 10B+ represent the training scale.} Additionally, we calculate the human world records normalized world score (HWRNS) of each game with each algorithm. First of all, we demonstrate the sources of the scores that we used.
Random scores  are from \citep{agent57}.
Human world records (HWR) are from \citep{dreamerv2,atarihuman}.
Rainbow's scores are from \citep{rainbow}.
IMPALA's scores are from \citep{impala}.
LASER's scores are from \citep{laser}, with no sweep at 200M.
As there are many versions of R2D2 and NGU, we use original papers'.
R2D2's scores are from \citep{r2d2}.
NGU's scores are from \citep{ngu}.
Agent57's scores are from \citep{agent57}.
MuZero's scores are from \citep{muzero}.
DreamerV2's scores are from \citep{dreamerv2}.
SimPLe's scores are from \citep{modelbasedatari}.
Go-Explore's scores are from \citep{goexplore}.
Muesli's scores are from \citep{muesli}.
In the following,  we detail the raw scores and HWRNS of each algorithm on 57 Atari games.

\begin{table}[!hb]
\footnotesize
\begin{center}
\caption{Score table of SOTA 200M model-free algorithms on HWRNS(\%)  (GDI-I$^3$).}
\setlength{\tabcolsep}{1.0pt}
\begin{tabular}{|c| c| c| c c| c c| c c| c c| }
\hline
Games               & RND       & HWR       & RAINBOW  & HWRNS & IMPALA  & HWRNS & LASER  & HWRNS & GDI-I$^3$ & HWRNS  \\
\hline
Scale               &           &           & 200M     &           &  200M   &            & 200M    &           &  200M    &            \\
\hline
 Alien              & 227.8     & \textbf{251916}    & 9491.7   &3.68    & 15962.1    & 6.25       & 976.51  & 14.04     &43384             &17.15     \\
 Amidar             & 5.8       & \textbf{104159}    & 5131.2   &4.92    & 1554.79    & 1.49       & 1829.2  & 1.75      &1442              &1.38           \\
 Assault            & 222.4     & 8647             & 14198.5  &165.90  & 19148.47   & 224.65     & 21560.4 & 253.28    &63876        &755.57  \\
 Asterix            & 210       & \textbf{1000000}   & 428200   &42.81   & 300732     & 30.06      & 240090  & 23.99     &759910            &75.99   \\
 Asteroids          & 719       & \textbf{10506650}  & 2712.8   &0.02    & 108590.05  & 1.03       & 213025  & 2.02      &751970            &7.15   \\
 Atlantis           & 12850     & \textbf{10604840}  & 826660   &7.68    & 849967.5   & 7.90       & 841200  & 7.82      &3803000           &35.78  \\
 Bank Heist         & 14.2      & \textbf{82058}     & 1358     &1.64    & 1223.15    & 1.47       & 569.4   & 0.68      &1401              &1.69    \\
 Battle Zone        & 236       & 801000    & 62010    &7.71    & 20885      & 2.58       & 64953.3 & 8.08      &478830            &59.77 \\
 Beam Rider         & 363.9     & \textbf{999999}    & 16850.2  &1.65    & 32463.47   & 3.21       & 90881.6 & 9.06      &162100            &16.18    \\
 Berzerk            & 123.7     & \textbf{1057940}   & 2545.6   &0.23    & 1852.7     & 0.16       & 25579.5 & 2.41      &7607                       &0.71            \\
 Bowling            & 23.1      & \textbf{300}       & 30       &2.49    & 59.92      & 13.30      & 48.3    & 9.10      &201.9             &64.57    \\
 Boxing             & 0.1       & \textbf{100}       & 99.6     &99.60   & 99.96      & 99.96      & \textbf{100}     & \textbf{100.00}    &\best{100}        &\best{100.00 }     \\
 Breakout           & 1.7       & \textbf{864}       & 417.5    &48.22   & 787.34     & 91.11      & 747.9   & 86.54     &\best{864}        &\best{100.00 }  \\
 Centipede          & 2090.9    & \textbf{1301709}   & 8167.3   &0.47    & 11049.75   & 0.69       & 292792  & 22.37     &155830            &11.83          \\
 Chopper Command    & 811       & \textbf{999999}    & 16654    &1.59    & 28255      & 2.75       & 761699  & 76.15     &\best{999999}     &\best{100.00 } \\
 Crazy Climber      & 10780.5   & 219900    & 168788.5 &75.56   & 136950     & 60.33      & 167820  & 75.10     &201000            &90.96          \\
 Defender           & 2874.5    & \textbf{6010500}   & 55105    &0.87    & 185203     & 3.03       & 336953  & 5.56      &893110            &14.82         \\
 Demon Attack       & 152.1     & \textbf{1556345}   & 111185   &7.13    & 132826.98  & 8.53       & 133530  & 8.57      &675530            &43.40          \\
 Double Dunk        & -18.6     & 21        & -0.3     &46.21   & -0.33      & 46.14      & 14      & 82.32     &\best{24}                  &\best{107.58 }   \\
 Enduro             & 0         & 9500      & 2125.9   &22.38   & 0          & 0.00       & 0       & 0.00      &\best{14330}               &\best{150.84  } \\
 Fishing Derby      & -91.7     & \textbf{71}        & 31.3     &75.60   & 44.85      & 83.93      & 45.2    & 84.14     &59                &92.89         \\
 Freeway            & 0         & \textbf{38}        & 34       &89.47   & 0          & 0.00       & 0       & 0.00      &34                &89.47          \\
 Frostbite          & 65.2      & \textbf{454830}    & 9590.5   &2.09    & 317.75     & 0.06       & 5083.5  & 1.10      &10485             &2.29           \\
 Gopher             & 257.6     & 355040    & 70354.6  &19.76   & 66782.3    & 18.75      & 114820.7& 32.29     &\best{488830}              &\best{137.71 } \\
 Gravitar           & 173       & \textbf{162850}    & 1419.3   &0.77    & 359.5      & 0.11       & 1106.2  & 0.57      &5905              &3.52           \\
 Hero               & 1027      & \textbf{1000000}   & 55887.4  &5.49    & 33730.55   & 3.27       & 31628.7 & 3.06      &38330                      &3.73           \\
 Ice Hockey         & -11.2     & 36        & 1.1      &26.06   & 3.48       & 31.10      & 17.4    & 60.59     &44.92              &118.94       \\
 Jamesbond          & 29        & 45550     & 19809    &43.45   & 601.5      & 1.26       & 37999.8 & 83.41     &594500              &1305.93 \\
 Kangaroo           & 52        & \textbf{1424600}   & 14637.5  &1.02    & 1632       & 0.11       & 14308   & 1.00      &14500                      &1.01            \\
 Krull              & 1598      & 104100    & 8741.5   &6.97    & 8147.4     & 6.39       & 9387.5  & 7.60      &97575             &93.63          \\
 Kung Fu Master     & 258.5     & 1000000   & 52181    &5.19    & 43375.5    & 4.31       & 607443  & 60.73     &140440            &14.02          \\
 Montezuma Revenge  &0          & \textbf{1219200}   & 384      &0.03    & 0          & 0.00       & 0.3     & 0.00      &3000              &0.25           \\
 Ms Pacman          & 307.3     & \textbf{290090}    & 5380.4   &1.75    & 7342.32    & 2.43       & 6565.5  & 2.16      &11536             &3.87           \\
 Name This Game     & 2292.3    & 25220     & 13136    &47.30   & 21537.2    & 83.94      & 26219.5 & 104.36    &34434               &140.19     \\
 Phoenix            & 761.5     & \textbf{4014440}   & 108529   &2.69    & 210996.45  & 5.24       & 519304  & 12.92     &894460            &22.27         \\
 Pitfall            & -229.4    & \textbf{114000}    & 0        &0.20    & -1.66      & 0.20       & -0.6    & 0.20      &\best{0    }      &0.20           \\
 Pong               & -20.7     & \textbf{21}        & 20.9     &99.76   & 20.98      & 99.95      & \textbf{21}      & \textbf{100.00}    &\best{21   }      &\best{100.00 }    \\
 Private Eye        & 24.9      & \textbf{101800}    & 4234     &4.14    & 98.5       & 0.07       & 96.3    & 0.07      &15100             &14.81          \\
 Qbert              & 163.9     & \textbf{2400000}   & 33817.5  &1.40    & 351200.12  & 14.63      & 21449.6 & 0.89      &27800             &1.15          \\
 Riverraid          & 1338.5    & \textbf{1000000}   & 22920.8  &2.16    & 29608.05   & 2.83       & 40362.7 & 3.91      &28075             &2.68           \\
 Road Runner        & 11.5      & \textbf{2038100}   & 62041    &3.04    & 57121      & 2.80       & 45289   & 2.22      &878600            &43.11          \\
 Robotank           & 2.2       & 76        & 61.4     &80.22   & 12.96      & 14.58      & 62.1    & 81.17     &108.2               &143.63  \\
 Seaquest           & 68.4      & 999999    & 15898.9  &1.58    & 1753.2     & 0.17       & 2890.3  & 0.28      &943910	             &94.39\\
 Skiing             & -17098    & \textbf{-3272}     & -12957.8 &29.95   & -10180.38  & 50.03      & -29968.4& -93.09    &-6774             &74.67        \\
 Solaris            & 1236.3    & \textbf{111420}    & 3560.3   &2.11    & 2365       & 1.02       & 2273.5  & 0.94      &11074             &8.93          \\
 Space Invaders     & 148       & \textbf{621535 }   & 18789    &3.00    & 43595.78   & 6.99       & 51037.4 & 8.19      &140460            &22.58          \\
 Star Gunner        & 664       & 77400     & 127029   &164.67  & 200625     & 260.58     & 321528  & 418.14    &465750              &606.09  \\
 Surround           & -10       & 9.6       & \textbf{9.7}      &\textbf{100.51}  & 7.56       & 89.59      & 8.4     & 93.88  &-7.8        &11.22          \\
 Tennis             & -23.8     & 21        & 0        &53.13   & 0.55       & 54.35      & 12.2    & 80.36     &\best{24       }           &\best{106.70    }\\
 Time Pilot         & 3568      & 65300     & 12926    &15.16   & 48481.5    & 72.76      & 105316  & 164.82    &216770              &345.37         \\
 Tutankham          & 11.4      & \textbf{5384}      & 241      &4.27    & 292.11     & 5.22       & 278.9   & 4.98      &423.9             &7.68           \\
 Up N Down          & 533.4     & 82840     & 125755   &152.14  & 332546.75  & 403.39     & 345727  & 419.40    &\best{986440}              &\best{1197.85  } \\
 Venture            & 0         & \textbf{38900}     & 5.5      &0.01    & 0          & 0.00       & 0       & 0.00      &2000              &5.23           \\
 Video Pinball      & 0         & \textbf{89218328}  & 533936.5 &0.60    & 572898.27  & 0.64       & 511835  & 0.57      &925830            &1.04          \\
 Wizard of Wor      & 563.5     & \textbf{395300}    & 17862.5  &4.38    & 9157.5     & 2.18       & 29059.3 & 7.22      &64439             &16.14         \\
 Yars Revenge       & 3092.9    & \textbf{15000105}  & 102557   &0.66    & 84231.14   & 0.54       & 166292.3& 1.09      &972000            &6.46           \\
 Zaxxon             & 32.5      & 83700     & 22209.5  &26.51   & 32935.5    & 39.33      & 41118   & 49.11     &109140              &130.41  \\
\hline
MEAN HWRNS(\%)      & 0.00      & 100.00    &          & \rainbowmeanHWRNS  &            & \impalameanHWRNS  &        & \lasermeanHWRNS &      & \GDIImeanHWRNS   \\
\hline   
MEDIAN HWRNS(\%)    & 0.00      & \textbf{100.00}    &          & \rainbowmedianHWRNS   &            & \impalamedianHWRNS &        & \lasermedianHWRNS &      &  \GDIImedianHWRNS  \\
\hline
\end{tabular}
\end{center}
\end{table}
\clearpage

\begin{table}[!hb]
\footnotesize
\begin{center}
\caption{Score table of SOTA 200M model-free algorithms on HWRNS(\%)  (GDI-H$^3$). }
\setlength{\tabcolsep}{1.0pt}
\begin{tabular}{|c| c| c| c c| c c| c c| c c| c c|}
\hline
Games               & RND       & HWR       & RAINBOW  & HWRNS & IMPALA  & HWRNS & LASER  & HWRNS  & GDI-H$^3$ & HWRNS \\
\hline
Scale               &           &           & 200M     &           &  200M   &            & 200M    &                      &    200M   &\\
\hline
 Alien              & 227.8     & \textbf{251916}    & 9491.7   &3.68    & 15962.1    & 6.25       & 976.51  & 14.04      &48735             &19.27   \\
 Amidar             & 5.8       & \textbf{104159}    & 5131.2   &4.92    & 1554.79    & 1.49       & 1829.2  & 1.75        &1065              &1.02          \\
 Assault            & 222.4     & 8647             & 14198.5  &165.90  & 19148.47   & 224.65     & 21560.4 & 253.28      &\textbf{97155}             &\textbf{1150.59} \\
 Asterix            & 210       & \textbf{1000000}   & 428200   &42.81   & 300732     & 30.06      & 240090  & 23.99     &999999            &100.00 \\
 Asteroids          & 719       & \textbf{10506650}  & 2712.8   &0.02    & 108590.05  & 1.03       & 213025  & 2.02        &760005            &7.23\\
 Atlantis           & 12850     & \textbf{10604840}  & 826660   &7.68    & 849967.5   & 7.90       & 841200  & 7.82      &3837300           &36.11\\
 Bank Heist         & 14.2      & \textbf{82058}     & 1358     &1.64    & 1223.15    & 1.47       & 569.4   & 0.68       &1380              &1.66  \\
 Battle Zone        & 236       & 801000    & 62010    &7.71    & 20885      & 2.58       & 64953.3 & 8.08       &\textbf{824360}            &102.92 \\
 Beam Rider         & 363.9     & \textbf{999999}    & 16850.2  &1.65    & 32463.47   & 3.21       & 90881.6 & 9.06        &422390            &42.22   \\
 Berzerk            & 123.7     & \textbf{1057940}   & 2545.6   &0.23    & 1852.7     & 0.16       & 25579.5 & 2.41        &14649             &1.37          \\
 Bowling            & 23.1      & \textbf{300}       & 30       &2.49    & 59.92      & 13.30      & 48.3    & 9.10      &205.2             &65.76   \\
 Boxing             & 0.1       & \textbf{100}       & 99.6     &99.60   & 99.96      & 99.96      & \textbf{100}     & \textbf{100.00}     &\textbf{100}               &\textbf{100.00}    \\
 Breakout           & 1.7       & \textbf{864}       & 417.5    &48.22   & 787.34     & 91.11      & 747.9   & 86.54     &\textbf{864}             &\textbf{100.00 }  \\
 Centipede          & 2090.9    & \textbf{1301709}   & 8167.3   &0.47    & 11049.75   & 0.69       & 292792  & 22.37            &195630            &14.89 \\
 Chopper Command    & 811       & \textbf{999999}    & 16654    &1.59    & 28255      & 2.75       & 761699  & 76.15     &\textbf{999999}            &\textbf{100.00}\\
 Crazy Climber      & 10780.5   & 219900    & 168788.5 &75.56   & 136950     & 60.33      & 167820  & 75.10            &\textbf{241170}            &\textbf{110.17}\\
 Defender           & 2874.5    & \textbf{6010500}   & 55105    &0.87    & 185203     & 3.03       & 336953  & 5.56              &970540            &16.11\\
 Demon Attack       & 152.1     & \textbf{1556345}   & 111185   &7.13    & 132826.98  & 8.53       & 133530  & 8.57            &\textbf{787985}   &\textbf{50.63}\\
 Double Dunk        & -18.6     & 21        & -0.3     &46.21   & -0.33      & 46.14      & 14      & 82.32     &\textbf{24}                &\textbf{107.58}  \\
 Enduro             & 0         & 9500      & 2125.9   &22.38   & 0          & 0.00       & 0       & 0.00      &14300             &150.53  \\
 Fishing Derby      & -91.7     & \textbf{71}        & 31.3     &75.60   & 44.85      & 83.93      & 45.2    & 84.14         &65               &96.31\\
 Freeway            & 0         & \textbf{38}        & 34       &89.47   & 0          & 0.00       & 0       & 0.00            &34               &89.47\\
 Frostbite          & 65.2      & \textbf{454830}    & 9590.5   &2.09    & 317.75     & 0.06       & 5083.5  & 1.10              &11330            &2.48 \\
 Gopher             & 257.6     & 355040    & 70354.6  &19.76   & 66782.3    & 18.75      & 114820.7& 32.29      &473560           &133.41 \\
 Gravitar           & 173       & \textbf{162850}    & 1419.3   &0.77    & 359.5      & 0.11       & 1106.2  & 0.57               &5915             &3.53\\
 Hero               & 1027      & \textbf{1000000}   & 55887.4  &5.49    & 33730.55   & 3.27       & 31628.7 & 3.06               &38225            &3.72  \\
 Ice Hockey         & -11.2     & 36        & 1.1      &26.06   & 3.48       & 31.10      & 17.4    & 60.59          &\textbf{47.11}           &\textbf{123.54} \\
 Jamesbond          & 29        & 45550     & 19809    &43.45   & 601.5      & 1.26       & 37999.8 & 83.41     &\textbf{620780}          &\textbf{1363.66}\\
 Kangaroo           & 52        & \textbf{1424600}   & 14637.5  &1.02    & 1632       & 0.11       & 14308   & 1.00              &14636           &1.02  \\
 Krull              & 1598      & 104100    & 8741.5   &6.97    & 8147.4     & 6.39       & 9387.5  & 7.60               &\textbf{594540}          &\textbf{578.47}\\
 Kung Fu Master     & 258.5     & 1000000   & 52181    &5.19    & 43375.5    & 4.31       & 607443  & 60.73             &\textbf{1666665}          &\textbf{166.68}\\
 Montezuma Revenge  &0          & \textbf{1219200}   & 384      &0.03    & 0          & 0.00       & 0.3     & 0.00              &2500            &0.21\\
 Ms Pacman          & 307.3     & \textbf{290090}    & 5380.4   &1.75    & 7342.32    & 2.43       & 6565.5  & 2.16           &11573           &3.89\\
 Name This Game     & 2292.3    & 25220     & 13136    &47.30   & 21537.2    & 83.94      & 26219.5 & 104.36      &\textbf{36296}  &\textbf{148.31}    \\
 Phoenix            & 761.5     & \textbf{4014440}   & 108529   &2.69    & 210996.45  & 5.24       & 519304  & 12.92            &959580          &23.89\\
 Pitfall            & -229.4    & \textbf{114000}    & 0        &0.20    & -1.66      & 0.20       & -0.6    & 0.20                &-4.3            &0.20\\
 Pong               & -20.7     & \textbf{21}        & 20.9     &99.76   & 20.98      & 99.95      & \textbf{21}      & \textbf{100.00}    &\textbf{21}     &\textbf{100.00}    \\
 Private Eye        & 24.9      & \textbf{101800}    & 4234     &4.14    & 98.5       & 0.07       & 96.3    & 0.07              &15100           &14.81\\
 Qbert              & 163.9     & \textbf{2400000}   & 33817.5  &1.40    & 351200.12  & 14.63      & 21449.6 & 0.89              &28657           &1.19\\
 Riverraid          & 1338.5    & \textbf{1000000}   & 22920.8  &2.16    & 29608.05   & 2.83       & 40362.7 & 3.91              &28349           &2.70\\
 Road Runner        & 11.5      & \textbf{2038100}   & 62041    &3.04    & 57121      & 2.80       & 45289   & 2.22              &999999          &49.06\\
 Robotank           & 2.2       & 76        & 61.4     &80.22   & 12.96      & 14.58      & 62.1    & 81.17      &\textbf{113.4}           &\textbf{150.68}\\
 Seaquest           & 68.4      & 999999    & 15898.9  &1.58    & 1753.2     & 0.17       & 2890.3  & 0.28      &\textbf{1000000}          &\textbf{100.00}\\
 Skiing             & -17098    & \textbf{-3272}     & -12957.8 &29.95   & -10180.38  & 50.03      & -29968.4& -93.09            &-6025	          &86.77\\
 Solaris            & 1236.3    & \textbf{111420}    & 3560.3   &2.11    & 2365       & 1.02       & 2273.5  & 0.94                &9105            &7.14\\
 Space Invaders     & 148       & \textbf{621535 }   & 18789    &3.00    & 43595.78   & 6.99       & 51037.4 & 8.19            &154380          &24.82\\
 Star Gunner        & 664       & 77400     & 127029   &164.67  & 200625     & 260.58     & 321528  & 418.14    &\textbf{677590} &\textbf{882.15}\\
 Surround           & -10       & 9.6       & \textbf{9.7}      &\textbf{100.51}  & 7.56       & 89.59      & 8.4     & 93.88           &2.606           &64.32\\
 Tennis             & -23.8     & 21        & 0        &53.13   & 0.55       & 54.35      & 12.2    & 80.36     &\textbf{24}  &\textbf{106.70}\\
 Time Pilot         & 3568      & 65300     & 12926    &15.16   & 48481.5    & 72.76      & 105316  & 164.82          &\textbf{450810} &\textbf{724.49}\\
 Tutankham          & 11.4      & \textbf{5384}      & 241      &4.27    & 292.11     & 5.22       & 278.9   & 4.98            &418.2           &7.57\\
 Up N Down          & 533.4     & 82840     & 125755   &152.14  & 332546.75  & 403.39     & 345727  & 419.40    &966590          &1173.73  \\
 Venture            & 0         & \textbf{38900}     & 5.5      &0.01    & 0          & 0.00       & 0       & 0.00              &2000            &5.14\\
 Video Pinball      & 0         & \textbf{89218328}  & 533936.5 &0.60    & 572898.27  & 0.64       & 511835  & 0.57              &978190          &1.10\\
 Wizard of Wor      & 563.5     & \textbf{395300}    & 17862.5  &4.38    & 9157.5     & 2.18       & 29059.3 & 7.22              &63735           &16.00\\
 Yars Revenge       & 3092.9    & \textbf{15000105}  & 102557   &0.66    & 84231.14   & 0.54       & 166292.3& 1.09              &968090          &6.43\\
 Zaxxon             & 32.5      & 83700     & 22209.5  &26.51   & 32935.5    & 39.33      & 41118   & 49.11      &\textbf{216020}          &\textbf{258.15}\\
\hline
MEAN HWRNS(\%)      & 0.00      & 100.00    &          & \rainbowmeanHWRNS  &            & \impalameanHWRNS  &        & \lasermeanHWRNS  &      & \textbf{\GDIHmeanHWRNS} \\
\hline   
MEDIAN HWRNS(\%)    & 0.00      & \textbf{100.00}    &          & \rainbowmedianHWRNS   &            & \impalamedianHWRNS &        & \lasermedianHWRNS  &      & \GDIHmedianHWRNS  \\
\hline   
\end{tabular}
\end{center}
\end{table}
\clearpage

\begin{table}[!hb]
\footnotesize
\begin{center}
\caption{Score table of SOTA 10B+ model-free algorithms on HWRNS(\%).}
\label{Tab:Score table of SOTA 10B+ model-free algorithms on HWRNS.}
\setlength{\tabcolsep}{1.0pt}
\begin{tabular}{|c| c c| c c| c c| c c| c c| }
\hline
 Games & R2D2 & HWRNS & NGU & HWRNS & AGENT57 & HWRNS & GDI-I$^3$ & HWRNS & GDI-H$^3$ & HWRNS \\
\hline
Scale  & 10B   &        & 35B &         & 100B     &        & 200M &  &    200M   &\\
\hline
 Alien              & 109038.4          & 43.23       & 248100          & 98.48          & \textbf{297638.17}   &\textbf{118.17}         &43384             &17.15   &48735             &19.27    \\
 Amidar             & 27751.24          & 26.64       & 17800           & 17.08          & \textbf{29660.08}    &\textbf{28.47}          &1442              &1.38    &1065              &1.02    \\
 Assault            & 90526.44          & 1071.91     & 34800           & 410.44         & 67212.67             &795.17         &63876             &755.57  &\textbf{97155}             &\textbf{1150.59}    \\
 Asterix            & 999080   & 99.91       & 950700          & 95.07          & 991384.42            &99.14          &759910            &75.99   &\textbf{999999}            &\textbf{100.00}    \\
 Asteroids          & 265861.2          & 2.52        & 230500          & 2.19           & 150854.61            &1.43           &751970     &7.15    &\textbf{760005}            &\textbf{7.23}    \\
 Atlantis           & 1576068           & 14.76       & 1653600         & 15.49          & 1528841.76           &14.31          &3803000    &35.78   &\textbf{3837300}           &\textbf{36.11}    \\
 Bank Heist         & \textbf{46285.6}  & \textbf{56.40}       & 17400           & 21.19          & 23071.5              &28.10          &1401              &1.69    &1380              &1.66    \\
 Battle Zone        & 513360            & 64.08       & 691700          & 86.35          & \textbf{934134.88}   &\textbf{116.63}         &478830            &59.77   &824360            &102.92    \\
 Beam Rider         & 128236.08         & 12.79       & 63600           & 6.33           & 300509.8    &30.03          &162100            &16.18   &\textbf{422390}            &\textbf{42.22}    \\
 Berzerk            & 34134.8           & 3.22        & 36200           & 3.41           & \textbf{61507.83}    &\textbf{5.80 }          &7607              &0.71    &14649             &1.37    \\
 Bowling            & 196.36            & 62.57       & 211.9           & 68.18          & \textbf{251.18}      &\textbf{82.37}          &201.9             &64.57   &205.2             &65.76    \\
 Boxing             & 99.16             & 99.16       & 99.7            & 99.70          & \textbf{100}         &\textbf{100.00}         &\best{100}        &\textbf{100.00} &\textbf{100}       &\textbf{100.00}     \\
 Breakout           & 795.36            & 92.04       & 559.2           & 64.65          & 790.4                &91.46          &\best{864}        &\textbf{100.00}  &\textbf{864}             &\textbf{100.00}    \\
 Centipede          & 532921.84         & 40.85       & \textbf{577800} & \textbf{44.30}          & 412847.86            &31.61          &155830            &11.83   &195630            &14.89    \\
 Chopper Command    &960648             & 96.06       &999900           & 99.99          &999900                &99.99          &\best{999999}     &\textbf{100.00}  &\textbf{999999}            &\textbf{100.00}    \\
 Crazy Climber      & 312768            & 144.41      & 313400          & 144.71         &\textbf{565909.85}    &\textbf{265.46}         &201000            &90.96   &241170            &110.17    \\
 Defender           & 562106            & 9.31        & 664100          & 11.01          & 677642.78            &11.23          &893110     &14.82   &\textbf{970540}            &\textbf{16.11}    \\
 Demon Attack       & 143664.6          & 9.22        & 143500          & 9.21           & 143161.44            &9.19           &675530     &43.40    &\textbf{787985}   &\textbf{50.63}   \\
 Double Dunk        & 23.12             & 105.35      & -14.1           & 11.36          & 23.93       &107.40         &\textbf{24}                &\textbf{107.58}  &\textbf{24}                &\textbf{107.58}    \\
 Enduro             & 2376.68           & 25.02       & 2000            & 21.05          & 2367.71              &24.92          &\best{14330}      &\textbf{150.84}  &14300             &150.53    \\
 Fishing Derby      & 81.96             & 106.74      & 32              & 76.03          & \textbf{86.97}       &\textbf{109.82}         &59                &92.89   &65                &96.31    \\
 Freeway            & \textbf{34}       & \textbf{89.47}       & 28.5            & 75.00          & 32.59                &85.76          &\best{34}         &\textbf{89.47} &\textbf{34}         &\textbf{89.47}       \\
 Frostbite          & 11238.4           & 2.46        & 206400          & 45.37          &\textbf{541280.88}    &\textbf{119.01}         &10485             &2.29    &11330            &2.48    \\
 Gopher             & 122196            & 34.37       & 113400          & 31.89          & 117777.08            &33.12          &\best{488830}     &\textbf{137.71}  &473560           &133.41    \\
 Gravitar           & 6750              & 4.04        & 14200           & 8.62           &\textbf{19213.96}     &\textbf{11.70}          &5905              &3.52    &5915             &3.53    \\
 Hero               & 37030.4           & 3.60        & 69400           & 6.84           &\textbf{114736.26}    &\textbf{11.38}          &38330             &3.73    &38225            &3.72    \\
 Ice Hockey         & \textbf{71.56}    & \textbf{175.34}      &-4.1             & 15.04          & 63.64                &158.56         &37.89             &118.94  &47.11           &123.54    \\
 Jamesbond          & 23266             & 51.05       & 26600           & 58.37          & 135784.96            &298.23         &594500              &1305.93 &\textbf{620780}          &\textbf{1363.66}    \\
 Kangaroo           & 14112             & 0.99        & \textbf{35100}  & \textbf{2.46}           &24034.16              &1.68           &14500             &1.01    &14636           &1.02    \\
 Krull              & 145284.8          & 140.18      & 127400          & 122.73         & 251997.31   &244.29         &97575             &93.63   &\textbf{594540}          &\textbf{578.47}    \\
 Kung Fu Master     & 200176            & 20.00       & 212100          & 21.19          & 206845.82            &20.66          &140440            &14.02            &\textbf{1666665}	          &\textbf{166.68}\\
 Montezuma Revenge  & 2504              & 0.21        & \textbf{10400}  & \textbf{0.85}           &9352.01               &0.77           &3000              &0.25    &2500            &0.21    \\
 Ms Pacman          & 29928.2           & 10.22       & 40800           & 13.97          & \textbf{63994.44}    &\textbf{21.98}          &11536             &3.87    &11573           &3.89    \\
 Name This Game     & 45214.8           & 187.21      & 23900           & 94.24          &\textbf{54386.77}     &\textbf{227.21}         &34434             &140.19  &36296           &148.31    \\
 Phoenix            & 811621.6          & 20.20       & \textbf{959100} & \textbf{23.88}          &908264.15             &22.61          &894460            &22.27   &959580          &23.89    \\
 Pitfall            & 0                 & 0.20        & 7800            & 7.03           &\textbf{18756.01}     &\textbf{16.62}          &0                 &0.20    &-4.3            &0.20    \\
 Pong               & \textbf{21}       & \textbf{100.00}      & 19.6            & 96.64          & 20.67                &99.21          &\best{21}         &\textbf{100.00} &\textbf{21}      &\textbf{100.00}      \\
 Private Eye        & 300               & 0.27        & \textbf{100000} & \textbf{98.23}          & 79716.46             &78.30          &15100             &14.81   &15100           &14.81    \\
 Qbert              & 161000            & 6.70        & 451900          & 18.82          &\textbf{580328.14}    &\textbf{24.18}          &27800             &1.15    &28657           &1.19   \\
 Riverraid          & 34076.4           & 3.28        & 36700           & 3.54           & \textbf{63318.67}    &\textbf{6.21}           &28075             &2.68    &28349           &2.70    \\
 Road Runner        & 498660            & 24.47       & 128600          & 6.31           & 243025.8             &11.92          &878600     &43.11   &\textbf{999999}          &\textbf{49.06}    \\
 Robotank           & \textbf{132.4}    & \textbf{176.42}      & 9.1             & 9.35           &127.32                &169.54         &108               &143.63  &113.4           &150.68    \\
 Seaquest           & 999991.84         & 100.00      & \textbf{1000000}& \textbf{100.00}         &999997.63             &100.00         &943910	             &94.39 &\textbf{1000000}  &\textbf{100.00}    \\
 Skiing             & -29970.32         & -93.10      & -22977.9        & -42.53         & \textbf{-4202.6}     &\textbf{93.27}          &-6774             &74.67           &-6025	          &86.77\\
 Solaris            & 4198.4            & 2.69        & 4700            & 3.14           & \textbf{44199.93}    &\textbf{38.99}          &11074             &8.93            &9105            &7.14\\
 Space Invaders     & 55889             & 8.97        & 43400           & 6.96           & 48680.86             &7.81           &140460     &22.58           &\textbf{154380} &\textbf{24.82}\\
 Star Gunner        & 521728            & 679.03      & 414600          & 539.43         &\textbf{839573.53}    &\textbf{1093.24}        &465750            &606.09          &677590          &882.15\\
 Surround           & \textbf{9.96}     & \textbf{101.84}      & -9.6            & 2.04           & 9.5                  &99.49          &-7.8              &11.22           &2.606           &64.32\\
 Tennis             & \textbf{24}       & \textbf{106.70}      & 10.2            & 75.89          & 23.84                &106.34         &\textbf{24}                &\textbf{106.70}          &\textbf{24}              &\textbf{106.70}\\
 Time Pilot         & 348932            & 559.46      & 344700          & 552.60         &\textbf{405425.31}    &\textbf{650.97}         &216770            &345.37          &450810          &724.49\\
 Tutankham          & 393.64            & 7.11        & 191.1           & 3.34           & \textbf{2354.91}     &\textbf{43.62}          &423.9             &7.68            &418.2           &7.57\\
 Up N Down          & 542918.8          & 658.98      & 620100          & 752.75         & 623805.73            &757.26         &\best{986440}     &\textbf{1197.85}         &966590          &1173.73\\
 Venture            & 1992              & 5.12        & 1700            & 4.37           &\textbf{2623.71}      &\textbf{6.74}           &2000              &5.23            &2000            &5.14\\
 Video Pinball      & 483569.72         & 0.54        & 965300          & 1.08           &\textbf{992340.74}    &\textbf{1.11}           &925830            &1.04            &978190          &1.10\\
 Wizard of Wor      & 133264            & 33.62       & 106200          & 26.76          &\textbf{157306.41}    &\textbf{39.71}          &64439             &16.14           &63735           &16.00\\
 Yars Revenge       & 918854.32         & 6.11        & 986000          & 6.55           &\textbf{998532.37}    &\textbf{6.64}           &972000            &6.46            &968090          &6.43\\
 Zaxxon             & 181372            & 216.74      & 111100          & 132.75         &\textbf{249808.9}     &\textbf{298.53}         &109140            &130.41          &216020          &258.15\\
\hline
MEAN HWRNS(\%)        &                   & \rtdtmeanHWRNS       &                 &  \ngumeanHWRNS         &                      & \agentmeanHWRNS           &                  & \GDIImeanHWRNS        &         & \textbf{\GDIHmeanHWRNS} \\
\hline
MEDIAN HWRNS(\%)      &                   &  \rtdtmedianHWRNS       &                 &  \ngumedianHWRNS      &                      &\agentmedianHWRNS &                  & \GDIImedianHWRNS          &         & \textbf{\GDIHmedianHWRNS} \\
\hline
\end{tabular}
\end{center}
\end{table}
\clearpage

\begin{table}[!hb]
\footnotesize
\begin{center}
\caption{Score table of SOTA model-based algorithms on HWRNS(\%). SimPLe \citep{modelbasedatari} and DreamerV2\citep{dreamerv2} haven't evaluated all 57 Atari Games in their paper. For fairness, we set the score on those games as N/A, which will not be considered when calculating the median and mean HWRNS and human world record breakthrough (HWRB). }
\label{Tab:Score table of SOTA model-based algorithms on HWRNS.}
\setlength{\tabcolsep}{1.0pt}
\begin{tabular}{| c |c c| c c| c c| c c| c c| }
\hline
 Games              & MuZero         & HWRNS      & DreamerV2 & HWRNS   & SimPLe             & HWRNS          & GDI-I$^3$     & HWRNS & GDI-H$^3$ & HWRNS \\
\hline
Scale               & 20B            &              & 200M      &            & 1M               &                  & 200M     & &    200M   &\\
\hline
 Alien              & \textbf{741812.63}      & \textbf{294.64  }     &3483       & 1.29     &616.9     & 0.15    & 43384       &  17.15            &48735             &19.27       \\
 Amidar             & \textbf{28634.39 }      & \textbf{27.49   }  &2028       & 1.94     &74.3      & 0.07    & 1442        &  1.38                &1065              &1.02     \\
 Assault            & \textbf{143972.03}      & \textbf{1706.31 }     &7679       & 88.51       &527.2     & 3.62       & 63876       &  755.57     &97155             &1150.59    \\
 Asterix            & 998425         & 99.84        &25669      & 2.55     &1128.3    & 0.09       & 759910      &  75.99         &\textbf{999999}            &\textbf{100.00} \\
 Asteroids          & 678558.64               & 6.45        &3064                & 0.02    &793.6              & 0.00    &751970         & 7.15   &\textbf{760005}            &\textbf{7.23}  \\
 Atlantis           & 1674767.2               & 15.69           &989207              & 9.22       &20992.5            & 0.08       &3803000        & 35.78  &\textbf{3837300}           &\textbf{36.11}  \\
 Bank Heist         & 1278.98                 & 1.54        &1043                & 1.25    &34.2               & 0.02    &\best{1401}           & \best{1.69  }  &1380              &1.66 \\
 Battle Zone        & \textbf{848623         }& \textbf{105.95}      &31225      & 3.87     &4031.2    & 0.47       & 478830      &  59.77      &824360            &102.92     \\
 Beam Rider         & \textbf{454993.53}      & \textbf{45.48}      &12413      & 1.21     &621.6     & 0.03    & 162100      &  16.18          &422390            &42.22 \\
 Berzerk            & \textbf{85932.6        }& \textbf{8.11}      &751        & 0.06     &N/A       & N/A     & 7607        &  0.71            &14649             &1.37 \\
 Bowling            & \textbf{260.13         }& \textbf{85.60}      &48         & 8.99     &30        & 2.49    & 202         &  64.57          &205.2             &65.76   \\
 Boxing             & \textbf{100}                     & \textbf{100.00}       &87                  & 86.99       &7.8         & 7.71       & \best{100}  & \best{100.00 } &\textbf{100}               &\textbf{100.00} \\
 Breakout           & \textbf{864}                     & \textbf{100.00}          &350                 & 40.39       &16.4     & 1.70       & \best{864}  & \best{100.00}  &\textbf{864}             &100.00 \\
 Centipede          & \textbf{1159049.27}     & \textbf{89.02}     &6601       & 0.35     &N/A       & N/A     & 155830      & 11.83                     &195630            &14.89\\
 Chopper Command    & 991039.7                & 99.10     &2833                & 0.20     & 979.4             & 0.02    & \best{999999}        & \best{100.00} &\textbf{999999}            &\textbf{100.00}  \\
 Crazy Climber      & \textbf{458315.4}       & \textbf{214.01    }  &141424     & 62.47       & 62583.6  & 24.77   & 201000      & 90.96                      &241170            &110.17 \\
 Defender           & 839642.95               & 13.93      & N/A                & N/A        & N/A               & N/A      & 893110 & 14.82 &\textbf{970540}            &\textbf{16.11}   \\
 Demon Attack       & 143964.26               & 9.24      & 2775              &0.17      & 208.1             & 0.00    & 675530      & 43.40  &\textbf{787985}   &\textbf{50.63} \\
 Double Dunk        & 23.94          & 107.42  & 22        &102.53        & N/A      & N/A     & \textbf{24}          & \textbf{107.58}                    &\textbf{24}                &\textbf{107.58}    \\
 Enduro             & 2382.44                 & 25.08       & 2112              &22.23         & N/A               & N/A     & \best{14330} & \best{150.84}&14300             &150.53   \\
 Fishing Derby      & \textbf{91.16}          & \textbf{112.39     }              &93.24          &286.77      &-90.7     & 0.61    & 59          & 92.89  &65               &96.31        \\
 Freeway            & 33.03                   & 86.92       & \textbf{34}                 &\textbf{89.47}          &16.7               & 43.95   & \best{34}      & \best{89.47} &\textbf{34}               &\textbf{89.47}   \\
 Frostbite          & \textbf{631378.53}      & \textbf{138.82}     & 15622    &3.42       &236.9     & 0.04    & 10485       & 2.29                           &11330            &2.48\\
 Gopher             & 130345.58               & 36.67      & 53853             &15.11         &596.8              & 0.10       & \best{488830} & \best{137.71} &473560           &133.41  \\
 Gravitar           & \textbf{6682.7     }    & \textbf{4.00    }  & 3554     &2.08      &173.4     & 0.00    & 5905        & 3.52     &5915             &3.53       \\
 Hero               & \textbf{49244.11}       & \textbf{4.83    }  & 30287    &2.93      &2656.6    & 0.16        & 38330       & 3.73 &38225            &3.72           \\
 Ice Hockey         & \textbf{67.04      }    & \textbf{165.76  }  & 29        &85.17         &-11.6     & -0.85       & 38          & 118.94 &47.11           &123.54         \\
 Jamesbond          & 41063.25                & 90.14     & 9269              &20.30         &100.5     & 0.16    &594500              &1305.93 &\textbf{620780}          &\textbf{ 1363.66}  \\
 Kangaroo           & \textbf{16763.6        }& \textbf{1.17}  & 11819     &0.83      &51.2      & 0.00    & 14500       & 1.01          &14636           &1.02   \\
 Krull              & 269358.27               & 261.22     & 9687     &7.89       &2204.8    & 0.59    & 97575       & 93.63    &\textbf{594540}          &\textbf{578.47}       \\
 Kung Fu Master     & 204824         & 20.46  & 66410    &6.62       &14862.5   & 1.46    & 140440      & 14.02        &\textbf{1666665}          &\textbf{166.68 } \\
 Montezuma Revenge  & 0                       & 0.00         & 1932              &0.16       &N/A       & N/A     & \best{3000}          & \best{0.25  } &2500            &0.21  \\
 Ms Pacman          & \textbf{243401.1 }      & \textbf{83.89   }  & 5651     &1.84       &1480      & 0.40    & 11536       & 3.87       &11573           &3.89     \\
 Name This Game     & \textbf{157177.85}      & \textbf{675.54  }  & 14472    &53.12          &2420.7    & 0.56    & 34434       & 140.19 &36296           &148.31        \\
 Phoenix            & \textbf{955137.84}      & \textbf{23.78   }  & 13342     &0.31       &N/A       & N/A     & 894460      & 22.27     &959580          &23.89     \\
 Pitfall            & \textbf{0}                       & \textbf{0.20}               & -1                 &0.20       &N/A       & N/A     & \best{0} & \best{0.20} &-4.3            &0.20     \\
 Pong               & \textbf{21}                      & 100.00             & 19                 &95.20          & 12.8     & 80.34   & \best{21} & \best{100.00} &\textbf{21}              &\textbf{100.00}   \\
 Private Eye        & \textbf{15299.98 }      & \textbf{15.01}     & 158       &0.13       & 35       & 0.01    & 15100       & 14.81             &15100           &14.81     \\
 Qbert              & \textbf{72276          }& \textbf{3.00}      & 162023    &6.74       & 1288.8   & 0.05    & 27800       & 1.15              &28657           &1.19 \\
 Riverraid          & \textbf{323417.18}      & \textbf{32.25}     & 16249    &1.49       & 1957.8   & 0.06    & 28075       & 2.68               &28349           &2.70\\
 Road Runner        & 613411.8                & 30.10              & 88772             &4.36       & 5640.6   & 0.28       & 878600 & 43.11   &\textbf{999999}          &\textbf{49.06}   \\
 Robotank           & \textbf{131.13}         & \textbf{174.70}    & 65        &85.09          & N/A      & N/A     & 108         & 143.63        &113.4           &150.68 \\
 Seaquest           & 999976.52               & 100.00             & 45898             &4.58       & 683.3             & 0.06    &943910	             &94.39 &\textbf{1000000}          &\textbf{100.00}\\
 Skiing             & -29968.36      & -93.09      & -8187    &64.45          & N/A      & N/A     & -6774       & 74.67        &\textbf{-6025}	         &\textbf{86.77} \\
 Solaris            & 56.62                   & -1.07              & 883                &-0.32          & N/A               & N/A     & \best{11074}         & \best{8.93 } &9105            &7.14\\
 Space Invaders     & 74335.3                 & 11.94                 & 2611               &0.40       & N/A               & N/A     & 140460        & 22.58  &\textbf{154380}          &\textbf{24.82}   \\
 Star Gunner        & 549271.7       & 714.93     & 29219    &37.21          & N/A      & N/A     & 465750      & 606.09      &\textbf{677590}          &\textbf{882.15}\\
 Surround           & \textbf{9.99       }    & \textbf{101.99}     & N/A       &N/A         & N/A      & N/A     & -8          & 11.22         &2.606           &64.32 \\
 Tennis             & 0                       & 53.13      & 23        &104.46        & N/A      & N/A     & \textbf{24}          & \textbf{106.70} &\textbf{24}              &\textbf{106.70}      \\
 Time Pilot         & \textbf{476763.9}       & \textbf{766.53}         & 32404    &46.71         & N/A      & N/A     & 216770      & 345.37       &450810          &724.49   \\
 Tutankham          & \textbf{491.48     }    & \textbf{8.94}   & 238       &4.22         & N/A      & N/A     & 424         & 7.68                 &418.2           &7.57 \\
 Up N Down          & 715545.61               & 868.72            & 648363            &787.09        & 3350.3            & 3.42   & \best{986440}   & \best{1197.85} &966590          &1173.73    \\
 Venture            & 0.4                     & 0.00        & 0                  &0.00       & N/A               & N/A     & \best{2030}          & \best{5.23}      &2000            &5.14  \\
 Video Pinball      & \textbf{981791.88}      & \textbf{1.10}      & 22218    &0.02     & N/A      & N/A     & 925830      & 1.04                                    &978190          &1.10\\
 Wizard of Wor      & \textbf{197126         }& \textbf{49.80}      & 14531    &3.54     & N/A      & N/A     & 64439       & 16.14                                  &63735           &16.00\\
 Yars Revenge       & 553311.46               & 3.67      & 20089             &0.11     & 5664.3            & 0.02    & \best{972000}        & \best{6.46}           &968090          &6.43\\
 Zaxxon             & \textbf{725853.9}       & \textbf{867.51}      & 18295    &21.83          & N/A      & N/A     & 109140      & 130.41                          &216020          &258.15\\
\hline
MEAN HWRNS(\%) &               &\muzeromeanHWRNS  &         &  \dreamermedianHWRNS &                                          & \simplemeanHWRNS  &     & \GDIImeanHWRNS &                  & \textbf{\GDIHmeanHWRNS} \\
\hline
MEDIAN HWRNS(\%) &            & \muzeromedianHWRNS &         &\dreamermedianHWRNS  &                                          & \simplemedianHWRNS &     & \textbf{\GDIImedianHWRNS} &                  & \GDIHmedianHWRNS \\
\hline
\end{tabular}
\end{center}
\end{table}
\clearpage

\begin{table}[!hb] 
\footnotesize
\begin{center}
\caption{Score table of other SOTA  algorithms on HWRNS(\%). Go-Explore \citep{goexplore} and Muesli \citep{muesli}.}
\label{Tab: Score table of other SOTA  algorithms on HWRNS.}
\setlength{\tabcolsep}{1.0pt}
\begin{tabular}{| c | c c |c c |c c |c c |}
\hline
 Games      & Muesli              & HWRNS        & Go-Explore              & HWRNS                   & GDI-I$^3$ & HWRNS & GDI-H$^3$ & HWRNS\\
\hline
Scale       &                     & 200M        & 10B                     &                             & 200M      &    &    200M   &\\
\hline    
 Alien        &139409          &55.30               &\textbf{959312}       &\textbf{381.06}                & 43384             &17.15     &48735             &19.27            \\
 Amidar       &\textbf{21653}  &\textbf{20.78}      &19083        &18.32                & 1442              &1.38                         &1065              &1.02          \\
 Assault      &36963           &436.11           &30773                 &362.64                         & 63876      &755.57&\textbf{97155}    &\textbf{1150.59}           \\
 Asterix      &316210          &31.61            &999500       &99.95                 & 759910            &75.99        &\textbf{999999}            &\textbf{100.00}    \\
 Asteroids    &484609          &4.61             &112952                &1.07                           & 751970     &7.15  &\textbf{760005}            &\textbf{7.23}        \\
 Atlantis     &1363427         &12.75            &286460                &2.58                           & 3803000    &35.78 &\textbf{3837300}           &\textbf{36.11}         \\
 Bank Heist   &1213            &1.46          &\textbf{3668}         &\textbf{4.45  }                & 1401              &1.69            &1380              &1.66    \\
 Battle Zone  &414107          &51.68            &\textbf{998800}       &\textbf{124.70}                & 478830            &59.77        &824360            &102.92     \\
 Beam Rider   &288870          &28.86         &371723       &37.15                & 162100            &16.18           &\textbf{422390}            &\textbf{42.22}     \\
 Berzerk      &44478           &4.19             &\textbf{131417}      &\textbf{12.41 }                & 7607              &0.71          &14649             &1.37       \\
 Bowling      &191             &60.64        &\textbf{247}           &\textbf{80.86 }                & 202               &64.57            &205.2             &65.76 \\
 Boxing       &99              &99.00        &91                     &90.99                          & \best{100}        &\best{100.00}    &\textbf{100}      &\textbf{100.00}        \\
 Breakout     &791             &91.53          &774                    &89.56                          & \best{864}        &\best{100.00}  &\textbf{864}              &\textbf{100.00}       \\
 Centipede    &\textbf{869751} &\textbf{66.76}          &613815      &47.07                 & 155830            &11.83                    &195630            &14.89  \\
 Chopper Command &101289       &10.06       &996220                &99.62                          & \best{999999}     &\best{100.00}     &\textbf{999999}            &\textbf{100.00}    \\
 Crazy Climber   &175322       &78.68     &235600       &107.51                & 201000            &90.96               &\textbf{241170}            &\textbf{110.17}\\
 Defender        &629482       &10.43        &N/A                    &N/A                            & 893110     &14.82    &\textbf{970540}                &\textbf{16.11}      \\
 Demon Attack    &129544       &8.31        &239895                 &15.41                          & 675530     &43.40     &\textbf{787985}   &\textbf{50.63}    \\
 Double Dunk     &-3           &39.39       &\textbf{24}                     &\textbf{107.58}       & \best{24}       &\best{107.58}      &\textbf{24}      &\textbf{107.58}      \\
 Enduro          &2362         &24.86                   &1031                   &10.85              & \best{14330}      &\best{150.84}    &14300             &150.53        \\
 Fishing Derby   &51           &87.71       &\textbf{67}            &\textbf{97.54 }                & 59                &92.89            &65               &96.31\\
 Freeway         &33           &86.84                   &\textbf{34}            &\textbf{89.47 }    & \best{34}         &\best{89.47}     &\textbf{34}      &\textbf{89.47}      \\
 Frostbite       &301694       &66.33            &\textbf{999990}       &\textbf{219.88}                & 10485             &2.29         &11330            &2.48   \\
 Gopher          &104441       &29.37            &134244                &37.77                          & \best{488830}     &\best{137.71} &473560           &133.41       \\
 Gravitar        &11660        &7.06           &\textbf{13385}        &\textbf{8.12}                  & 5905              &3.52           &5915             &3.53  \\
 Hero            &37161        &3.62                     &37783                  &3.68                           & \best{38330}      &\best{3.73}   &38225           &3.72           \\
 Ice Hockey      &25           &76.69          &33                     &93.64                          & 45         &118.94           &\textbf{47.11}  &\textbf{123.54} \\
 Jamesbond       &19319        &42.38            &200810                &441.07                         &594500              &1305.93         &\textbf{620780}       &\textbf{ 1363.66}\\
 Kangaroo        &14096        &0.99                   &\textbf{24300}        &\textbf{1.70}                  & 14500             &1.01             &14636           &1.02\\
 Krull           &34221        &31.83                     &63149                 &60.05                          & 97575      &93.63  &\textbf{594540}          &\textbf{578.47}        \\
 Kung Fu Master  &134689       &13.45      &24320                 &2.41                           & 140440     &14.02                 &\textbf{1666665}          &\textbf{166.68}\\
 Montezuma Revenge  &2359      &0.19             &\textbf{24758}        &\textbf{2.03}                  & 3000              &0.25                   &2500            &0.21\\
 Ms Pacman          &65278     &22.42     &\textbf{456123}       &\textbf{157.30}                & 11536             &3.87                          &11573           &3.89\\
 Name This Game     &105043    &448.15     &\textbf{212824}       &\textbf{918.24}               & 34434             &140.19                        &36296           &148.31\\
 Phoenix        &805305        &20.05                                &19200                 &0.46                           & 894460     &22.27 &\textbf{959580}          &\textbf{23.89}   \\
 Pitfall        &0             &0.20                   &\textbf{7875}          &\textbf{7.09   }               & 0                 &0.2             &-4.3            &0.20 \\
 Pong           &20            &97.60                       &\textbf{21}            &\textbf{100.00 }               & \best{21}         &\best{100} &\textbf{21}              &\textbf{100.00}          \\
 Private Eye    &10323         &10.12                &\textbf{69976}        &\textbf{68.73  }               & 15100             &14.81               &15100           &14.81 \\
 Qbert          &157353        &6.55                            &\textbf{999975}       &\textbf{41.66  }               & 27800             &1.15     &28657           &1.19       \\
 Riverraid      &\textbf{47323}&\textbf{4.60}               &35588        &3.43               & 28075             &2.68                              &28349           &2.70\\
 Road Runner    &327025        &16.05                        &999900        &49.06                 & 878600           &43.11       &\textbf{999999}          &\textbf{49.06}    \\
 Robotank       &59            &76.96                             &\textbf{143}           &\textbf{190.79 }               & 108         &143.63      &113.4           &150.68          \\
 Seaquest       &815970        &81.60                &539456       &53.94                 &943910	             &94.39       &\textbf{1000000}          &\textbf{100.00}    \\
 Skiing         &-18407        &-9.47                      &\textbf{-4185}        &\textbf{93.40  }               & -6774             &74.67         &-6025	          &86.77      \\
 Solaris        &3031          &1.63                         &\textbf{20306}        &\textbf{17.31  }               & 11074             &8.93        &9105            &7.14        \\
 Space Invaders &59602         &9.57                 &93147                 &14.97                          & 140460     &22.58        &\textbf{154380}          &\textbf{24.82} \\
 Star Gunner    &214383        &278.51   &609580                 &793.52                         &465750     &606.09                  &\textbf{677590}          &\textbf{882.15}\\
 Surround       &\textbf{9}    &\textbf{96.94}             &N/A                    &N/A                            & -8         &11.22               &2.606           &64.32\\
 Tennis         &12            &79.91                           &\best{24}              &\best{106.7}                   & \best{24}         &\best{106.70   }  &\textbf{24}              &\textbf{106.70}          \\
 Time Pilot     &359105 &575.94    &183620                &291.67                         & 216770     & 345.37                    &\textbf{450810}          &\textbf{724.49}\\
 Tutankham      &252           &4.48         &\textbf{528}           &\textbf{9.62}                  & 424               &7.68                       &418.2           &7.57\\
 Up N Down      &649190        &788.10                &553718                &672.10                         & \best{986440}     &\best{1197.85}     &966590          &1173.73      \\
 Venture        &2104          &5.41            &\textbf{3074}         &\textbf{7.90}                & 2035              &5.23                       &2000            &5.14\\
 Video Pinball  &685436        &0.77                 &\textbf{999999}       &\textbf{1.12}               & 925830            &1.04                   &978190          &1.10\\
 Wizard of Wor  &93291         &23.49                &\textbf{199900}       &\textbf{50.50}               & 64293             &16.14                 &63735           &16.00\\
 Yars Revenge   &557818        &3.70               &\textbf{999998}       &\textbf{6.65}               & 972000            &6.46                     &968090          &6.43\\
 Zaxxon         &65325         &78.04                 &18340                 &21.88                        & 109140     &130.41        &\textbf{216020} &\textbf{258.15}  \\
\hline    
MEAN HWRNS(\%)  &    & \mueslimeanHWRNS &                       & \goexploremeanHWRNS                       &            &  \GDIImeanHWRNS  &                  & \textbf{\GDIHmeanHWRNS} \\
\hline
MEDIAN HWRNS(\%)&   & \mueslimedianHWRNS  &                       & \goexploremedianHWRNS              &            & \GDIImedianHWRNS   &                  & \textbf{\GDIHmedianHWRNS} \\
\hline
\end{tabular}
\end{center}
\end{table}

\subsection{Atari Games Table of Scores Based on SABER}
\label{app: Atari Games Table of Scores Based on SABER}
In this part, we detail the raw score of several representative SOTA algorithms , including the SOTA 200M model-free algorithms, SOTA 10B+ model-free algorithms, SOTA model-based algorithms and other SOTA algorithms.\footnote{200M and 10B+ represent the training scale.} Additionally, we calculate the capped human world records normalized world score (CHWRNS) or called SABER \citep{atarihuman} of each game with each algorithm. First of all, we demonstrate the sources of the scores that we used.
Random scores  are from \citep{agent57}.
Human world records (HWR) are from \citep{dreamerv2,atarihuman}.
Rainbow's scores are from \citep{rainbow}.
IMPALA's scores are from \citep{impala}.
LASER's scores are from \citep{laser}, with no sweep at 200M.
As there are many versions of R2D2 and NGU, we use original papers'.
R2D2's scores are from \citep{r2d2}.
NGU's scores are from \citep{ngu}.
Agent57's scores are from \citep{agent57}.
MuZero's scores are from \citep{muzero}.
DreamerV2's scores are from \citep{dreamerv2}.
SimPLe's scores are from \citep{modelbasedatari}.
Go-Explore's scores are from \citep{goexplore}.
Muesli's scores are from \citep{muesli}.
In the following,  we detail the raw scores and SABER of each algorithm on 57 Atari games.
\clearpage

\begin{table}[!hb]
\footnotesize
\begin{center}
\caption{Score table of SOTA 200M model-free algorithms on SABER(\%)  (GDI-I$^3$).}
\setlength{\tabcolsep}{1.0pt}
\begin{tabular}{ |c |c |c| c c| c c|  c c |c c |c c |}
\hline
Games & RND & HWR & RAINBOW & SABER & IMPALA & SABER & LASER & SABER & GDI-I$^3$ & SABER \\
\hline
Scale  &     &       & 200M   &       &  200M    &        & 200M   & &  200M   & \\
\hline
 Alien              & 227.8     & \textbf{251916}    & 9491.7   &3.68    & 15962.1    & 6.25       & 976.51  & 14.04     &43384      &17.15   \\
 Amidar             & 5.8       & \textbf{104159}    & 5131.2   &4.92    & 1554.79    & 1.49       & 1829.2  & 1.75      &1442              &1.38           \\
 Assault            & 222.4     & 8647               & 14198.5  &165.90  & 19148.47   & 200.00     & 21560.4 & 200.00    &63876      &200.00   \\
 Asterix            & 210       & \textbf{1000000}   & 428200   &42.81   & 300732     & 30.06      & 240090  & 23.99     &759910     &75.99   \\
 Asteroids          & 719       & \textbf{10506650}  & 2712.8   &0.02    & 108590.05  & 1.03       & 213025  & 2.02      &751970     &7.15    \\
 Atlantis           & 12850     & \textbf{10604840}  & 826660   &7.68    & 849967.5   & 7.90       & 841200  & 7.82      &3803000    &35.78   \\
 Bank Heist         & 14.2      & \textbf{82058}     & 1358     &1.64    & 1223.15    & 1.47       & 569.4   & 0.68      &1401       &1.69     \\
 Battle Zone        & 236       &801000    & 62010    &7.71    & 20885      & 2.58       & 64953.3 & 8.08      &478830     &59.77    \\
 Beam Rider         & 363.9     & \textbf{999999}    & 16850.2  &1.65    & 32463.47   & 3.21       & 90881.6 & 9.06      &162100     &16.18    \\
 Berzerk            & 123.7     & \textbf{1057940}            & 2545.6   &0.23    & 1852.7     & 0.16       & 25579.5 & 2.41      &7607              &0.71              \\
 Bowling            & 23.1      & \textbf{300}       & 30       &2.49    & 59.92      & 13.30      & 48.3    & 9.10      &201.9      &64.57  \\
 Boxing             & 0.1       & \textbf{100}                & 99.6     &99.60   & 99.96      & 99.96      & \textbf{100}     & \textbf{100.00}    &\best{100}        &\best{100.00 }    \\
 Breakout           & 1.7       & \textbf{864}                & 417.5    &48.22   & 787.34     & 91.11      & 747.9   & 86.54     &\best{864}        &\best{100.00 }   \\
 Centipede          & 2090.9    & \textbf{1301709}   & 8167.3   &0.47    & 11049.75   & 0.69       & 292792  & 22.37     &155830            &11.83         \\
 Chopper Command    & 811       & \textbf{999999}             & 16654    &1.59    & 28255      & 2.75       & 761699  & 76.15     &\best{999999}     &\best{100.00 } \\
 Crazy Climber      & 10780.5   & 219900    & 168788.5 &75.56   & 136950     & 60.33      & 167820  & 75.10     &201000     &90.96                 \\
 Defender           & 2874.5    & \textbf{6010500}   & 55105    &0.87    & 185203     & 3.03       & 336953  & 5.56      &893110     &14.82                \\
 Demon Attack       & 152.1     & \textbf{1556345}   & 111185   &7.13    & 132826.98  & 8.53       & 133530  & 8.57      &675530     &43.10                 \\
 Double Dunk        & -18.6     & 21                 & -0.3     &46.21   & -0.33      & 46.14      & 14      & 82.32     &\best{24}      &\best{107.58} \\
 Enduro             & 0         & 9500               & 2125.9   &22.38   & 0          & 0.00       & 0       & 0.00      &\best{14330}      &\best{150.84  }\\
 Fishing Derby      & -91.7     & \textbf{71}        & 31.3     &75.60   & 44.85      & 83.93      & 45.2    & 84.14     &59                &95.08         \\
 Freeway            & 0         & \textbf{38}        & 34       &89.47   & 0          & 0.00       & 0       & 0.00      &34                &89.47          \\
 Frostbite          & 65.2      & \textbf{454830}    & 9590.5   &2.09    & 317.75     & 0.06       & 5083.5  & 1.10      &10485      &2.29                  \\          
 Gopher             & 257.6     & 355040             & 70354.6  &19.76   & 66782.3    & 18.75      & 114820.7& 32.29     &\best{488830}     &\best{137.71 } \\
 Gravitar           & 173       & \textbf{162850}    & 1419.3   &0.77    & 359.5      & 0.11       & 1106.2  & 0.57      &5905       &3.52                  \\
 Hero               & 1027      & \textbf{1000000}            & 55887.4  &5.49    & 33730.55   & 3.27       & 31628.7 & 3.06      &38330             &3.73           \\
 Ice Hockey         & -11.2     & 36                 & 1.1      &26.06   & 3.48       & 31.10      & 17.4    & 60.59     &44.92              &118.94      \\
 Jamesbond          & 29        & 45550              & 19809    &43.45   & 601.5      & 1.26       & 37999.8 & 83.41     &594500              &200.00 \\
 Kangaroo           & 52        & \textbf{1424600}   & 14637.5  &1.02    & 1632       & 0.11       & 14308   & 1.00      &14500             &1.01         \\
 Krull              & 1598      & 104100    & 8741.5   &6.97    & 8147.4     & 6.39       & 9387.5  & 7.60      &97575     &93.63                  \\
 Kung Fu Master     & 258.5     & 1000000   & 52181    &5.19    & 43375.5    & 4.31       & 607443  & 60.73     &140440            &14.02          \\
 Montezuma Revenge  &0          & \textbf{1219200}   & 384      &0.03    & 0          & 0.00       & 0.3     & 0.00      &3000              &0.25           \\
 Ms Pacman          & 307.3     & \textbf{290090}    & 5380.4   &1.75    & 7342.32    & 2.43       & 6565.5  & 2.16      &11536              &3.87          \\
 Name This Game     & 2292.3    & 25220              & 13136    &47.30   & 21537.2    & 83.94      & 26219.5 & 104.36    &34434      &140.19 \\
 Phoenix            & 761.5     & \textbf{4014440}   & 108529   &2.69    & 210996.45  & 5.24       & 519304  & 12.92     &894460         &22.27             \\
 Pitfall            & -229.4    & \textbf{114000}    & \textbf{0}        &\textbf{0.20}    & -1.66      & 0.20       & -0.6    & 0.20      &\best{0    }      &0.20          \\
 Pong               & -20.7     & \textbf{21}                 & 20.9     &99.76   & 20.98      & 99.95      & \textbf{21}      & \textbf{100.00}    &\best{21   }      &\best{100.00 }     \\
 Private Eye        & 24.9      & \textbf{101800}    & 4234     &4.14    & 98.5       & 0.07       & 96.3    & 0.07      &15100      &14.81                 \\
 Qbert              & 163.9     & \textbf{2400000}   & 33817.5  &1.40    & 351200.12  & 14.63      & 21449.6 & 0.89      &27800             &1.03            \\
 Riverraid          & 1338.5    & \textbf{1000000}   & 22920.8  &2.16    & 29608.05   & 2.83       & 40362.7 & 3.91      &28075             &2.68           \\
 Road Runner        & 11.5      & \textbf{2038100}   & 62041    &3.04    & 57121      & 2.80       & 45289   & 2.22      &878600    &43.11                  \\
 Robotank           & 2.2       & 76                 & 61.4     &80.22   & 12.96      & 14.58      & 62.1    & 81.17     &108.2      &143.63   \\
 Seaquest           & 68.4      & 999999             & 15898.9  &1.58    & 1753.2     & 0.17       & 2890.3  & 0.28      &943910	             &94.39 \\
 Skiing             & -17098    & \textbf{-3272}     & -12957.8 &29.95   & -10180.38  & 50.03      & -29968.4& -93.09    &-6774             &74.67            \\
 Solaris            & 1236.3    & \textbf{111420}    & 3560.3   &2.11    & 2365       & 1.02       & 2273.5  & 0.94      &11074             &8.93           \\
 Space Invaders     & 148       & \textbf{621535 }   & 18789    &3.00    & 43595.78   & 6.99       & 51037.4 & 8.19      &140460    &22.58                  \\
 Star Gunner        & 664       & 77400              & 127029   &164.67  & 200625     & 200.00     & 321528  & 418.14    &465750     &200.00    \\
 Surround           & -10       & 9.6                & \textbf{9.7}      &\textbf{100.51}  & 7.56       & 89.59      & 8.4     & 93.88     &-7.8    &11.22  \\
 Tennis             & -23.8     & 21                 & 0        &53.13   & 0.55       & 54.35      & 12.2    & 80.36     &\best{24       }  &\best{106.70    }    \\
 Time Pilot         & 3568      & 65300              & 12926    &15.16   & 48481.5    & 72.76      & 105316  & 164.82    &216770     &200.00  \\
 Tutankham          & 11.4      & \textbf{5384}      & 241      &4.27    & 292.11     & 5.22       & 278.9   & 4.98      &423.9     &7.68                   \\
 Up N Down          & 533.4     & 82840              & 125755   &152.14  & 332546.75  & 200.00     & 345727  & 200.00    &\best{986440}     &\best{200.00  } \\
 Venture            & 0         & \textbf{38900}     & 5.5      &0.01    & 0          & 0.00       & 0       & 0.00      &2000    &5.14                     \\
 Video Pinball      & 0         & \textbf{89218328}  & 533936.5 &0.60    & 572898.27  & 0.64       & 511835  & 0.57      &925830  &1.04                      \\\
 Wizard of Wor      & 563.5     & \textbf{395300}    & 17862.5  &4.38    & 9157.5     & 2.18       & 29059.3 & 7.22      &64439   &16.18                     \\
 Yars Revenge       & 3092.9    & \textbf{15000105}  & 102557   &0.66    & 84231.14   & 0.54       & 166292.3& 1.09      &972000  &6.46                      \\
 Zaxxon             & 32.5      & 83700              & 22209.5  &26.51   & 32935.5    & 39.33      & 41118   & 49.11     &109140     &130.41   \\
\hline
MEAN SABER(\%) &     0.00 & \textbf{100.00}   &         & \rainbowmeanSABER &         & \impalameanSABER  &        & \lasermeanSABER &      & \GDIImeanSABER\\
\hline
MEDIAN SABER(\%) & 0.00   & \textbf{100.00}   &         &\rainbowmedianSABER &         & \impalamedianSABER  &        & \lasermedianSABER  &      & \GDIImedianSABER  \\
\hline
\end{tabular}
\end{center}
\end{table}
\clearpage

\begin{table}[!hb]
\footnotesize
\begin{center}
\caption{Score table of SOTA 200M model-free algorithms on SABER(\%)  (GDI-H$^3$).}
\setlength{\tabcolsep}{1.0pt}
\begin{tabular}{ |c |c |c| c c| c c|  c c |c c |c c |}
\hline
Games & RND & HWR & RAINBOW & SABER & IMPALA & SABER & LASER & SABER  & GDI-H$^3$ & SABER\\
\hline
Scale  &     &       & 200M   &       &  200M    &        & 200M   & &  200M   & \\
\hline
 Alien              & 227.8     & \textbf{251916}    & 9491.7   &3.68    & 15962.1    & 6.25       & 976.51  & 14.04       &48735             &19.27  \\
 Amidar             & 5.8       & \textbf{104159}    & 5131.2   &4.92    & 1554.79    & 1.49       & 1829.2  & 1.75        &1065              &1.02           \\
 Assault            & 222.4     & 8647               & 14198.5  &165.90  & 19148.47   & 200.00     & 21560.4 & 200.00    &\textbf{97155}     &\textbf{200.00} \\
 Asterix            & 210       & \textbf{1000000}   & 428200   &42.81   & 300732     & 30.06      & 240090  & 23.99     &999999            &100.00\\
 Asteroids          & 719       & \textbf{10506650}  & 2712.8   &0.02    & 108590.05  & 1.03       & 213025  & 2.02         &760005            &7.23\\
 Atlantis           & 12850     & \textbf{10604840}  & 826660   &7.68    & 849967.5   & 7.90       & 841200  & 7.82       &3837300           &36.11\\
 Bank Heist         & 14.2      & \textbf{82058}     & 1358     &1.64    & 1223.15    & 1.47       & 569.4   & 0.68        &1380              &1.66 \\
 Battle Zone        & 236       &801000    & 62010    &7.71    & 20885      & 2.58       & 64953.3 & 8.08        &\textbf{824360}            &\textbf{102.92} \\
 Beam Rider         & 363.9     & \textbf{999999}    & 16850.2  &1.65    & 32463.47   & 3.21       & 90881.6 & 9.06       &422390            &42.22 \\
 Berzerk            & 123.7     & \textbf{1057940}            & 2545.6   &0.23    & 1852.7     & 0.16       & 25579.5 & 2.41      &14649             &1.37             \\
 Bowling            & 23.1      & \textbf{300}       & 30       &2.49    & 59.92      & 13.30      & 48.3    & 9.10         &205.2             &65.76\\
 Boxing             & 0.1       & \textbf{100}                & 99.6     &99.60   & 99.96      & 99.96      & \textbf{100}     & \textbf{100.00}    &\textbf{100}       &\textbf{100.00}    \\
 Breakout           & 1.7       & \textbf{864}                & 417.5    &48.22   & 787.34     & 91.11      & 747.9   & 86.54      &\textbf{864}     &\textbf{100.00}   \\
 Centipede          & 2090.9    & \textbf{1301709}   & 8167.3   &0.47    & 11049.75   & 0.69       & 292792  & 22.37              &195630    &14.89\\
 Chopper Command    & 811       & \textbf{999999}             & 16654    &1.59    & 28255      & 2.75       & 761699  & 76.15     &\textbf{999999}    &\textbf{100.00}\\
 Crazy Climber      & 10780.5   & 219900    & 168788.5 &75.56   & 136950     & 60.33      & 167820  & 75.10                      &\textbf{241170}    &\textbf{110.17}\\
 Defender           & 2874.5    & \textbf{6010500}   & 55105    &0.87    & 185203     & 3.03       & 336953  & 5.56                  &970540    &16.11\\
 Demon Attack       & 152.1     & \textbf{1556345}   & 111185   &7.13    & 132826.98  & 8.53       & 133530  & 8.57                       &787985   &50.63\\
 Double Dunk        & -18.6     & 21                 & -0.3     &46.21   & -0.33      & 46.14      & 14      & 82.32      &\textbf{24}        &\textbf{107.58}\\
 Enduro             & 0         & 9500               & 2125.9   &22.38   & 0          & 0.00       & 0       & 0.00      &14300     &150.53\\
 Fishing Derby      & -91.7     & \textbf{71}        & 31.3     &75.60   & 44.85      & 83.93      & 45.2    & 84.14             &65        &96.31\\
 Freeway            & 0         & \textbf{38}        & 34       &89.47   & 0          & 0.00       & 0       & 0.00           &34        &89.47\\
 Frostbite          & 65.2      & \textbf{454830}    & 9590.5   &2.09    & 317.75     & 0.06       & 5083.5  & 1.10                      &11330     &2.48\\          
 Gopher             & 257.6     & 355040             & 70354.6  &19.76   & 66782.3    & 18.75      & 114820.7& 32.29     &473560    &133.41\\
 Gravitar           & 173       & \textbf{162850}    & 1419.3   &0.77    & 359.5      & 0.11       & 1106.2  & 0.57                      &5915      &3.53\\
 Hero               & 1027      & \textbf{1000000}            & 55887.4  &5.49    & 33730.55   & 3.27       & 31628.7 & 3.06               &38225     &3.72\\
 Ice Hockey         & -11.2     & 36                 & 1.1      &26.06   & 3.48       & 31.10      & 17.4    & 60.59           &\textbf{47.11}           &\textbf{123.54}\\
 Jamesbond          & 29        & 45550              & 19809    &43.45   & 601.5      & 1.26       & 37999.8 & 83.41      &\textbf{620780}          &\textbf{200.00}\\
 Kangaroo           & 52        & \textbf{1424600}   & 14637.5  &1.02    & 1632       & 0.11       & 14308   & 1.00          &14636           &1.02\\
 Krull              & 1598      & 104100    & 8741.5   &6.97    & 8147.4     & 6.39       & 9387.5  & 7.60                    &\textbf{594540}          &\textbf{200.00}\\
 Kung Fu Master     & 258.5     & 1000000   & 52181    &5.19    & 43375.5    & 4.31       & 607443  & 60.73         &\textbf{1666665}          &\textbf{166.68}\\
 Montezuma Revenge  &0          & \textbf{1219200}   & 384      &0.03    & 0          & 0.00       & 0.3     & 0.00                &2500            &0.21\\
 Ms Pacman          & 307.3     & \textbf{290090}    & 5380.4   &1.75    & 7342.32    & 2.43       & 6565.5  & 2.16        &11573           &3.89\\
 Name This Game     & 2292.3    & 25220              & 13136    &47.30   & 21537.2    & 83.94      & 26219.5 & 104.36      &\textbf{36296}  &\textbf{148.31}\\
 Phoenix            & 761.5     & \textbf{4014440}   & 108529   &2.69    & 210996.45  & 5.24       & 519304  & 12.92                &959580          &23.89\\
 Pitfall            & -229.4    & \textbf{114000}    & \textbf{0}        &\textbf{0.20}    & -1.66      & 0.20       & -0.6    & 0.20                &-4.3            &0.20\\
 Pong               & -20.7     & \textbf{21}                 & 20.9     &99.76   & 20.98      & 99.95      & \textbf{21}      & \textbf{100.00}    &\textbf{21}     &\textbf{100.00}    \\
 Private Eye        & 24.9      & \textbf{101800}    & 4234     &4.14    & 98.5       & 0.07       & 96.3    & 0.07                  &15100           &14.81\\
 Qbert              & 163.9     & \textbf{2400000}   & 33817.5  &1.40    & 351200.12  & 14.63      & 21449.6 & 0.89               &28657           &1.19  \\
 Riverraid          & 1338.5    & \textbf{1000000}   & 22920.8  &2.16    & 29608.05   & 2.83       & 40362.7 & 3.91                &28349           &2.70  \\
 Road Runner        & 11.5      & \textbf{2038100}   & 62041    &3.04    & 57121      & 2.80       & 45289   & 2.22                     &999999          &49.06\\
 Robotank           & 2.2       & 76                 & 61.4     &80.22   & 12.96      & 14.58      & 62.1    & 81.17      &\textbf{113.4}  &\textbf{150.68}  \\
 Seaquest           & 68.4      & 999999             & 15898.9  &1.58    & 1753.2     & 0.17       & 2890.3  & 0.28     &\textbf{1000000}          &\textbf{100.00}  \\
 Skiing             & -17098    & \textbf{-3272}     & -12957.8 &29.95   & -10180.38  & 50.03      & -29968.4& -93.09             &-6025	         &86.77   \\
 Solaris            & 1236.3    & \textbf{111420}    & 3560.3   &2.11    & 2365       & 1.02       & 2273.5  & 0.94               &9105            &7.14  \\
 Space Invaders     & 148       & \textbf{621535 }   & 18789    &3.00    & 43595.78   & 6.99       & 51037.4 & 8.19                      &154380          &24.82\\
 Star Gunner        & 664       & 77400              & 127029   &164.67  & 200625     & 200.00     & 321528  & 418.14      &\textbf{677590} &\textbf{200.00}  \\
 Surround           & -10       & 9.6                & \textbf{9.7}      &\textbf{100.51}  & 7.56       & 89.59      & 8.4     & 93.88       &2.606           &64.32\\
 Tennis             & -23.8     & 21                 & 0        &53.13   & 0.55       & 54.35      & 12.2    & 80.36      &\textbf{24} &\textbf{106.70}  \\
 Time Pilot         & 3568      & 65300              & 12926    &15.16   & 48481.5    & 72.76      & 105316  & 164.82      &\textbf{450810}          &\textbf{200.00}\\
 Tutankham          & 11.4      & \textbf{5384}      & 241      &4.27    & 292.11     & 5.22       & 278.9   & 4.98                         &418.2           &7.57\\
 Up N Down          & 533.4     & 82840              & 125755   &152.14  & 332546.75  & 200.00     & 345727  & 200.00     &966590        &200.00\\
 Venture            & 0         & \textbf{38900}     & 5.5      &0.01    & 0          & 0.00       & 0       & 0.00                       &2000            &5.14\\
 Video Pinball      & 0         & \textbf{89218328}  & 533936.5 &0.60    & 572898.27  & 0.64       & 511835  & 0.57                      &978190          &1.10\\\
 Wizard of Wor      & 563.5     & \textbf{395300}    & 17862.5  &4.38    & 9157.5     & 2.18       & 29059.3 & 7.22                          &63735           &16.00\\
 Yars Revenge       & 3092.9    & \textbf{15000105}  & 102557   &0.66    & 84231.14   & 0.54       & 166292.3& 1.09                          &968090          &6.43\\
 Zaxxon             & 32.5      & 83700              & 22209.5  &26.51   & 32935.5    & 39.33      & 41118   & 49.11       &\textbf{216020}          &\textbf{200.00}\\
\hline
MEAN SABER(\%) &     0.00 & \textbf{100.00}   &         & \rainbowmeanSABER &         & \impalameanSABER  &        & \lasermeanSABER &            &\GDIHmeanSABER\\
\hline
MEDIAN SABER(\%) & 0.00   & \textbf{100.00}   &         &\rainbowmedianSABER &         & \impalamedianSABER  &        & \lasermedianSABER  &           & \GDIHmedianSABER  \\
\hline
\end{tabular}
\end{center}
\end{table}
\clearpage

\begin{table}[!hb]
\footnotesize
\begin{center}
\caption{Score table of SOTA  10B+ model-free algorithms on SABER(\%).}
\label{Tab:Score table of SOTA  10B+ model-free algorithms on SABER.}
\setlength{\tabcolsep}{1.0pt}
\begin{tabular}{ |c| c c| c c| c c| c c| c c| }
\hline
 Games & R2D2 & SABER & NGU & SABER & AGENT57 & SABER & GDI-I$^3$ & SABER & GDI-H$^3$ & SABER\\
\hline
Scale  & 10B   &        & 35B &         & 100B     &        & 200M & &  200M   & \\
\hline
 Alien              & 109038.4          & 43.23             & 248100          & 98.48          & \textbf{297638.17}   &\textbf{118.17}           &43384             &17.15              &48735  &19.27        \\
 Amidar             & 27751.24          & 26.64             & 17800           & 17.08          & \textbf{29660.08}    &\textbf{28.47}            &1442              &1.38               &1065              &1.02     \\
 Assault            & 90526.44          & 200.00            & 34800           & 200.00         & 67212.67             &200.00                    &63876             &200.00             &\textbf{97155}     &\textbf{ 200.00}  \\
 Asterix            & 999080            & 99.91             & 950700          & 95.07          & 991384.42            &99.14                     &759910            &75.99              &\textbf{999999}     &\textbf{100.00}\\
 Asteroids          & 265861.2          & 2.52              & 230500          & 2.19           & 150854.61            &1.43                      &751970            &7.15               &\textbf{760005}     &\textbf{7.23}    \\
 Atlantis           & 1576068           & 14.76             & 1653600         & 15.49          & 1528841.76           &14.31                     &3803000           &35.78              &\textbf{3837300}    &\textbf{36.11}            \\
 Bank Heist         & \textbf{46285.6}  & \textbf{56.40}    & 17400           & 21.19          & 23071.5              &28.10                     &1401              &1.69               &1380       &1.66\\
 Battle Zone        & 513360            & 64.08             & 691700          & 86.35          & \textbf{934134.88}   &\textbf{116.63}           &478830            &59.77              &824360     &102.92  \\
 Beam Rider         & 128236.08         & 12.79             & 63600           & 6.33           & 300509.8             &30.03                     &162100            &16.18              &\textbf{422390}     &\textbf{42.22}            \\
 Berzerk            & 34134.8           & 3.22              & 36200           & 3.41           & \textbf{61507.83}    &\textbf{5.80}             &7607              &0.71               &14649      &1.37         \\
 Bowling            & 196.36            & 62.57             & 211.9           & 68.18          & \textbf{251.18}      &\textbf{82.37}            &201.9             &64.57              &205.2      &65.76   \\
 Boxing             & 99.16             & 99.16             & 99.7            & 99.70          & \textbf{100}         &\textbf{100.00}           &\best{100}        &\textbf{100.00}    &\textbf{100} &\textbf{100.00}      \\
 Breakout           & 795.36            & 92.04             & 559.2           & 64.65          & 790.4                &91.46                     &\best{864}        &\textbf{100.00}    &\textbf{864}     &\textbf{100}          \\
 Centipede          & 532921.84         & 40.85             & \textbf{577800} & 44.30          & 412847.86            &31.61                     &155830            &11.83              &195630    &14.89\\
 Chopper Command    &960648             & 96.06             &999900           & 99.99          &999900                &99.99                     &\best{999999}     &\textbf{100.00}    &\textbf{999999}  &\textbf{100.00} \\
 Crazy Climber      & 312768            & 144.41            & 313400          & 144.71         &\textbf{565909.85}    &\textbf{200.00}           &201000            &90.96              &241170    &110.17      \\
 Defender           & 562106            & 9.31              & 664100          & 11.01          & 677642.78            &11.23                     &893110            &14.82              &\textbf{970540}&\textbf{16.11}          \\
 Demon Attack       & 143664.6          & 9.22              & 143500          & 9.21           & 143161.44            &9.19                      &675530            &43.10              &\textbf{787985}   &\textbf{50.63}           \\
 Double Dunk        & 23.12             & 105.35            & -14.1           & 11.36          & 23.93                &107.40                    &\textbf{24}       &\textbf{107.58}    &\textbf{24}&\textbf{107.58}     \\
 Enduro             & 2376.68           & 25.02             & 2000            & 21.05          & 2367.71              &24.92                     &\best{14330}      &\textbf{150.84}    &14300     &150.53          \\
 Fishing Derby      & 81.96             & 106.74            & 32              & 76.03          & \textbf{86.97}       &\textbf{109.82}           &59                &95.08   &65        &96.31           \\
 Freeway            & \textbf{34}       & \textbf{89.47}    & 28.5            & 75.00          & 32.59                &85.76                     &\best{34}         &\textbf{89.47}     &\textbf{34} &\textbf{89.47}           \\
 Frostbite          & 11238.4           & 2.46              & 206400          & 45.37          &\textbf{541280.88}    &\textbf{119.01}           &10485             &2.29               &11330     &2.48            \\
 Gopher             & 122196            & 34.37             & 113400          & 31.89          & 117777.08            &33.12                     &\best{488830}     &\textbf{137.71}    &473560    &133.41          \\
 Gravitar           & 6750              & 4.04              & 14200           & 8.62           &\textbf{19213.96}     &\textbf{11.70}            &5905              &3.52               &5915      &3.53    \\
 Hero               & 37030.4           & 3.60              & 69400           & 6.84           &\textbf{114736.26}    &\textbf{11.38}            &38330             &3.73               &38225     &3.72    \\
 Ice Hockey         & \textbf{71.56}    & \textbf{175.34}   &-4.1             & 15.04          & 63.64                &158.56                    &44.92             &118.94             &\textbf{47.11}  &\textbf{123.54}\\
 Jamesbond          & 23266             & 51.05             & 26600           & 58.37          & 135784.96            &200.00                    &594500            &200.00             &\textbf{620780} &\textbf{200.00}   \\
 Kangaroo           & 14112             & 0.99              & \textbf{35100}  & 2.46           &24034.16              &1.68                      &14500             &1.01               &14636           &1.02\\
 Krull              & 145284.8          & 140.18            & 127400          & 122.73         & 251997.31            &200.00                    &97575             &93.63              &\textbf{594540} &\textbf{200.00}      \\
 Kung Fu Master     & 200176            & 20.00             & 212100          & 21.19          & 206845.82            &20.66                     &140440            &14.02              &\textbf{1666665}&\textbf{166.68}\\
 Montezuma Revenge  & 2504              & 0.21              & \textbf{10400}  & 0.85           &9352.01               &0.77                      &3000              &0.25               &2500            &0.21   \\
 Ms Pacman          & 29928.2           & 10.22             & 40800           & 13.97          & \textbf{63994.44}    &\textbf{21.98}            &11536             &3.87               &11573           &3.89 \\
 Name This Game     & 45214.8           & 187.21            & 23900           & 94.24          &\textbf{54386.77}     &\textbf{200.00}           &34434             &140.19             &36296           &148.31\\
 Phoenix            & 811621.6          & 20.20             & 959100          & 23.88          &908264.15             &22.61                     &894460            &22.27              &\textbf{959580} &\textbf{23.89}\\
 Pitfall            & \textbf{0}        & \textbf{0.20}     & 7800            & 7.03           &\textbf{18756.01}     &\textbf{16.62}            &0                 &0.20               &-4.3            &0.20\\
 Pong               & \textbf{21}       & \textbf{100.00}   & 19.6            & 96.64          & 20.67                &99.21                     &\best{21}         &\textbf{100.00}    &\textbf{21}     &\textbf{100.00}\\
 Private Eye        & 300               & 0.27              & \textbf{100000} & 98.23          & 79716.46             &78.30                     &15100             &14.81              &15100           &14.81\\
 Qbert              & 161000            & 6.70              & 451900          & 18.82          &\textbf{580328.14}    &\textbf{24.18}            &27800             &1.03               &28657           &1.19\\
 Riverraid          & 34076.4           & 3.28              & 36700           & 3.54           & \textbf{63318.67}    &\textbf{6.21}             &28075             &2.68               &28349           &2.70\\
 Road Runner        & 498660            & 24.47             & 128600          & 6.31           & 243025.8             &11.92                     &\best{878600}     &\textbf{43.11}     &999999          &49.06\\
 Robotank           & \textbf{132.4}    & \textbf{176.42}   & 9.1             & 9.35           &127.32                &169.54                    &108               &143.63             &113.4           &150.68\\
 Seaquest           & 999991.84         & 100.00            & \textbf{1000000}& 100.00         &999997.63             &100.00                    &943910	        &94.39              &\textbf{1000000}&\textbf{100.00}\\
 Skiing             & -29970.32         & -93.10            & -22977.9        & -42.53         & \textbf{-4202.6}     &\textbf{93.27}            &-6774             &74.67              &-6025	         &86.77\\
 Solaris            & 4198.4            & 2.69              & 4700            & 3.14           & \textbf{44199.93}    &\textbf{38.99}            &11074             &8.93               &9105            &7.14\\
 Space Invaders     & 55889             & 8.97              & 43400           & 6.96           & 48680.86             &7.81                      &140460            &22.58              &\textbf{154380} &\textbf{24.82}\\
 Star Gunner        & 521728            & 200.00            & 414600          & 200.00         &\textbf{839573.53}    &\textbf{200.00}           &465750            &200.00             &677590          &200.00\\
 Surround           & \textbf{9.96}     & \textbf{101.84}   & -9.6            & 2.04           & 9.5                  &99.49                     &-7.8              &11.22              &2.606           &64.32\\
 Tennis             & \textbf{24}       & \textbf{106.70}   & 10.2            & 75.89          & 23.84                &106.34                    &\textbf{24}       &\textbf{106.70}    &\textbf{24}     &\textbf{106.70}\\
 Time Pilot         & 348932            & 200.00            & 344700          & 200.00         &405425.31             &200.00                    &216770            &200.00             &\textbf{450810} &\textbf{200.00}\\
 Tutankham          & 393.64            & 7.11              & 191.1           & 3.34           & \textbf{2354.91}     &\textbf{43.62}            &423.9             &7.68               &418.2           &7.57\\
 Up N Down          & 542918.8          & 200.00            & 620100          & 200.00         & 623805.73            &200.00                    &\best{986440}     &\textbf{200.00}    &966590          &200.00\\
 Venture            & 1992              & 5.12              & 1700            & 4.37           &\textbf{2623.71}      &\textbf{6.74}             &2000              &5.14               &2000            &5.14\\
 Video Pinball      & 483569.72         & 0.54              & 965300          & 1.08           &\textbf{992340.74}    &\textbf{1.11}             &925830            &1.04               &978190          &1.10\\
 Wizard of Wor      & 133264            & 33.62             & 106200          & 26.76          &\textbf{157306.41}    &\textbf{39.71}            &64439             &16.18              &63735           &16.00\\
 Yars Revenge       & 918854.32         & 6.11              & 986000          & 6.55           &\textbf{998532.37}    &\textbf{6.64}             &972000            &6.46               &968090          &6.43\\
 Zaxxon             & 181372            & 200.00            & 111100          & 132.75         &\textbf{249808.9}     &\textbf{200.00}           &109140            &130.41             &216020          &200.00\\
\hline
MEAN SABER(\%)        &                   & \rtdtmeanSABER            &                 &  \ngumeanSABER         &                      & \textbf{\agentmeanSABER}  &                  & \GDIImeanSABER &      & \GDIHmeanSABER \\
\hline
MEDIAN SABER(\%)      &                   & \rtdtmedianSABER            &                 & \ngumedianSABER          &                      & \agentmedianSABER  &                  & \GDIImedianSABER &      & \textbf{\GDIHmedianSABER}  \\
\hline
\end{tabular}
\end{center}
\end{table}
\clearpage

\begin{table}[!hb]
\footnotesize
\begin{center}
\caption{Score table of  SOTA  model-based algorithms on SABER(\%). SimPLe \citep{modelbasedatari} and DreamerV2 \citep{dreamerv2} haven't evaluated all 57 Atari Games in their paper. For fairness, we set the score on those games as N/A, which will not be considered when calculating the median and mean SABER.}
\label{Tab: Score table of  SOTA  model-based algorithms on SABER.}
\setlength{\tabcolsep}{1.0pt}
\begin{tabular}{|c |c c| c c| c c| c c| c c| }
\hline
 Games              & MuZero         & SABER      & DreamerV2 & SABER   & SimPLe             & SABER          & GDI-I$^3$     & SABER & GDI-H$^3$ & SABER(\%)\\
\hline
Scale               & 20B            &              & 200M      &            & 1M               &                  & 200M     & &  200M   & \\
\hline
 Alien              & \textbf{741812.63}      & \textbf{200.00  }     &3483       & 1.29     &616.9     & 0.15    & 43384                       &  17.15      &48735    &19.27   \\
 Amidar             & \textbf{28634.39 }      & \textbf{27.49   }  &2028       & 1.94     &74.3      & 0.07    & 1442                           &  1.38       &1065     &1.02          \\
 Assault            & \textbf{143972.03}      & \textbf{200.00 }     &7679       & 88.51       &527.2     & 3.62       & 63876                  &  200.00     &97155     &200.00    \\
 Asterix            & 998425         & 99.84        &25669      & 2.55     &1128.3    & 0.09       & 759910                     &  75.99      &\textbf{999999}            &\textbf{100.00}    \\
 Asteroids          & 678558.64               & 6.45        &3064                & 0.02    &793.6              & 0.00           &751970         & 7.15   &\textbf{760005}            &\textbf{7.23}   \\
 Atlantis           & 1674767.2               & 15.69           &989207              & 9.22       &20992.5            & 0.08    &3803000    & 35.78  &\textbf{3837300}           &\textbf{36.11} \\
 Bank Heist         & 1278.98                 & 1.54        &1043                & 1.25    &34.2               & 0.02    &\best{1401}           & \best{1.69  } &1380              &1.66   \\
 Battle Zone        & \textbf{848623         }& \textbf{105.95}      &31225      & 3.87     &4031.2    & 0.47       & 478830                    &  59.77        &824360            &102.92  \\
 Beam Rider         & \textbf{454993.53}      & \textbf{45.48}      &12413      & 1.21     &621.6     & 0.03    & 162100                        &  16.18        &422390            &42.22  \\
 Berzerk            & \textbf{85932.6        }& \textbf{8.11}      &751        & 0.06     &N/A       & N/A     & 7607                           &  0.71         &14649             &1.37  \\
 Bowling            & \textbf{260.13         }& \textbf{85.60}      &48         & 8.99     &30        & 2.49    & 202                           &  64.57        &205.2             &65.76 \\
 Boxing             & \textbf{100}                     & \textbf{100.00}       &87                  & 86.99       &7.8         & 7.71       & \best{100}          & \best{100.00 }&\textbf{100}       &\textbf{100.00}\\
 Breakout           & \textbf{864}                     & \textbf{100.00}          &350                 & 40.39       &16.4     & 1.70       & \best{864}          & \best{100.00} &\textbf{864}     &100.00\\
 Centipede          & \textbf{1159049.27}     & \textbf{89.02}     &6601       & 0.35     &N/A       & N/A     & 155830                         & 11.83         &195630    &14.89\\
 Chopper Command    & 991039.7                & 99.10     &2833                & 0.20     & 979.4             & 0.02    & \best{999999}         & \best{100.00} &\textbf{999999}    &\textbf{100.00}\\
 Crazy Climber      & \textbf{458315.4}       & \textbf{200.00    }  &141424     & 62.47       & 62583.6  & 24.77   & 201000                    & 90.96         &241170    &110.17\\
 Defender           & 839642.95               & 13.93      & N/A                & N/A        & N/A               & N/A      & 893110     &14.82  &\textbf{970540}    &\textbf{16.11}\\
 Demon Attack       & 143964.26               & 9.24      & 2775              &0.17      & 208.1             & 0.00    & 675530         & 43.40  &\textbf{787985}   &\textbf{50.63}\\
 Double Dunk        & 23.94          & 107.42     & 22        &102.53        & N/A      & N/A                                               & \textbf{24}& \textbf{107.58} &\textbf{24}&\textbf{107.58} \\
 Enduro             & 2382.44                 & 25.08       & 2112              &22.23         & N/A               & N/A     & \best{14330}    & \best{150.84}  &14300     &150.53\\
 Fishing Derby      & \textbf{91.16}          & \textbf{112.39     }              &93.24          &200.00      &-90.7     & 0.61    & 59        & 92.89         &65        &96.31\\
 Freeway            & 33.03                   & 86.92       & \textbf{34}                 &\textbf{89.47}          &16.7               & 43.95   & \best{34}      & \best{89.47}  &\textbf{34}        &\textbf{89.47}\\
 Frostbite          & \textbf{631378.53}      & \textbf{138.82}     & 15622    &3.42       &236.9     & 0.04    & 10485                        & 2.29           &11330     &2.48\\
 Gopher             & 130345.58               & 36.67      & 53853             &15.11         &596.8              & 0.10       & \best{488830} & \best{137.71}  &473560    &133.41\\
 Gravitar           & \textbf{6682.7     }    & \textbf{4.00    }  & 3554     &2.08      &173.4     & 0.00    & 5905                           & 3.52           &5915      &3.53\\
 Hero               & \textbf{49244.11}       & \textbf{4.83    }  & 30287    &2.93      &2656.6    & 0.16        & 38330                      & 3.73           &38225     &3.72\\
 Ice Hockey         & \textbf{67.04      }    & \textbf{165.76  }  & 29        &85.17         &-11.6     & -0.85       &44.92              &118.94        &47.11           &123.54\\
 Jamesbond          & 41063.25                & 90.14     & 9269              &20.30         &100.5     & 0.16    &594500              &200.00  &\textbf{620780}          &\textbf{200.00}\\
 Kangaroo           & \textbf{16763.6        }& \textbf{1.17}  & 11819     &0.83      &51.2      & 0.00    & 14500                              & 1.01          &14636           &1.02\\
 Krull              & 269358.27      & 200.00 & 9687     &7.89       &2204.8    & 0.59    & 97575                        & 93.63          &\textbf{594540}          &\textbf{200.00}\\
 Kung Fu Master     & \textbf{204824         }& \textbf{20.46}  & 66410    &6.62       &14862.5   & 1.46    & 140440                           & 14.02          &1666665          &166.68\\
 Montezuma Revenge  & 0                       & 0.00         & 1932              &0.16       &N/A       & N/A     & \best{3000}                & \best{0.25  }  &2500            &0.21\\
 Ms Pacman          & \textbf{243401.1 }      & \textbf{83.89   }  & 5651     &1.84       &1480      & 0.40    & 11536                         & 3.87           &11573           &3.89\\
 Name This Game     & \textbf{157177.85}      & \textbf{200.00  }  & 14472    &53.12          &2420.7    & 0.56    & 34434                     & 140.19         &36296           &148.31\\
 Phoenix            & \textbf{955137.84}      & \textbf{23.78   }  & 13342     &0.31       &N/A       & N/A     & 894460                        & 22.27         &959580          &23.89\\
 Pitfall            & \textbf{0}              & \textbf{0.20}               & -1                 &0.20       &N/A       & N/A     & \best{0}             & \best{0.20}   &-4.3            &0.20\\
 Pong               & \textbf{21}             & \textbf{100.00}             & 19                 &95.20          & 12.8     & 80.34   & \best{21}        & \best{100.00} &\textbf{21}   &\textbf{100.00}\\
 Private Eye        & \textbf{15299.98 }      & \textbf{15.01}     & 158       &0.13       & 35       & 0.01    & 15100                         & 14.81         &15100           &14.81\\
 Qbert              & \textbf{72276          }& \textbf{3.00}      & 162023    &6.74       & 1288.8   & 0.05    & 27800                         & 1.15          &28657           &1.19\\
 Riverraid          & \textbf{323417.18}      & \textbf{32.25}     & 16249    &1.49       & 1957.8   & 0.06    & 28075                         & 2.68           &28349           &2.70\\
 Road Runner        & 613411.8                & 30.10              & 88772             &4.36       & 5640.6   & 0.28       & 878600     & 43.11   &\textbf{999999}          &\textbf{49.06}\\
 Robotank           & \textbf{131.13}         & \textbf{174.70}    & 65        &85.09          & N/A      & N/A     & 108                       & 143.63        &113.4           &150.68\\
 Seaquest           & 999976.52               & 100.00             & 45898             &4.58       & 683.3             & 0.06 &943910	             &94.39&\textbf{1000000}          &\textbf{100.00}\\
 Skiing             & \textbf{-29968.36}      & \textbf{-93.09}      & -8187    &64.45          & N/A      & N/A     & -6774                   & 74.67          &-6025	          &86.77\\
 Solaris            & 56.62                   & -1.07              & 883                &-0.32          & N/A               & N/A   & \best{11074}  & \best{8.93 }&9105            &7.14\\
 Space Invaders     & 74335.3                 & 11.94                 & 2611               &0.40       & N/A               & N/A    & 140460 & 22.58&\textbf{154380}          &\textbf{24.82}    \\
 Star Gunner        &549271.7       & 200.00     & 29219    &37.21          & N/A      & N/A     & 465750      & 200.00                         &\textbf{677590}          &\textbf{200.00}\\
 Surround           & \textbf{9.99       }    & \textbf{101.99}     & N/A       &N/A         & N/A      & N/A     & -8          & 11.22                            &2.606           &64.32\\
 Tennis             & 0       & 53.13      & 23        &104.46        & N/A      & N/A     & \textbf{24}          & \textbf{106.70}       &\textbf{24}           &\textbf{106.70}\\
 Time Pilot         & \textbf{476763.9}       & \textbf{200.00}         & 32404    &46.71         & N/A      & N/A     & 216770      & 200.00    &450810          &200.00     \\
 Tutankham          & \textbf{491.48     }    & \textbf{8.94}   & 238       &4.22         & N/A      & N/A     & 424         & 7.68              &418.2           &7.57\\
 Up N Down          & 715545.61               & 200.00            & 648363            &200.00        & 3350.3            & 3.42   & \best{986440}        & \best{200.00} &966590        &200.00    \\
 Venture            & 0.4                     & 0.00        & 0                  &0.00       & N/A               & N/A     & \best{2000}          & \best{5.23}   &2000            &5.14    \\
 Video Pinball      & \textbf{981791.88}      & \textbf{1.10}      & 22218    &0.02     & N/A      & N/A     & 925830      & 1.04                                &978190          &1.10\\
 Wizard of Wor      & \textbf{197126         }& \textbf{49.80}      & 14531    &3.54     & N/A      & N/A     & 64439       & 16.14                             &63735           &16.00\\
 Yars Revenge       & 553311.46               & 3.67      & 20089             &0.11     & 5664.3            & 0.02    & \best{972000}        & \best{6.46}      &968090          &6.43\\
 Zaxxon             & \textbf{725853.9}       & \textbf{200.00}      & 18295    &21.83          & N/A      & N/A     & 109140      & 130.41                     &216020          &200.00\\
\hline
MEAN SABER(\%) &               &\textbf{\muzeromeanSABER}  &         &  \dreamermeanSABER &                                          & \simplemeanSABER  &     & \GDIImeanSABER &      &\GDIHmeanSABER\\
\hline
MEDIAN SABER(\%) &            &\muzeromedianSABER   &         & \dreamermedianSABER  &                                          & \simplemedianSABER  &     & \GDIImedianSABER &      & \textbf{\GDIHmedianSABER}  \\
\hline
\end{tabular}
\end{center}
\end{table}
\clearpage

\begin{table}[!hb]
\footnotesize
\begin{center}
\caption{Score table of other SOTA  algorithms on SABER(\%). Go-Explore \citep{goexplore} and Muesli \citep{muesli}.}
\label{tab: Score table of other SOTA  algorithms on SABER.}
\setlength{\tabcolsep}{1.0pt}
\begin{tabular}{| c | c c |  c c| c c| c c|}
\hline
 Games           & Muesli              & SABER          & Go-Explore              & SABER                     & GDI-I$^3$ & SABER        & GDI-H$^3$ & SABER         \\
\hline
Scale            & 200M           &           & 10B                     &                             & 200M      &        & 200M      &           \\
\hline    
 Alien           &139409          &55.30               &\textbf{959312}       &\textbf{200.00}                & 43384             &17.15     &48735             &19.27         \\
 Amidar          &\textbf{21653}  &\textbf{20.78}      &19083                 &18.32                          & 1442              &1.38      &1065              &1.02          \\
 Assault         &36963           &200.00              &30773                 &200.00                         & 63876             &200.00  &\textbf{97155}&\textbf{200.00}        \\
 Asterix         &316210          &31.61               &999500       &99.95                 & 759910            &75.99      &\textbf{999999}            &\textbf{100.00}     \\
 Asteroids       &484609          &4.61                &112952                &1.07                           & 751970     &7.15&\textbf{760005}            &\textbf{7.23}           \\
 Atlantis        &1363427         &12.75               &286460                &2.58                           & 3803000    &35.78&\textbf{3837300}           &\textbf{36.11}          \\
 Bank Heist      &1213            &1.46                &\textbf{3668}         &\textbf{4.45  }                & 1401              &1.69        &1380              &1.66      \\
 Battle Zone     &414107          &51.68               &\textbf{998800}       &\textbf{124.70}                & 478830            &59.77       &824360            &102.92      \\
 Beam Rider      &288870          &28.86               &371723       &37.15                 & 162100            &16.18       &\textbf{422390}            &\textbf{42.22}      \\
 Berzerk         &44478           &4.19                &\textbf{131417}       &\textbf{12.41 }                & 7607              &0.71        &14649             &1.37      \\
 Bowling         &191             &60.64               &\textbf{247}           &\textbf{80.86 }                & 202               &64.57      &205.2             &65.76       \\
 Boxing          &99              &99.00               &91                     &90.99                          & \best{100}        &\best{100.00} &\textbf{100}       &\textbf{100.00}           \\
 Breakout        &791             &91.53               &774                    &89.56                          & \best{864}        &\best{100.00} &\textbf{864}     &\textbf{100.00}          \\
 Centipede       &\textbf{869751} &\textbf{66.76}      &613815                &47.07                          & 155830            &11.83        &195630    &14.89  \\
 Chopper Command &101289       &10.06               &996220                &99.62                          & \best{999999}     &\best{100.00}   &\textbf{999999}    &\textbf{100.00}       \\
 Crazy Climber   &175322       &78.68               &235600       &107.51                & 201000            &90.96           &\textbf{241170}    &\textbf{110.17} \\
 Defender        &629482       &10.43               &N/A                    &N/A                            & 893110     &14.82   &\textbf{970540}    &\textbf{16.11}        \\
 Demon Attack    &129544       &8.31                &239895                 &15.41                          & 675530     &43.40   &\textbf{787985}   &\textbf{50.63}       \\
 Double Dunk     &-3           &39.39               &\textbf{24}            &\textbf{107.58}                         & \best{24}         &\best{107.58}  &\textbf{24}        &\textbf{107.58}         \\
 Enduro          &2362         &24.86               &1031                   &10.85                          & \best{14330}      &\best{150.84}  &14300     &150.53         \\
 Fishing Derby   &51           &87.71               &\textbf{67}            &\textbf{97.54 }                & 59                &92.89          &65        &96.31  \\
 Freeway         &33           &86.84               &\textbf{34}            &\textbf{89.47 }                & \best{34}         &\best{89.47}   &\textbf{34}        &\textbf{89.47}         \\
 Frostbite       &301694       &66.33               &\textbf{999990}       &\textbf{200.00}                & 10485             &2.29            &11330     &2.48\\
 Gopher          &104441       &29.37               &134244                &37.77                          & \best{488830}     &\best{137.71}   &473560    &133.41       \\
 Gravitar        &11660        &7.06                &\textbf{13385}        &\textbf{8.12}                  & 5905              &3.52            &5915      &3.53\\
 Hero            &37161        &3.62                &37783                  &3.68                           &38330      &3.73    &\textbf{38225}     &\textbf{3.72}       \\
 Ice Hockey      &25           &76.69               &33                     &93.64                          &44.92              &118.94       &\textbf{47.11}           &\textbf{123.54}          \\
 Jamesbond       &19319        &42.38               &200810                &200.00                         &594500              &200.00   &\textbf{620780}          &\textbf{200.00}     \\
 Kangaroo        &14096        &0.99                &\textbf{24300}        &\textbf{1.70}                  & 14500             &1.01            &14636              &1.02\\
 Krull           &34221        &31.83               &63149                 &60.05                          & 97575      &93.63    &\textbf{594540}    &\textbf{200.00}       \\
 Kung Fu Master  &134689       &13.45               &24320                 &2.41                           & 140440     &14.02    &\textbf{1666665}          &\textbf{166.68}        \\
 Montezuma Revenge  &2359      &0.19                &\textbf{24758}        &\textbf{2.03}                  & 3000              &0.25            &2500            &0.21\\
 Ms Pacman          &65278     &22.42               &\textbf{456123}       &\textbf{157.30}                & 11536             &3.87            &11573           &3.89\\
 Name This Game     &105043    &200.00              &\textbf{212824}       &\textbf{200.00}                & 34434             &140.19          &36296           &148.31\\
 Phoenix        &805305        &20.05               &19200                 &0.46                           & 894460     &22.27    &\textbf{959580}          &\textbf{23.89}\\
 Pitfall        &0             &0.20                &\textbf{7875}          &\textbf{7.09   }               & 0                 &0.2            &-4.3            &0.20 \\
 Pong           &20            &97.60               &\textbf{21}            &\textbf{100.00 }               & \best{21}         &\best{100}     &\textbf{21}              &\textbf{100.00}       \\
 Private Eye    &10323         &10.12               &\textbf{69976}        &\textbf{68.73  }               & 15100             &14.81           &15100           &14.81  \\
 Qbert          &157353        &6.55                &\textbf{999975}       &\textbf{41.66  }               & 27800             &1.15            &28657           &1.19\\
 Riverraid      &\textbf{47323}&\textbf{4.60}       &35588                 &3.43                           & 28075             &2.68            &28349           &2.70\\
 Road Runner    &327025        &16.05               &999900        &49.06                 & 878600            &43.11          &\textbf{999999}          &\textbf{49.06}\\
 Robotank       &59            &76.96               &\textbf{143}           &\textbf{190.79 }               & 108               &143.63         &113.4           &150.68\\
 Seaquest       &815970        &81.60               &539456                 &53.94                 &943910	             &94.39   &\textbf{1000000}          &\textbf{100.00}\\
 Skiing         &-18407        &-9.47               &\textbf{-4185}        &\textbf{93.40  }               & -6774             &74.67           &-6025	         &86.77\\
 Solaris        &3031          &1.63                &\textbf{20306}        &\textbf{17.31  }               & 11074             &8.93            &9105            &7.14\\
 Space Invaders &59602         &9.57                &93147                 &14.97                          & 140460     &22.58    &\textbf{154380}          &\textbf{24.82}\\
 Star Gunner    &214383        &200.00              &609580      &200.00                         & 465750     &200.00&\textbf{677590}          &\textbf{200.00}     \\
 Surround       &\textbf{9}    &\textbf{96.94}      &N/A                    &N/A                            & -8         &11.22                 &2.606           &64.32\\
 Tennis         &12            &79.91               &\best{24}              &\best{106.7}                   & \best{24}         &\best{106.70   }&\textbf{24}           &\textbf{106.70}            \\
 Time Pilot     &359105 &200.00   &183620                &200.00                         & 216770     & 200.00               &\textbf{450810}          &\textbf{200.00}\\
 Tutankham      &252           &4.48                &\textbf{528}           &\textbf{9.62}                  & 424               &7.68           &418.2           &7.57\\
 Up N Down      &649190        &200.00              &553718                &200.00                         & \best{986440}     &\best{11.9785}  &966590        &200.00         \\
 Venture        &2104          &5.41                &\textbf{3074}         &\textbf{7.90}                  & 2035              &5.23            &2000            &5.14\\
 Video Pinball  &685436        &0.77                &\textbf{999999}       &\textbf{1.12}                  & 925830            &1.04            &978190          &1.10\\
 Wizard of Wor  &93291         &23.49               &\textbf{199900}       &\textbf{50.50}                 & 64293             &16.14           &63735           &16.00\\
 Yars Revenge   &557818        &3.70                &\textbf{999998}       &\textbf{6.65}                  & 972000            &6.46            &968090          &6.43\\
 Zaxxon         &65325         &78.04               &18340                 &21.88                          & 109140         &130.41   &\textbf{216020} &\textbf{200.00}        \\
\hline    
MEAN SABER(\%)  &              & \mueslimeanSABER              &                       & \textbf{\goexploremeanSABER}                &                   & \GDIImeanSABER &      &\GDIHmeanSABER\\
\hline
MEDIAN SABER(\%)&              &  \mueslimedianSABER             &                       &\goexploremedianSABER                 &                   & \GDIImedianSABER &      & \textbf{\GDIHmedianSABER}  \\
\hline
\end{tabular}
\end{center}
\end{table}

\clearpage

\subsection{Atari Games Learning Curves}
\label{app: Atari Games Learning Curves}

\subsubsection{Atari Games Learning Curves of GDI-I$^3$}
\renewcommand{\thesubfigure}{\arabic{subfigure}.}
\begin{figure}[!ht] 
    \subfigure[Alien]{
    \includegraphics[width=0.3\textwidth]{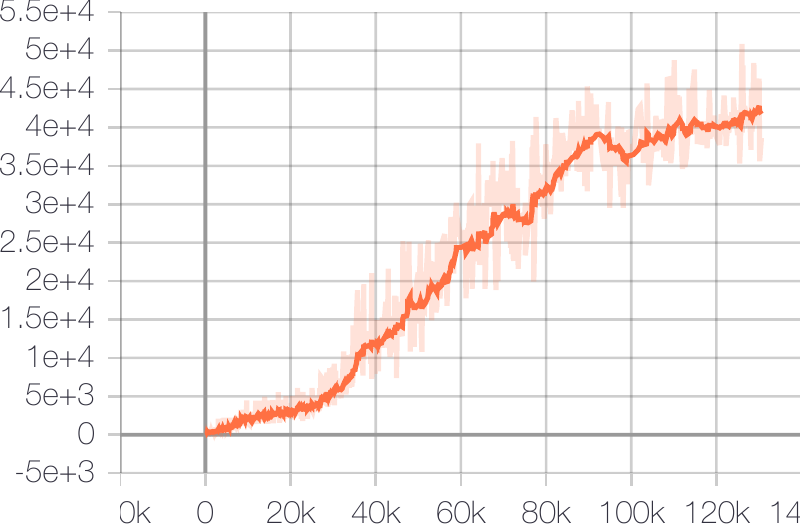}
    }
    \subfigure[Amidar]{
    \includegraphics[width=0.3\textwidth]{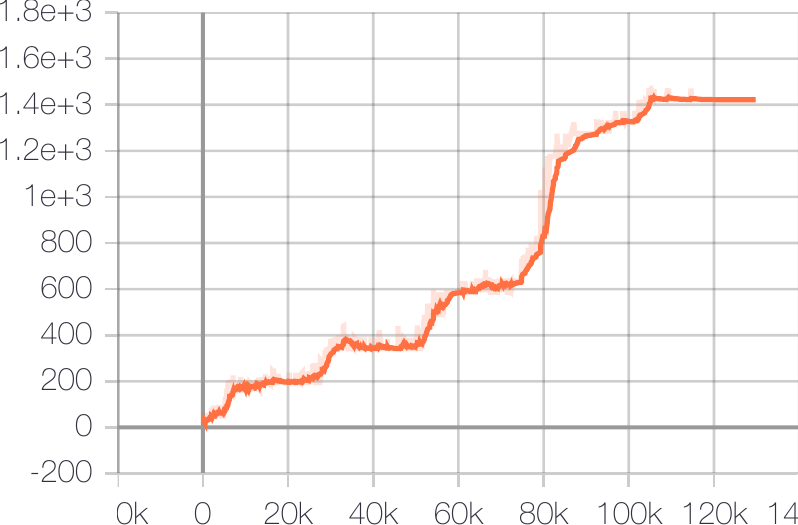}
    }
    \subfigure[Assault]{
    \includegraphics[width=0.3\textwidth]{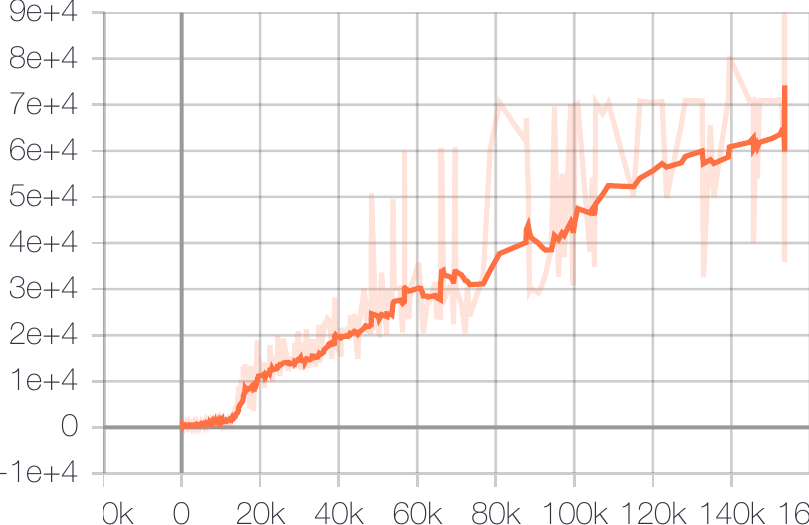}
    }
\end{figure}

\begin{figure}[!ht]
    \subfigure[Asterix]{
    \includegraphics[width=0.3\textwidth]{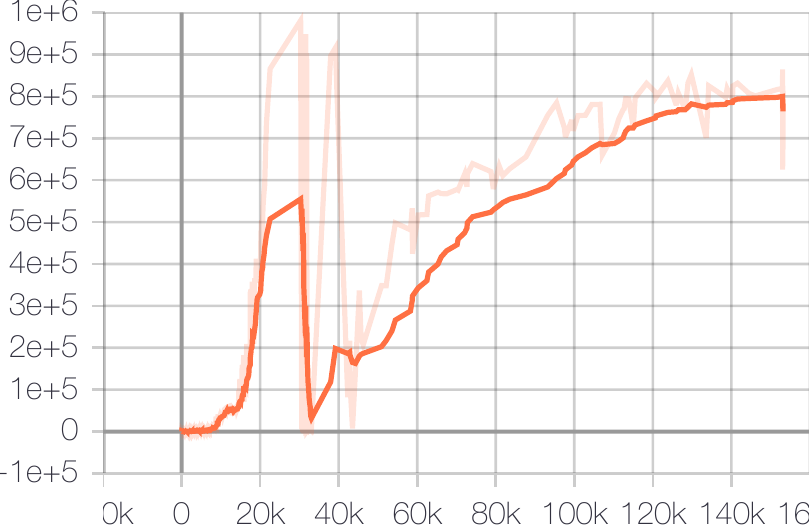}
    }
    \subfigure[Asteroids]{
    \includegraphics[width=0.3\textwidth]{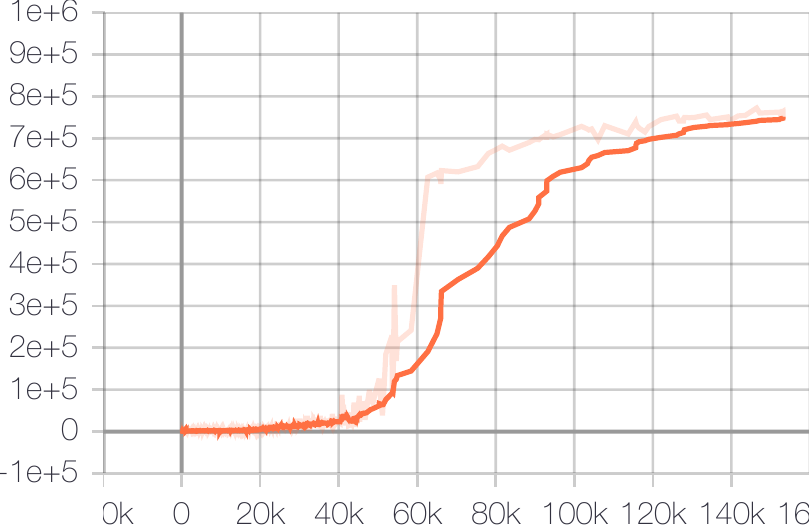}
    }
    \subfigure[Atlantis]{
    \includegraphics[width=0.3\textwidth]{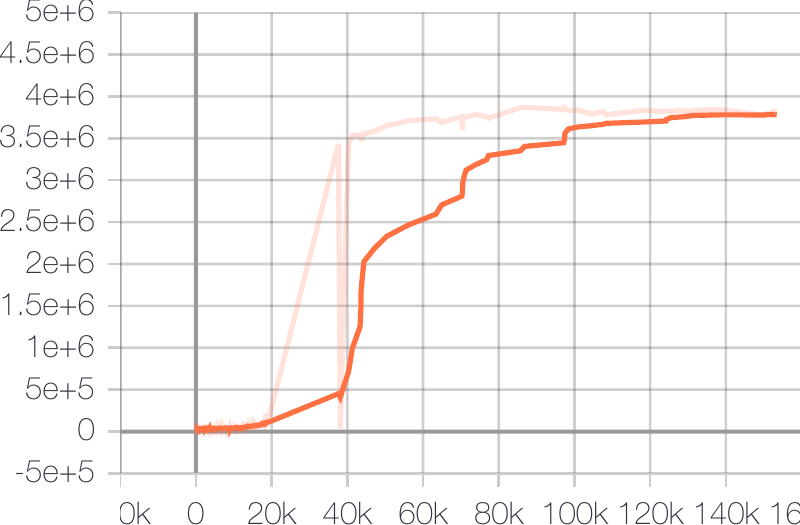}
    }
\end{figure}

\begin{figure}[!ht]
    \subfigure[Bank\_Heist]{
    \includegraphics[width=0.3\textwidth]{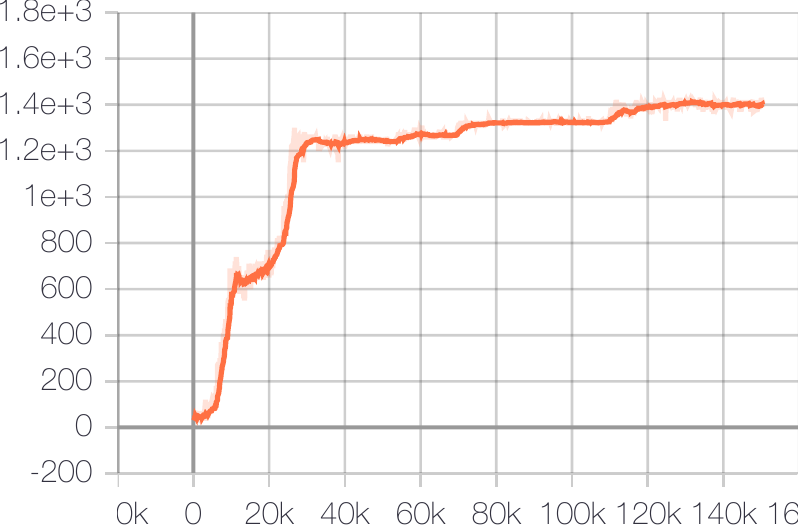}
    }
    \subfigure[Battle\_Zone]{
    \includegraphics[width=0.3\textwidth]{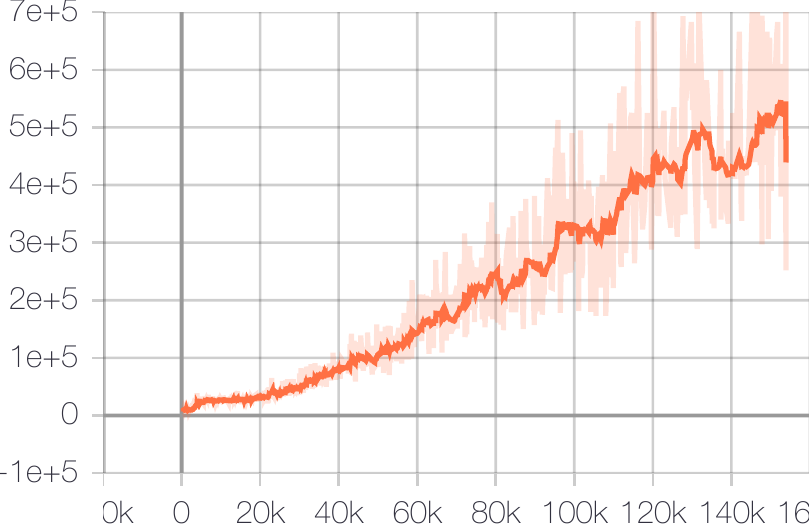}
    }
    \subfigure[Beam\_Rider]{
    \includegraphics[width=0.3\textwidth]{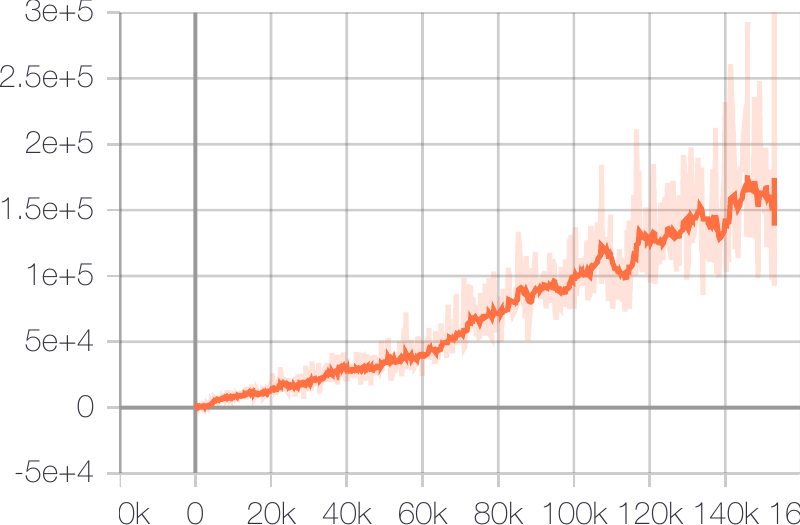}
    }
\end{figure}

\begin{figure}[!ht]
    \subfigure[Berzerk]{
    \includegraphics[width=0.3\textwidth]{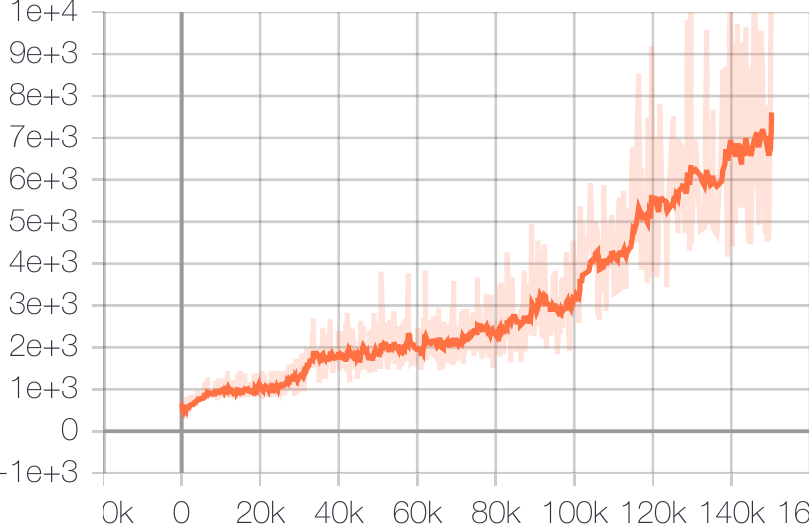}
    }
    \subfigure[Bowling]{
    \includegraphics[width=0.3\textwidth]{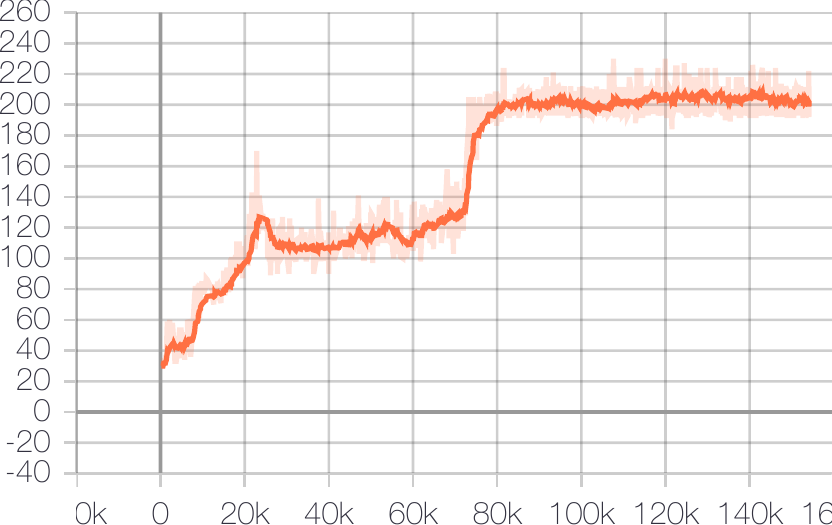}
    }
    \subfigure[Boxing]{
    \includegraphics[width=0.3\textwidth]{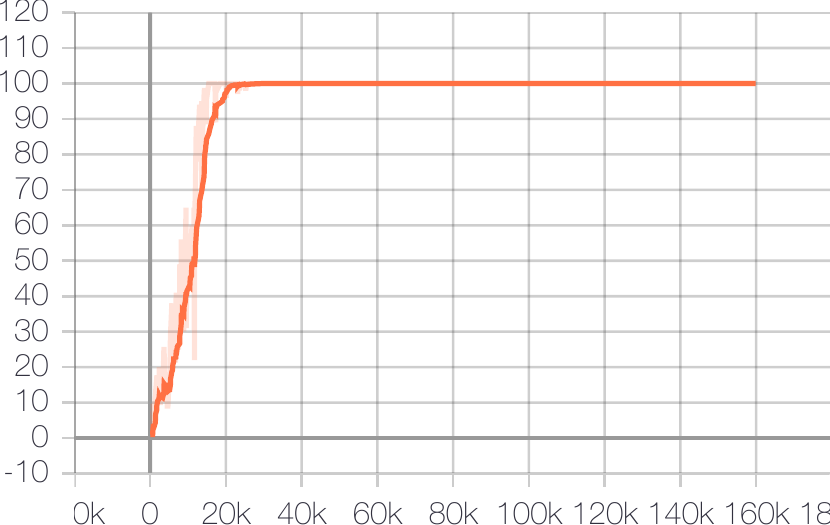}
    }
\end{figure}

\begin{figure}[!ht]
    \subfigure[Breakout]{
    \includegraphics[width=0.3\textwidth]{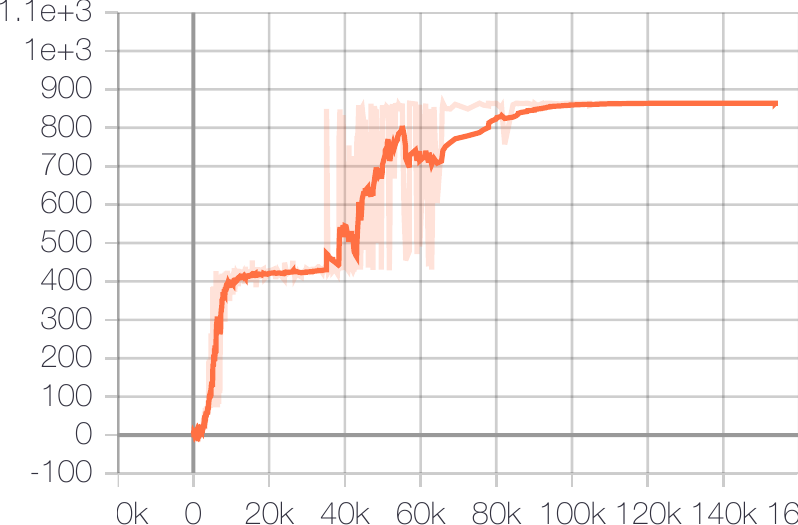}
    }
    \subfigure[Centipede]{
    \includegraphics[width=0.3\textwidth]{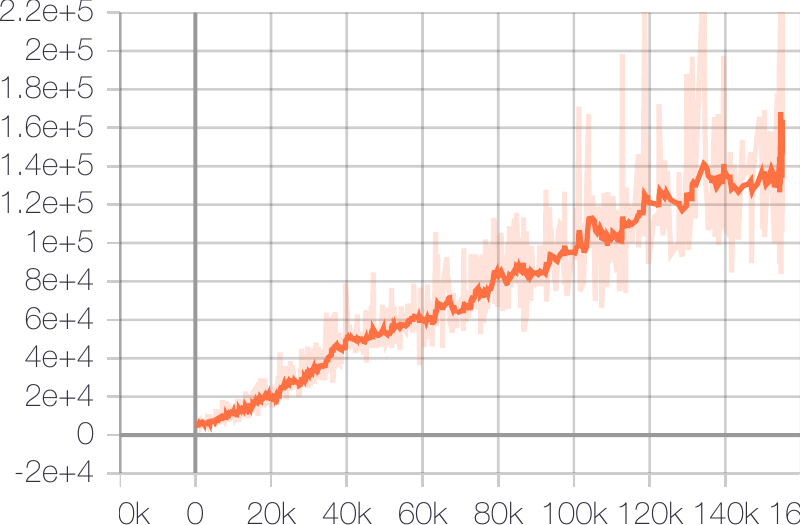}
    }
    \subfigure[Chopper\_Command]{
    \includegraphics[width=0.3\textwidth]{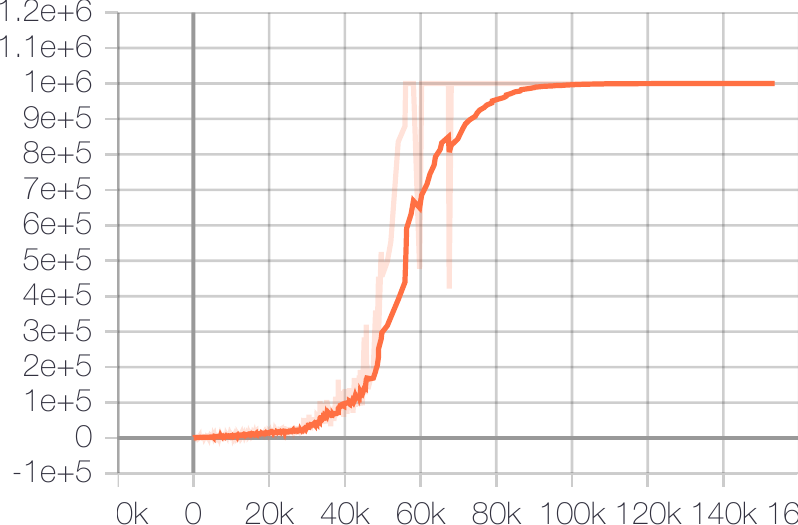}
    }
\end{figure}

\begin{figure}[!ht]
    \subfigure[Crazy\_Climber]{
    \includegraphics[width=0.3\textwidth]{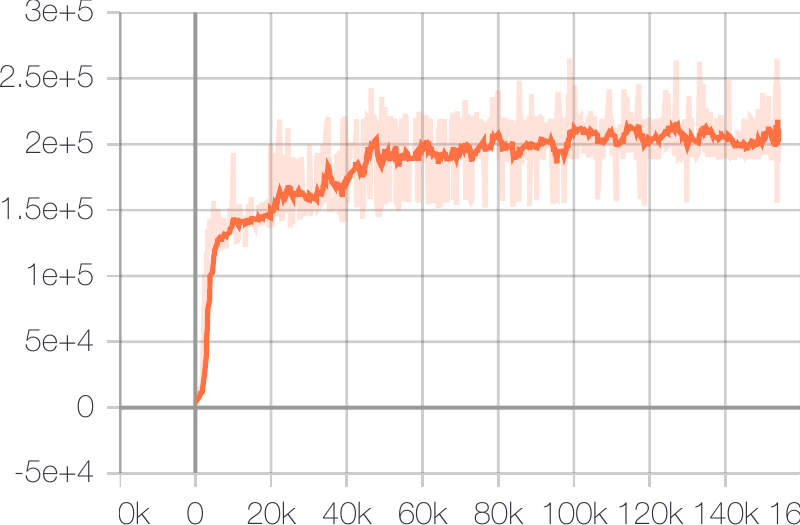}
    }
    \subfigure[Defender]{
    \includegraphics[width=0.3\textwidth]{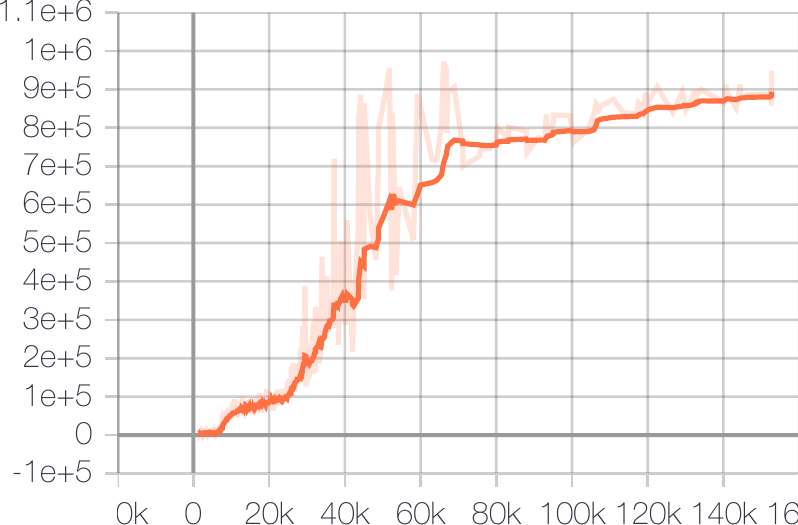}
    }
    \subfigure[Demon\_Attack]{
    \includegraphics[width=0.3\textwidth]{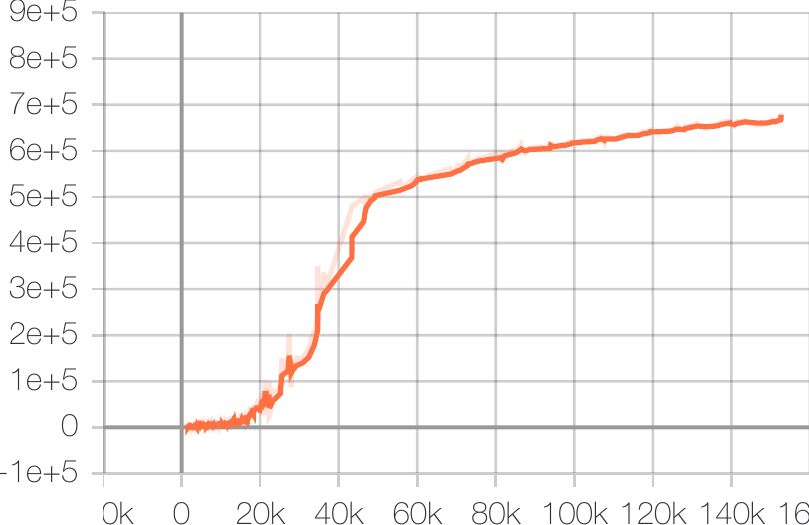}
    }
\end{figure}

\begin{figure}[!ht]
    \subfigure[Double\_Dunk]{
    \includegraphics[width=0.3\textwidth]{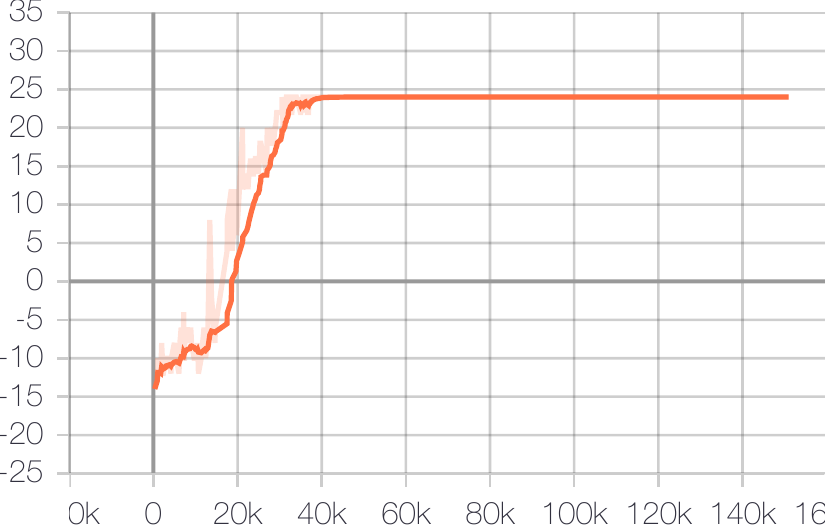}
    }
    \subfigure[Enduro]{
     \includegraphics[width=0.3\textwidth]{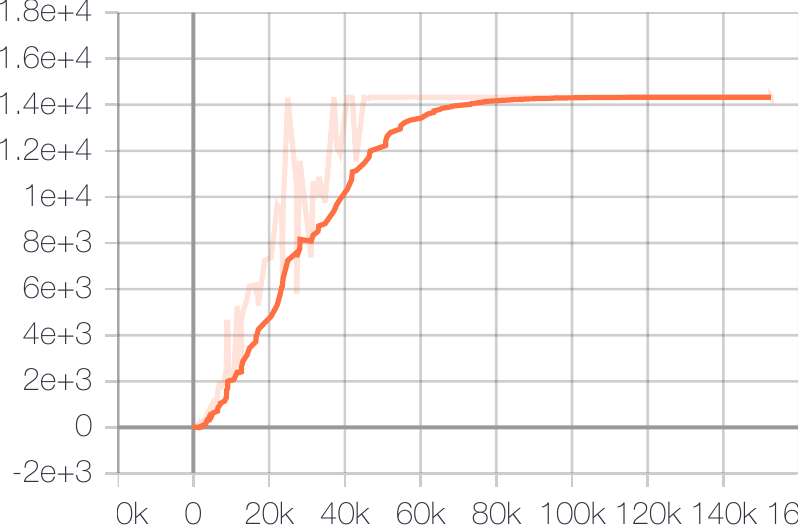}
    }
    \subfigure[Fishing\_Derby]{
    \includegraphics[width=0.3\textwidth]{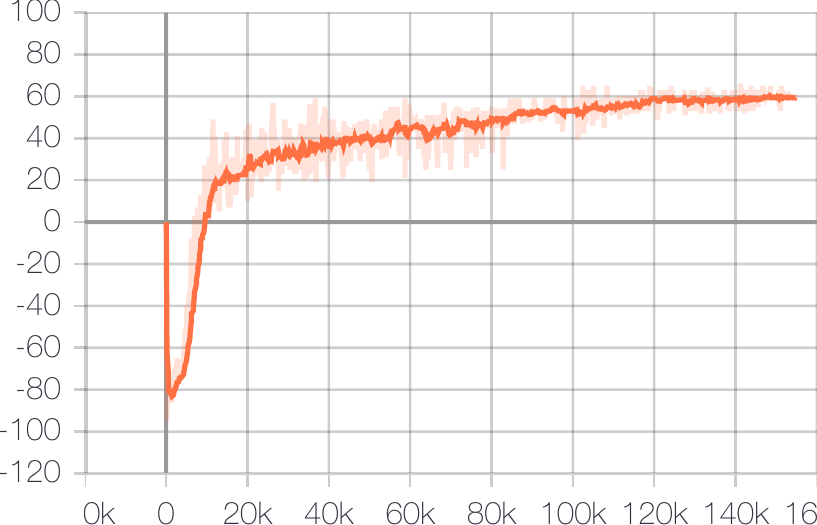}
    }
\end{figure}

\begin{figure}[!ht]
    \subfigure[Freeway]{
    \includegraphics[width=0.3\textwidth]{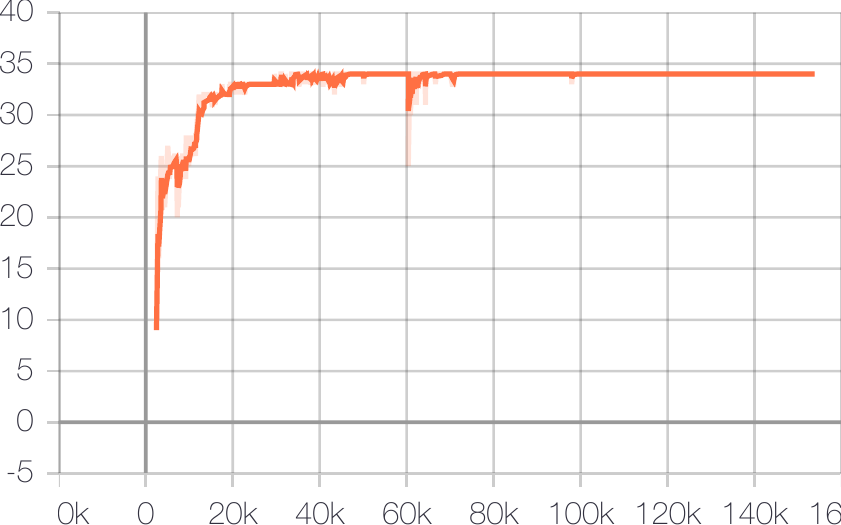}
    }
    \subfigure[Frostbite]{
    \includegraphics[width=0.3\textwidth]{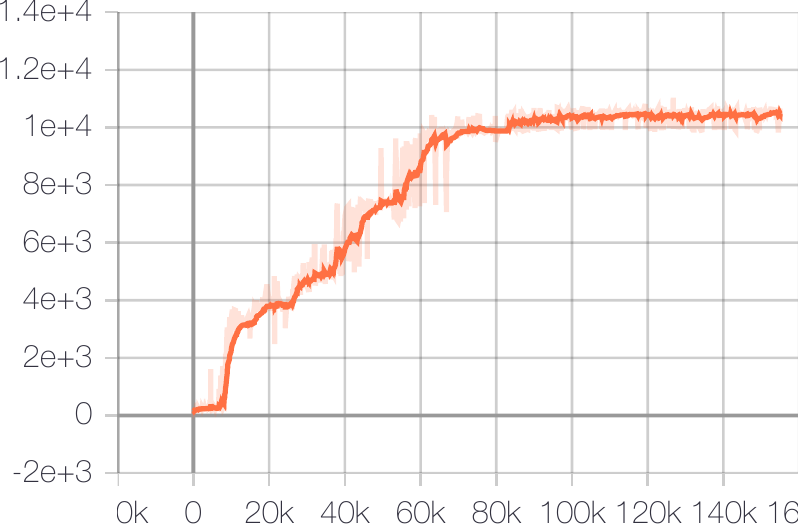}
    }
    \subfigure[Gopher]{
    \includegraphics[width=0.3\textwidth]{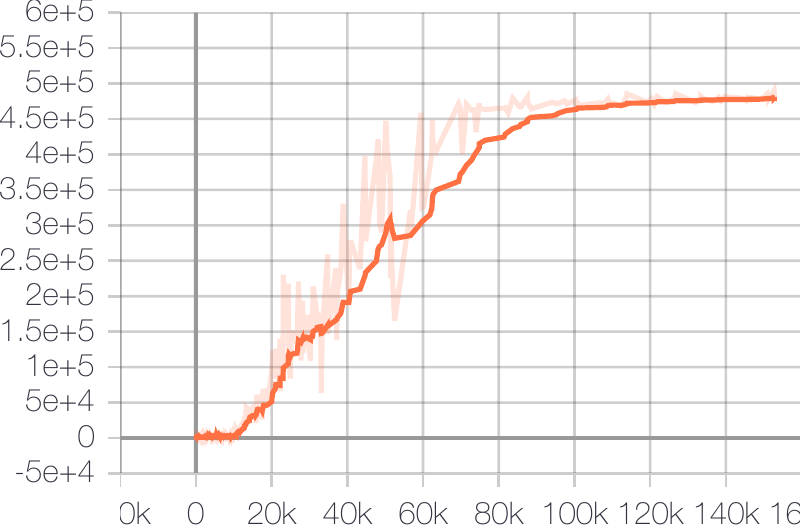}
    }
\end{figure}

\begin{figure}[!ht]
    \subfigure[Gravitar]{
    \includegraphics[width=0.3\textwidth]{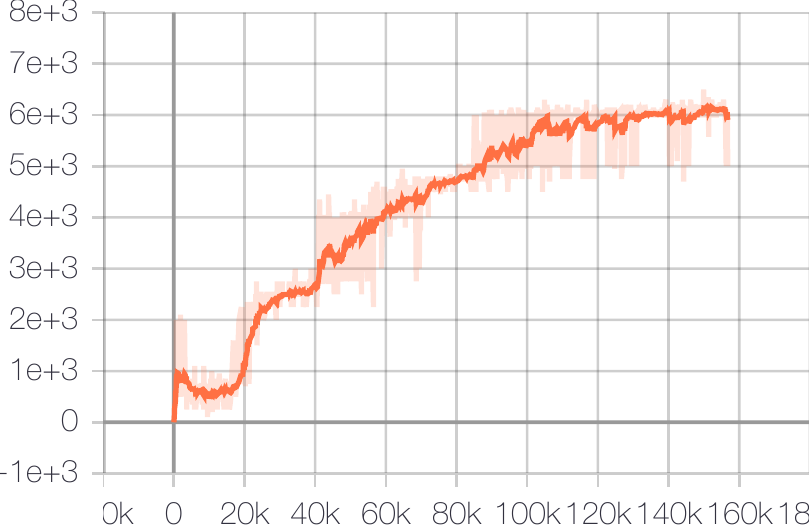}
    }
    \subfigure[Hero]{
    \includegraphics[width=0.3\textwidth]{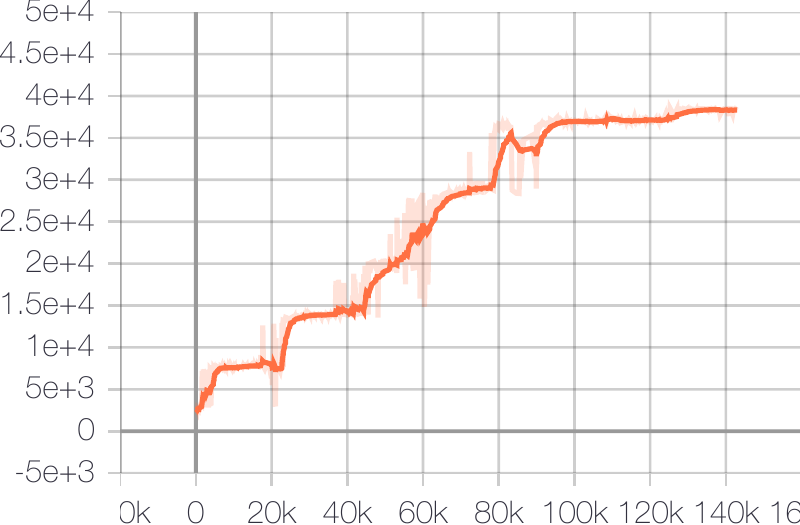}
    }
    \subfigure[Ice\_Hockey]{
    \includegraphics[width=0.3\textwidth]{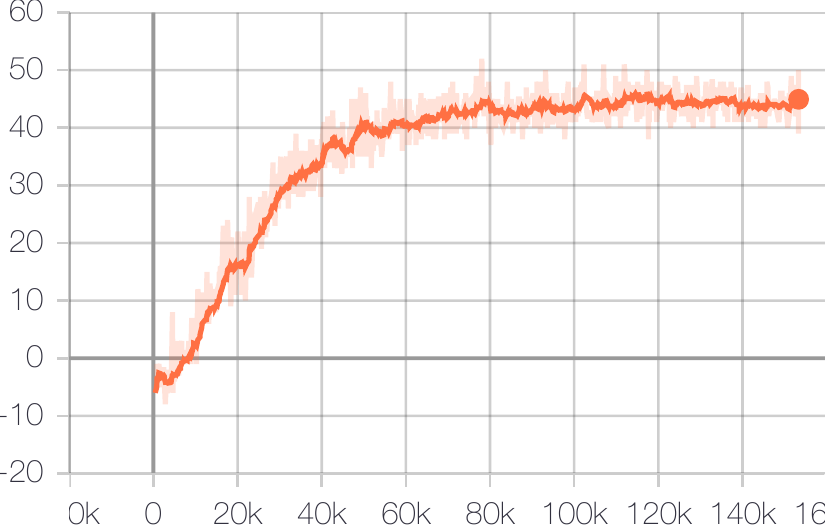}
    }
\end{figure}

\begin{figure}[!ht]
    \subfigure[Jamesbond]{
    \includegraphics[width=0.3\textwidth]{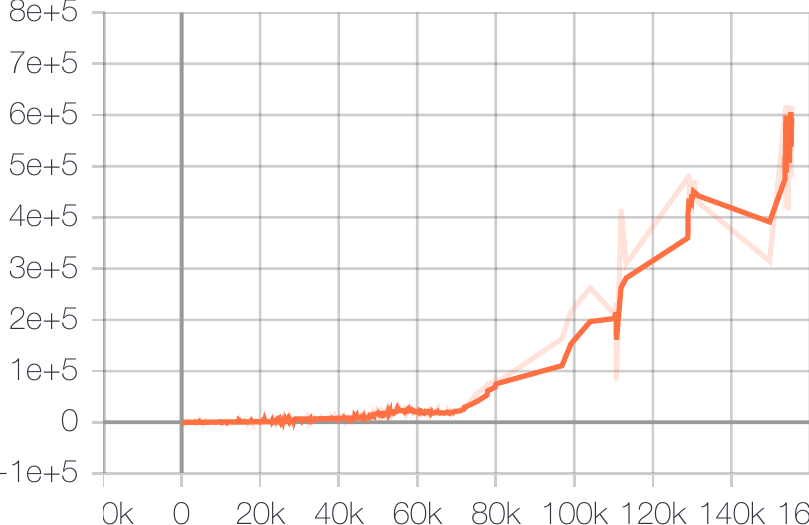}
    }
    \subfigure[Kangaroo]{
    \includegraphics[width=0.3\textwidth]{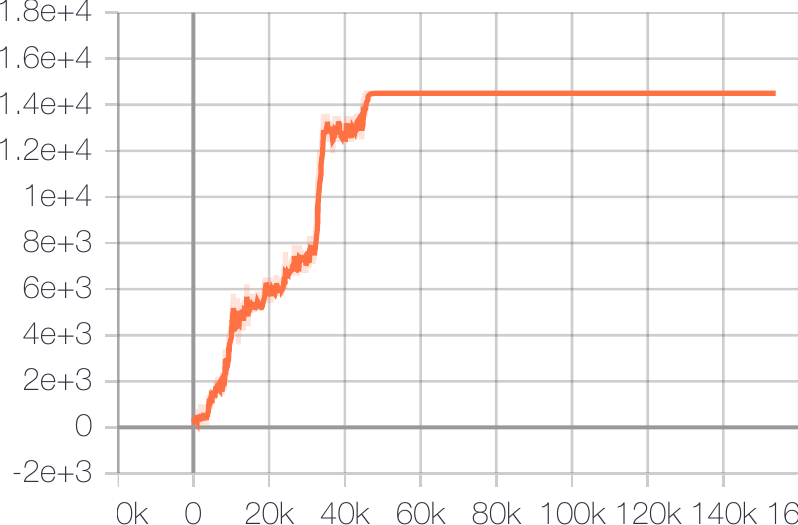}
    }
    \subfigure[Krull]{
    \includegraphics[width=0.3\textwidth]{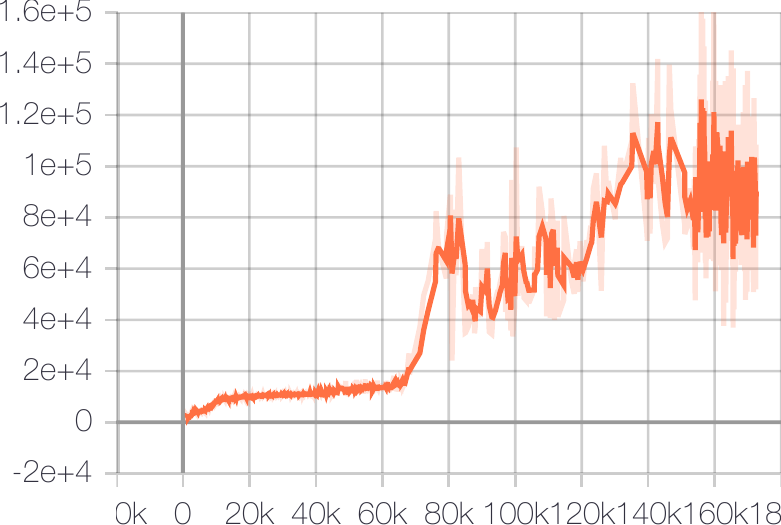}
    }
\end{figure}

\begin{figure}[!ht]
    \subfigure[Kung\_Fu\_Master]{
    \includegraphics[width=0.3\textwidth]{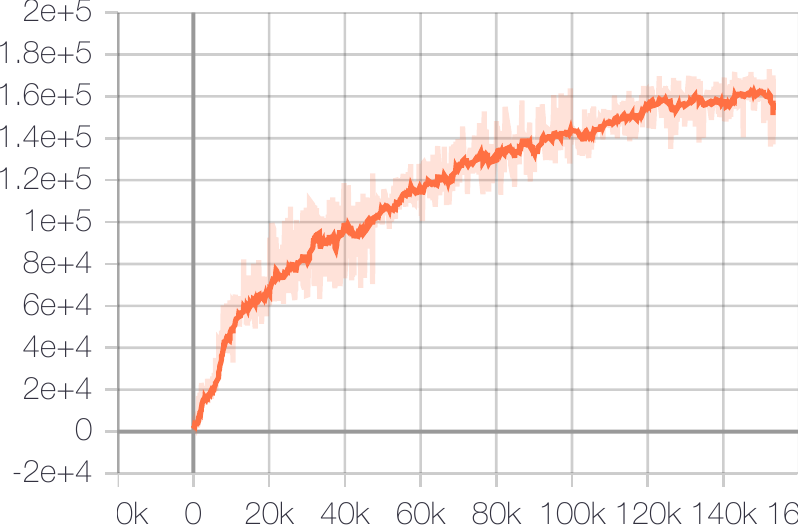}
    }
    \subfigure[Montezuma\_Revenge]{
    \includegraphics[width=0.3\textwidth]{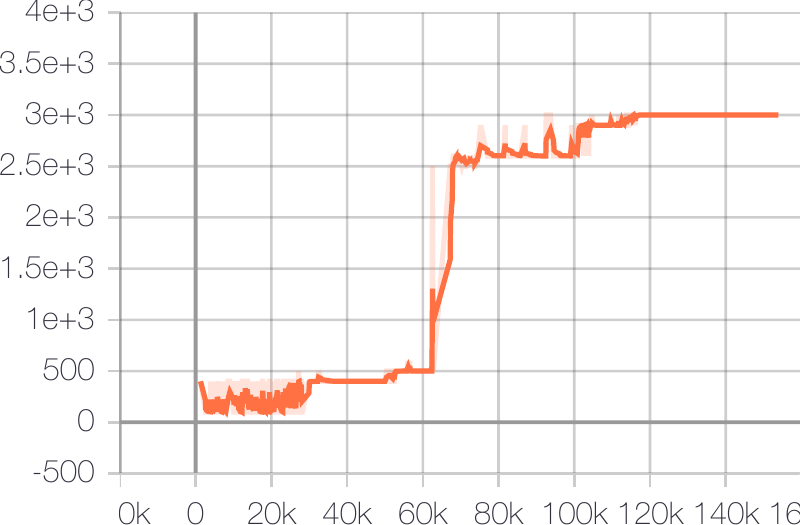}
    }
    \subfigure[Ms\_Pacman]{
    \includegraphics[width=0.3\textwidth]{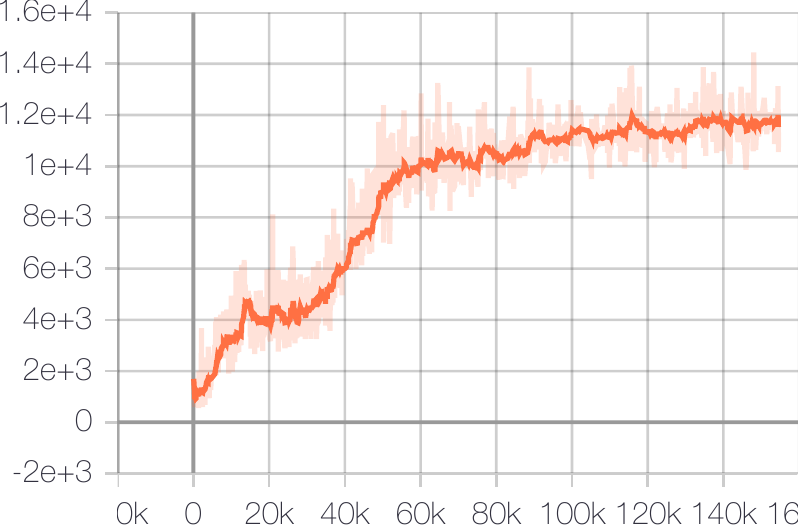}
    }
\end{figure}

\begin{figure}[!ht]
    \subfigure[Name\_This\_Game]{
    \includegraphics[width=0.3\textwidth]{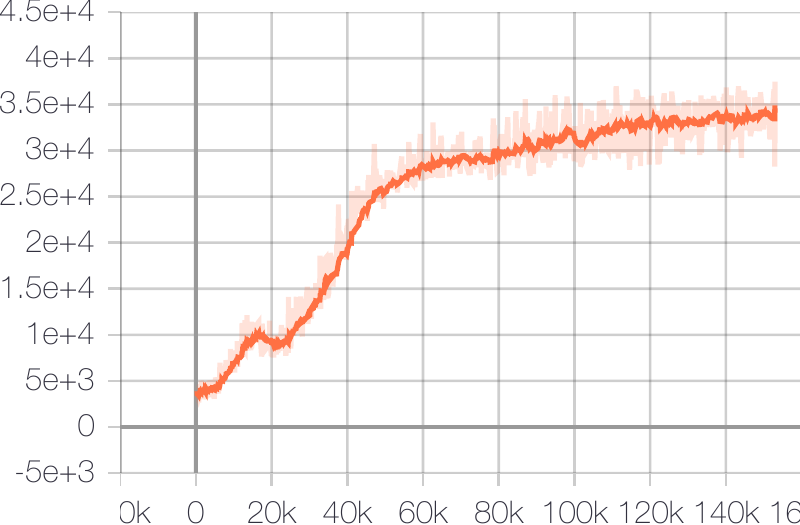}
    }
    \subfigure[Phoenix]{
    \includegraphics[width=0.3\textwidth]{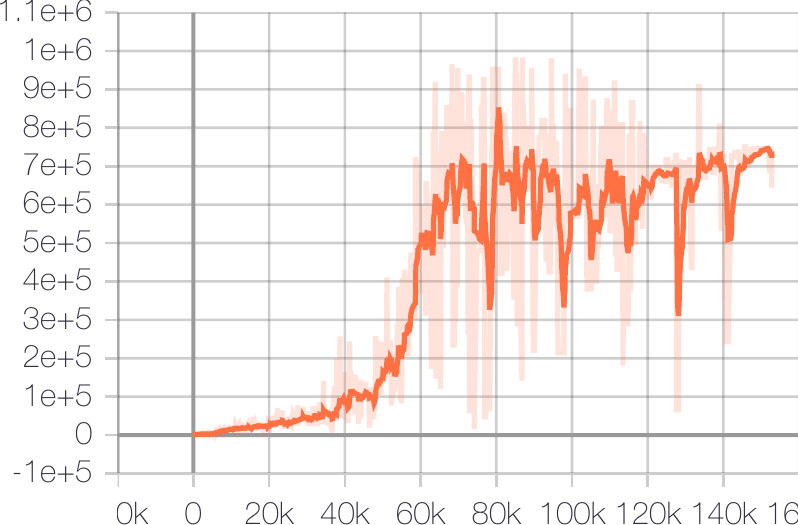}
    }
    \subfigure[Pitfall]{
    \includegraphics[width=0.3\textwidth]{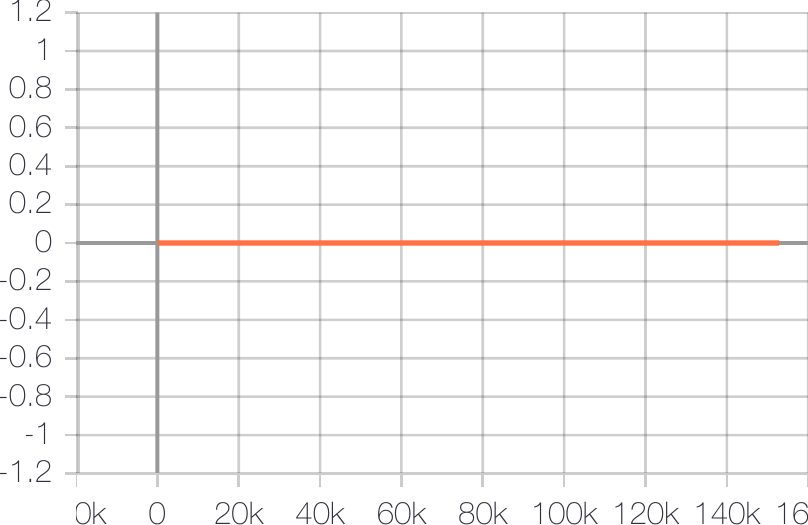}
    }
\end{figure}

\begin{figure}[!ht]
    \subfigure[Pong]{
    \includegraphics[width=0.3\textwidth]{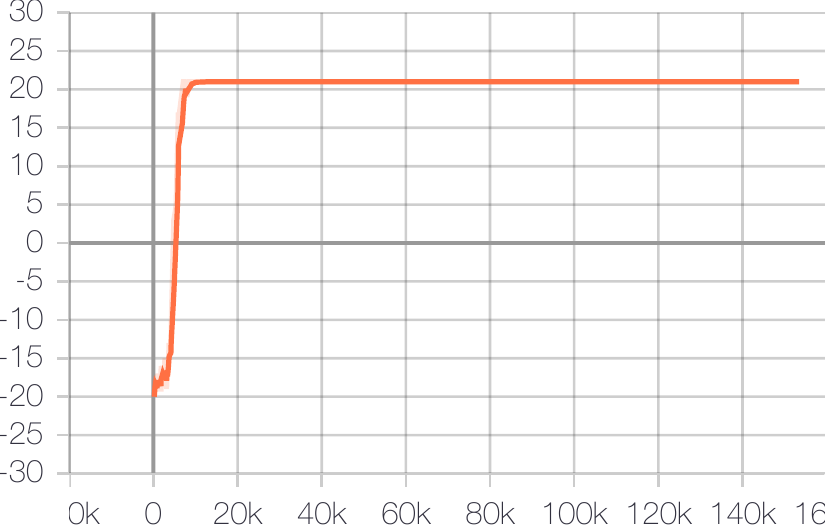}
    }
    \subfigure[Private\_Eye]{
    \includegraphics[width=0.3\textwidth]{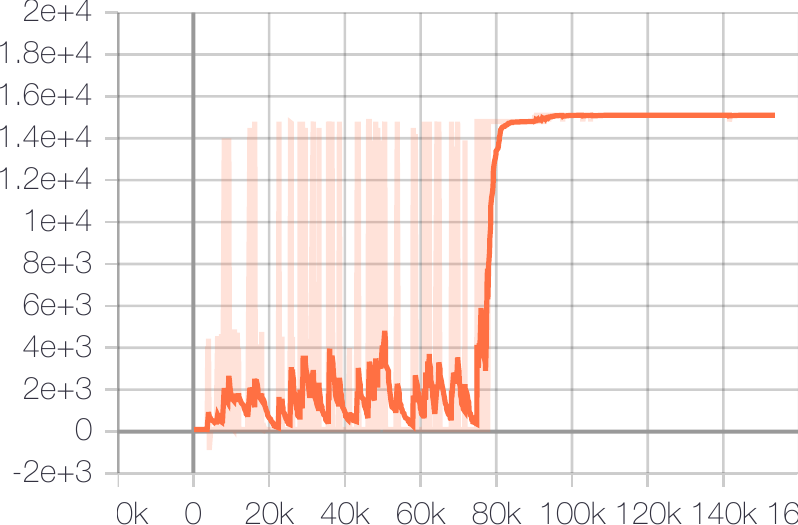}
    }
    \subfigure[Qbert]{
    \includegraphics[width=0.3\textwidth]{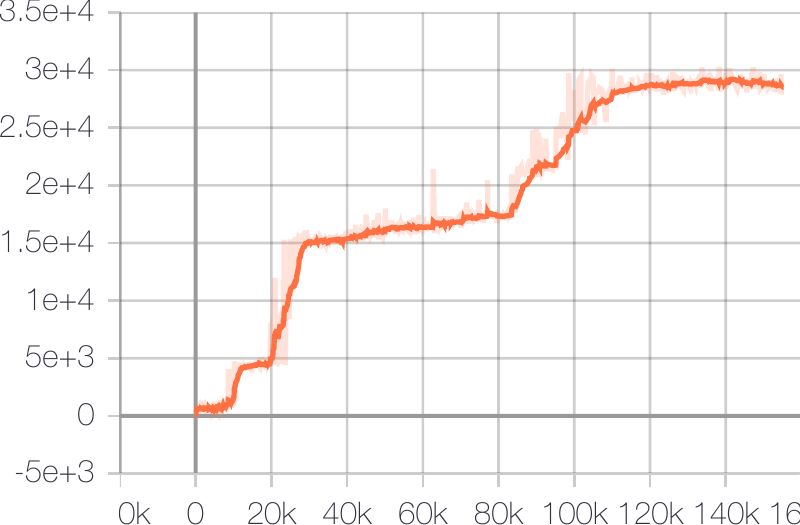}
    }
\end{figure}

\begin{figure}[!ht]
    \subfigure[Riverraid]{
    \includegraphics[width=0.3\textwidth]{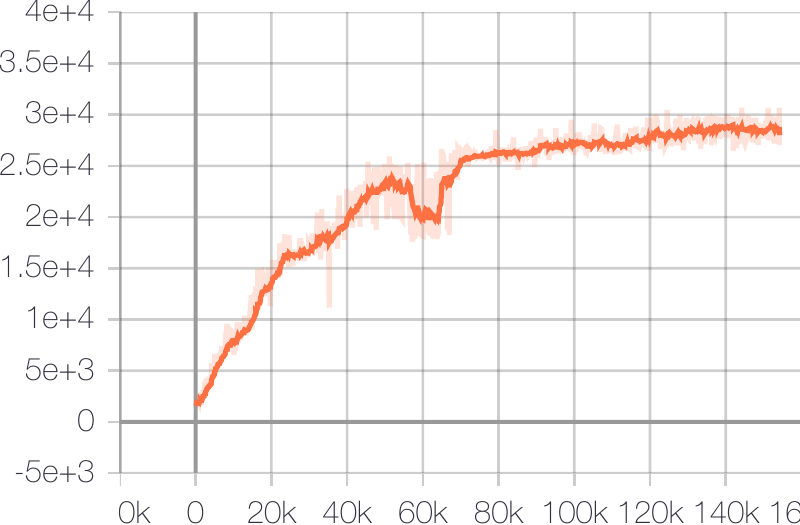}
    }
    \subfigure[Road\_Runner]{
    \includegraphics[width=0.3\textwidth]{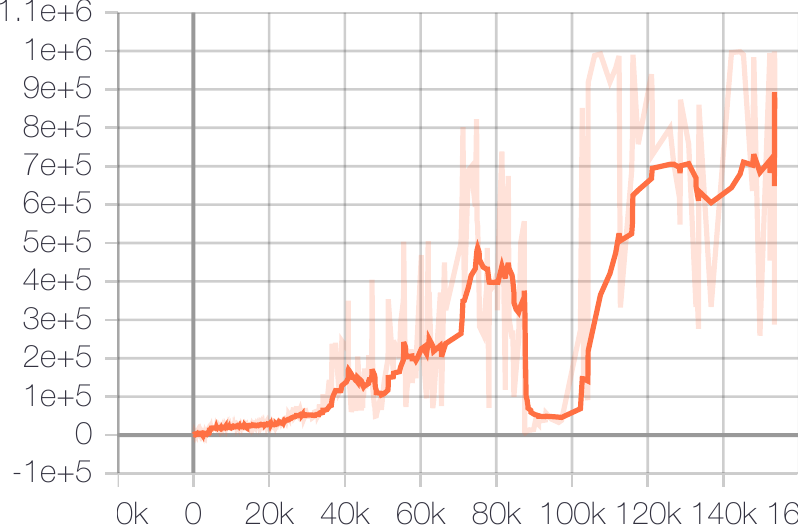}
    }
    \subfigure[Robotank]{
    \includegraphics[width=0.3\textwidth]{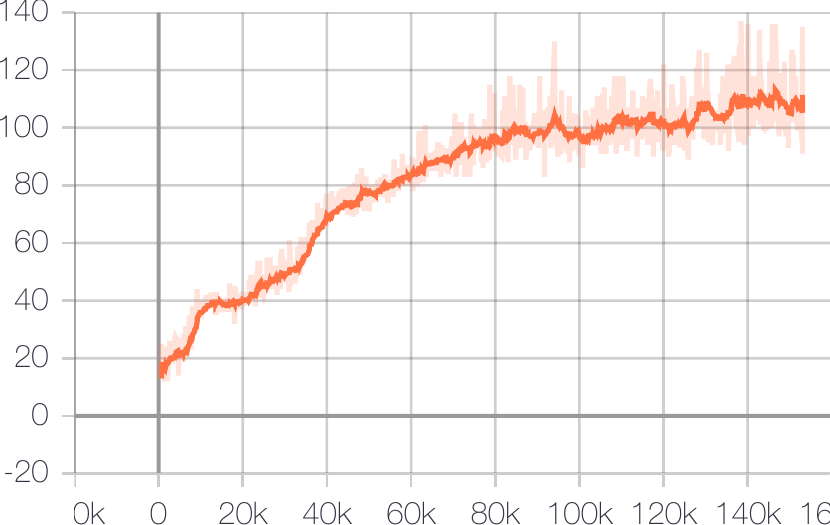}
    }
\end{figure}

\begin{figure}[!ht]
    \subfigure[Seaquest]{
    \includegraphics[width=0.3\textwidth]{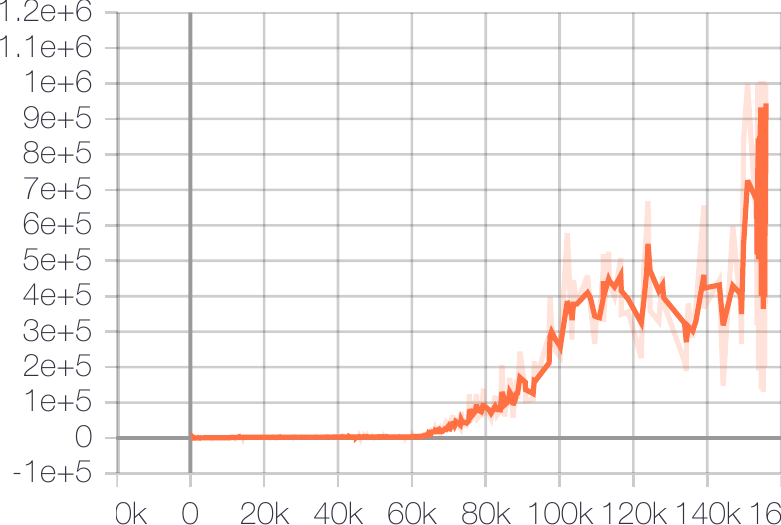}
    }
    \subfigure[Skiing]{
    \includegraphics[width=0.3\textwidth]{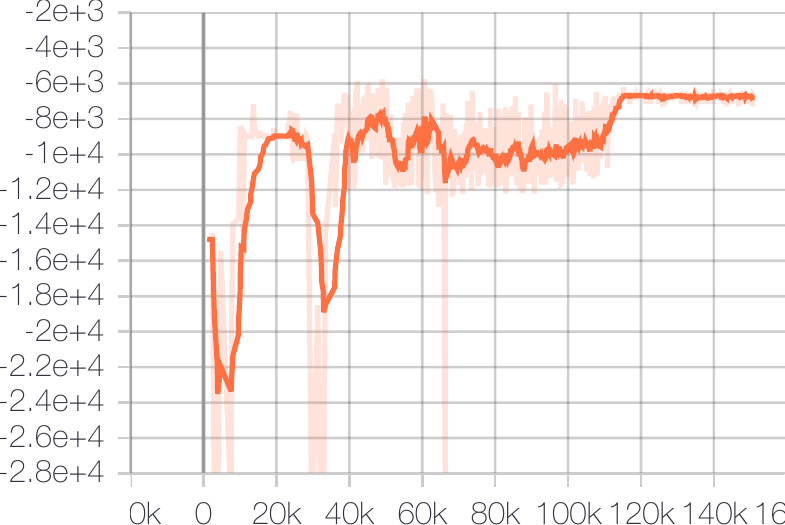}
    }
    \subfigure[Solaris]{
    \includegraphics[width=0.3\textwidth]{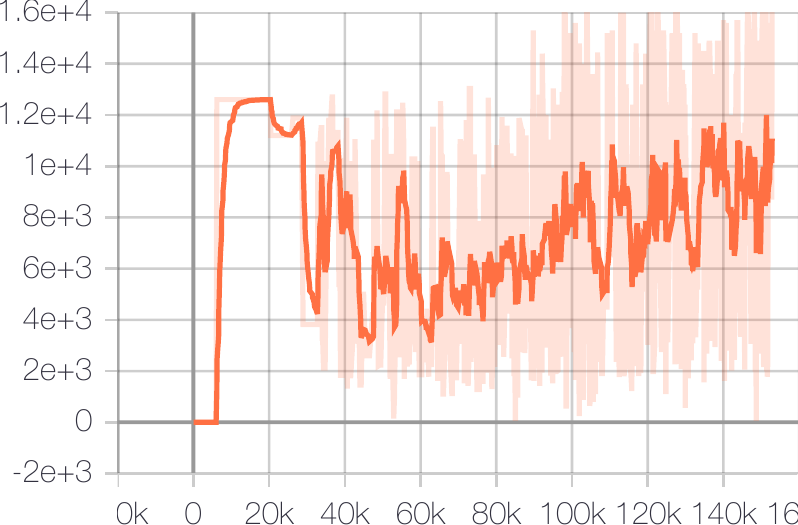}
    }
\end{figure}

\begin{figure}[!ht]
    \subfigure[Space\_Invaders]{
    \includegraphics[width=0.3\textwidth]{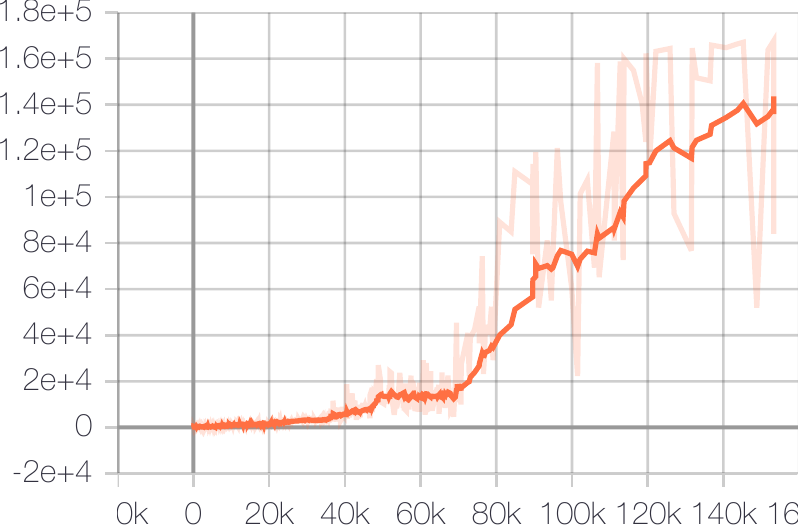}
    }
    \subfigure[Star\_Gunner]{
    \includegraphics[width=0.3\textwidth]{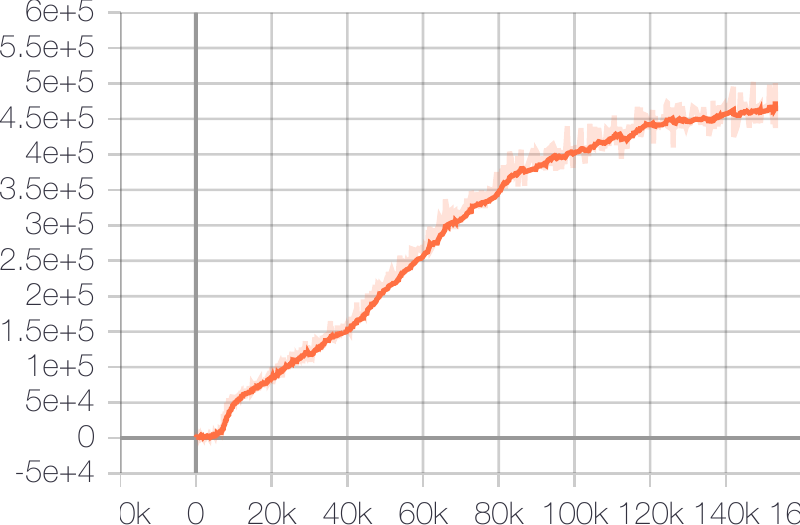}
    }
    \subfigure[Surround]{
    \includegraphics[width=0.3\textwidth]{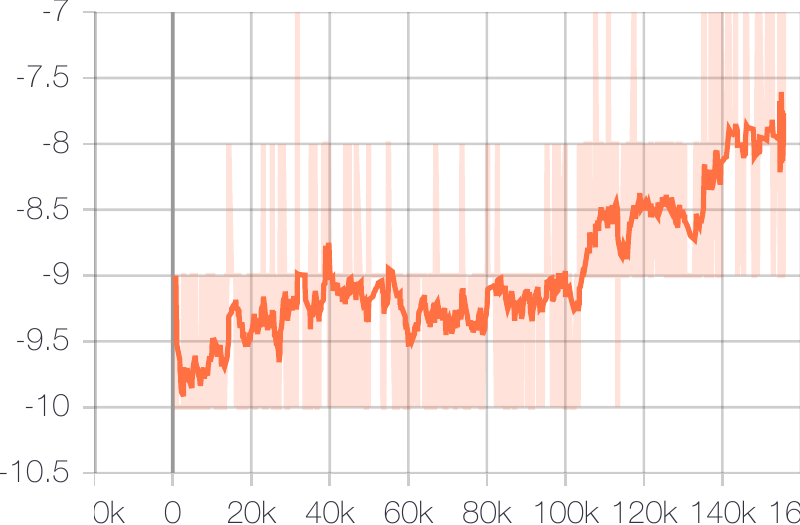}
    }
\end{figure}

\begin{figure}[!ht]
    \subfigure[Tennis]{
    \includegraphics[width=0.3\textwidth]{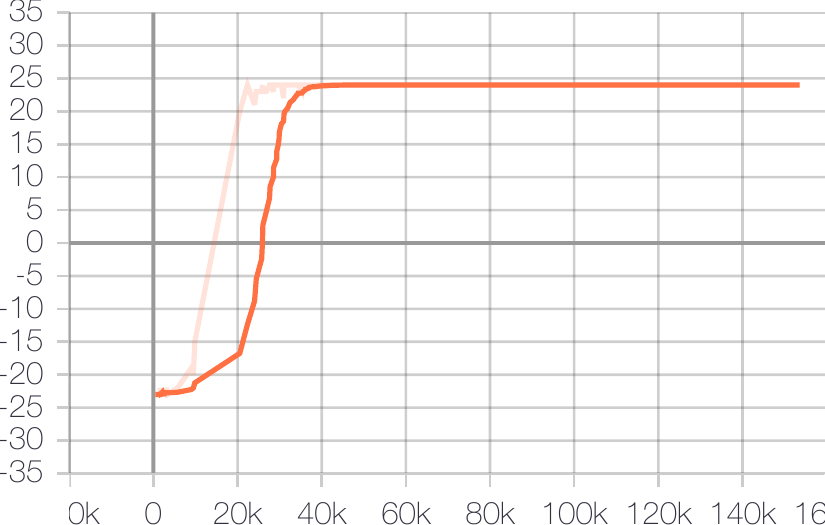}
    }
    \subfigure[Time\_Pilot]{
    \includegraphics[width=0.3\textwidth]{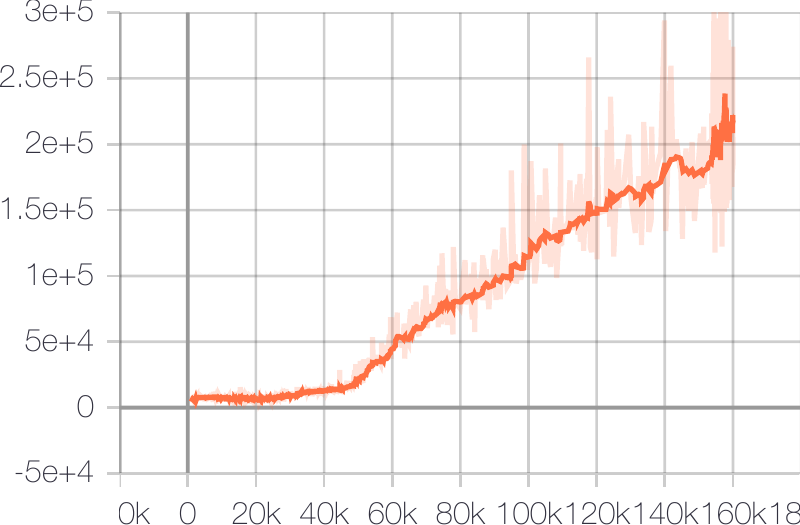}
    }
    \subfigure[Tutankham]{
    \includegraphics[width=0.3\textwidth]{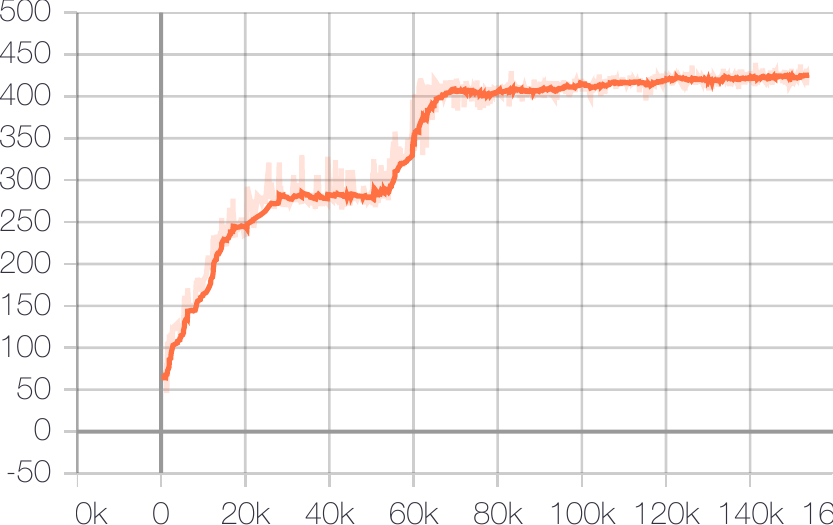}
    }
\end{figure}

\begin{figure}[!ht]
    \subfigure[Up\_N\_Down]{
    \includegraphics[width=0.3\textwidth]{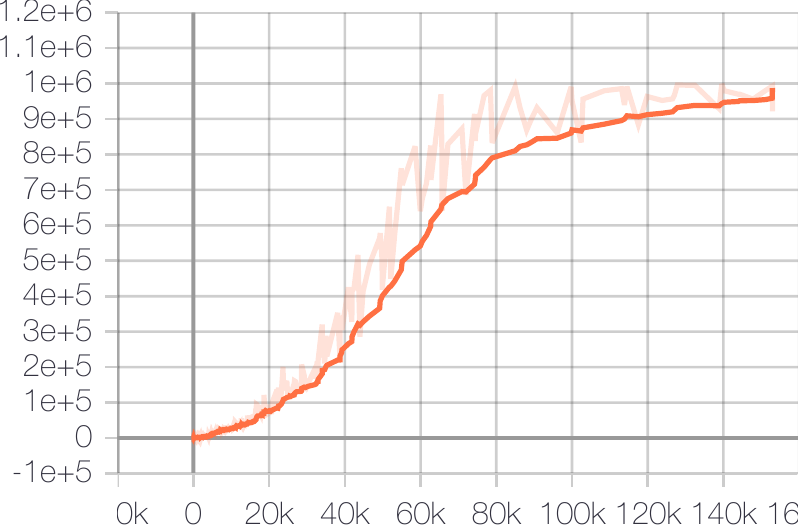}
    }
    \subfigure[Venture]{
    \includegraphics[width=0.3\textwidth]{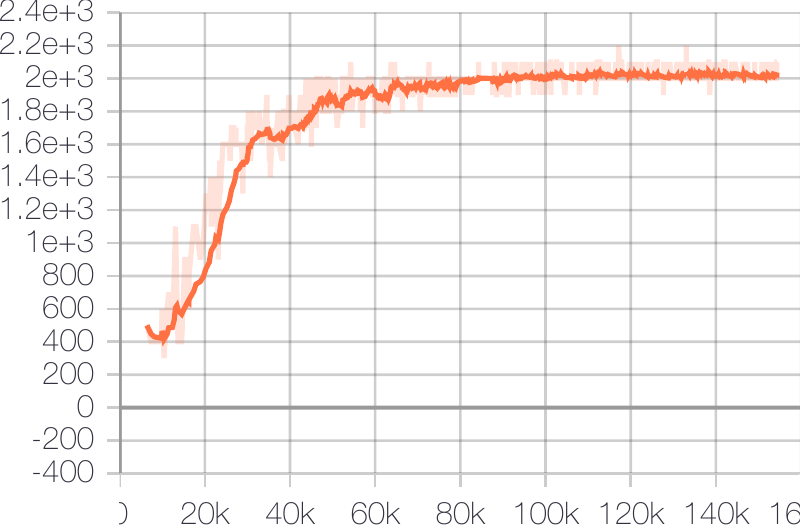}
    }
    \subfigure[Video\_Pinball]{
    \includegraphics[width=0.3\textwidth]{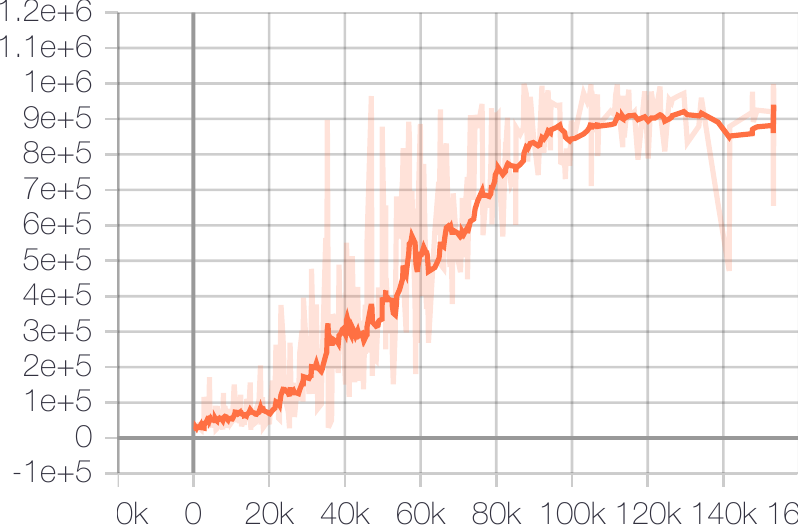}
    }
\end{figure}

\begin{figure}[!ht]
    \subfigure[Wizard\_of\_Wor]{
    \includegraphics[width=0.3\textwidth]{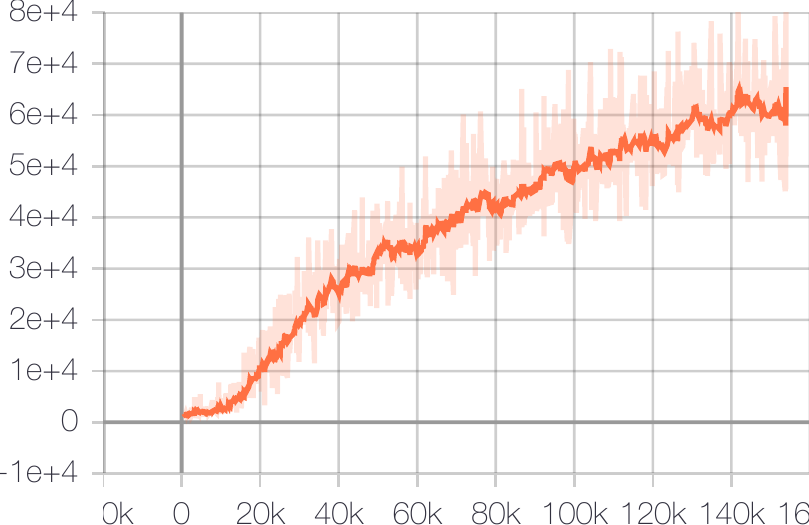}
    }
    \subfigure[Yars\_Revenge]{
    \includegraphics[width=0.3\textwidth]{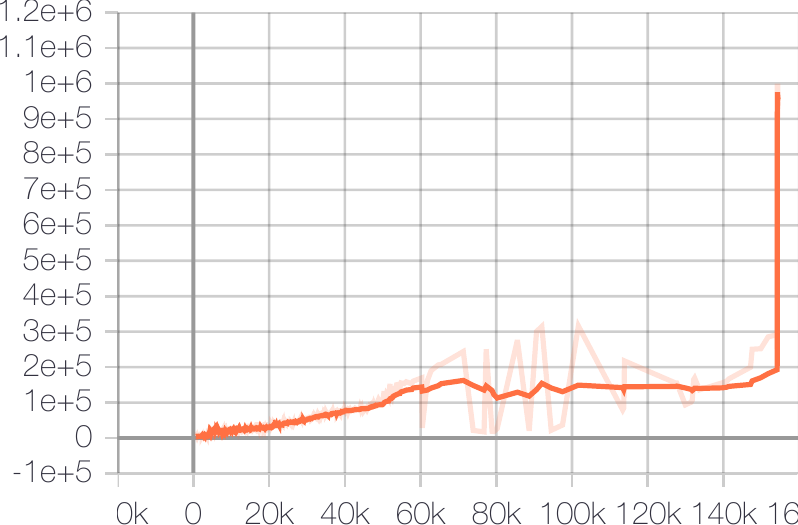}
    }
    \subfigure[Zaxxon]{
    \includegraphics[width=0.3\textwidth]{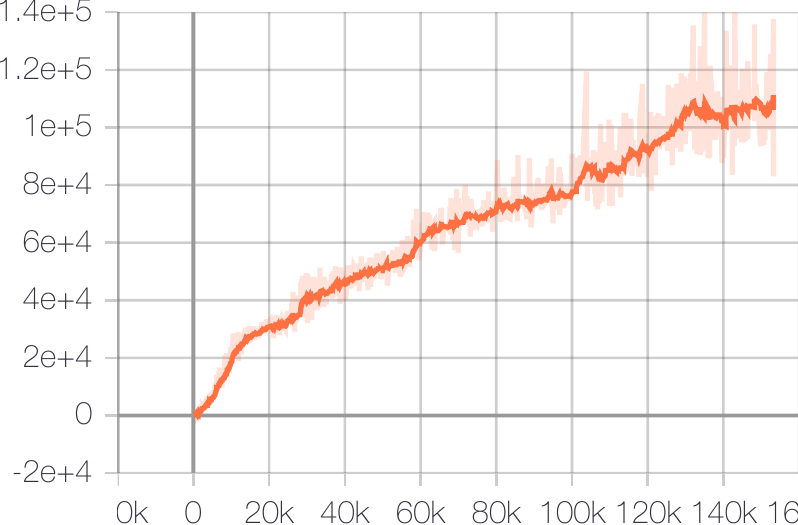}
    }
\end{figure}

\clearpage

\subsubsection{Atari Games Learning Curves of GDI-H$^3$}

\setcounter{subfigure}{0}

\begin{figure}[!ht] 
    \subfigure[alien]{
    \includegraphics[width=0.3\textwidth]{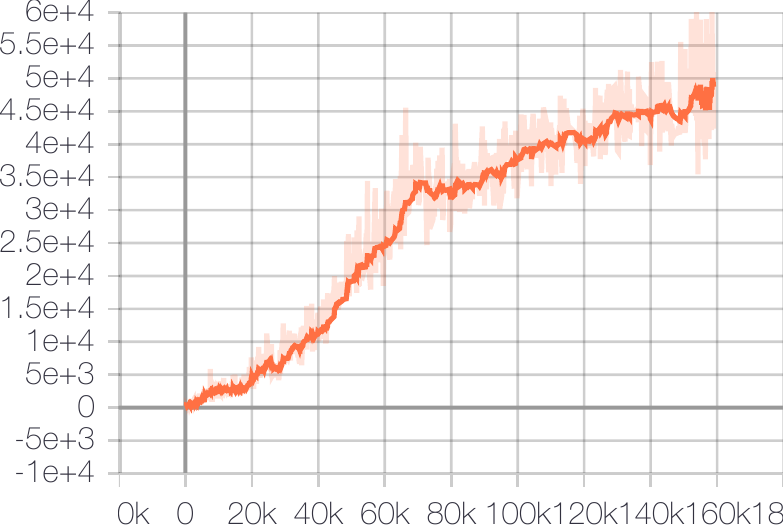}
    }
    \subfigure[amidar]{
    \includegraphics[width=0.3\textwidth]{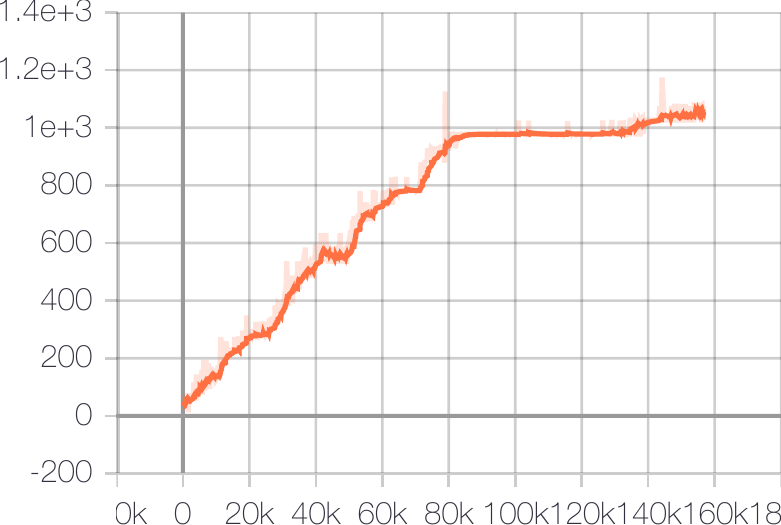}
    }
    \subfigure[assault]{
    \includegraphics[width=0.3\textwidth]{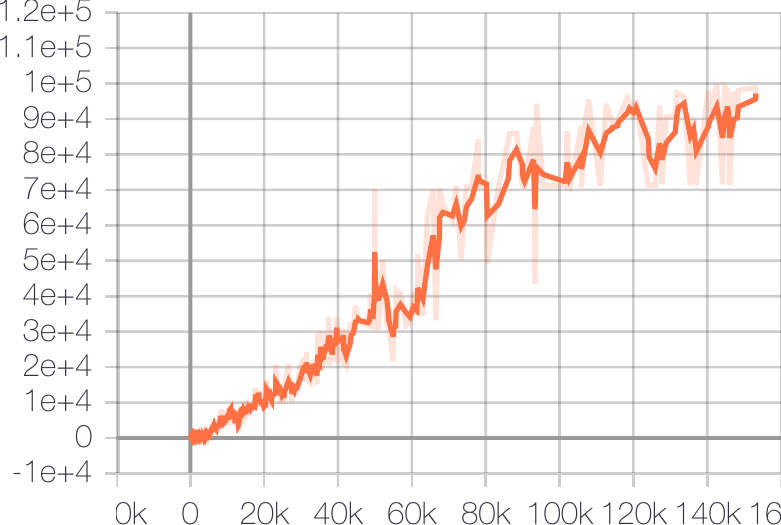}
    }
\end{figure}

\begin{figure}[!ht]
    \subfigure[asterix]{
    \includegraphics[width=0.3\textwidth]{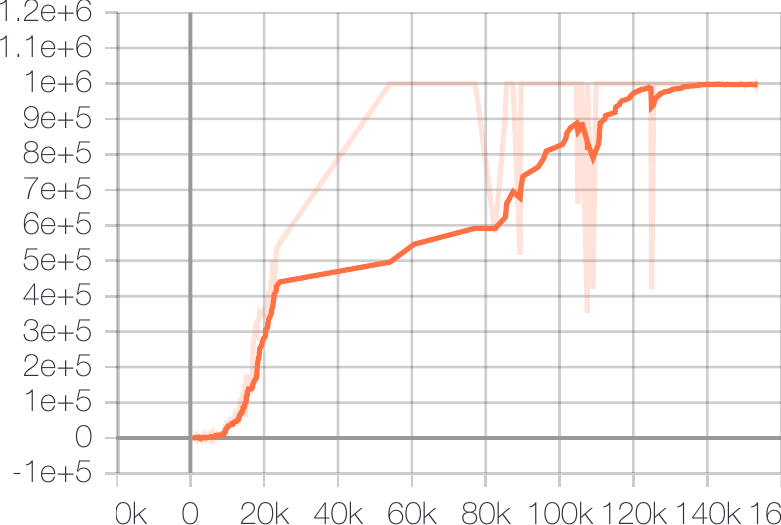}
    }
    \subfigure[asteroids]{
    \includegraphics[width=0.3\textwidth]{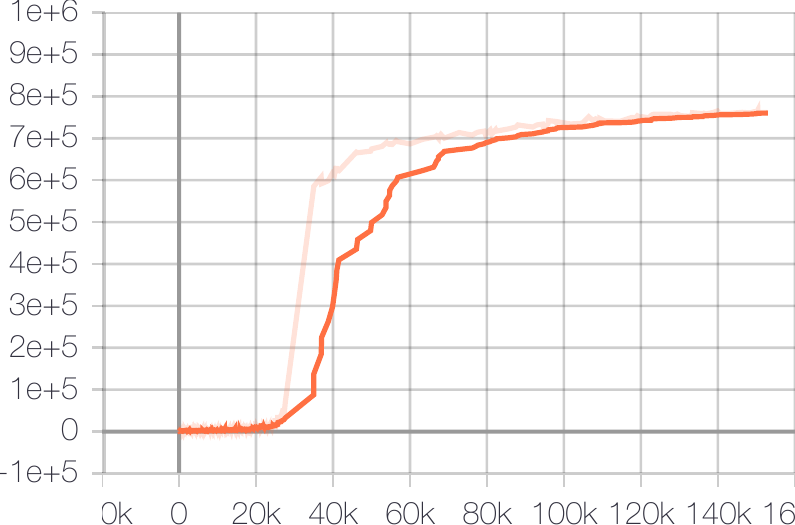}
    }
    \subfigure[atlantis]{
    \includegraphics[width=0.3\textwidth]{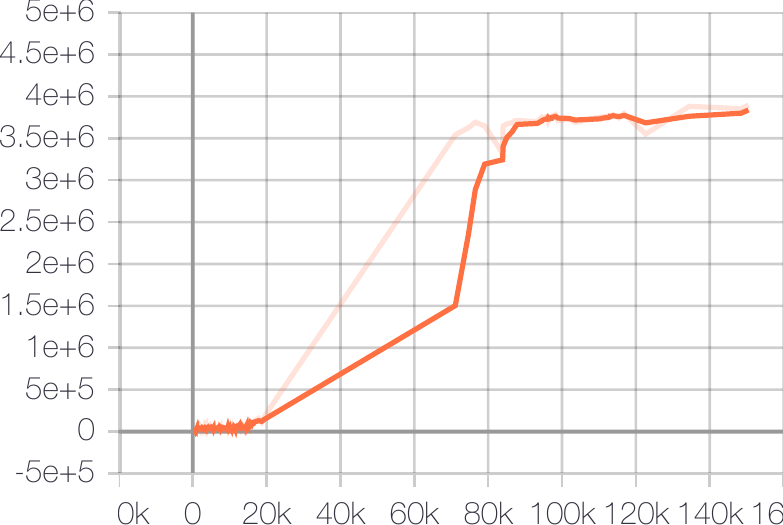}
    }
\end{figure}

\begin{figure}[!ht]
    \subfigure[bank\_heist]{
    \includegraphics[width=0.3\textwidth]{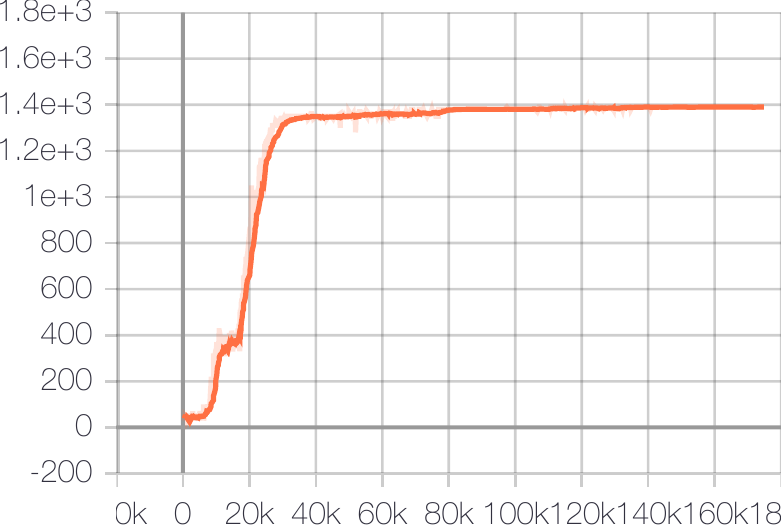}
    }
    \subfigure[battle\_zone]{
    \includegraphics[width=0.3\textwidth]{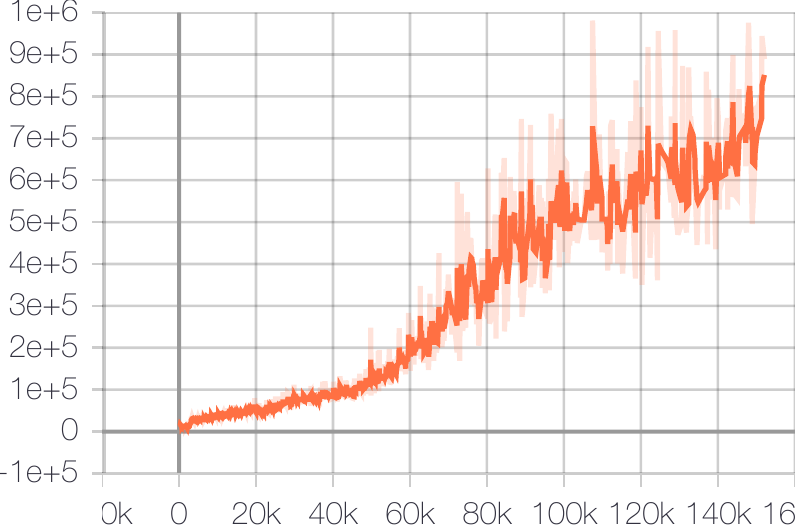}
    }
    \subfigure[beam\_rider]{
    \includegraphics[width=0.3\textwidth]{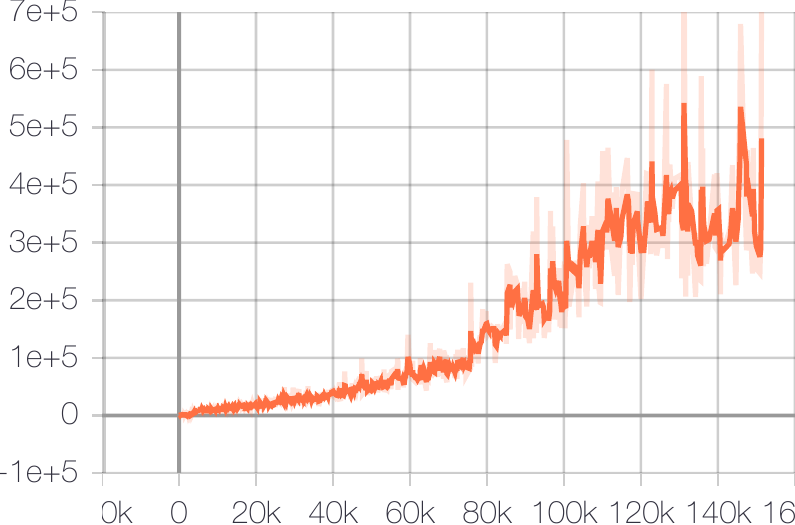}
    }
\end{figure}

\begin{figure}[!ht]
    \subfigure[berzerk]{
    \includegraphics[width=0.3\textwidth]{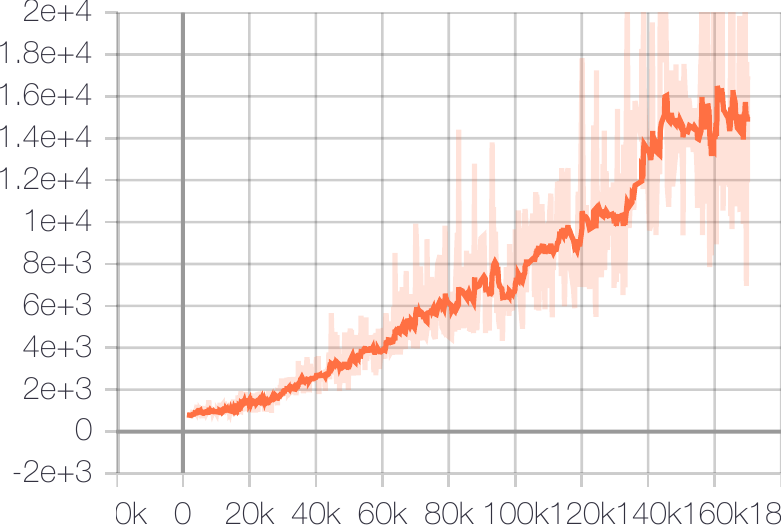}
    }
    \subfigure[bowling]{
    \includegraphics[width=0.3\textwidth]{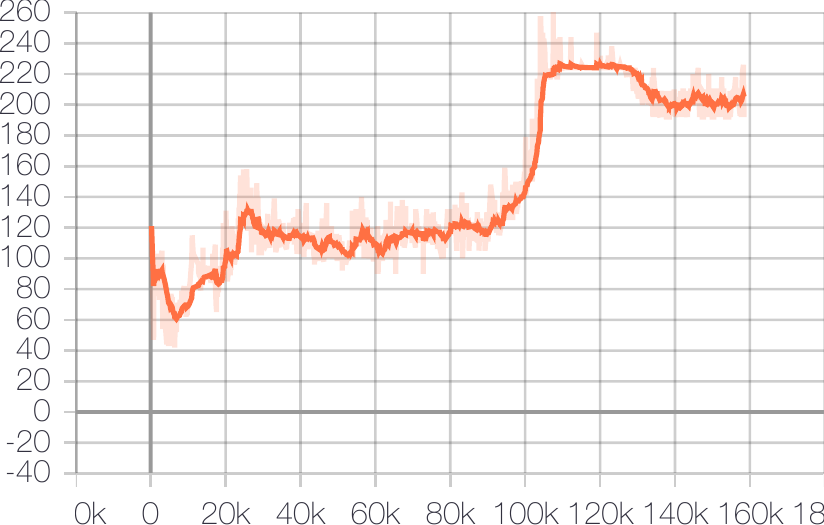}
    }
    \subfigure[boxing]{
    \includegraphics[width=0.3\textwidth]{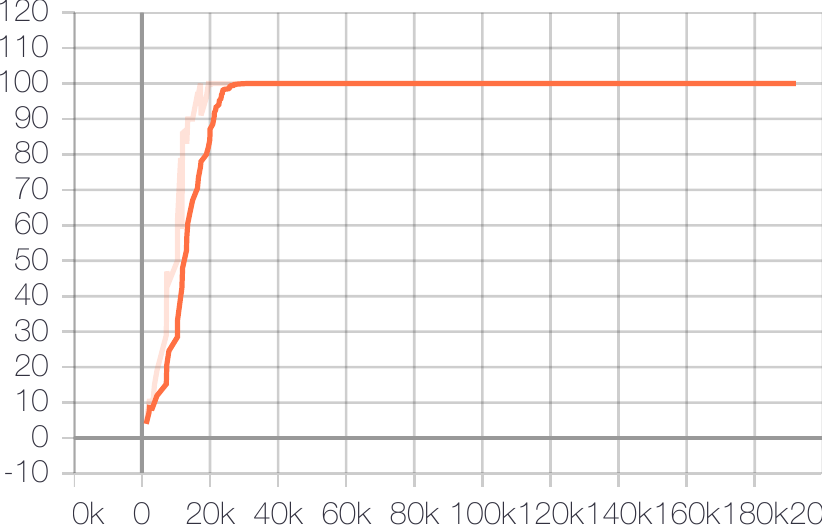}
    }
\end{figure}

\begin{figure}[!ht]
    \subfigure[breakout]{
    \includegraphics[width=0.3\textwidth]{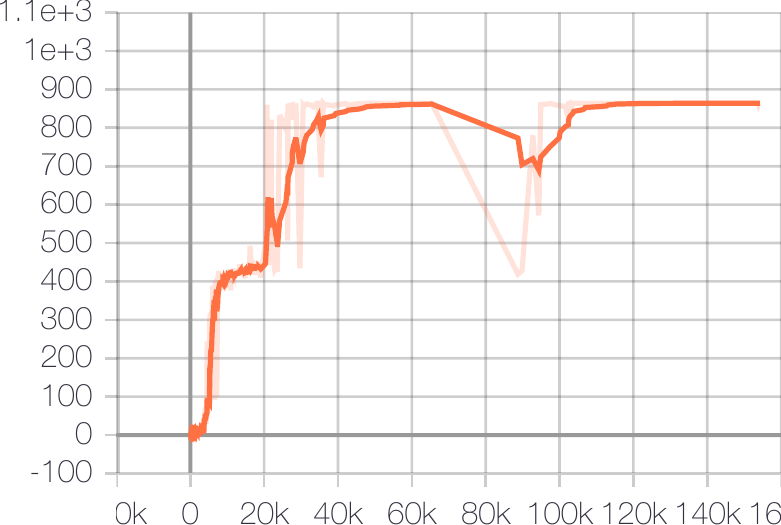}
    }
    \subfigure[centipede]{
    \includegraphics[width=0.3\textwidth]{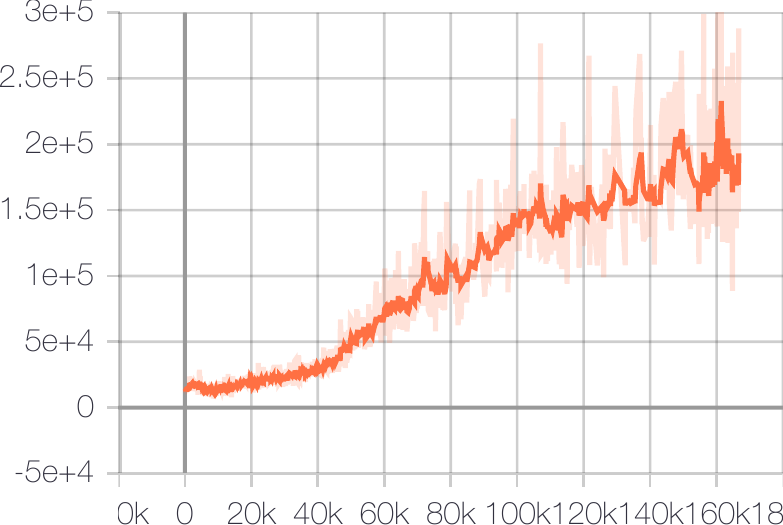}
    }
    \subfigure[chopper\_command]{
    \includegraphics[width=0.3\textwidth]{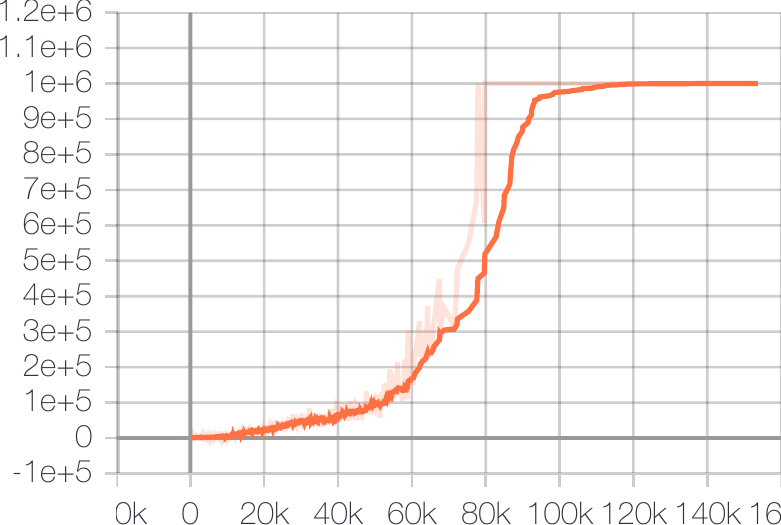}
    }
\end{figure}

\begin{figure}[!ht]
    \subfigure[crazy\_climber]{
    \includegraphics[width=0.3\textwidth]{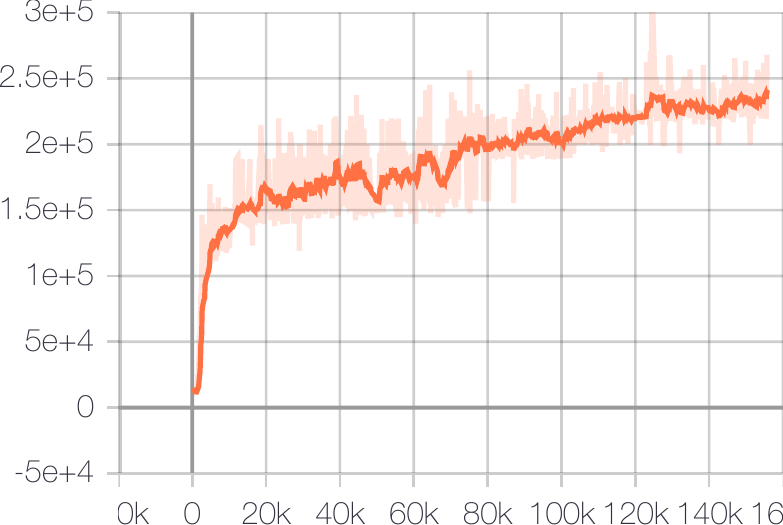}
    }
    \subfigure[defender]{
    \includegraphics[width=0.3\textwidth]{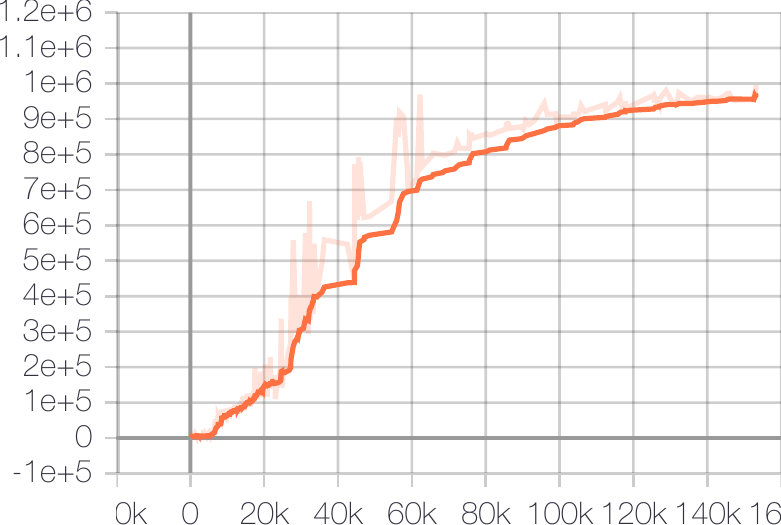}
    }
    \subfigure[demon\_attack]{
    \includegraphics[width=0.3\textwidth]{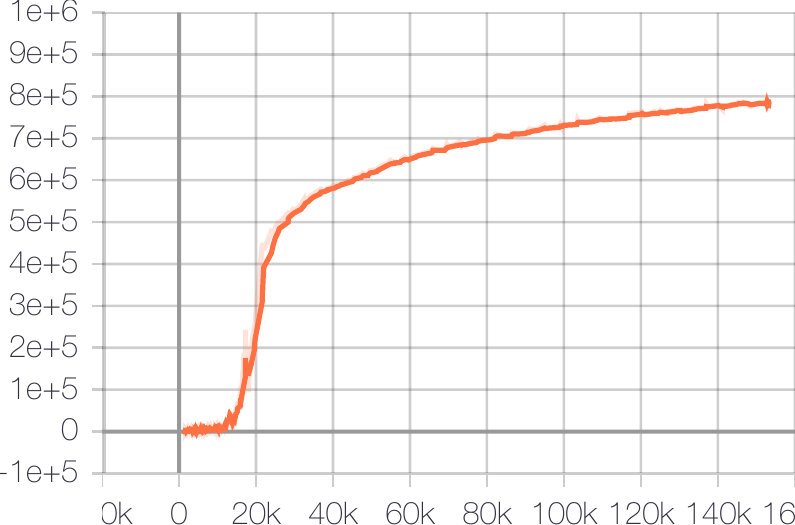}
    }
\end{figure}

\begin{figure}[!ht]
    \subfigure[double\_dunk]{
    \includegraphics[width=0.3\textwidth]{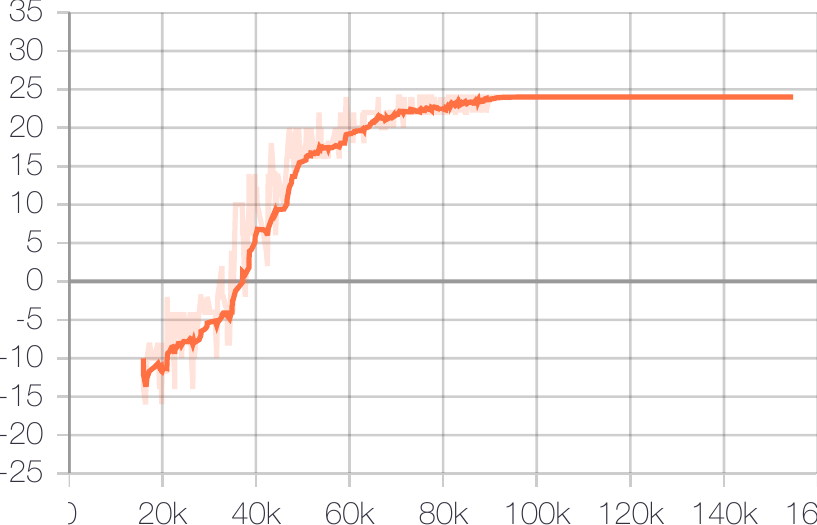}
    }
    \subfigure[enduro]{
    \includegraphics[width=0.3\textwidth]{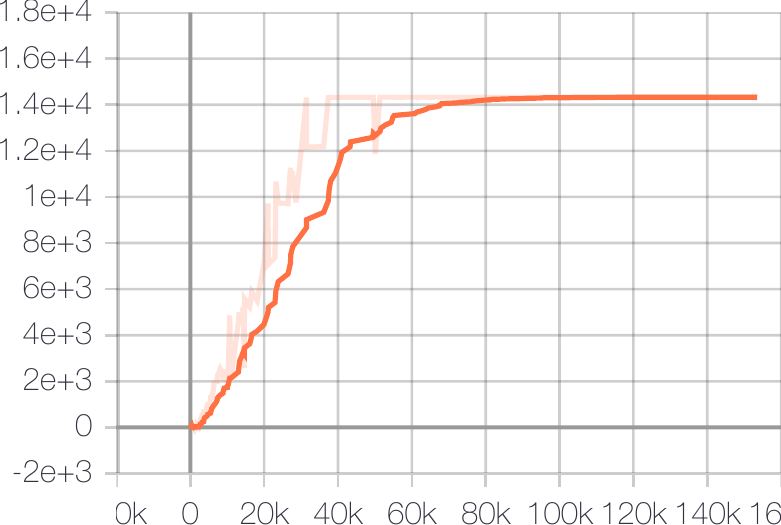}
    }
    \subfigure[fishing\_derby]{
    \includegraphics[width=0.3\textwidth]{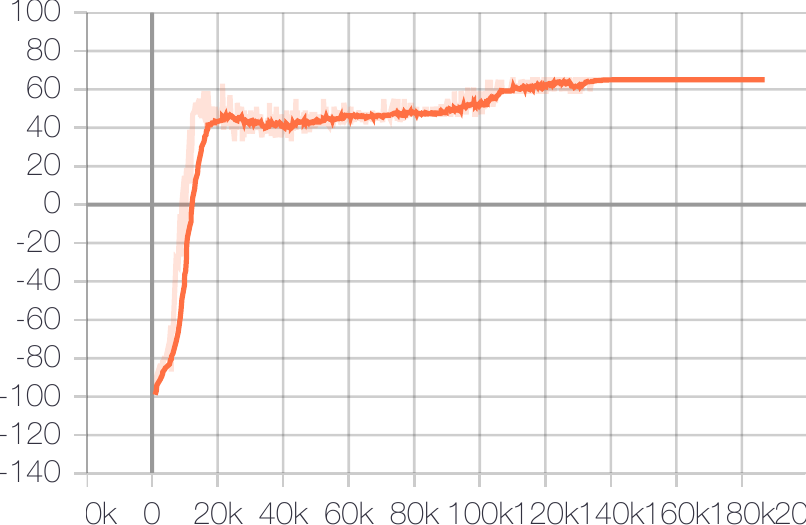}
    }
\end{figure}

\begin{figure}[!ht]
    \subfigure[freeway]{
    \includegraphics[width=0.3\textwidth]{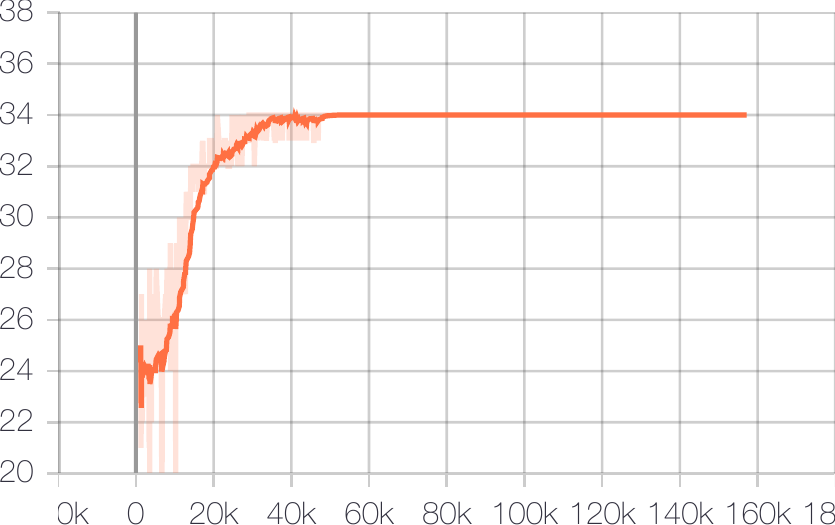}
    }
    \subfigure[frostbite]{
    \includegraphics[width=0.3\textwidth]{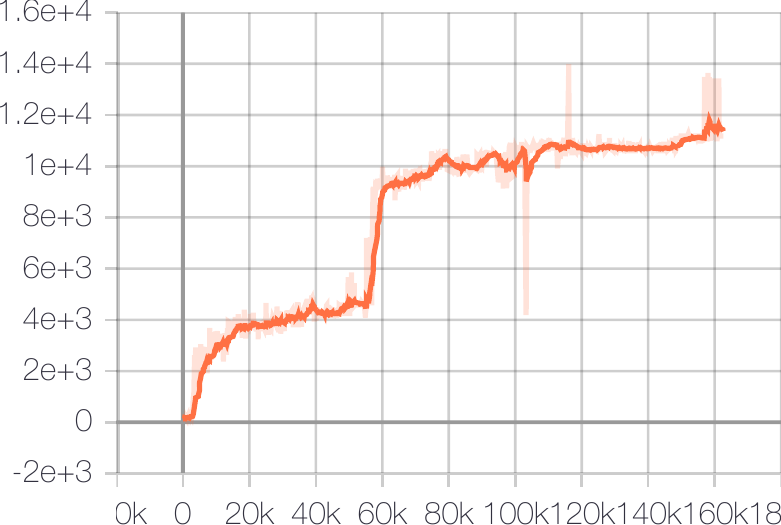}
    }
    \subfigure[gopher]{
    \includegraphics[width=0.3\textwidth]{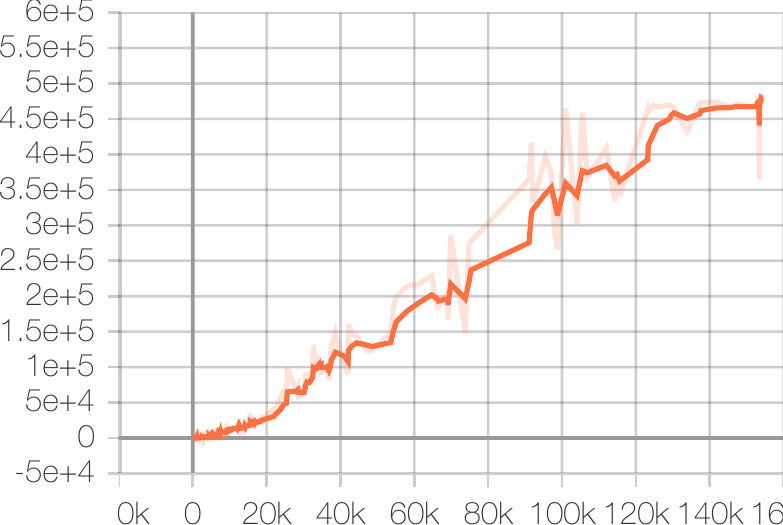}
    }
\end{figure}

\begin{figure}[!ht]
    \subfigure[gravitar]{
    \includegraphics[width=0.3\textwidth]{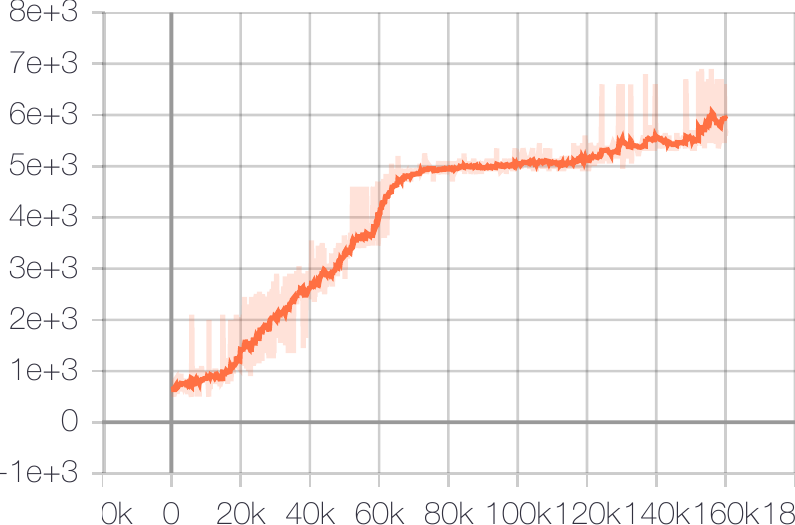}
    }
    \subfigure[hero]{
    \includegraphics[width=0.3\textwidth]{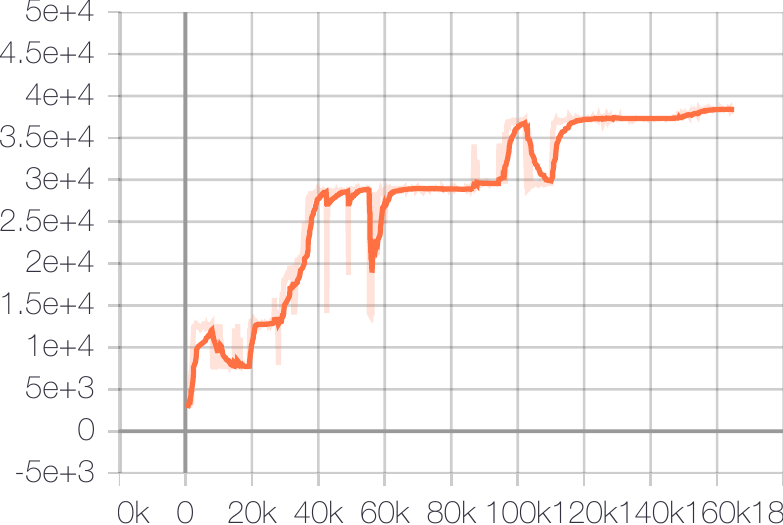}
    }
    \subfigure[ice\_hockey]{
    \includegraphics[width=0.3\textwidth]{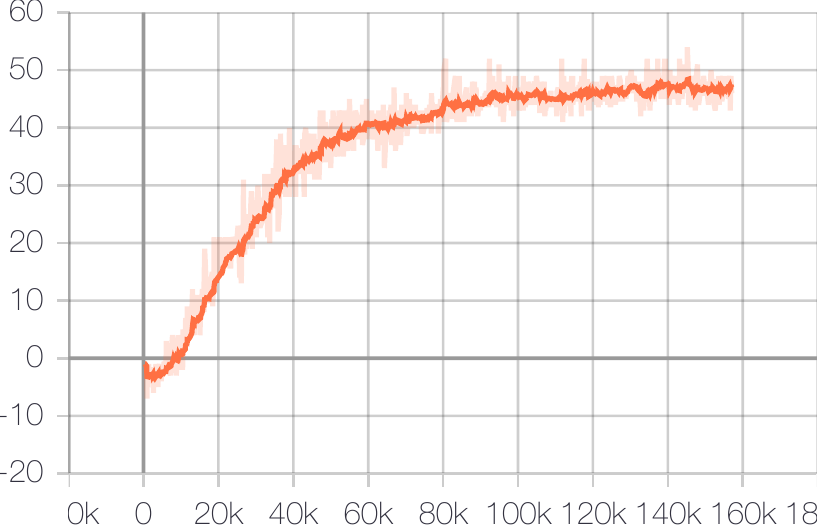}
    }
\end{figure}

\begin{figure}[!ht]
    \subfigure[jamesbond]{
    \includegraphics[width=0.3\textwidth]{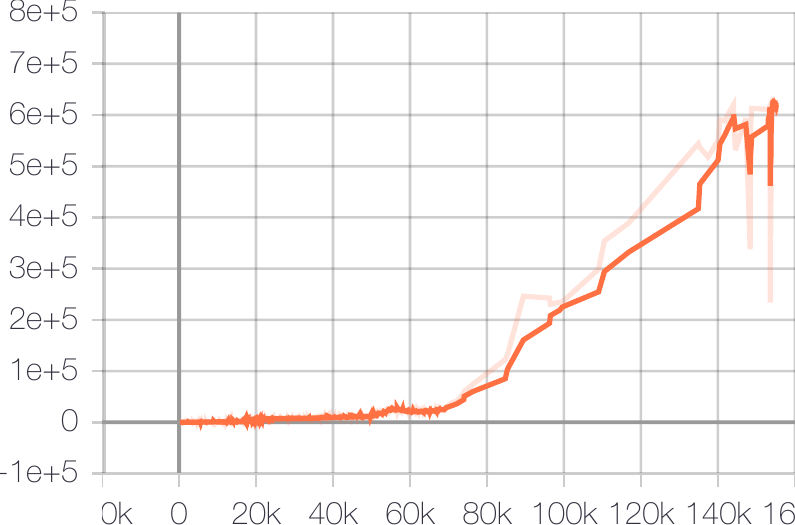}
    }
    \subfigure[kangaroo]{
    \includegraphics[width=0.3\textwidth]{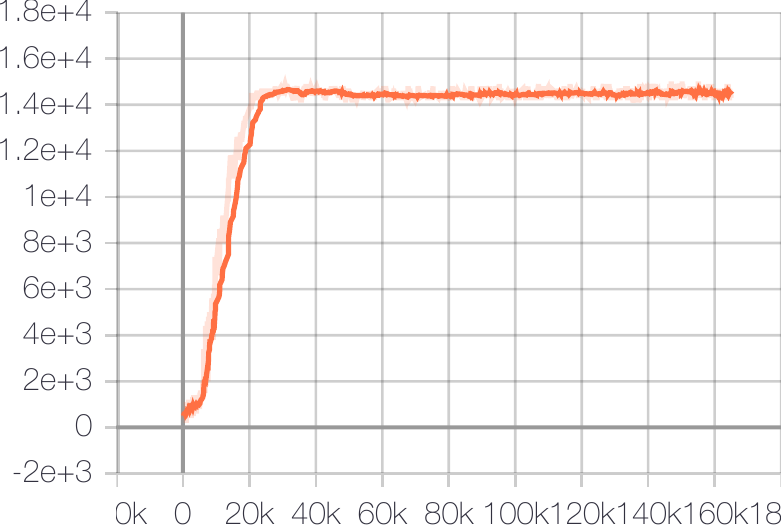}
    }
    \subfigure[krull]{
    \includegraphics[width=0.3\textwidth]{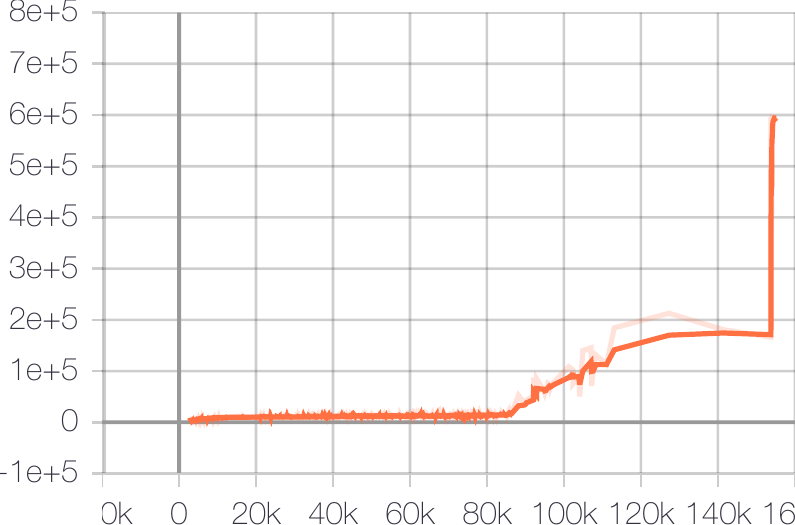}
    }
\end{figure}

\begin{figure}[!ht]
    \subfigure[kung\_fu\_master]{
    \includegraphics[width=0.3\textwidth]{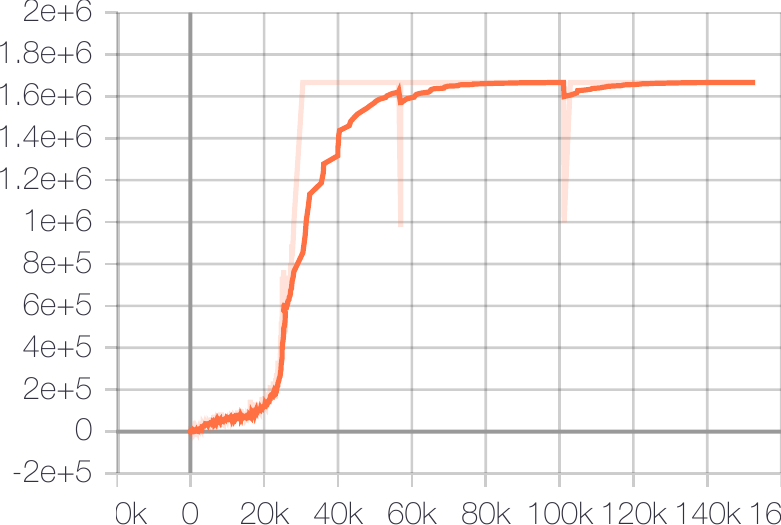}
    }
    \subfigure[montezuma\_revenge]{
     \includegraphics[width=0.3\textwidth]{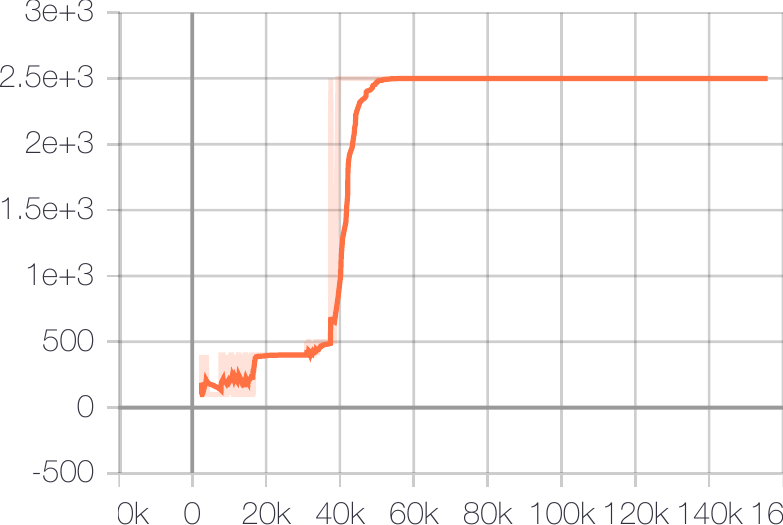}
    }
    \subfigure[ms\_pacman]{
    \includegraphics[width=0.3\textwidth]{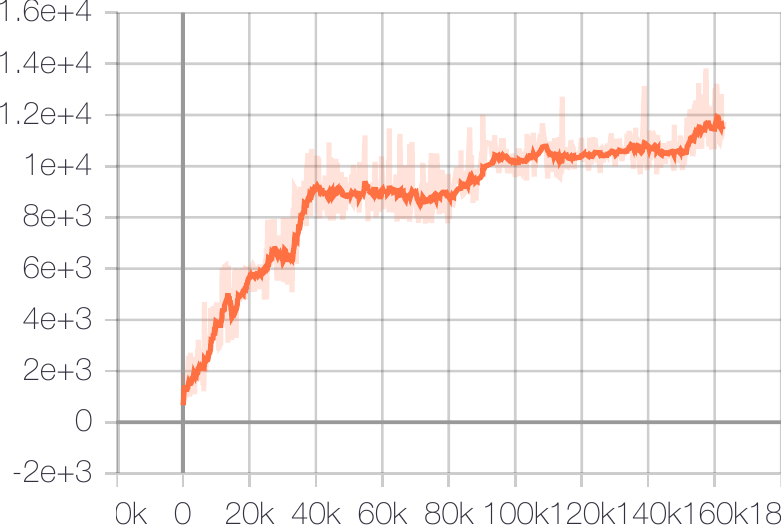}
    }
\end{figure}
 
\begin{figure}[!ht]
    \subfigure[name\_this\_game]{
    \includegraphics[width=0.3\textwidth]{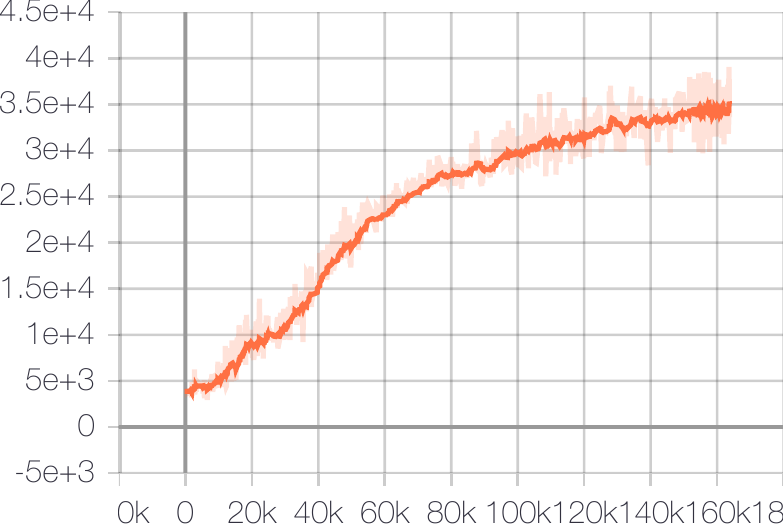}
    }
    \subfigure[phoenix]{
     \includegraphics[width=0.3\textwidth]{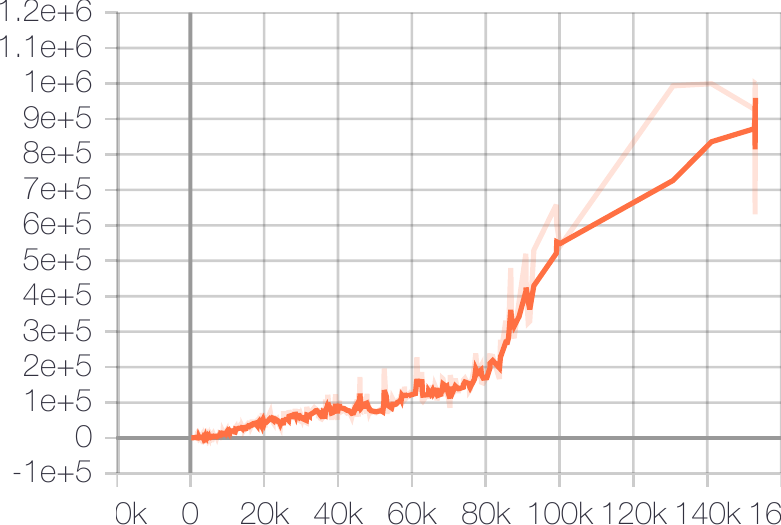}
    }
    \subfigure[pitfall]{
    \includegraphics[width=0.3\textwidth]{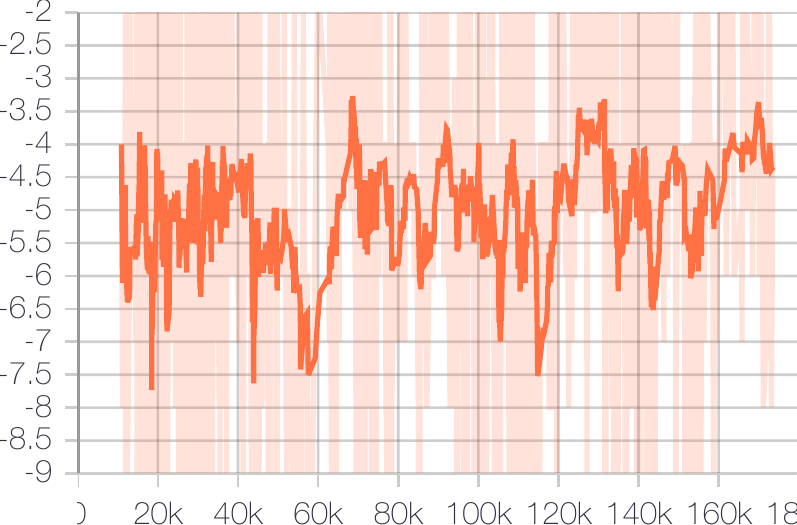}
    }
\end{figure}

\begin{figure}[!ht]
    \subfigure[pong]{
    \includegraphics[width=0.3\textwidth]{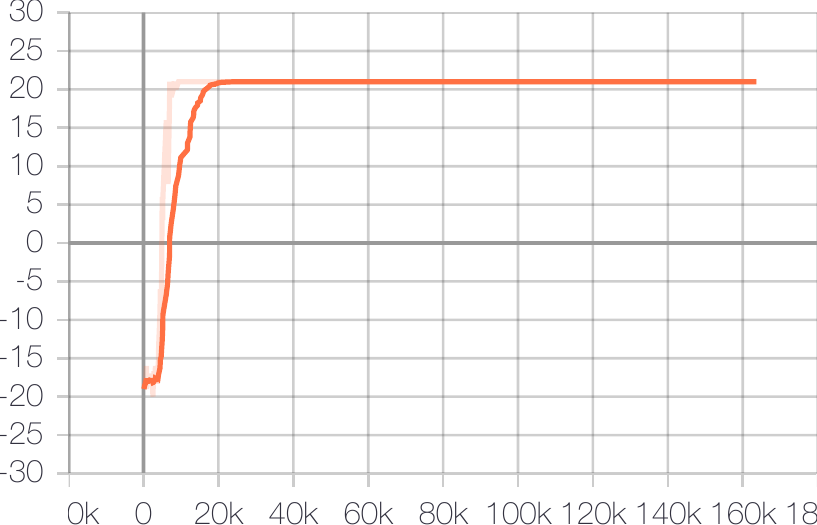}
    }
    \subfigure[private\_eye]{
    \includegraphics[width=0.3\textwidth]{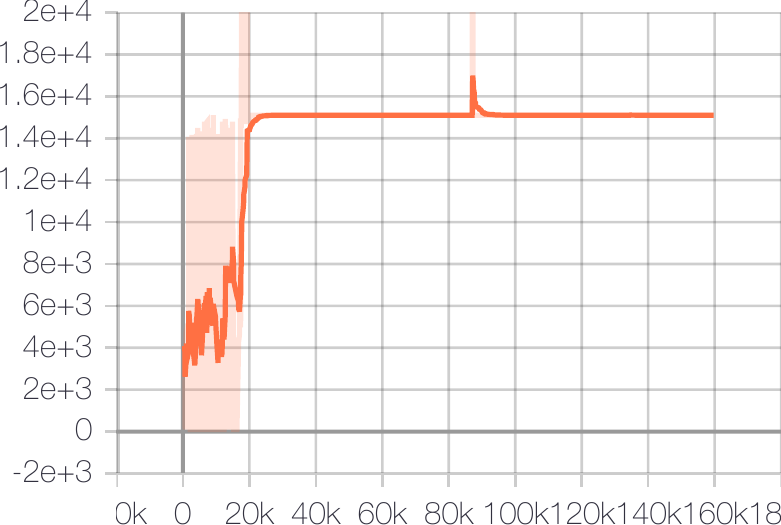}
    }
    \subfigure[qbert]{
     \includegraphics[width=0.3\textwidth]{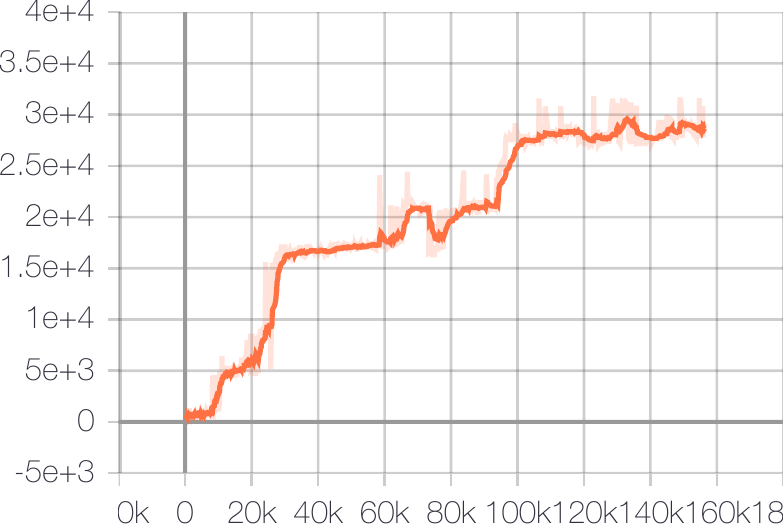}
    }
\end{figure}

\begin{figure}[!ht]
    \subfigure[riverraid]{
    \includegraphics[width=0.3\textwidth]{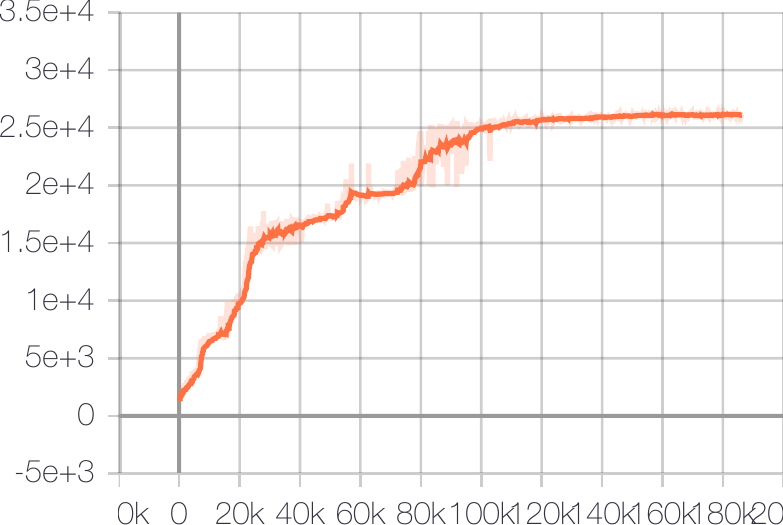}
    }
    \subfigure[road\_runner]{
    \includegraphics[width=0.3\textwidth]{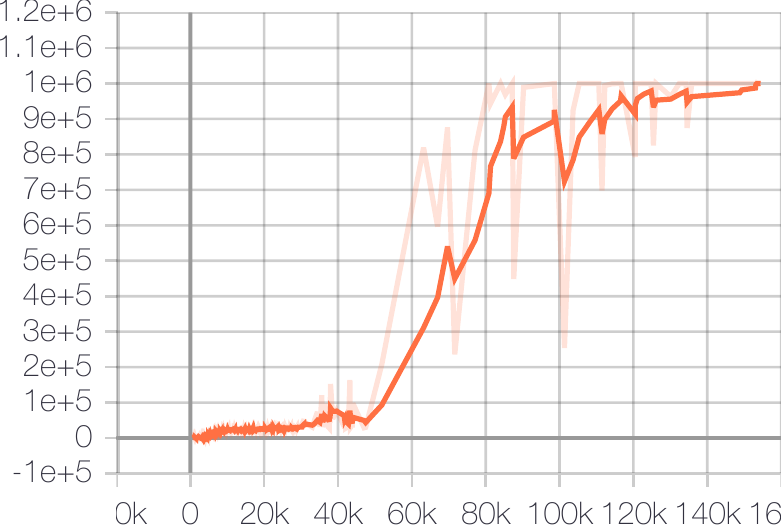}
    }
    \subfigure[robotank]{
    \includegraphics[width=0.3\textwidth]{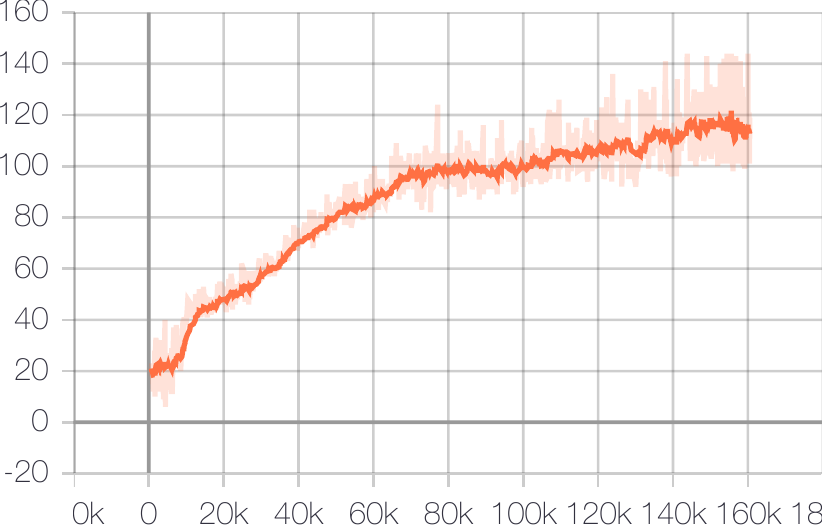}
    }
\end{figure}

\begin{figure}[!ht]
    \subfigure[seaquest]{
    \includegraphics[width=0.3\textwidth]{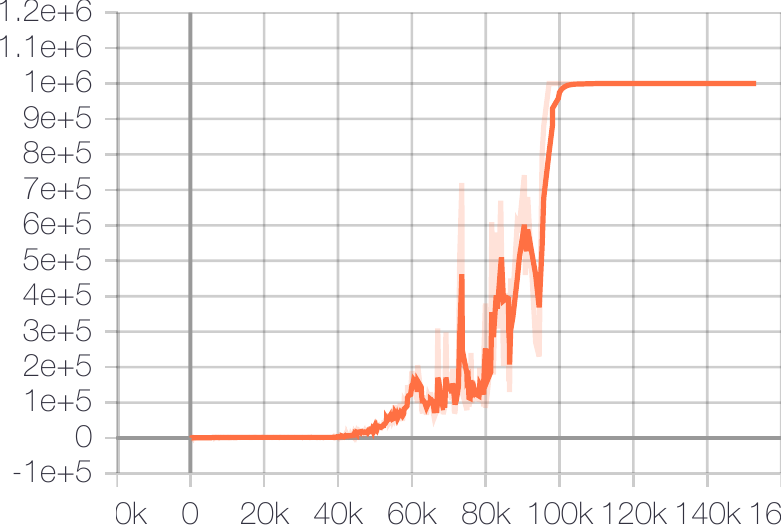}
    }
    \subfigure[skiing]{
    \includegraphics[width=0.3\textwidth]{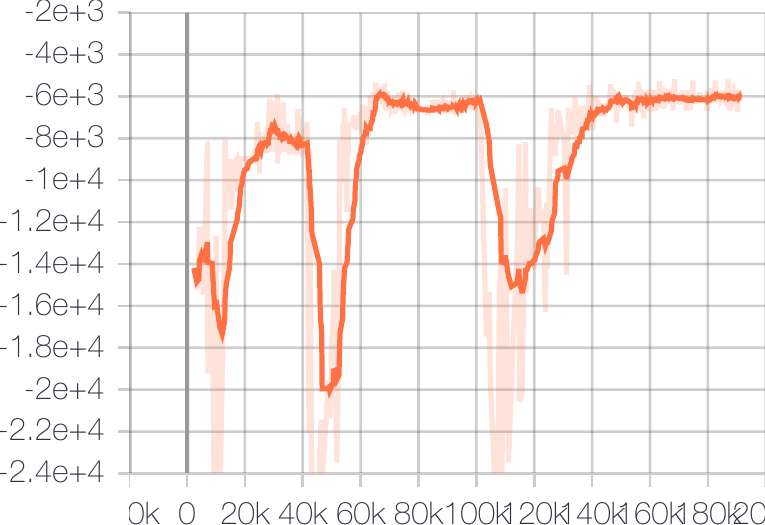}
    }
    \subfigure[solaris]{
    \includegraphics[width=0.3\textwidth]{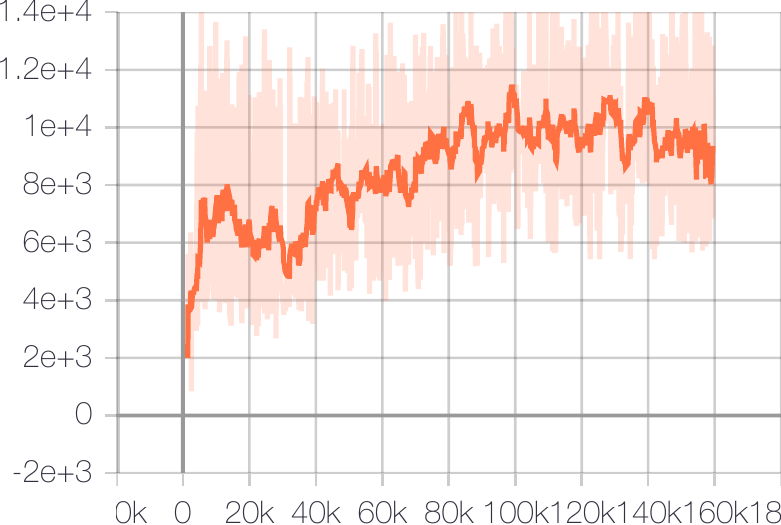}
    }
\end{figure}

\begin{figure}[!ht]
    \subfigure[space\_invaders]{
     \includegraphics[width=0.3\textwidth]{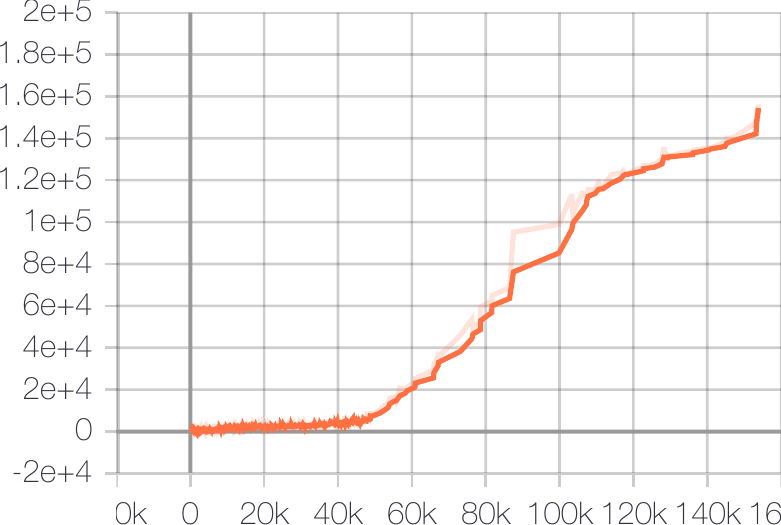}
    }
    \subfigure[star\_gunner]{
    \includegraphics[width=0.3\textwidth]{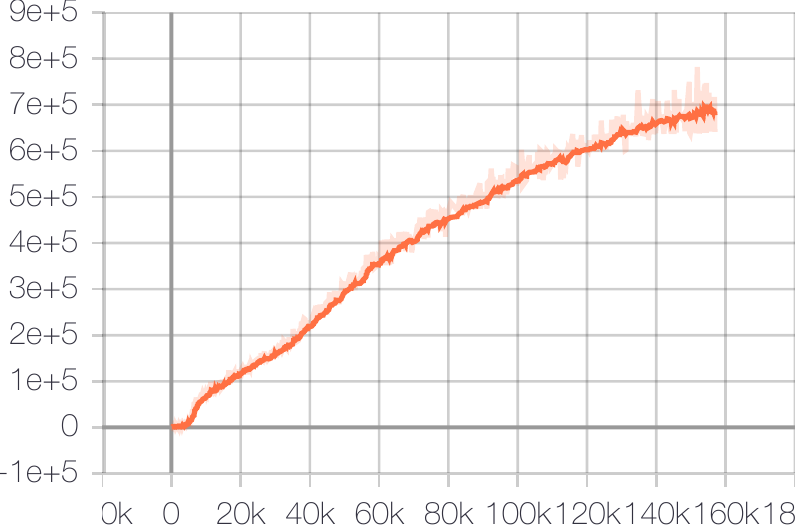}
    }
    \subfigure[surround]{
    \includegraphics[width=0.3\textwidth]{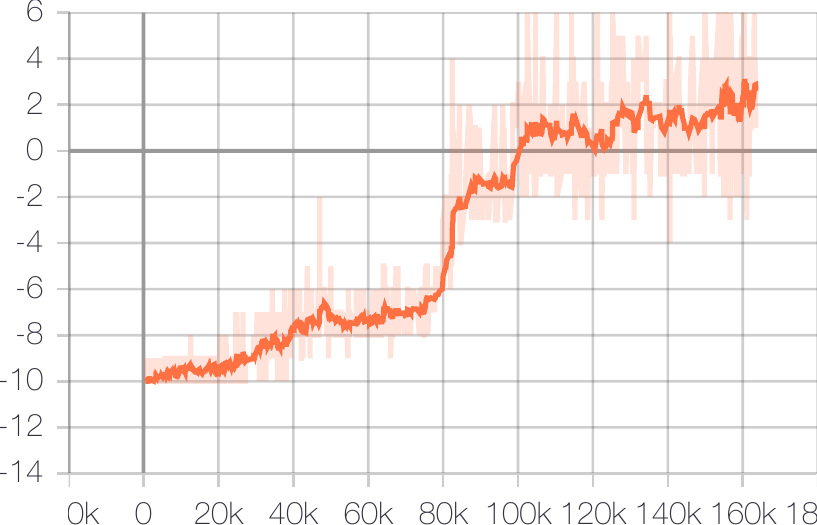}
    }
\end{figure}

\begin{figure}[!ht]
    \subfigure[tennis]{
    \includegraphics[width=0.3\textwidth]{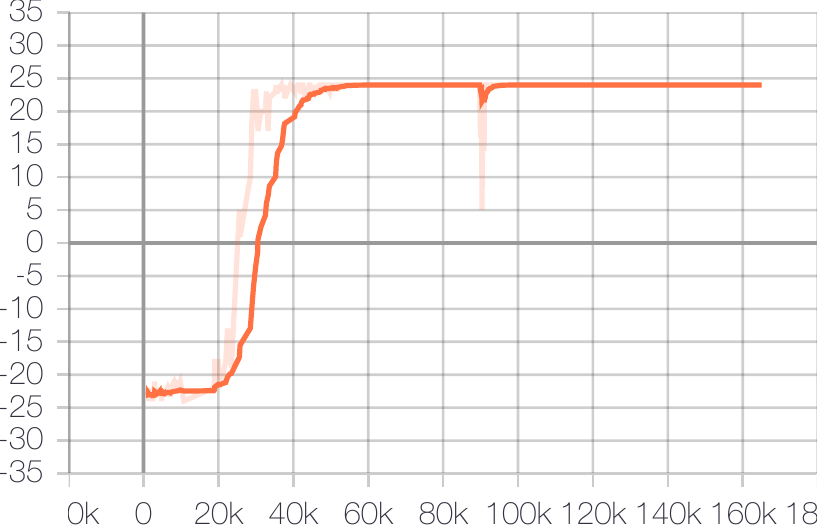}
    }
    \subfigure[time\_pilot]{
    \includegraphics[width=0.3\textwidth]{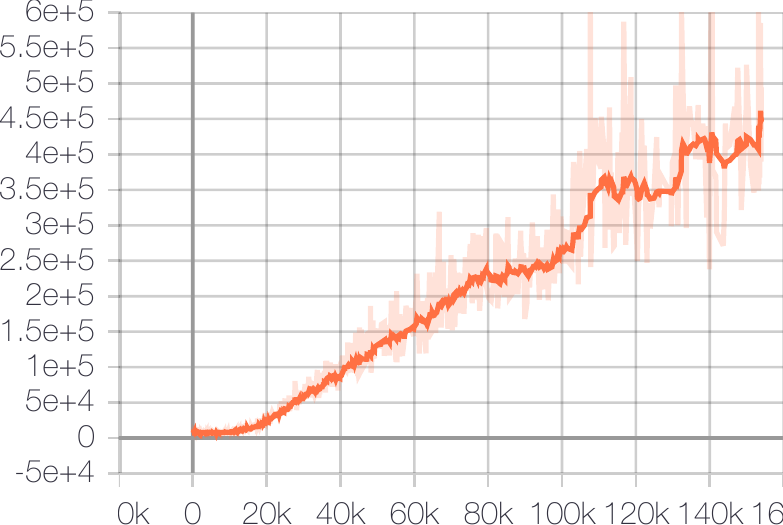}
    }
    \subfigure[tutankham]{
    \includegraphics[width=0.3\textwidth]{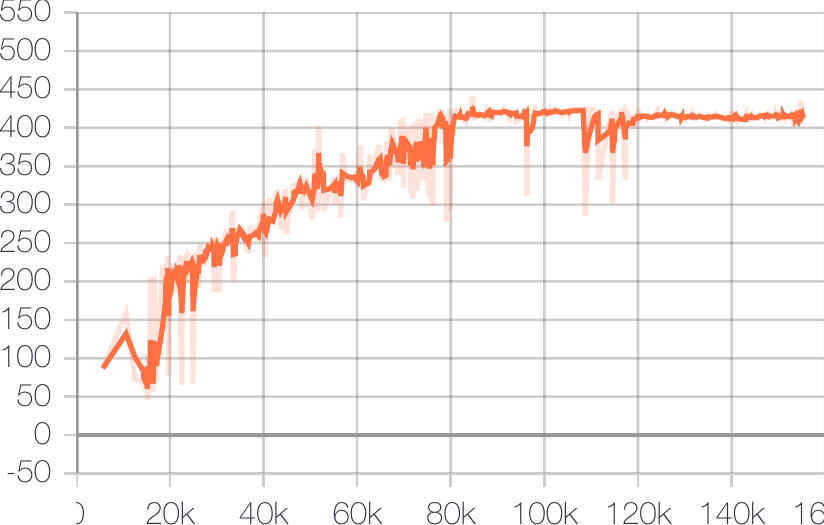}
    }
\end{figure}

\clearpage

\begin{figure}[!ht]
    \subfigure[up\_n\_down]{
    \includegraphics[width=0.3\textwidth]{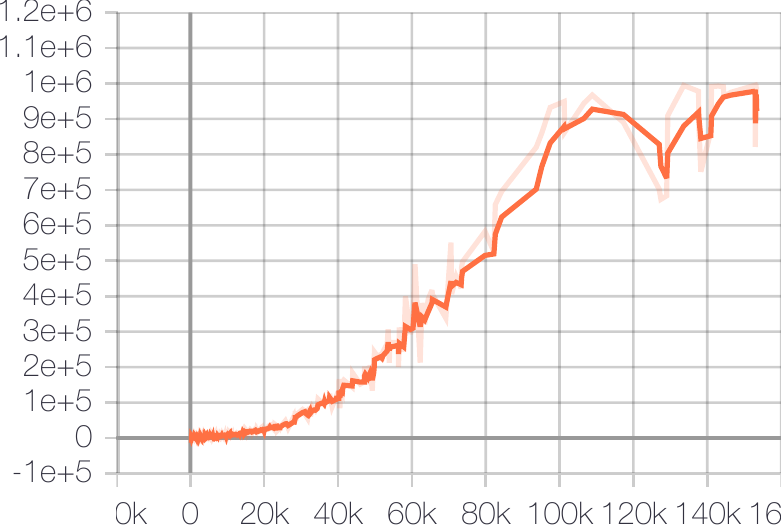}
    }
    \subfigure[venture]{
    \includegraphics[width=0.3\textwidth]{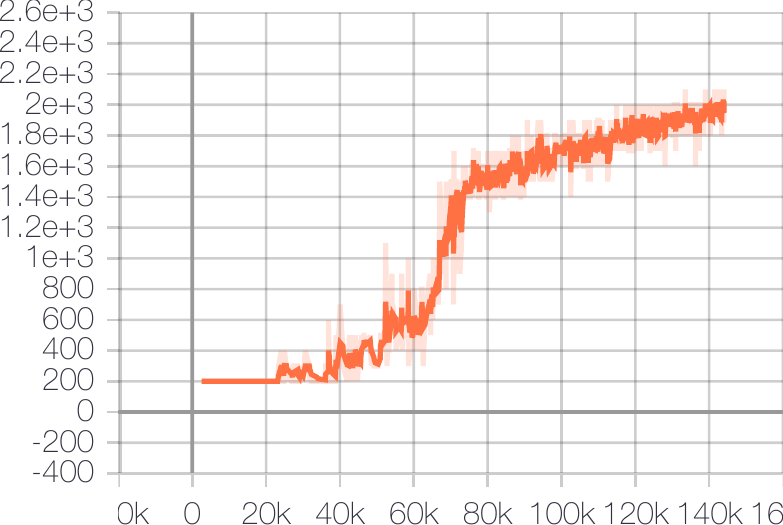}
    }
    \subfigure[video\_pinball]{
    \includegraphics[width=0.3\textwidth]{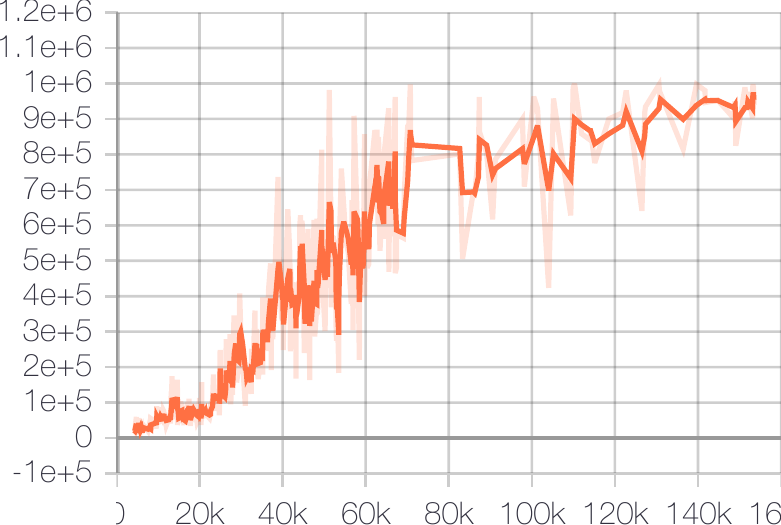}
    }
\end{figure}

\begin{figure}[!ht]
    \subfigure[wizard\_of\_wor]{
    \includegraphics[width=0.3\textwidth]{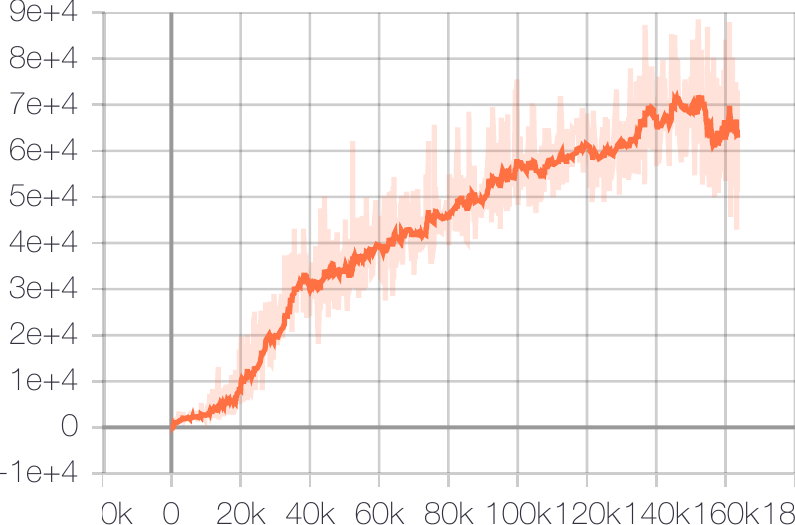}
    }
    \subfigure[yars\_revenge]{
    \includegraphics[width=0.3\textwidth]{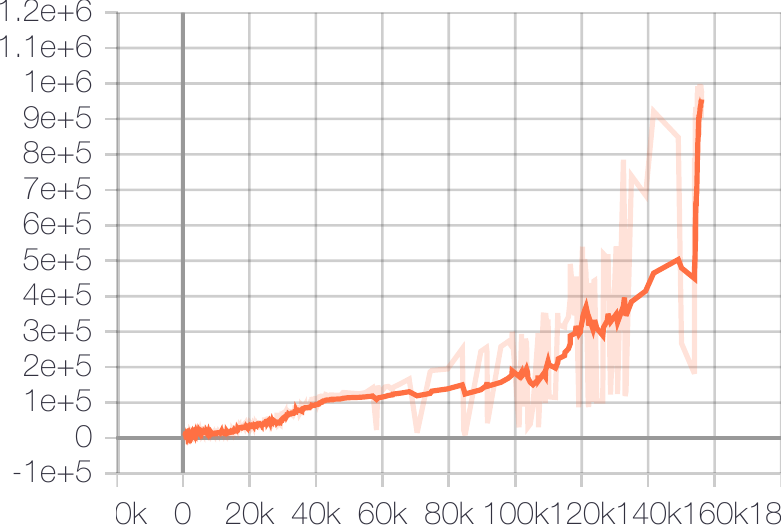}
    }
    \subfigure[zaxxon]{
    \includegraphics[width=0.3\textwidth]{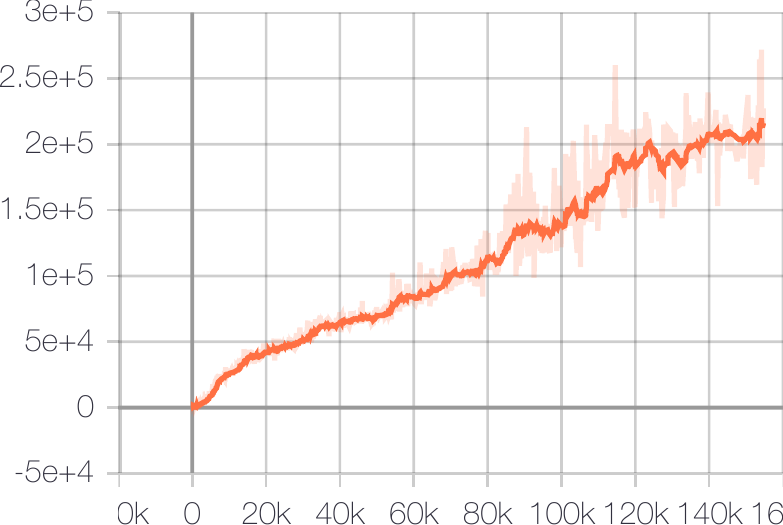}
    }
\end{figure}

\clearpage

\section{Ablation Study}
\label{Sec: appendix Ablation Study}

In this section, we firstly demonstrate the settings of our ablation studies. Then, we offer the t-SNE of three Atari games as a study case to further show the data richness among different capacities of the policy space via t-SNE.

\begin{table}[H]
\begin{center}
\caption{Summary of Algorithms of Ablation Study. The behavior policies are sampled from the policy space $\{\pi_{\theta_{\lambda}} | \lambda \in \Lambda\}$ which is parameterized by the policy network and indexed by the index set $\Lambda$ via a sampling distribution $P_{\Lambda}$. Wherein the sampling distribution is iteratively optimized via a data distribution optimization $\mathcal{E}$.}
\begin{tabular}{|c| c| c| c| c| c|}
\hline
\textbf{Name}               & \textbf{Category}          &$\pi_{\theta_{\lambda}}$       & $\Lambda$                                                  & $P_{\Lambda}^{(0)}$ & $\mathcal{E}$  \\
\hline
GDI-I$^{3}$                 &   GDI-I$^3$              & \makecell[c]{$\epsilon \cdot \operatorname{Softmax}\left(\frac{A_{\theta}}{\tau_{1}}\right)+$\\$(1-\epsilon) \cdot \operatorname{Softmax}\left(\frac{A_{\theta}}{\tau_{2}}\right)$} &$ \{\lambda| \lambda =(\epsilon,\tau_1,\tau_2)\}$  & Uniform              & MAB\\
\hline
GDI-H$^{3}$                 &   GDI-H$^3$              &  \makecell[c]{$\epsilon \cdot \operatorname{Softmax}\left(\frac{A_{\theta_1}}{\tau_{1}}\right)+$\\$(1-\epsilon) \cdot \operatorname{Softmax}\left(\frac{A_{\theta_2}}{\tau_{2}}\right)$}         &$ \{\lambda| \lambda =(\epsilon,\tau_1,\tau_2)\}$  & Uniform              & MAB\\
\hline
GDI-I$^0$ w/o  $\mathcal{E}$&   GDI-I$^0$              & \makecell[c]{$\epsilon \cdot \operatorname{Softmax}\left(\frac{A_{\theta}}{\tau_{1}}\right)+$\\$(1-\epsilon) \cdot \operatorname{Softmax}\left(\frac{A_{\theta}}{\tau_{2}}\right)$}   &$ \{\lambda| \lambda =(\epsilon,\tau_1,\tau_2)\}$  & One Point           & Identical Mapping\\
\hline
GDI-I$^1$         &   GDI-I$^1$                        & $  \operatorname{Softmax}\left(\frac{A_{\theta}}{\tau}\right)$ &$ \{\lambda| \lambda =(\tau)\}$                    & Uniform              & MAB\\
\hline
\end{tabular}
\end{center}
\label{tab:Summary of Algorithms of Ablation Study}
\end{table}

\begin{table}[H]
\begin{center}
\caption{Summary of the ablation  groups in the Ablation Study. The corresponding algorithms can see Tab. \ref{tab:Summary of Algorithms of Ablation Study}. We investigate the effects of several properties of GDI via ablating the Ablation Variables (e.g., removing the meta-controller from GDI-I$^3$, and explore the impact of the  meta-controller) and keeping the Control Variables (e.g., Hyperparameters) remains the same. }
\begin{tabular}{|c| c| c| c| c| c|}
\hline
\makecell[c]{\textbf{Group}\\ \textbf{Name}} & \makecell[c]{\textbf{Ablation}\\ \textbf{Variable}} & \makecell[c]{\textbf{Control} \\ \textbf{Variable}} &  \makecell[c]{\textbf{Corresponding} \\ \textbf{Algorithm}}& \makecell[c]{\textbf{Corresponding}\\ \textbf{Problem}} & \makecell[c]{\textbf{Corresponding}\\ \textbf{Results}} \\
\hline
GDI-I$^3$   &  N/A             & $\pi_{\theta}$, $\Lambda$ , $\mathcal{E}$ & GDI-I$^3$  & Baseline Group & N/A\\
\hline
w/o $\mathcal{E}$   &  $\mathcal{E}$             & $\Lambda$, $\pi_{\theta}$ & GDI-I$^0$ w/o  $\mathcal{E}$ & \makecell[c]{Exploration-Exploitation \\ Trade-off} & Fig. \ref{fig:ablation study}\\
\hline
GDI-I$^1$   &  $\Lambda$             & $\pi_{\theta}$, $\mathcal{E}$ & GDI-I$^1$  & \makecell[c]{Data Richness} & \makecell[c]{
Fig. \ref{fig:ablation study}\\
Fig. \ref{Fig: t-SNE of Seaquest}\\
Fig. \ref{Fig: t-SNE of ChopperCommand}\\
Fig. \ref{Fig: t-SNE of Krull}\\
Fig. \ref{Fig: Overview of t-SNE in Atari games.}\\
}\\
\hline
GDI-H$^3$   &  $\pi_{\theta}$             & $\Lambda$ , $\mathcal{E}$ & GDI-H$^3$  & \makecell[c]{Data Richness}& \makecell[c]{
Fig. \ref{fig:ablation study}\\
Fig. \ref{Fig: t-SNE of Krull}\\
Fig. \ref{Fig: Overview of t-SNE in Atari games.}\\
}\\
\hline
\end{tabular}
\end{center}
\label{tab:Summary of the ablation experimental groups in Ablation Study}
\end{table}

\subsection{Ablation Study Design}
\label{app: Ablation Study Design}

To prove the effectiveness of capacity and diversity control and the data distribution optimization operator $\mathcal{E}$. All the implemented algorithms in the ablation study have been summarized in Tab. \ref{tab:Summary of Algorithms of Ablation Study}.

 We have summarized all the ablation experimental groups of the ablation study in Tab. \ref{tab:Summary of the ablation experimental groups in Ablation Study}. The operator $\mathcal{E}$ is achieved with MAB (see App. \ref{Sec: appendix MAB}).The operator $\mathcal{T}$ is achieved with Vtrace, Retrace and policy gradient. Except for the Control Variable listed in the Tab. \ref{tab:Summary of the ablation experimental groups in Ablation Study}, other settings and the shared hyperparameters remain the same in all ablation  groups. The hyperparameters can see App. \ref{Sec: appendix hyperparameters}.

\subsection{t-SNE}
\label{app: tsne}
In all the t-SNE, we mark the state generated by GDI-I$^3$ as A$_i$ and mark the state generated by GDI-I$^1$ as B$_i$, where i = 1, 2, 3 represents three stages of the training process.

\setcounter{subfigure}{0}
\begin{figure}[!ht]
	\begin{center}
	    \subfigure[Early stage of GDI-I$^3$]{
		\includegraphics[width=0.4\textwidth,height=0.25\textheight]{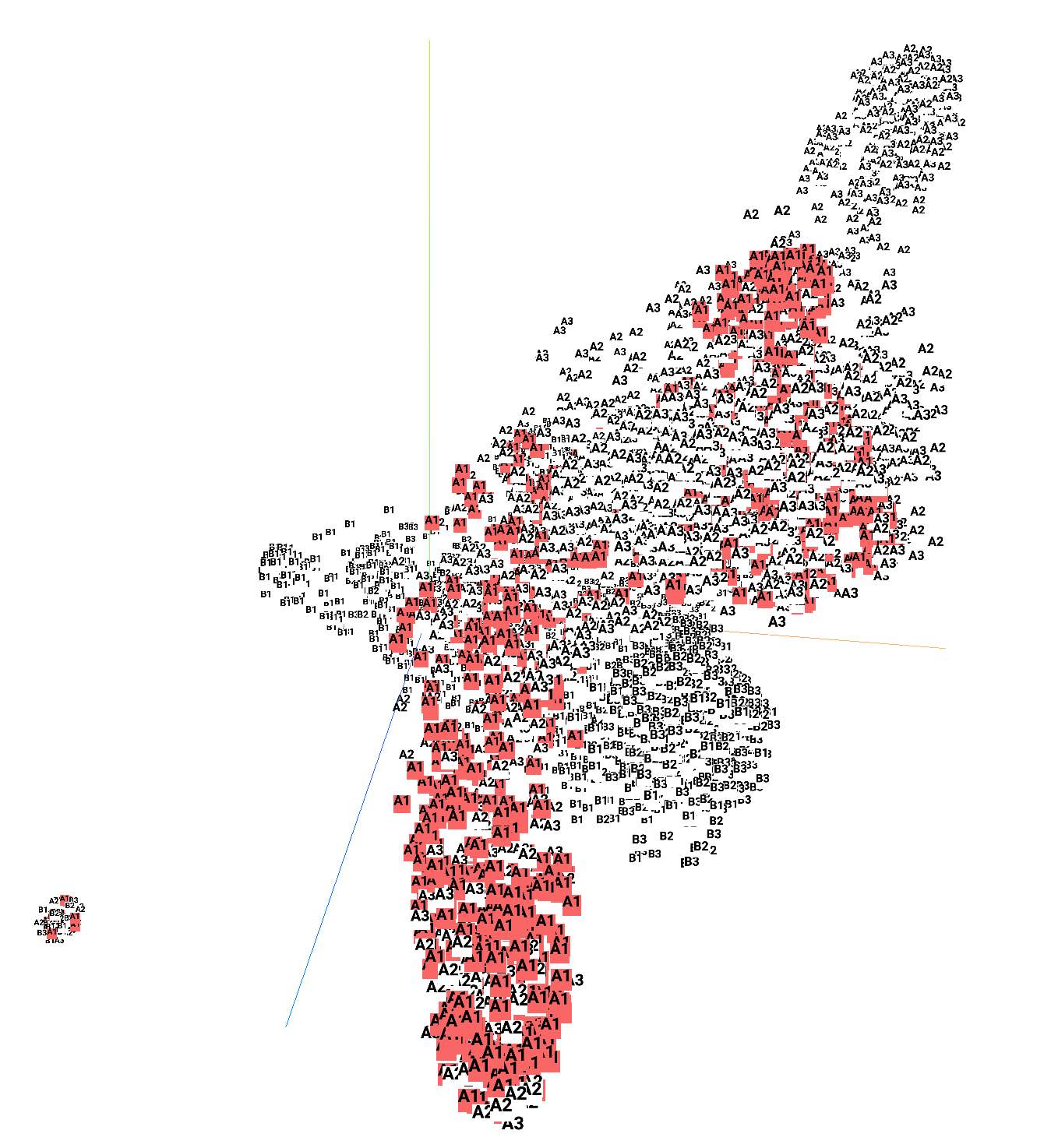}
	   }
	\subfigure[Early stage of GDI-I$^1$]{
		\includegraphics[width=0.4\textwidth,height=0.25\textheight]{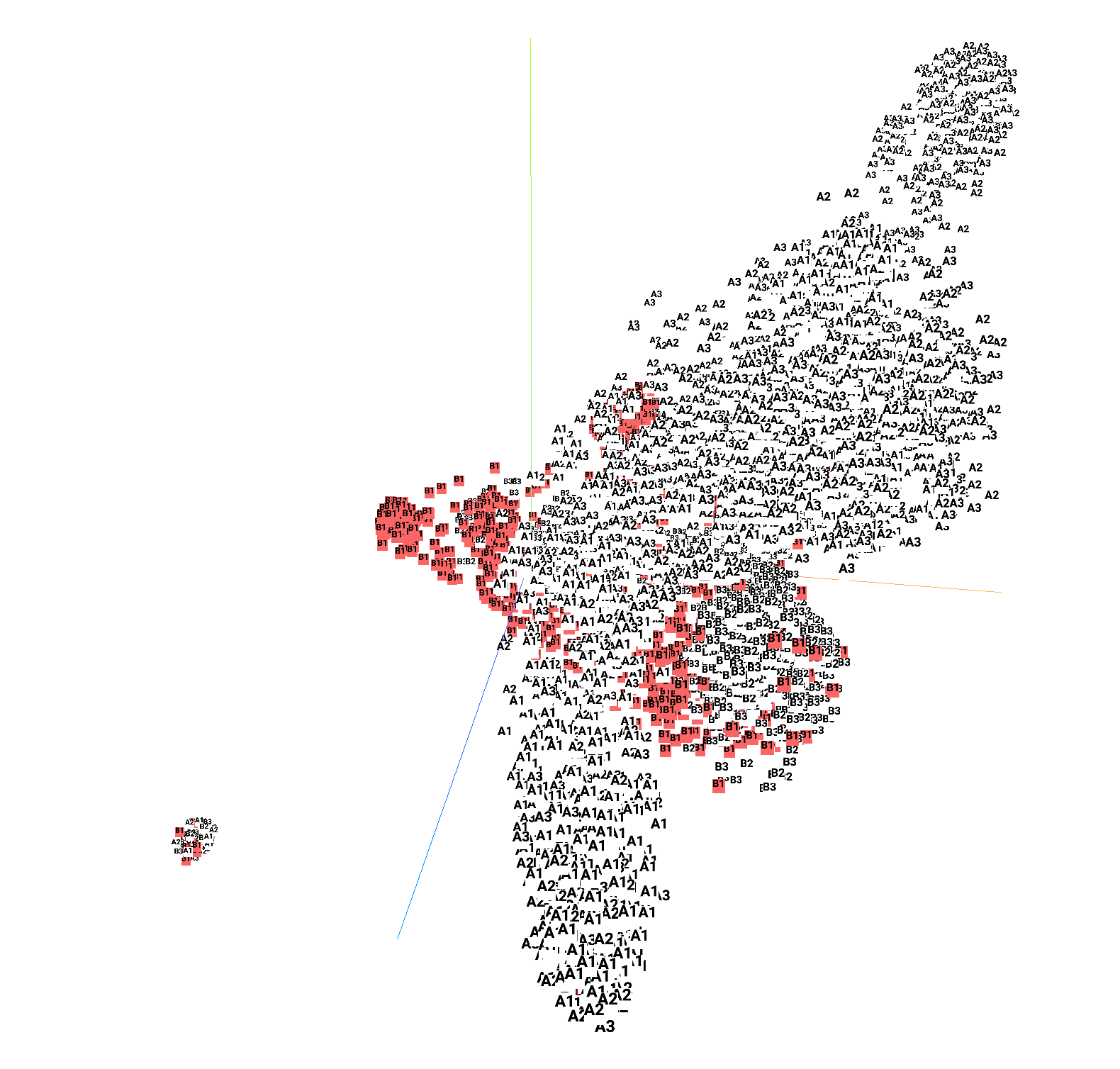}
	}
	
	\subfigure[Middle stage of GDI-I$^3$]{
		\includegraphics[width=0.4\textwidth,height=0.25\textheight]{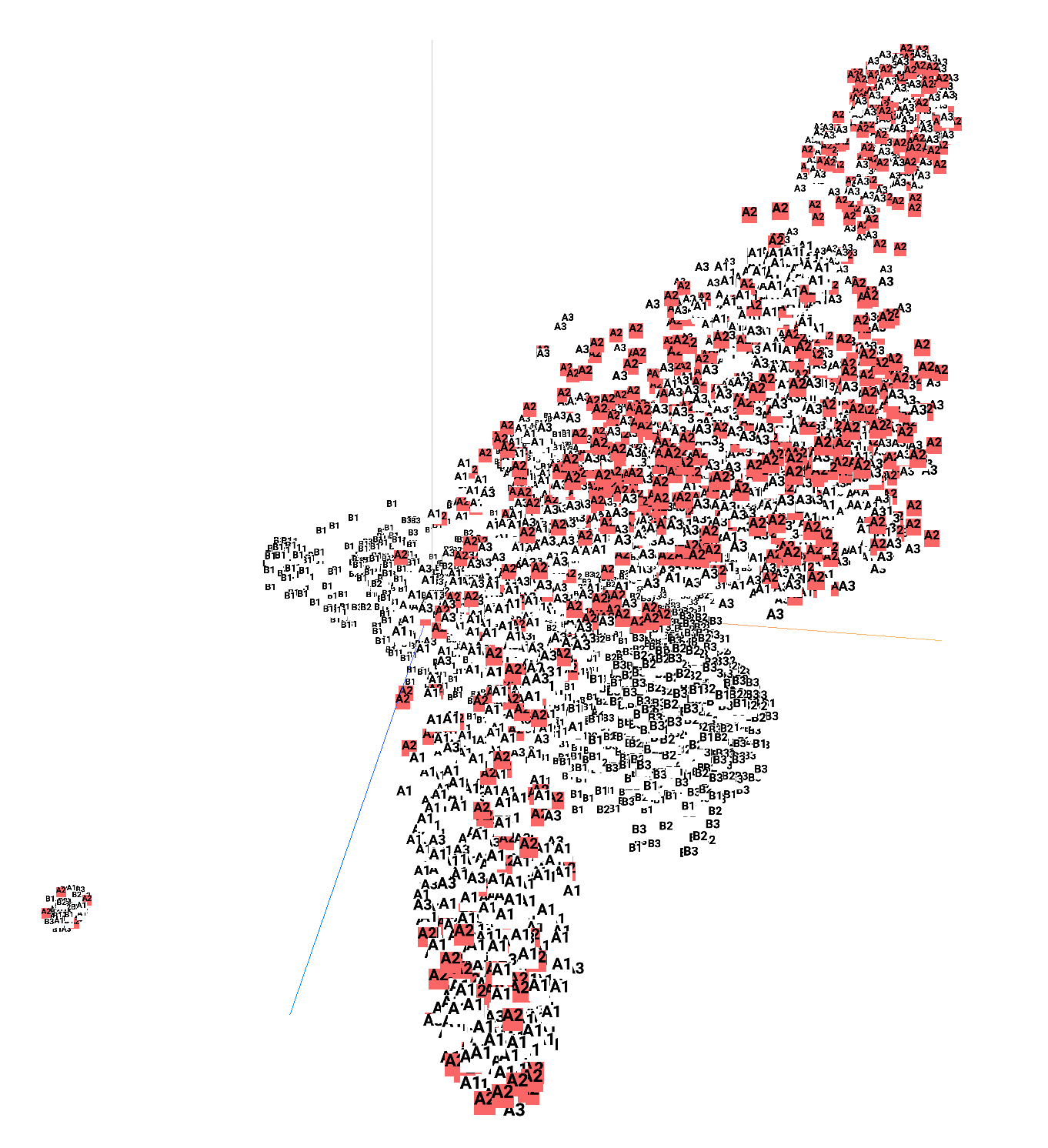}
	}
	\subfigure[Middle stage of GDI-I$^1$]{
		\includegraphics[width=0.4\textwidth,height=0.25\textheight]{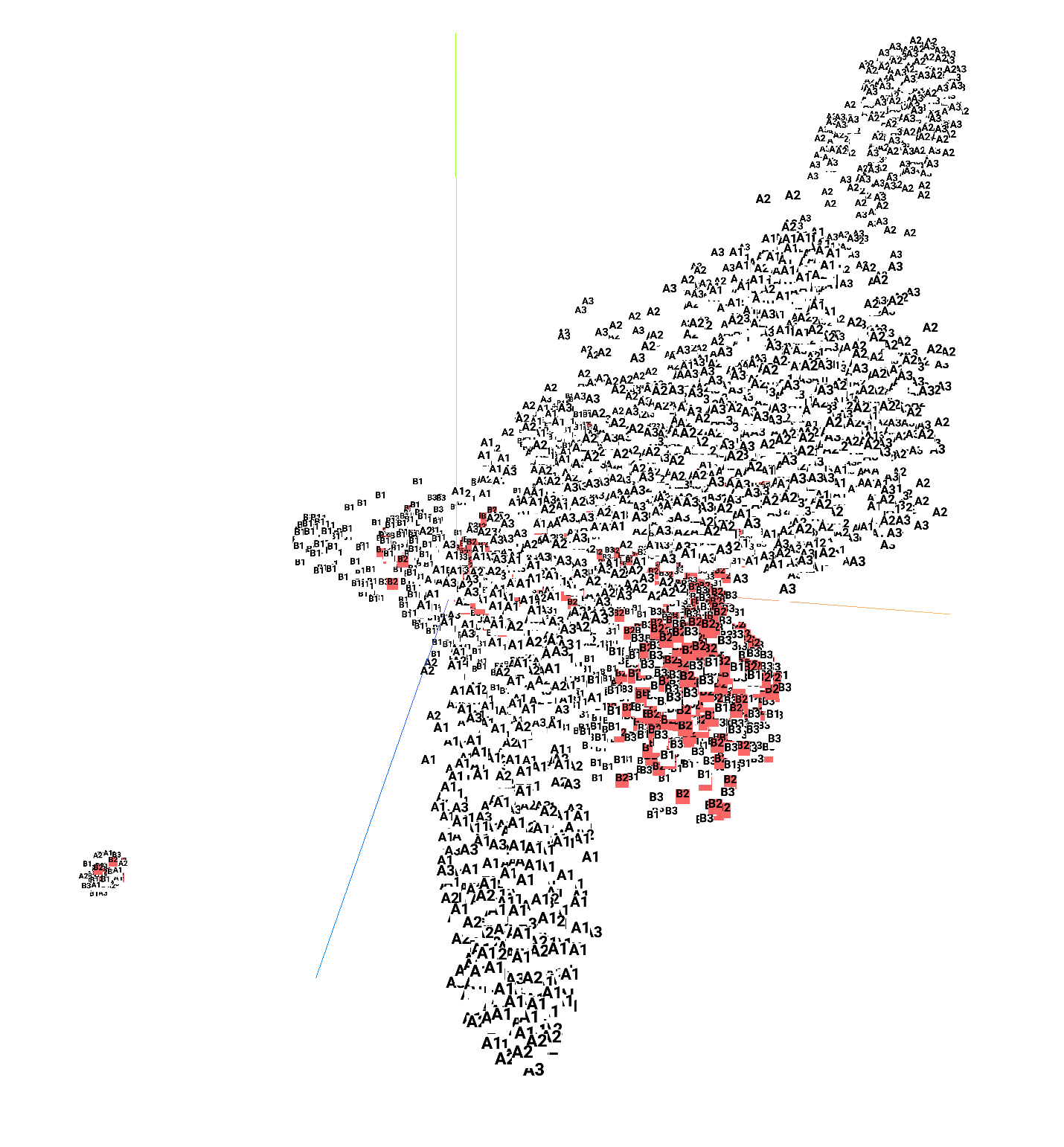}
	}
	
	\subfigure[Later stage of GDI-I$^3$]{
		\includegraphics[width=0.4\textwidth,height=0.25\textheight]{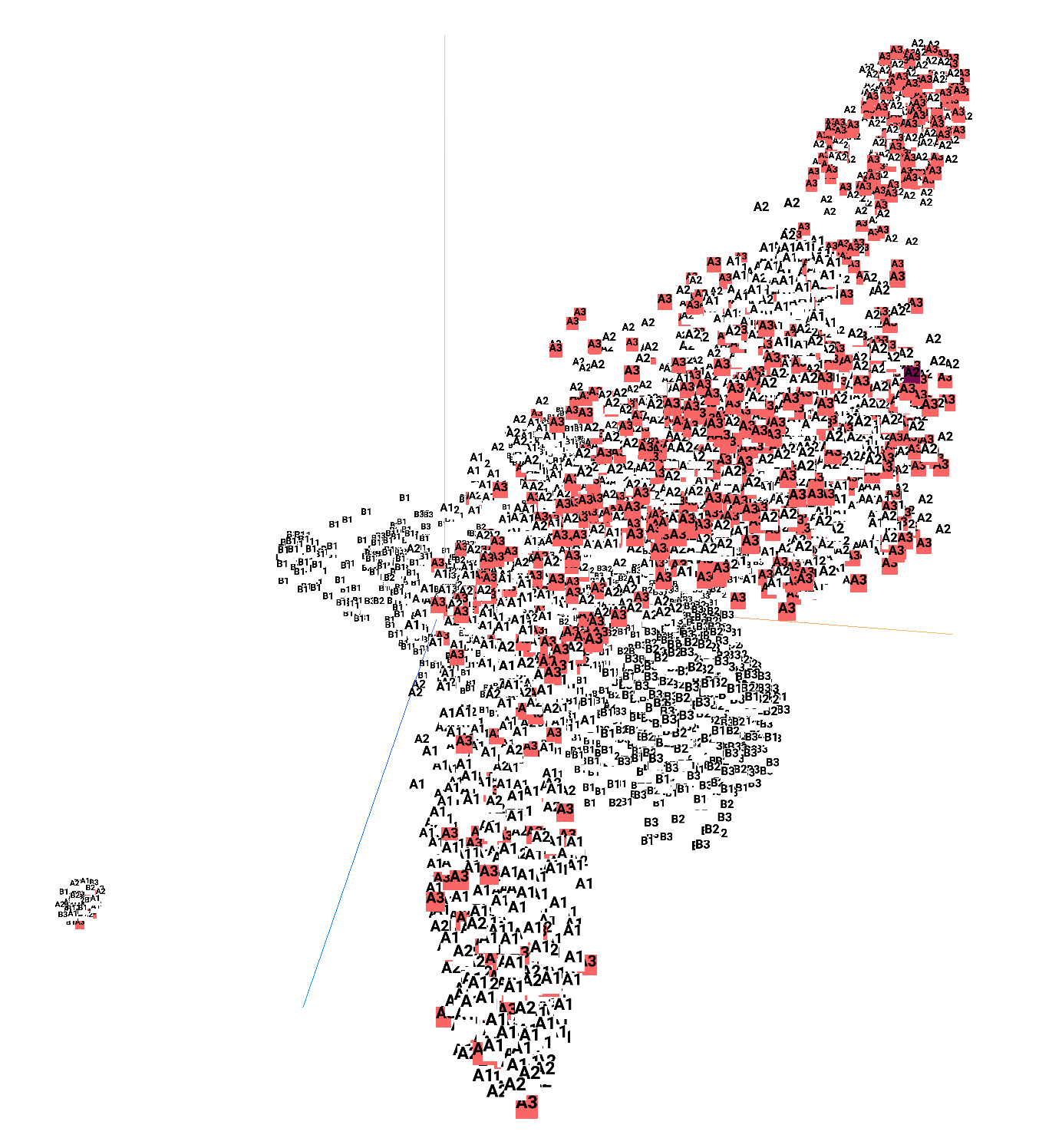}
	}
	\subfigure[Later stage of GDI-I$^1$]{
		\includegraphics[width=0.4\textwidth,height=0.25\textheight]{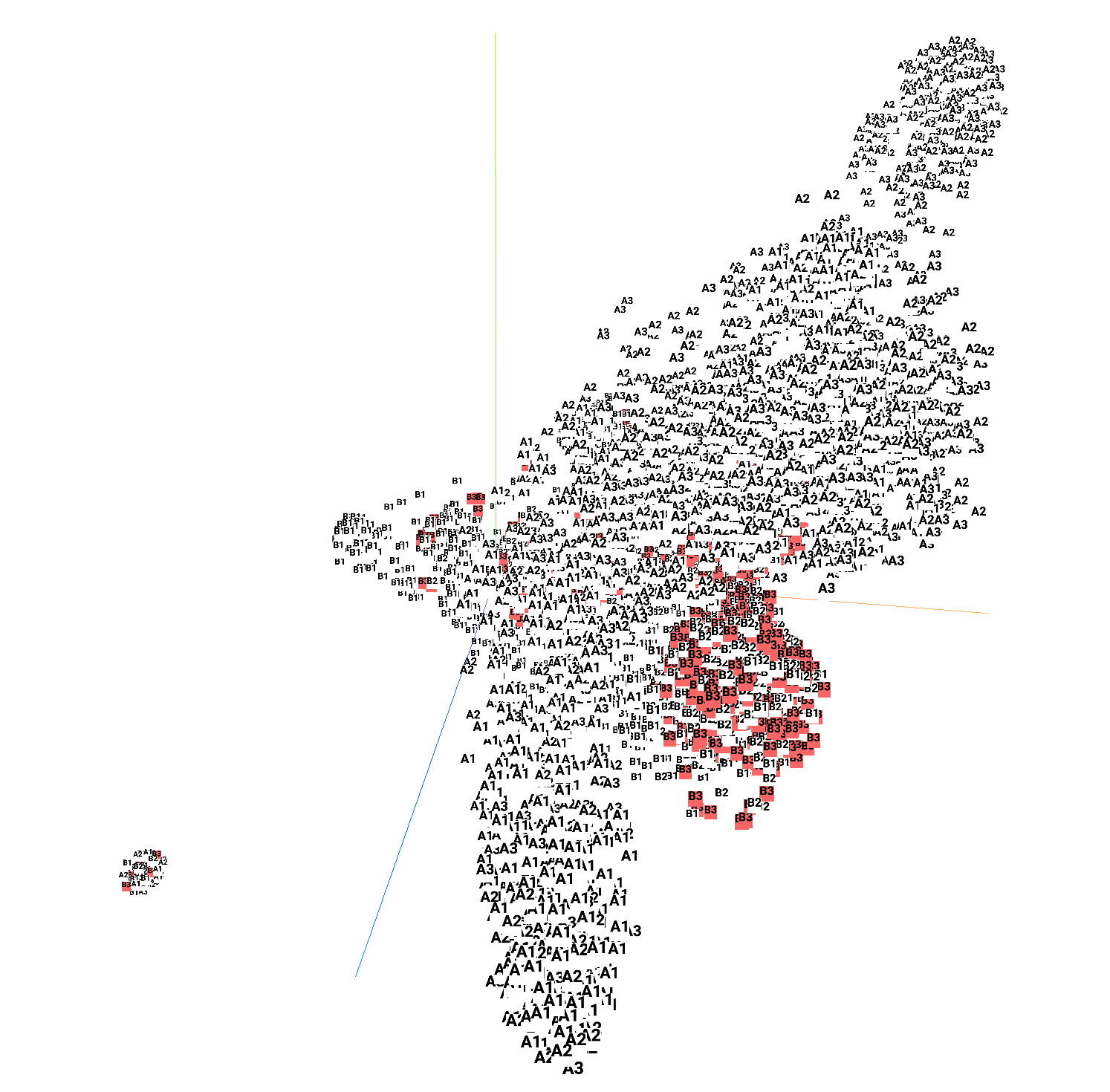}
	}
	\end{center}
	\caption{t-SNE of Seaquest. 
t-SNE is drawn from 6k states.
	We sample 1k states from each stage of GDI-I$^3$ and GDI-I$^1$.
	We highlight 1k states of each stage of GDI-I$^3$ and GDI-I$^1$.}
	\label{Fig: t-SNE of Seaquest}
\end{figure}

\begin{figure}[!ht]
	\begin{center}
	    \subfigure[Early stage of GDI-I$^3$]{
		\includegraphics[width=0.4\textwidth]{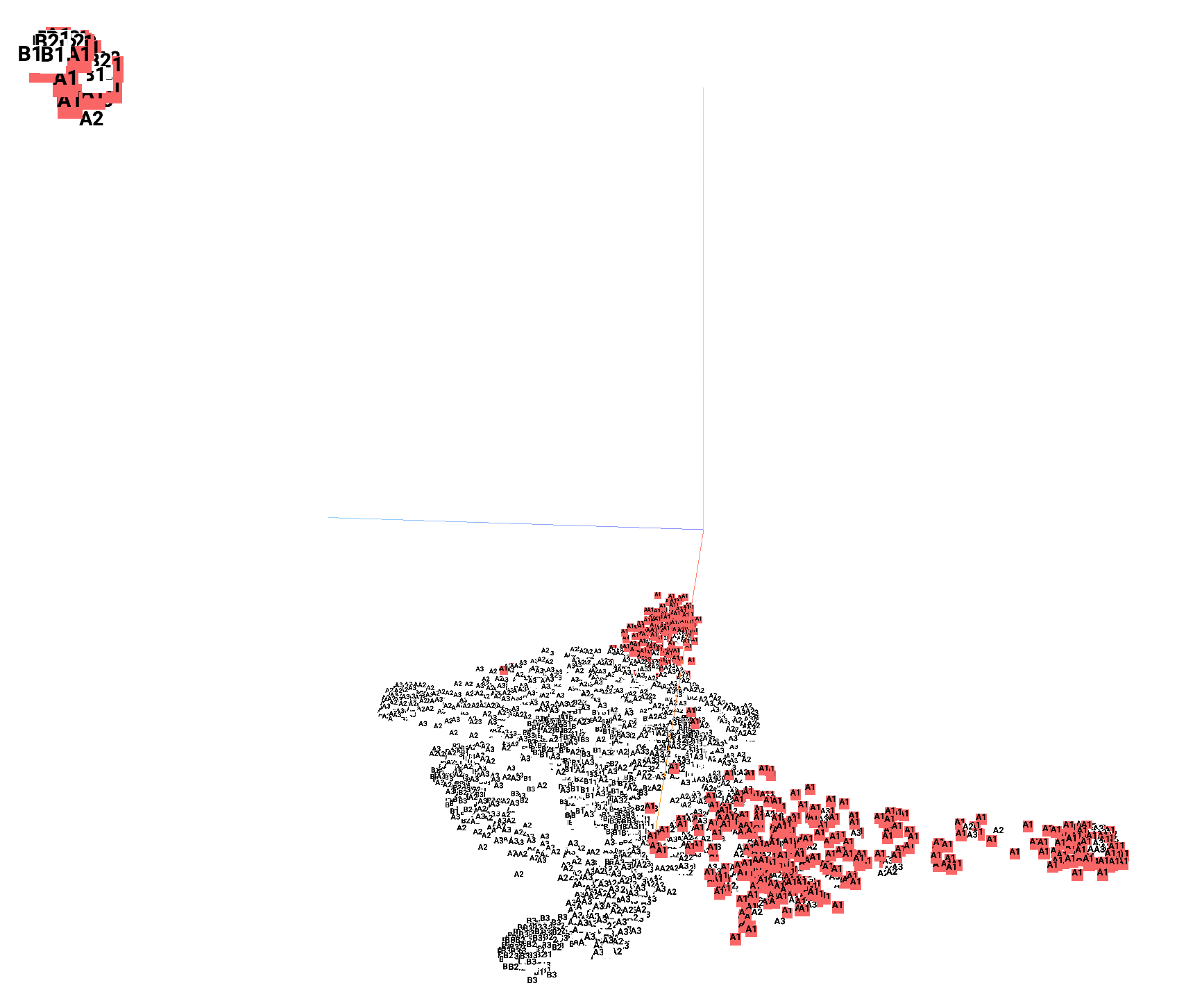}
	   }
	\subfigure[Early stage of GDI-I$^1$]{
		\includegraphics[width=0.4\textwidth]{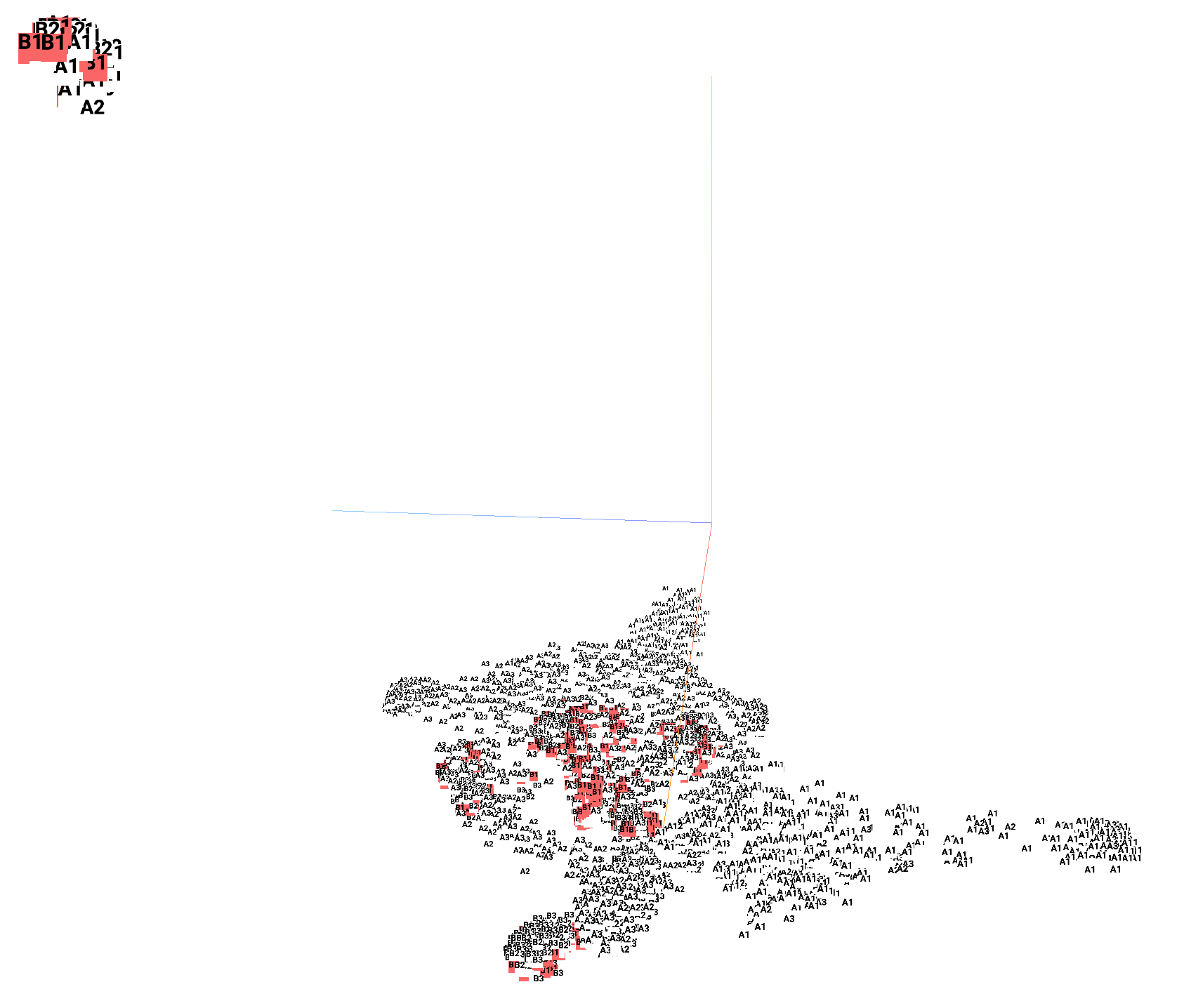}
	}
	
	\subfigure[Middle stage of GDI-I$^3$]{
		\includegraphics[width=0.4\textwidth]{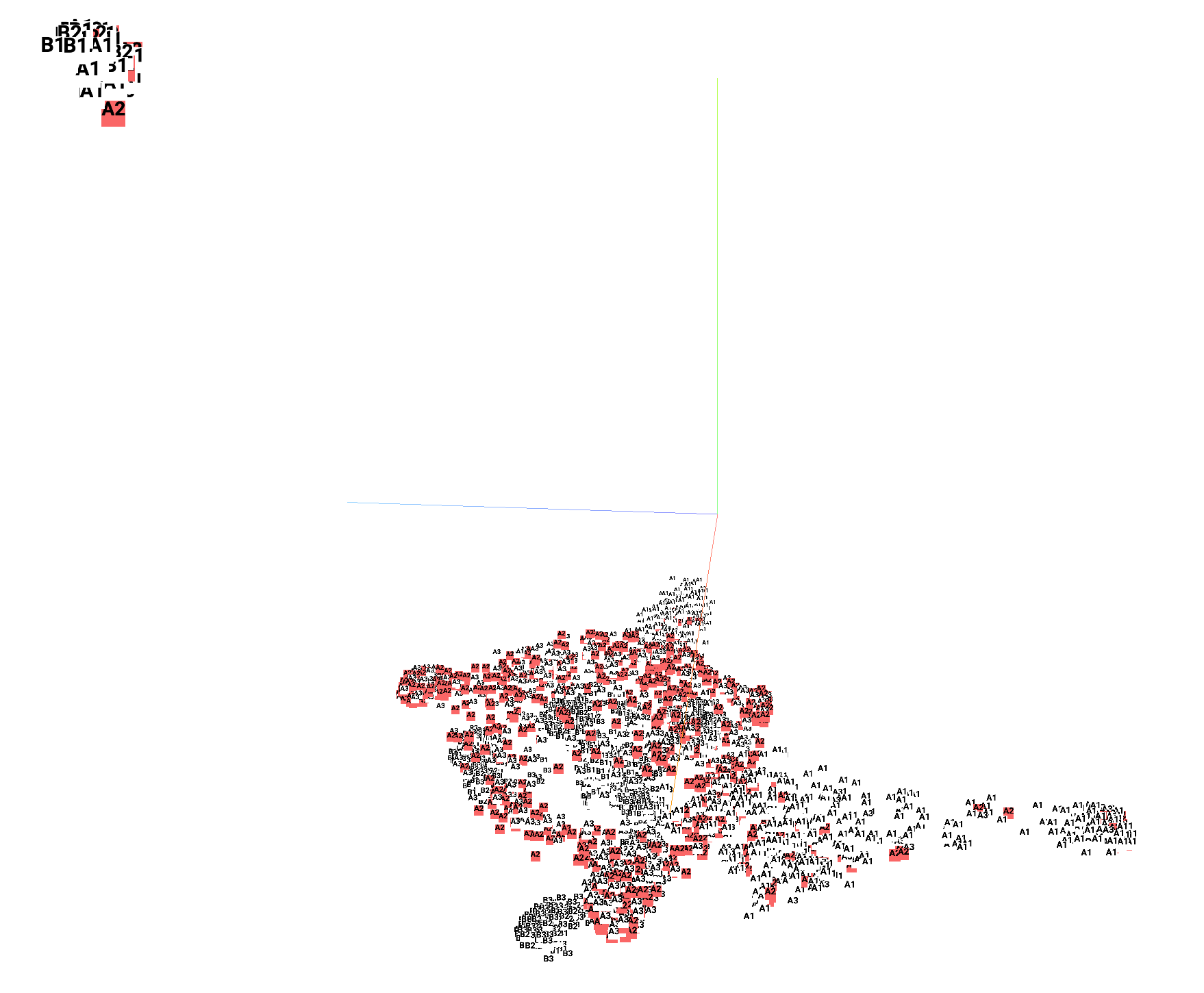}
	}
	\subfigure[Middle stage of GDI-I$^1$]{
		\includegraphics[width=0.4\textwidth]{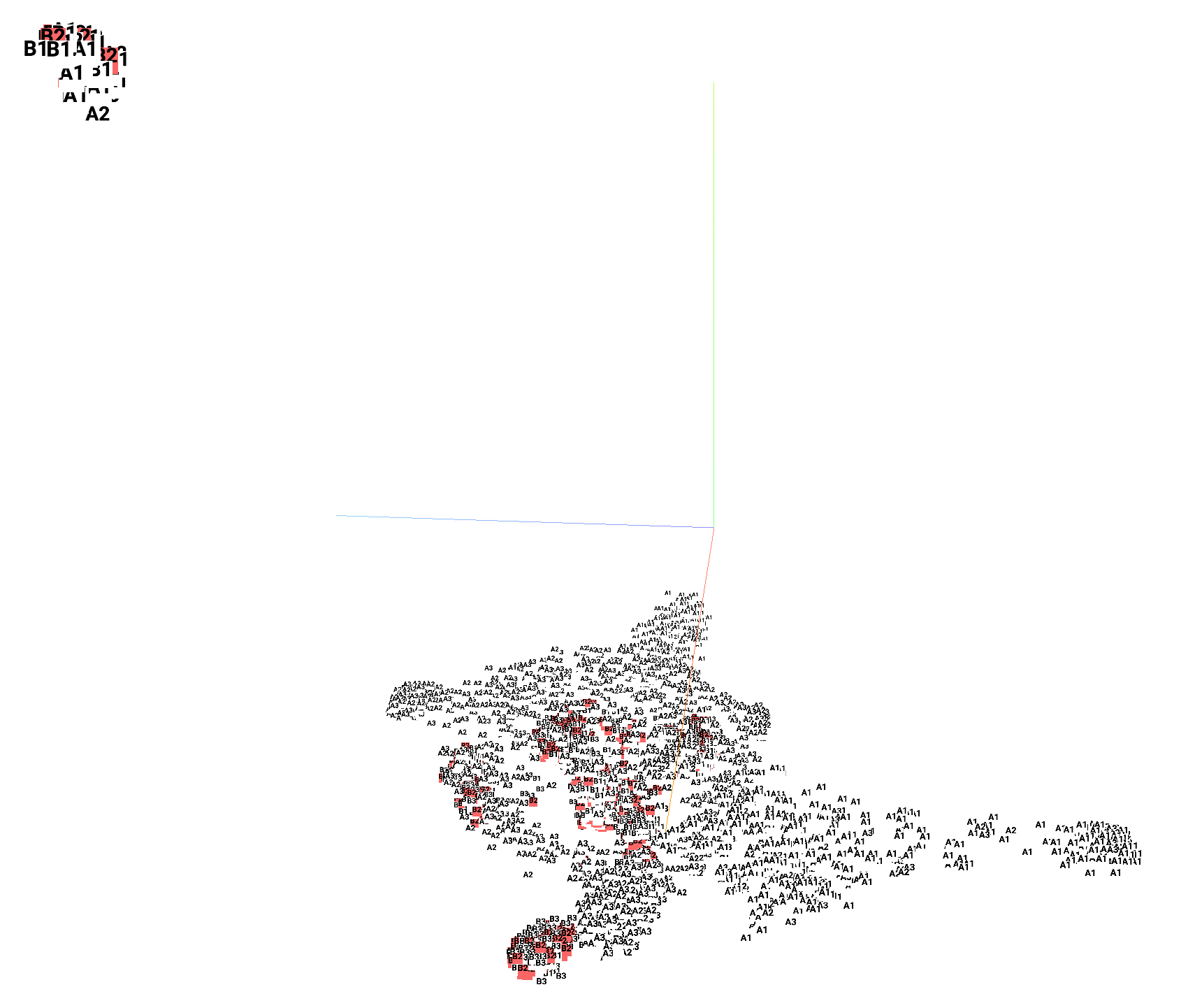}
	}
	
	\subfigure[Later stage of GDI-I$^3$]{
		\includegraphics[width=0.4\textwidth]{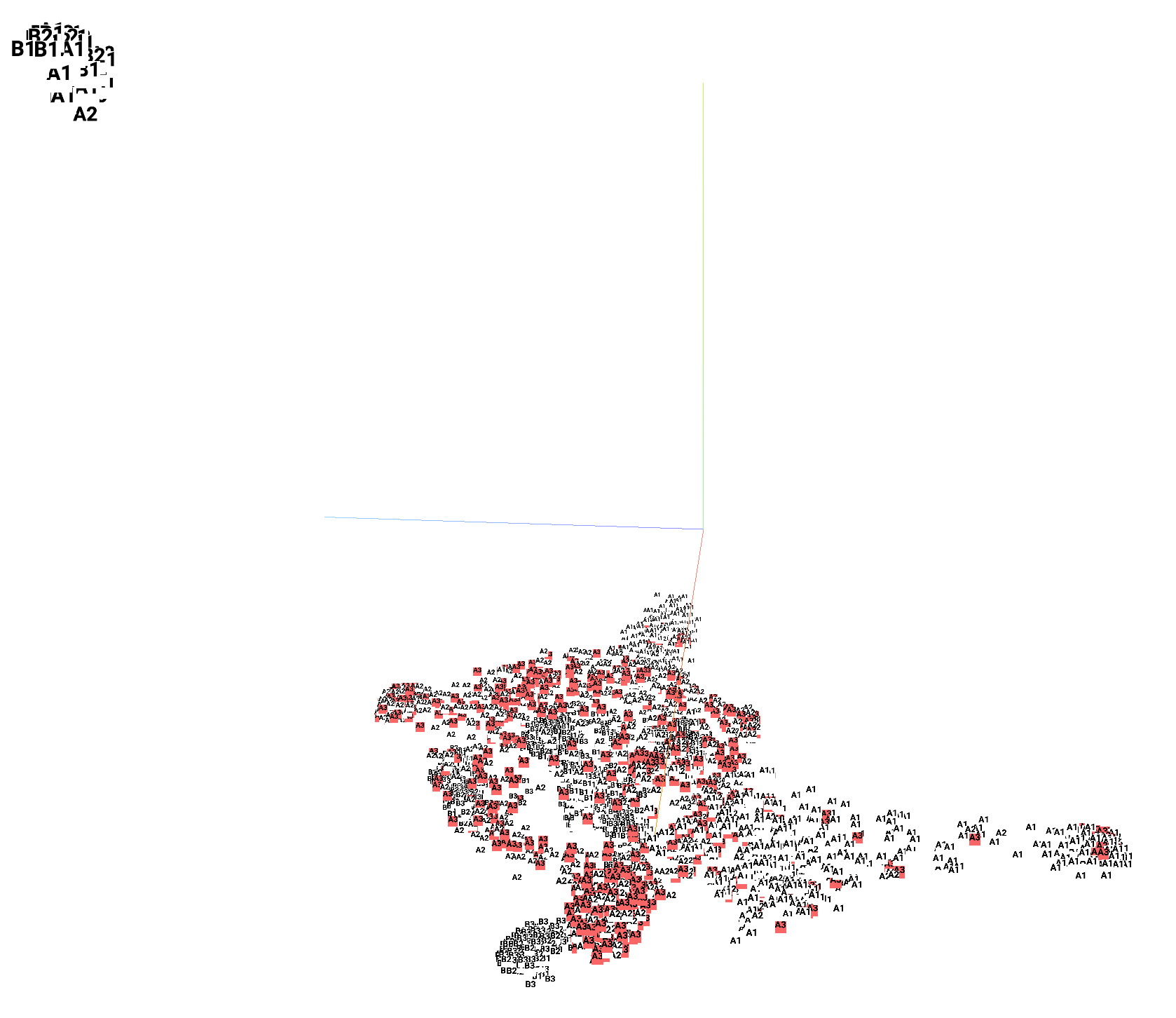}
	}
	\subfigure[Later stage of GDI-I$^1$]{
		\includegraphics[width=0.4\textwidth]{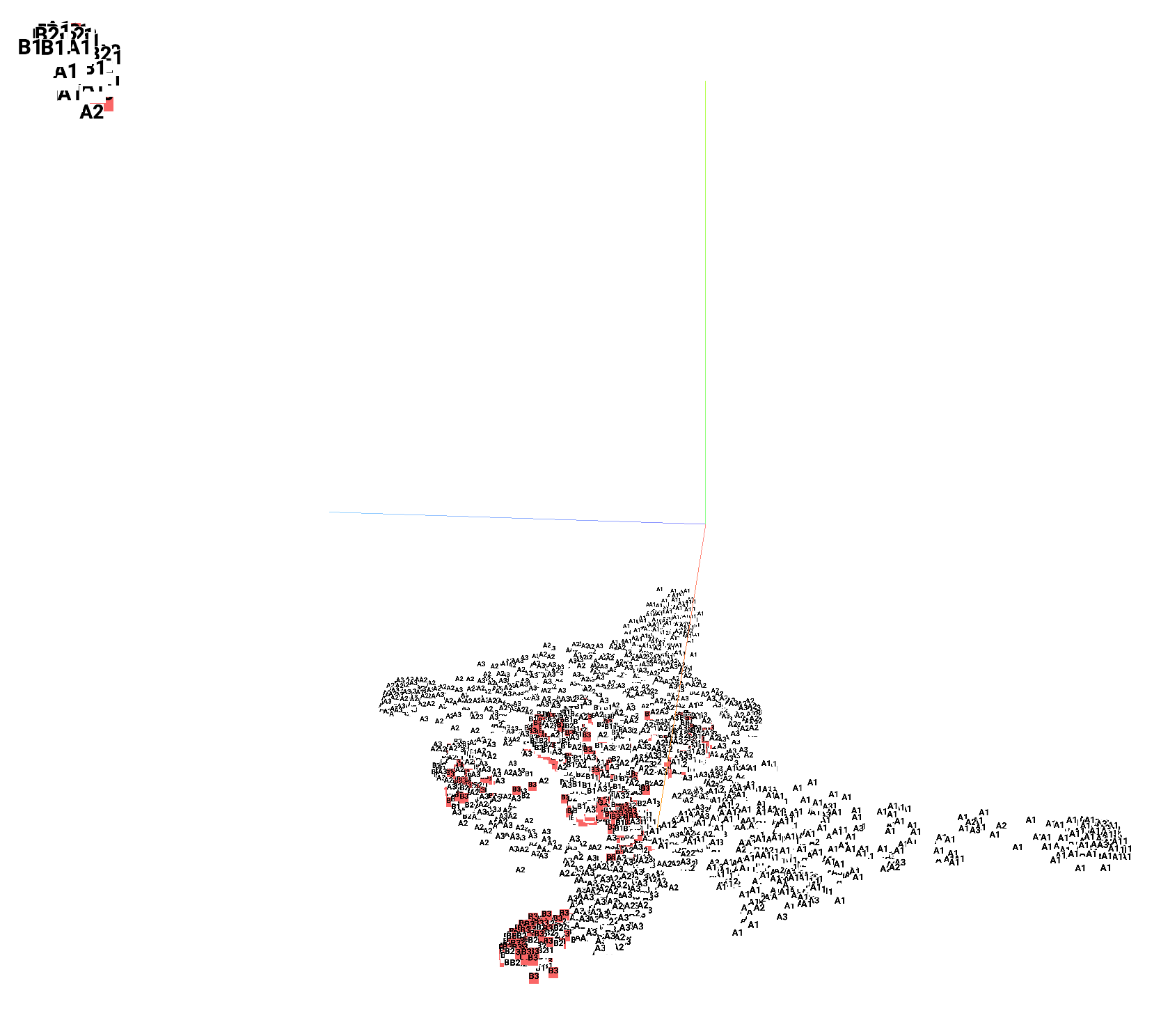}
	}
	\end{center}
	\caption{t-SNE of ChopperCommand. 
	t-SNE is drawn from 6k states.
	We sample 1k states from each stage of GDI-I$^3$ and GDI-I$^1$.
	We highlight 1k states of each stage of GDI-I$^3$ and GDI-I$^1$.}
	\label{Fig: t-SNE of ChopperCommand}
\end{figure}

\begin{figure}[!ht]
	\begin{center}
	    \subfigure[Early stage of GDI-H$^3$]{
		\includegraphics[width=0.4\textwidth]{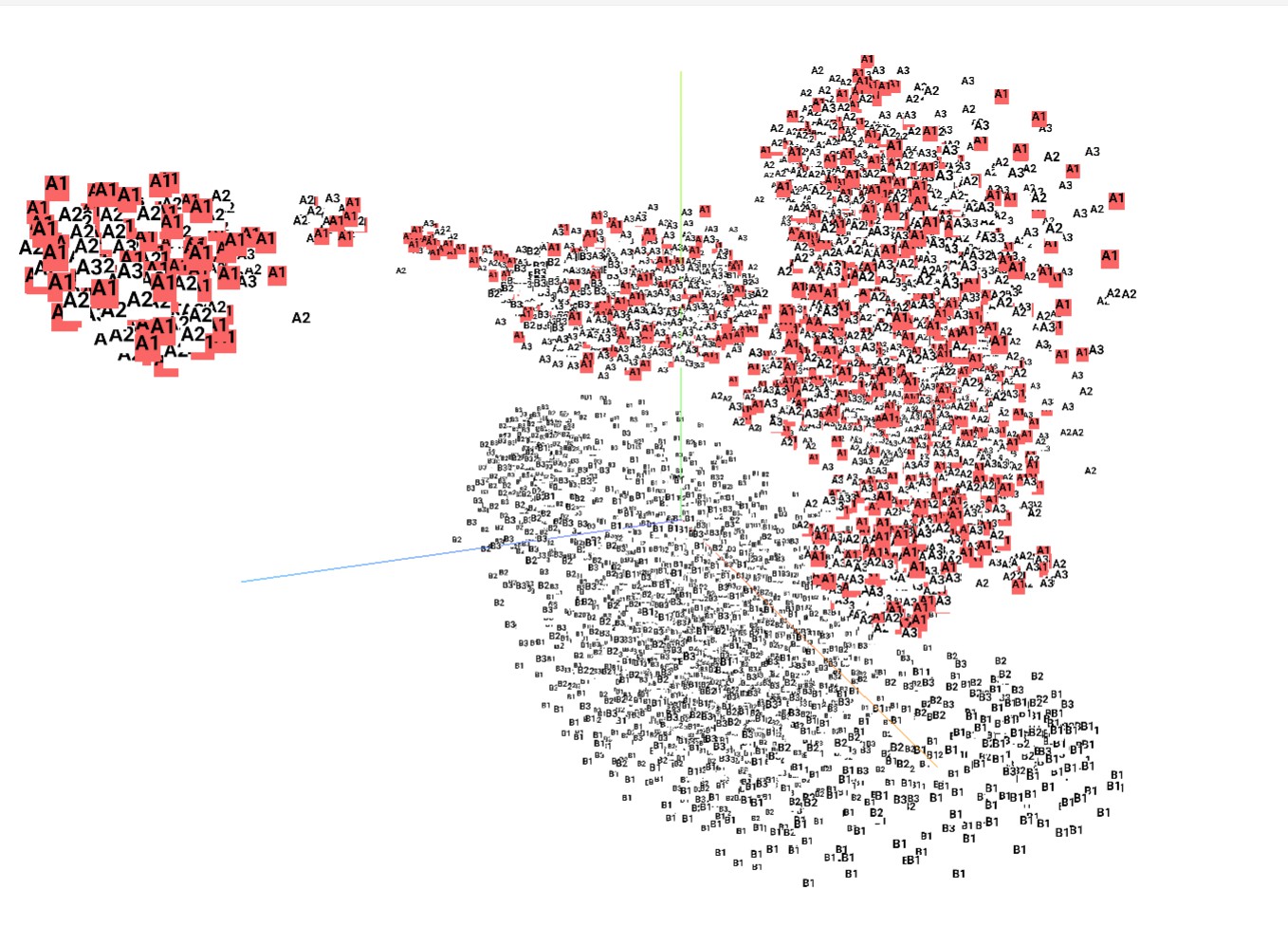}
	   }
	\subfigure[Early stage of GDI-I$^1$]{
		\includegraphics[width=0.4\textwidth]{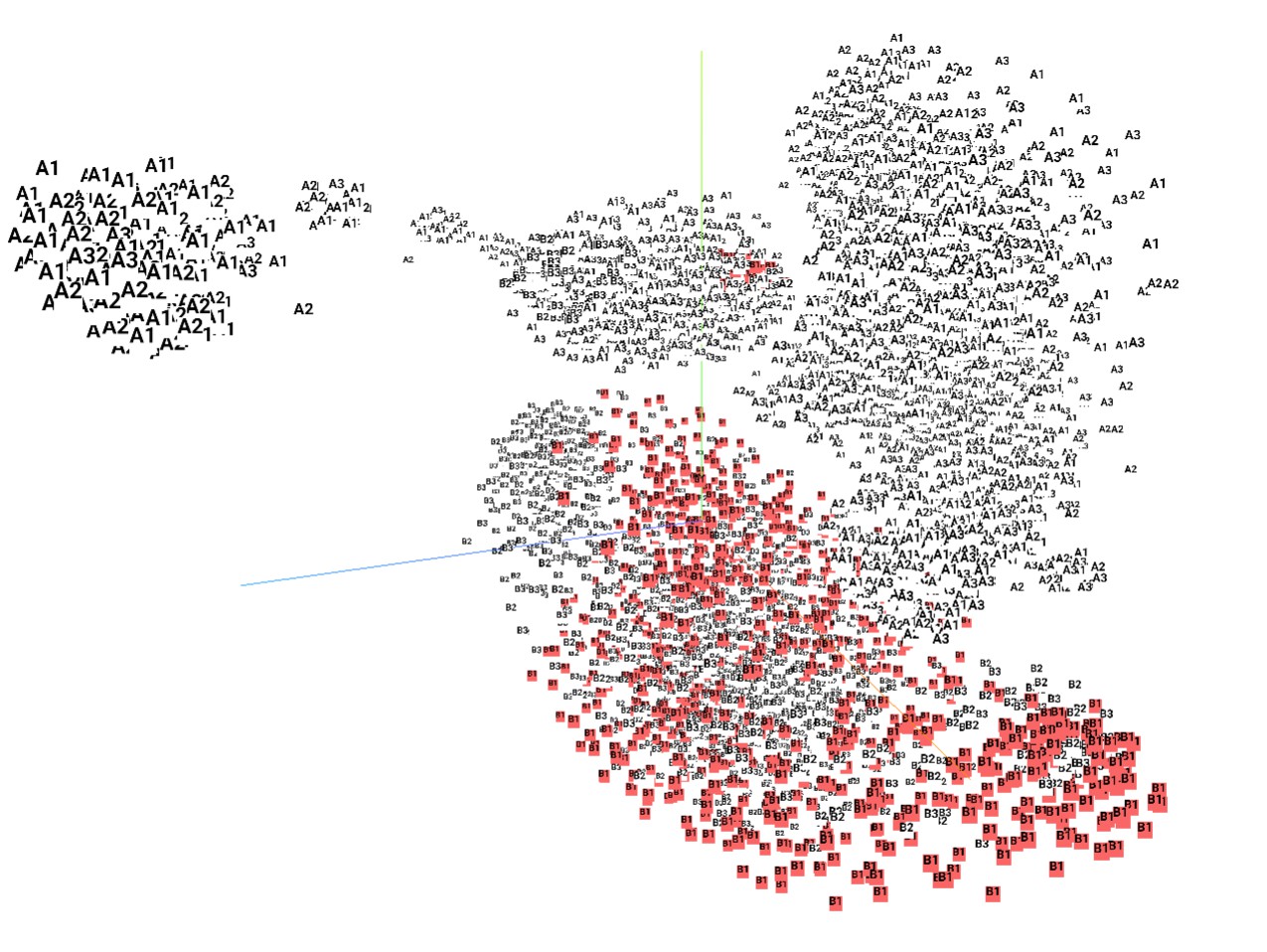}
	}
	
	\subfigure[Middle stage of GDI-H$^3$]{
		\includegraphics[width=0.4\textwidth]{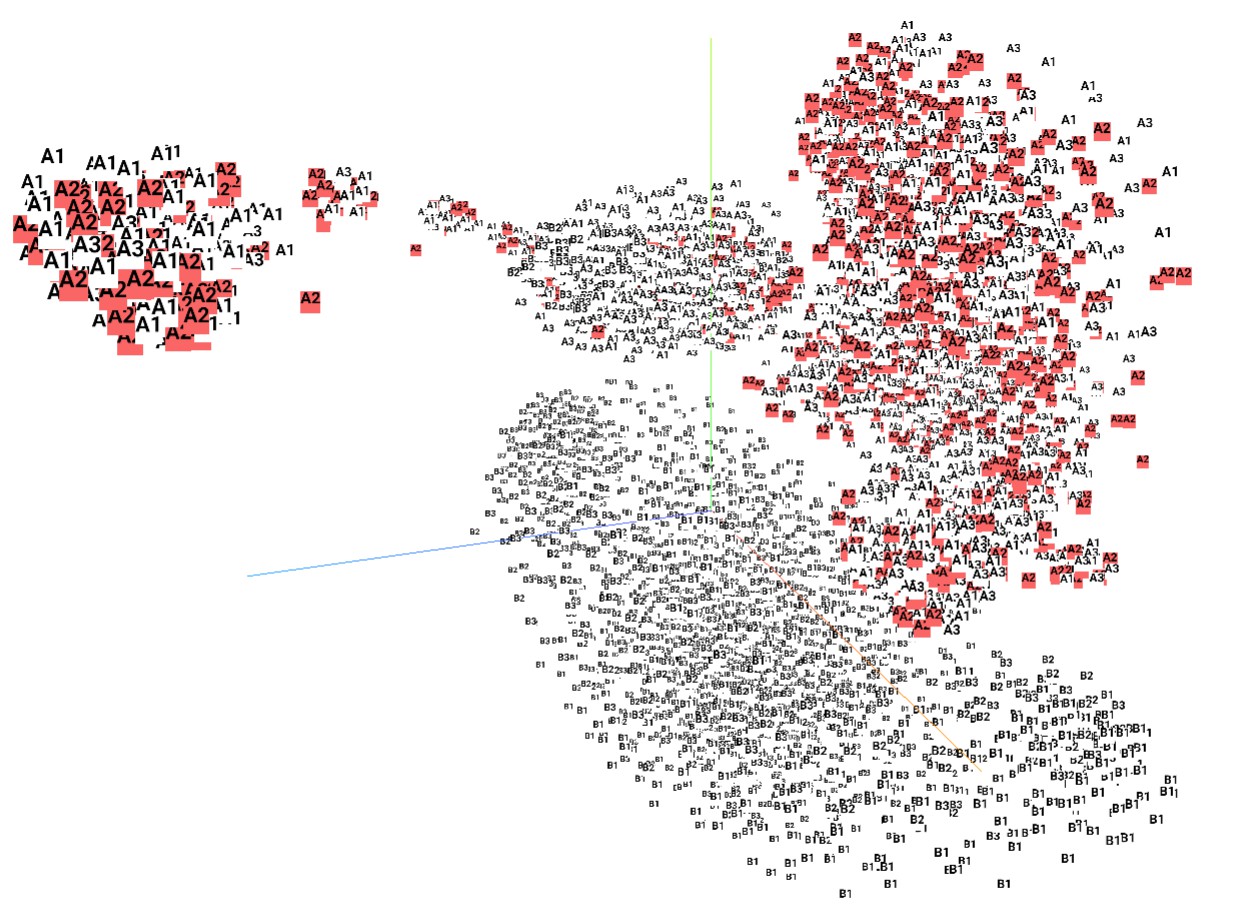}
	}
	\subfigure[Middle stage of GDI-I$^1$]{
		\includegraphics[width=0.4\textwidth]{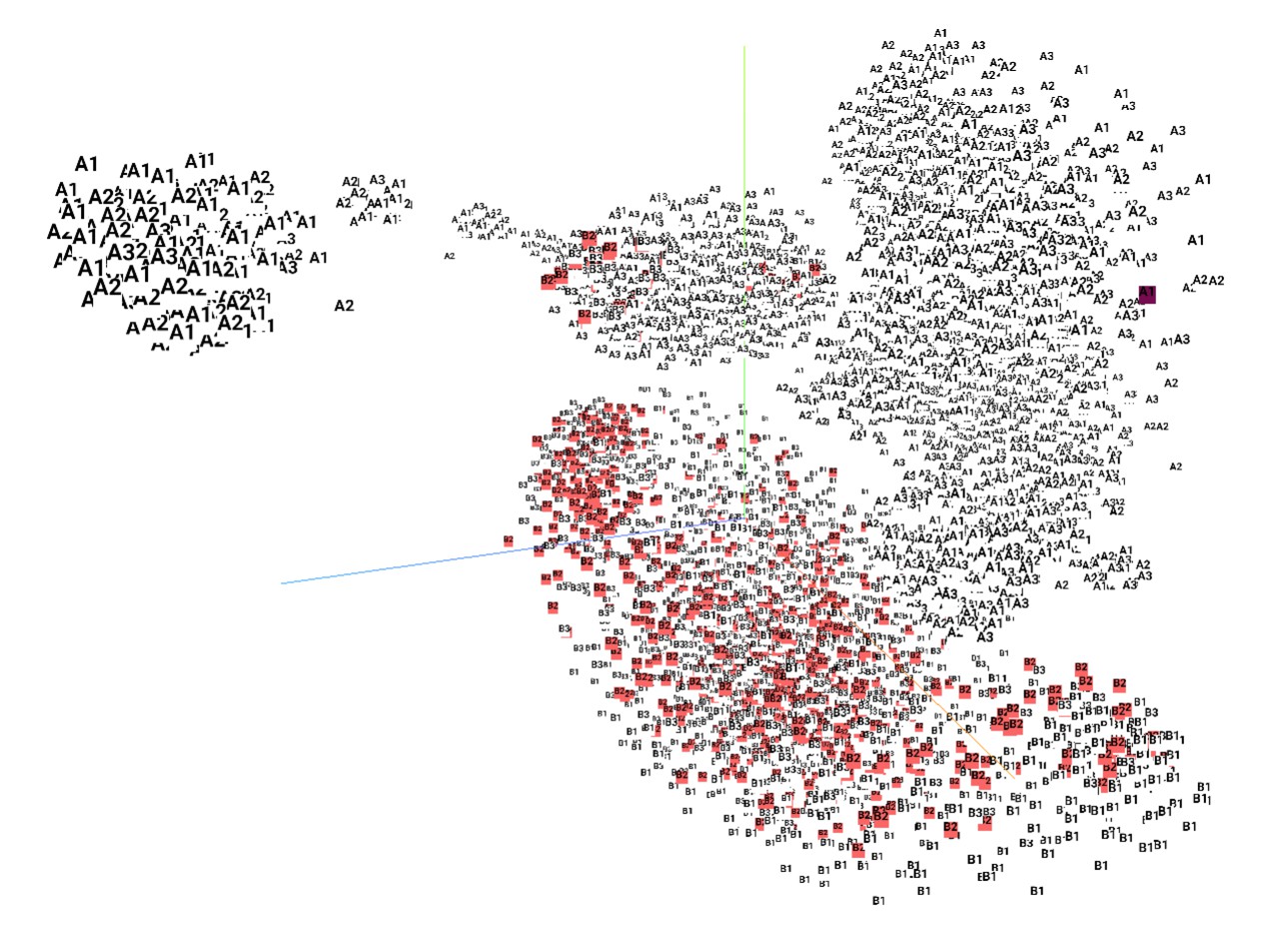}
	}
	
	\subfigure[Later stage of GDI-H$^3$]{
		\includegraphics[width=0.4\textwidth]{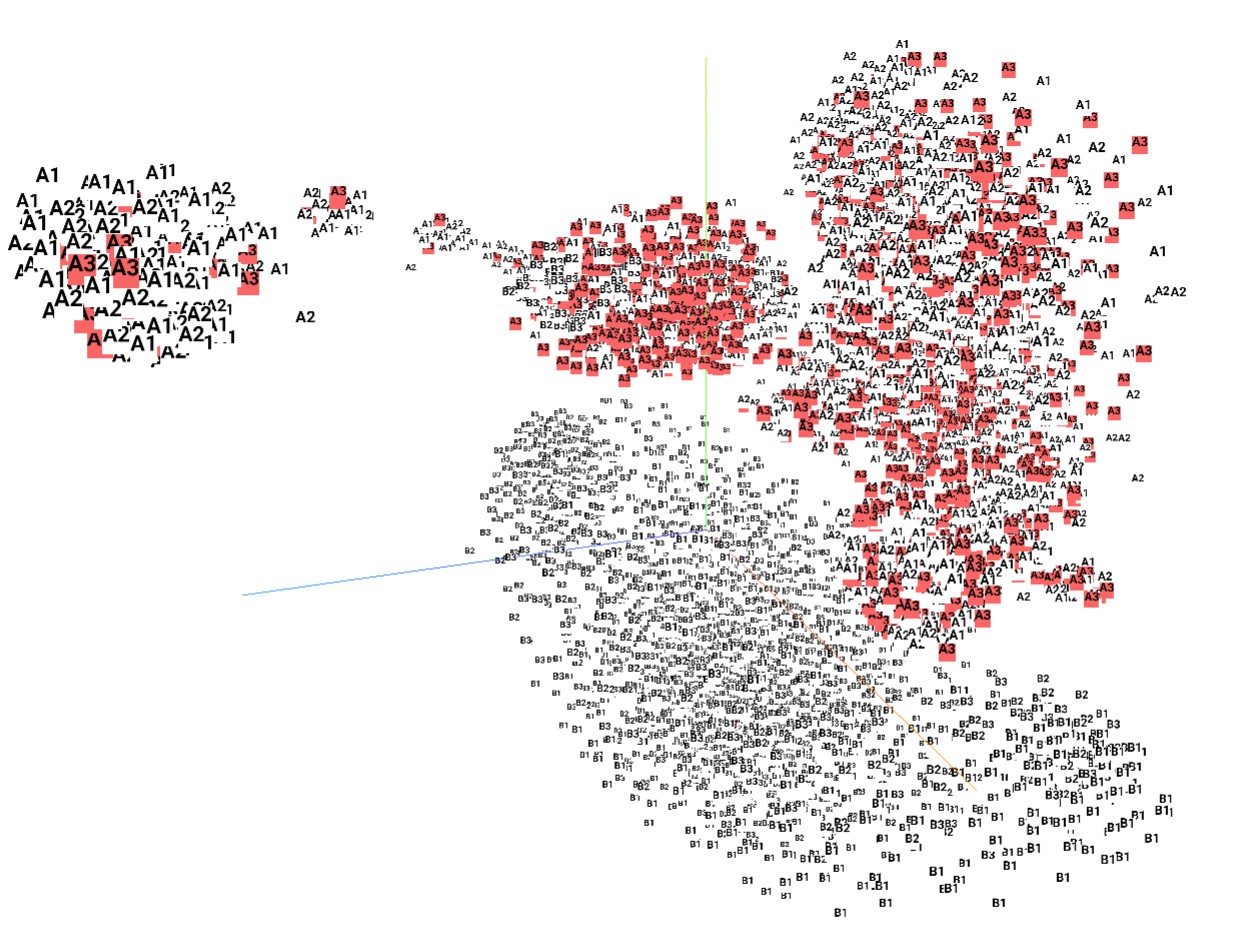}
	}
	\subfigure[Later stage of GDI-I$^1$]{
		\includegraphics[width=0.4\textwidth]{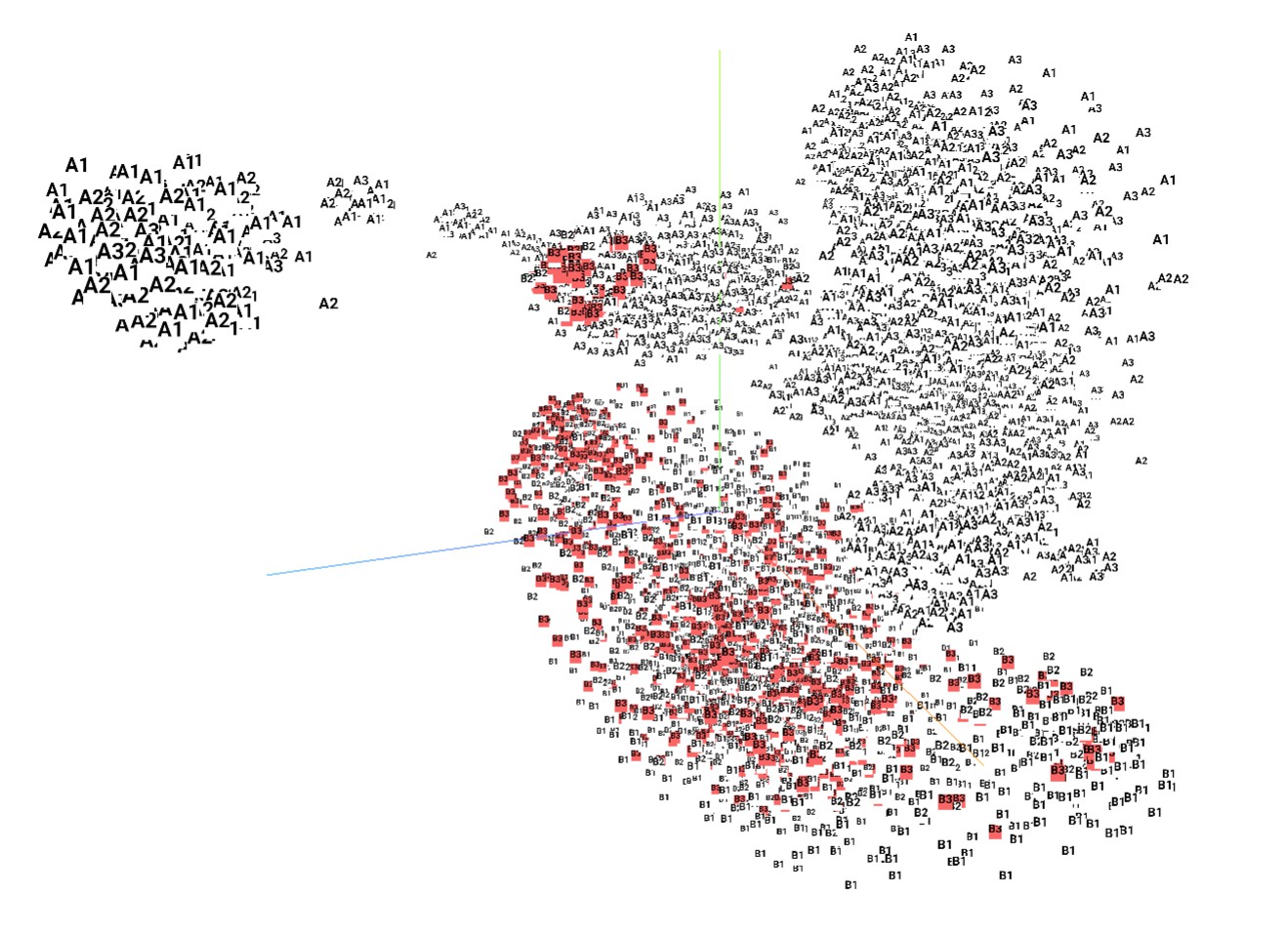}
	}
	\end{center}
	\caption{t-SNE of  Krull. 
	t-SNE is drawn from 6k states.
	We sample 1k states from each stage of GDI-H$^3$ and GDI-I$^1$.
	We highlight 1k states of each stage of GDI-H$^3$ and GDI-I$^1$.}
	\label{Fig: t-SNE of Krull}
\end{figure}

\begin{figure}[!ht]
    \begin{center}
        \subfigure[GDI-I$^3$ on Seaquest]{
		\includegraphics[width=0.4\textwidth]{photo/TSNE/sq/A.png}
	}
	\subfigure[GDI-I$^1$  on Seaquest]{
		\includegraphics[width=0.4\textwidth]{photo/TSNE/sq/B.png}
	}
	
	\subfigure[GDI-I$^3$ on ChopperCommand]{
		\includegraphics[width=0.4\textwidth]{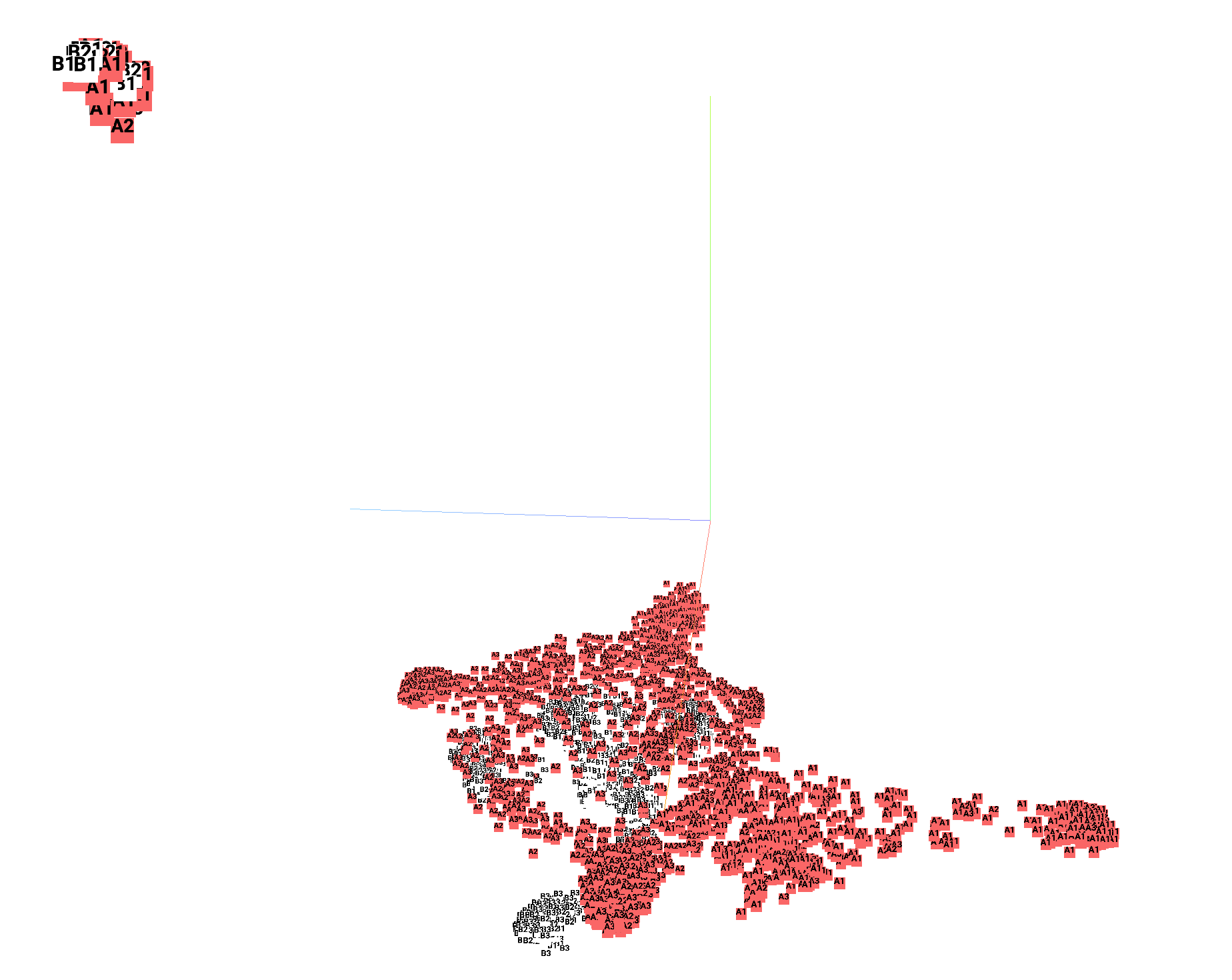}
	   }
	\subfigure[GDI-I$^1$ on ChopperCommand]{
		\includegraphics[width=0.4\textwidth]{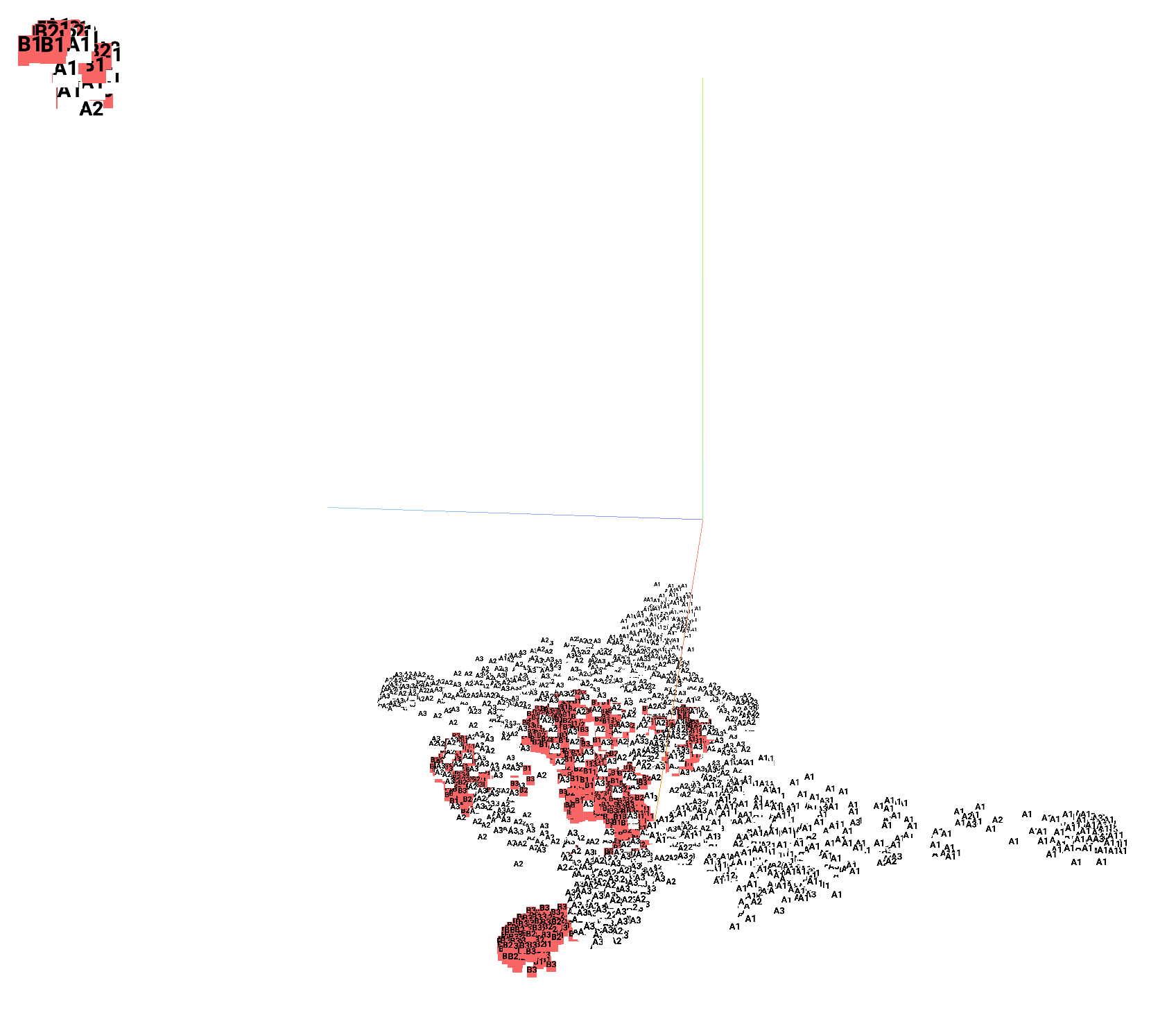}
	}
	
	\subfigure[ GDI-H$^3$ on  Krull]{
		\includegraphics[width=0.4\textwidth]{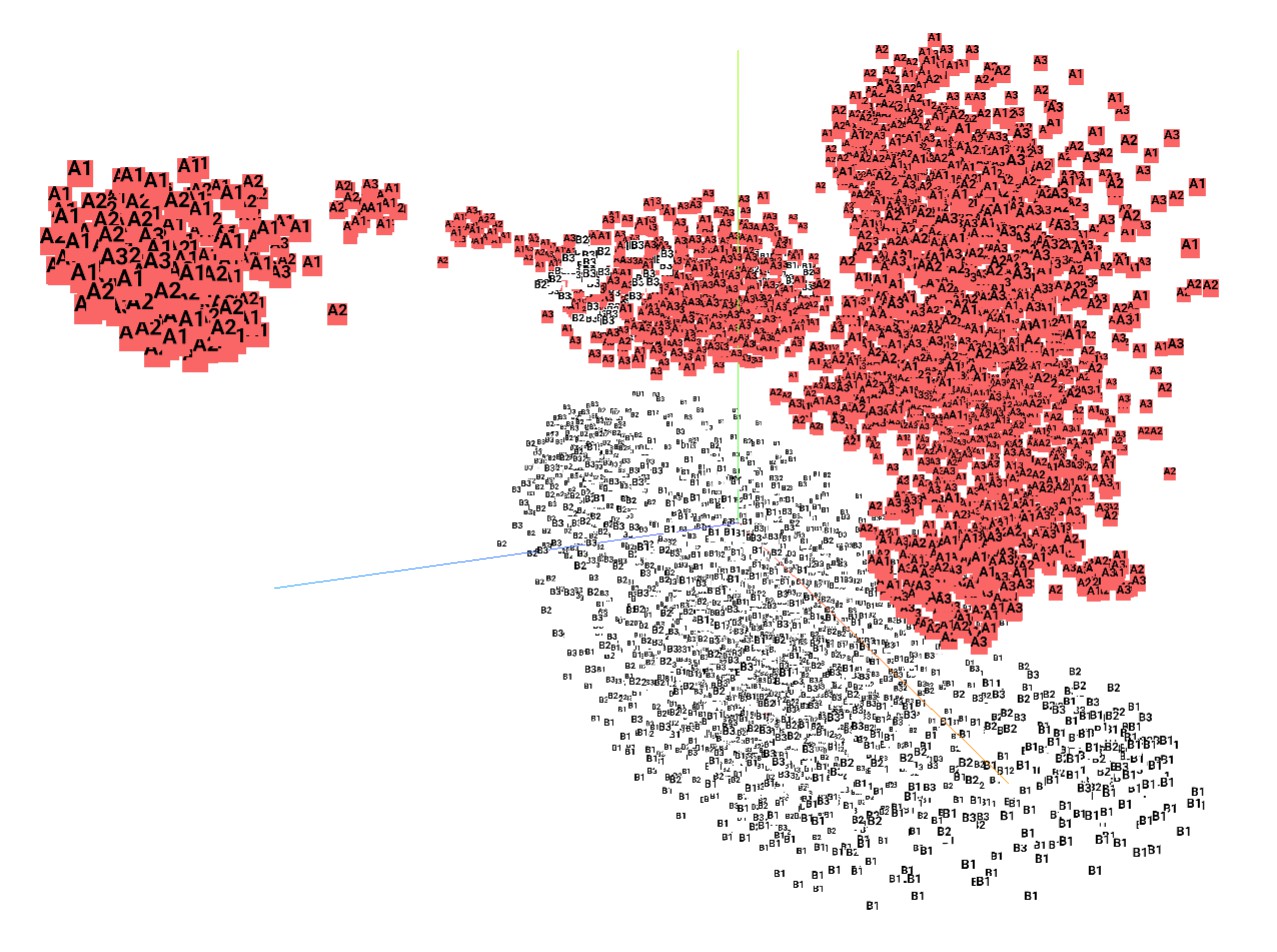}
	}
	\subfigure[GDI-I$^1$ on Krull]{
		\includegraphics[width=0.4\textwidth]{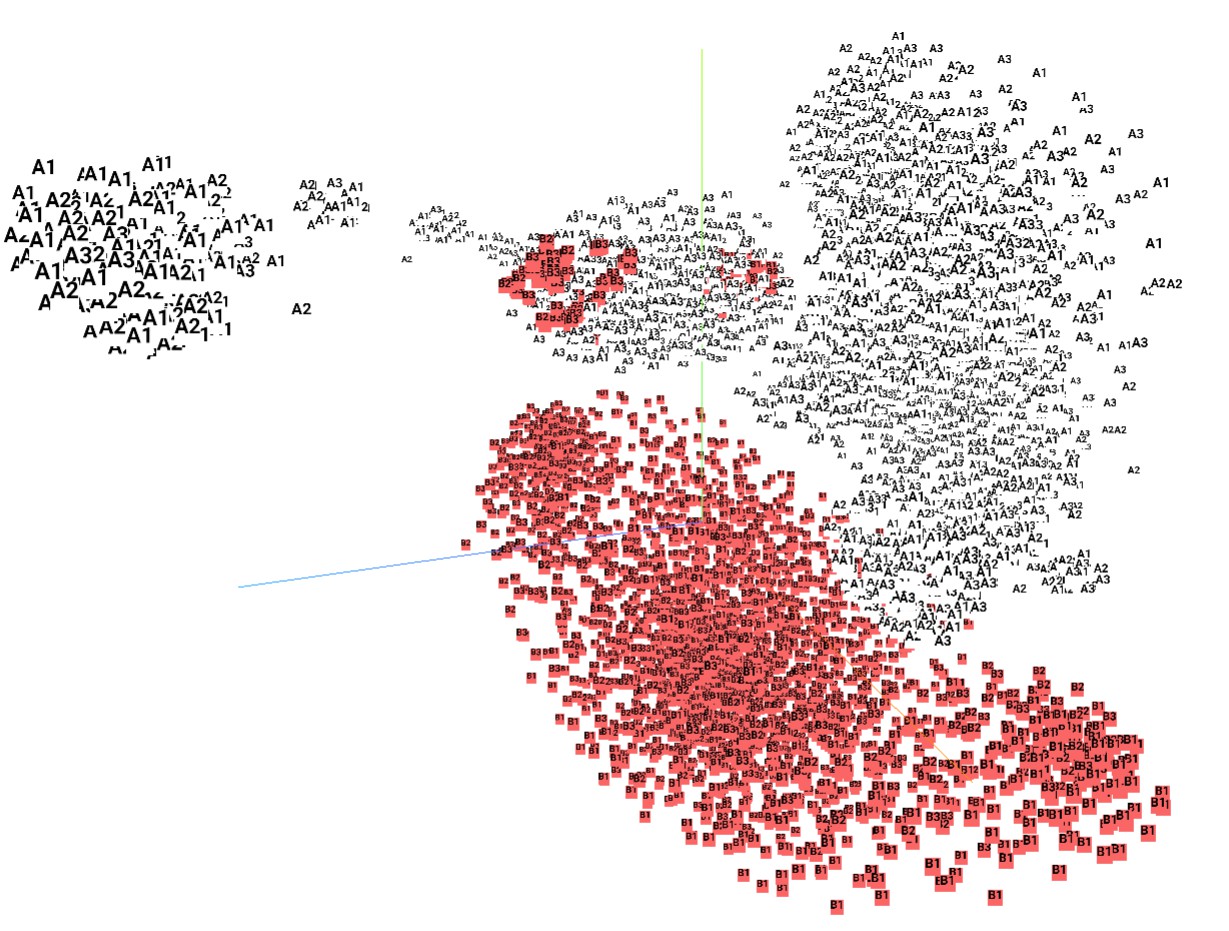}
	}
    \end{center}
	\caption{Overview of t-SNE in Atari games. 
	Each t-SNE figure is drawn from 6k states.
	We highlight 3k states of GDI-I$^3$, GDI-H$^3$ and GDI-I$^1$, respectively.}
	\label{Fig: Overview of t-SNE in Atari games.}
\end{figure}

\clearpage

\end{document}